\documentclass[11pt]{article}
\usepackage{lipsum}
\usepackage{amsfonts}
\usepackage{graphicx}
\usepackage{epstopdf}

\usepackage{mathtools}
\usepackage{empheq}
\usepackage{amsfonts}
\usepackage{amsgen,amsthm,amsmath,amstext,amsbsy,amsopn,amssymb,subfigure,stmaryrd}
\usepackage{comment}
\usepackage{enumitem}
\usepackage{booktabs}
\usepackage{blkarray}
\usepackage{algorithm}
\usepackage{algpseudocode}
\usepackage{url}
\usepackage{longtable}
\usepackage{mathtools}
\usepackage{multirow}

\ifpdf
  \DeclareGraphicsExtensions{.eps,.pdf,.png,.jpg}
\else
  \DeclareGraphicsExtensions{.eps}
\fi

\newtheorem{Theorem}{Theorem}

\newtheorem{Lemma}{Lemma}
\newtheorem{Remark}{Remark}

\newtheorem{Condition}{Condition}

\DeclareMathAlphabet\mathbfcal{OMS}{cmsy}{b}{n}

\newcommand{\be}{\begin{equation}}
\newcommand{\ee}{\end{equation}}
\newcommand{\bea}{\begin{eqnarray}}
\newcommand{\eea}{\end{eqnarray}}
\newcommand{\beas}{\begin{eqnarray*}}
	\newcommand{\eeas}{\end{eqnarray*}}

\renewcommand{\a}{{\mathbf{a}}}
\renewcommand{\b}{{\mathbf{b}}}
\renewcommand{\c}{{\mathbf{c}}}
\newcommand{\e}{{\mathbf{e}}}

\renewcommand{\u}{{\mathbf{u}}}
\renewcommand{\v}{{\mathbf{v}}}
\newcommand{\w}{{\mathbf{w}}}

\newcommand{\x}{{\mathbf{x}}}
\newcommand{\y}{{\mathbf{y}}}
\newcommand{\z}{{\mathbf{z}}}
\newcommand{\X}{{\mathbf{X}}}
\newcommand{\B}{{\mathbf{B}}}
\newcommand{\C}{{\mathbf{C}}}
\newcommand{\D}{{\mathbf{D}}}
\renewcommand{\H}{{\mathbf{H}}}
\renewcommand{\L}{{\mathbf{L}}}
\newcommand{\I}{{\mathbf{I}}}
\newcommand{\M}{{\mathbf{M}}}
\newcommand{\E}{{\mathbf{E}}}

\renewcommand{\S}{{\mathbf{S}}}
\newcommand{\U}{{\mathbf{U}}}
\newcommand{\V}{{\mathbf{V}}}
\newcommand{\W}{{\mathbf{W}}}

\newcommand{\A}{{\mathbf{A}}}

\newcommand{\Z}{{\mathbf{Z}}}

\newcommand{\F}{{\mathbf{F}}}
\newcommand{\G}{{\mathbf{G}}}
\newcommand{\bcX}{{\mathbfcal{X}}}
\newcommand{\bcL}{{\mathbfcal{L}}}

\newcommand{\bcA}{{\mathbfcal{A}}}

\newcommand{\bcB}{{\mathbfcal{B}}}
\newcommand{\bcS}{{\mathbfcal{S}}}
\newcommand{\bcF}{{\mathbfcal{F}}}
\newcommand{\bcZ}{{\mathbfcal{Z}}}
\newcommand{\bcT}{{\mathbfcal{T}}}
\newcommand{\bcH}{{\mathbfcal{H}}}

\newcommand{\bcR}{{\mathbfcal{R}}}

\newcommand{\bgamma}{{\boldsymbol{\gamma}}}
\newcommand{\bbeta}{\boldsymbol{\beta}}

\newcommand{\bvarepsilon}{\boldsymbol{\varepsilon}}

\newcommand{\bSigma}{\boldsymbol{\Sigma}}

\newcommand{\rank}{{\rm rank}}

\newcommand{\tr}{{\rm tr}}

\newcommand{\SVD}{{\rm SVD}}
\newcommand{\rmvec}{{\rm vec}}
\newcommand{\tHS}{{\rm HS}}

\newcommand{\argmin}{\mathop{\rm arg\min}}

\newcommand{\supp}{{\rm supp}}
\newcommand{\bbP}{\mathbb{P}}

\newcommand{\bp}{\boldsymbol{p}}
\newcommand{\br}{\boldsymbol{r}}
\newcommand{\bs}{\boldsymbol{s}}

\makeatletter
\newcommand*{\rom}[1]{\expandafter\@slowromancap\romannumeral #1@}
\makeatother

\allowdisplaybreaks

\graphicspath{ {images/} }

\begin{document}
	\title{ISLET: Fast and Optimal Low-rank Tensor Regression via Importance Sketching}
	
	\author{Anru Zhang$^1$, ~ Yuetian Luo$^1$, ~ Garvesh Raskutti$^1$, ~ and ~ Ming Yuan$^2$}
	\date{}
	\maketitle

	\footnotetext[1]{Department of Statistics, University of Wisconsin-Madison. (\texttt{anruzhang@stat.wisc.edu}, \texttt{yluo86@wisc.edu}, \texttt{raskutti@stat.wisc.edu})}

	\footnotetext[2]{Department of Statistics, Columbia University (\texttt{my2550@columbia.edu})}

	\bigskip

\begin{abstract}
		In this paper, we develop a novel procedure for low-rank tensor regression, namely \emph{\underline{I}mportance \underline{S}ketching \underline{L}ow-rank \underline{E}stimation for \underline{T}ensors} (ISLET).
		The central idea behind ISLET is \emph{importance sketching}, i.e., carefully designed sketches based on both the responses and low-dimensional structure of the parameter of interest. 
		We show that the proposed method is sharply minimax optimal in terms of the mean-squared error under low-rank Tucker assumptions and under randomized Gaussian ensemble design. In addition, if a tensor is low-rank with group sparsity, our procedure also achieves minimax optimality.	Further, we show through numerical study that ISLET achieves comparable or better mean-squared error performance to existing state-of-the-art methods while having substantial storage and run-time advantages including capabilities for parallel and distributed computing. In particular, our procedure performs reliable estimation with tensors of dimension $p = O(10^8)$ and is $1$ or $2$ orders of magnitude faster than baseline methods.

\end{abstract}


\section{Introduction}\label{sec:intro}

The past decades have seen a large body of work on tenors or multiway arrays \cite{kolda2009tensor, sidiropoulos2017tensor, cichocki2015tensor,kroonenberg2008applied}. Tensors arise in numerous applications involving multiway data (e.g., brain imaging \cite{zhou2013tensor}, hyperspectral imaging \cite{li2010tensor}, or recommender system design \cite{bi2018multilayer}). In addition, tensor methods have been applied to many problems in statistics and machine learning where the observations are not necessarily tensors, such as topic and latent variable models \cite{anandkumar2014tensor}, additive index models \cite{balasubramanian2018tensor}, and high-order interaction pursuit \cite{hao2018sparse}, among others. In many of these settings, the tensor of interest is \emph{high-dimensional} in that the ambient dimension, i.e, the dimension of the target parameter is substantially larger than the sample size. However in practice, the tensor parameter often has intrinsic dimension-reduced structure, such as low-rankness and sparsity \cite{kolda2009tensor,sun2016sparse,udell2019big}, which makes inference possible. How to exploit such structure for tensors poses new \emph{statistical} and \emph{computational challenges} \cite{raskutti2015convex}.

From a statistical perspective, a key question is how many samples are required to learn the suitable dimension-reduced structure and what the optimal mean-squared error rates are. Prior work has developed various tensor-based methods with theoretical guarantees based on regularization approaches ~\cite{lee2010practical,mu2014square,raskutti2015convex,tomioka2011statistical}, the spectral method and projected gradient descent~\cite{chen2016non}, alternating gradient descent~\cite{li2017parsimonious,sun2017store,zhou2013tensor}, stochastic gradient descent~\cite{ge2015escaping}, and power iteration methods~\cite{anandkumar2014tensor}. However, a number of these methods are not statistically optimal. Furthermore, some of these methods rely on evaluation of a full gradient, which is typically costly in the high-dimensional setting. This leads to computational challenges including both the \emph{storage} of tensors and \emph{run time} of the algorithm. 

From a computational perspective, one approach to addressing both the storage and run-time challenge is \emph{randomized sketching}. Sketching methods have been widely studied (see e.g.
\cite{avron2016sharper,avron2014subspace, ban2019ptas,boutsidis2017optimal,clarkson2015input,clarkson2017low,dasarathy2015sketching,diao2018sketching,dobriban2018new,haupt2017near,mahoney2011randomized,nelson2013osnap,pagh2013compressed,pham2013fast,pilanci2015randomized,raskutti2014statistical,song2017low,song2019relative,sun2019low,tropp2017practical,wang2015fast,woodruff2014sketching}). Many of these prior works on matrix or tensor sketching mainly focused on relative approximation error \cite{boutsidis2017optimal,clarkson2017low,nelson2013osnap,raskutti2014statistical} after randomized sketching which either may not yield optimal mean-squared error rates under statistical settings \cite{raskutti2014statistical} or requires multiple sketching iterations \cite{pilanci2015randomized, pilanci2016iterative}.

In this article, we address both computational and statistical challenges by developing a novel sketching-based estimating procedure for tensor regression. The proposed procedure is provably fast and sharply minimax optimal in terms of mean-squared error under randomized Gaussian design. The central idea lies in constructing specifically designed structural sketches, namely \emph{importance sketching}. In contrast with randomized sketching methods, importance sketching utilizes both the response and structure of the target tensor parameter and reduces the dimension of parameters (i.e., the number of columns) instead of samples (i.e., the number of rows), which leads to statistical optimality while maintaining the computational advantages of many randomized sketching methods. See more comparison between importance sketching in this work and sketching in prior literature in Section \ref{sec:literature-review}.

\subsection{Problem Statement}\label{sec:problem-statement}

Specifically, we focus on the following low-rank tensor regression model,
\begin{equation}\label{eq:model}
y_j = \langle \bcX_j, \bcA \rangle +\varepsilon_j,\quad j=1,\ldots, n,
\end{equation}
where $y_j$ and $\varepsilon_j$ are responses and observation noise, respectively; $\{\bcX_j\}_{j=1}^n$ are tensor covariates with randomized design; and $\bcA\in \mathbb{R}^{p_1 \times\cdots\times p_d}$ is the order-$d$ tensor with parameters aligned in $d$ ways. Here $\langle \cdot, \cdot\rangle$ stands for the usual vectorized inner product. The goal is to recover $\bcA$ based on observations $\{y_j, \bcX_j\}_{j=1}^n$. In particular, when $d=2$, this becomes a low-rank matrix regression problem, which has been widely studied in recent years \cite{candes2011tight,koltchinskii2011nuclear,recht2010guaranteed}. The main focus of this paper is solving the underdetermined equation system, where the sample size $n$ is much smaller than the number of coefficients $\prod_{i=1}^d p_i$. This is because many applications belong to this regime. In particular, in the real data example to be discussed later, one MRI image is 121-by-145-by-121, which includes 2,122,945 parameters. Typically we can collect far fewer MRI images in practice.

The general regression model~\eqref{eq:model} includes specific problem instances with different choices of design $\bcX$. Examples include matrix/tensor regression with general random or deterministic design \cite{chen2016non,li2013tucker,raskutti2015convex,zhou2013tensor}, matrix trace regression~\cite{baldin2018optimal,candes2011tight,fan2017generalized,fan2016shrinkage,koltchinskii2011nuclear,recht2010guaranteed}, and matrix sparse recovery \cite{yu2018recovery}. Another example is \emph{matrix/tensor recovery via rank-{\rm 1} projections} \cite{cai2015rop,chen2015exact,hao2018sparse}, which arise by setting $\bcX_j = \u_j \circ \v_j \circ \w_j$, where $\u_j, \v_j, \w_j$ are random vectors and ``$\circ$" represents the outer product, which includes phase retrieval~\cite{cai2016optimal,candes2015phase} as a special case. The very popular matrix/tensor completion example~\cite{candes2010power,liu2013tensor,montanari2016spectral,xia2017polynomial,xia2017statistically,yuan2014tensor} arises by setting $\bcX_j = \left(\e_{a_j} \circ \e_{b_j} \circ \e_{c_j}\right)$, where $\e_j$ is the $j$th canonical vector and $\{a_j, b_j, c_j\}_{j=1}^n$ are randomly selected integers from $\{1,\ldots,  p_1\}\times \{1,\ldots, p_2\}\times \{1,\ldots, p_3\}$. Specific applications of this low-rank tensor regression model include neuroimaging analysis \cite{guhaniyogi2015bayesian,li2017parsimonious,zhou2013tensor}, longitudinal relational data analysis \cite{hoff2015multilinear}, 3D imaging processing \cite{guo2012tensor}, etc.

For convenience of presentation, we specialize the discussions on order-3 tensors later, while the results can be extended to the general order-$d$ tensors. In the modern high-dimensional setting, a variety of matrix/tensor data satisfy intrinsic structural assumptions, such as low-rankness \cite{udell2019big} or sparsity \cite{zhou2013tensor}, which makes the accurate estimation of $\bcA$ possible even if the sample size $n$ is smaller than the number of coefficients in the target tensor $\bcA$.
We thus focus on the low Tucker rank $(r_1, r_2, r_3)$ tensor $\bcA$ with the following Tucker decomposition \cite{tucker1966some}:
\begin{equation}\label{eq:tucker-decomposition}
\bcA = \llbracket \bcS; \U_1, \U_2, \U_3\rrbracket := \bcS\times_1 \U_1 \times_2 \U_2 \times_3 \U_3,
\end{equation}
where $\bcS$ is an $r_1$-by-$r_2$-by-$r_3$ core tensor and $\U_k$ is a $p_k$-by-$r_k$ matrix with orthonormal columns for $k=1,2,3$. The rigorous definition of Tucker rank of a tensor and more discussions on tensor algebra are postponed to Section \ref{sec:notations}. In addition, the canonical polyadic (CP) low-rank tensors have also been widely considered in recent literature \cite{hao2018sparse,haupt2017near,sun2017store,zhou2013tensor}. Since any CP-rank-$r$ tensor $\bcA = \sum_{i=1}^r \lambda_i \a_i \circ \b_i \circ \c_i$ has the Tucker decomposition $\bcA = \llbracket \bcL; \A, \B, \C \rrbracket$, where $\bcL$ is the $r$-by-$r$-by-$r$ diagonal tensor with diagonal entries $\lambda_1,\ldots, \lambda_r$, $\A = [\a_1, \ldots, \a_r]$, and likewise for $\B, \C$ \cite{kolda2009tensor}, our results naturally adapt to low CP-rank tensor regression. Also, with a slight abuse of notation, we will refer to low-rank and low Tucker rank interchangeably throughout the paper.
Moreover, we also consider a sparse setting where there may exist a subset of modes, say $J_s \subseteq \{1,2, 3\}$, such that $\bcA$ is sparse along these modes, i.e.
\begin{equation}\label{eq:tucker-decomposition-sparse}
\bcA = \llbracket\bcS; \U_1, \U_2, \U_3\rrbracket, \quad \|\U_k\|_0 = \sum_{i=1}^{p_k} 1_{\{(\U_k)_{[i, :]}\neq 0\}} \leq s_k,\quad k \in J_s.
\end{equation}

\subsection{Our Contributions}

We make the following major contributions to low-rank tensor regression in this article. First, we introduce the main algorithm -- \emph{\underline{I}mportance \underline{S}ketching \underline{L}ow-rank \underline{E}stimation for \underline{T}ensors} (ISLET). Our algorithm has three steps: (i) first we use the tensor technique high-order orthogonal iteration (HOOI) \cite{de2000best} or sparse tensor alternating thresholding - singular value decomposition (STAT-SVD) \cite{zhang2017optimal-statsvd} to determine the importance sketching directions. Here HOOI and STAT-SVD are regular and sparse tensor low-rank decomposition methods, respectively, whose explanations are postponed to Sections \ref{sec:regular-islet-procedure} and \ref{sec:sparse-procedure}; (ii) using the sketching directions from the first step, we perform importance sketching, and then evaluate the dimension-reduced regression using the sketched tensors/matrices (to incorporate sparsity, we add a group-sparsity regularizer); (iii) we construct the final tensor estimator using the sketched components. Although the focus of this work is on low-rank tensor regression, we point out that our three-step procedure applies to general high-dimensional statistics problems with low-dimensional structure, provided that we can find a suitable projection operator in  step (i) and inverse projection operator in step (iii). 

One of the main advantages of ISLET is the scalability of the algorithm. The proposed procedure is computationally efficient due to the dimension reduction by importance sketchings. 
Most importantly, ISLET only require access to the full data twice, which significantly saves run time for large-scale settings when it is not possible to store all samples into the core memory. We also show that our algorithm can be naturally distributed across multiple machines that can significantly reduce computation time. 

Second, we prove a deterministic oracle inequality for the ISLET procedure under the low-Tucker-rank assumption and general noise and design (Theorems \ref{th:upper_bound_general} and \ref{th:upper_bound_sparse_general}). We additionally show that ISLET achieves the optimal mean-squared error (with the optimal constant for nonsparse ISLET) under randomized Gaussian design (Theorems \ref{th:upper_bound_regression}, \ref{th:lower-bound-regression}, \ref{th:upper_bound_sparse_tensor_regression}, and \ref{th:lower_bound_sparse_tensor_regression}). The following informal statement summarizes two of the main results of the article.
\begin{Theorem}[ISLET for tensor regression: informal]
\label{ThmInformal}
	Consider the regular tensor regression problem with Gaussian ensemble design, where $\bcA$ is Tucker rank-$(r_1,r_2,r_3)$, $\bcX_j$ has i.i.d. standard normal entries, $\varepsilon_j\overset{i.i.d.}{\sim}N(0, \sigma^2)$, and $\varepsilon_j, \bcX_j$ are independent: 
	\begin{itemize}
	    \item[(a)]
	Under regularity conditions, ISLET achieves the following optimal rate of convergence with the matching constant,
	$$\mathbb{E}\left\|\widehat{\bcA} - \bcA\right\|_{\tHS}^2 = \left(1 + o(1)\right) \frac{m\sigma^2}{n},$$
	where $m = r_1r_2r_3 + r_1(p_1-r_1) + r_2(p_2-r_2) + r_3(p_3-r_3)$ is exactly the degree of freedom of all Tucker rank-$(r_1, r_2, r_3)$ tensors in $\mathbb{R}^{p_1\times p_2\times p_3}$ and $\left\| \cdot \right\|_{\tHS}$ is the Hilbert-Schmidt norm to be defined in Section \ref{sec:notations}.
	\item[(b)]
	If, in addition, \eqref{eq:tucker-decomposition-sparse} holds with sparsity level $s_k$, then under regularity conditions, ISLET achieves the following optimal rate of convergence: 
	$$\mathbb{E}\left\|\widehat{\bcA} - \bcA\right\|_{\tHS}^2 \asymp \frac{m_s\sigma^2}{n},$$
	where $m_s = r_1r_2r_3 + \sum_{k\in J_s} s_k \left(r_k + \log(p_k/s_k)\right) + \sum_{k\notin J_s} p_kr_k$ and ``$\asymp$" denotes the asymptotic equivalence between two number series (see a more formal definition in Section \ref{sec:notations}).
	\end{itemize}
\end{Theorem}
To the best of our knowledge, we are the first to develop the matching-constant optimal rate results for regular tensor regression under randomized Gaussian ensemble design, even for the low-rank matrix recovery case since it is not clear whether prior approaches (e.g. nuclear norm minimization) achieve sharp constants. We are also the first to develop the optimal rate results for tensor regression with sparsity condition (\ref{eq:tucker-decomposition-sparse}). 

Third, proving the optimal mean-squared error bound presents a number of technical challenges and we introduce novel proof ideas to overcome these difficulties. In particular, one major difficulty lies in the analysis of reduced-dimensional regressions (see \eqref{eq:partial-regression} in Section \ref{sec:procedure}) since we analyze sketched regression models. To this end, we introduce partial linear models for these reduced-dimensional regressions from which we develop estimation error upper bounds.

The final and most important computational contribution is to display through numerical studies the advantages of our ISLET algorithms. Compared to state-of-the-art tensor estimation algorithms including nonconvex projected gradient descent (PGD) \cite{chen2016non}, Tucker regression \cite{zhou2013tensor}, and convex regularization \cite{tomioka2013convex}, we show that our ISLET algorithm achieves comparable statistical performance with substantially faster computation. In particular, the run time is 1-3 orders of magnitude faster than existing methods. In the most prominent example, our ISLET procedure can efficiently solve the ultrahigh-dimensional tensor regression with covariates of 7.68 terabytes. For the order-2 case, i.e., low-rank matrix regression, our simulation studies show that ISLET outperforms the classic nuclear norm minimization estimator. We also provide a real data application where we study the association between the attention-deficit/hyperactivity disorder disease and the high-dimensional MRI image tensors. We show that the proposed procedure provides significantly better prediction performance in much less time compared to state-of-the-art methods.

\subsection{Related Literature}\label{sec:literature-review}

Our work is related to a broad range of literature varying from a number of communities including scientific computing, computer science, signal processing, applied mathematics, and statistics. Here we make an attempt to discuss existing results from these various communities; however, we do not claim that our literature survey is exhaustive.

Large-scale linear systems where the solution admits a low-rank tensor structure commonly arise after discretizing high-dimensional partial differential equations \cite{hofreither2018black,hughes2005isogeometric,lynch1964tensor} and various methods have been proposed. For example, \cite{bousse2017linear} developed algebraic and Gauss-Newton methods to solve the linear system with a CP low-rank tensor solution. \cite{ballani2013projection, beylkin2005algorithms} proposed iterative projection methods to solve large-scale linear systems with Kronecker-product-type design matrices. \cite{georgieva2019greedy} introduced a greedy approach. \cite{kressner2016preconditioned,kressner2010krylov} considered Riemannian optimization methods and tensor Krylov subspace methods, respectively. The readers are referred to \cite{grasedyck2013literature} for a recent survey. Different from these works, our proposed ISLET is a one-step procedure that only involves solving a simple least squares regression after performing dimension reduction on covariates by importance sketching (see Steps 1 and 2 in Section \ref{sec:regular-islet-procedure}).  Moreover, many prior works mainly focused on computational aspects of their proposed methods \cite{ballani2013projection,bousse2018linear,espig2012variational,georgieva2019greedy,grasedyck2013literature}, while we show that ISLET is not only computationally efficient (see more discussion and comparison on computation complexity in the Computation and Implementation part of Section \ref{sec:regular-islet-procedure}) but also has optimal theoretical guarantees in terms of mean square error under the statistical setting.

In addition, sketching methods play an important role in computation acceleration and have been widely considered in previous literature. For example, \cite{clarkson2017low,meng2013low, nelson2013osnap} provided accurate approximation algorithms based on sketching with novel embedding matrices, where the run time is proportional to the number of the nonzero entries of the input matrix. Sketching methods have also been studied in robust $\ell_1$ low-rank matrix approximation \cite{markopoulos2014optimal, markopoulos2017efficient, meng2013cyclic,song2017low,zheng2012practical}, general $\ell_p$ low-rank matrix approximation \cite{ban2019ptas,chierichetti2017algorithms}, low-rank tensor approximation \cite{song2019relative}, etc. In the regression context, the sketching method has been considered for the least squares regression \cite{clarkson2017low,diao2018sketching, nelson2013osnap, pilanci2016iterative, raskutti2014statistical}, $\ell_p$ regression \cite{clarkson2017low,  meng2013low,nelson2013osnap}, Kronecker product regression \cite{diao2018sketching}, ridge regression \cite{avron2016sharper, wang2017sketching}, regularized kernel regression \cite{camoriano2016nytro,zhang2014random}, etc. Various types of random sketching matrices have been developed, including random sub-Gaussian \cite{pilanci2016iterative}, random sampling \cite{drineas2012fast,drineas2010effective}, CountSketch \cite{charikar2002finding,clarkson2015input}, Sparse Johnson-Lindenstrauss transformation \cite{kane2014sparser}, among many others. The readers are also referred to survey papers on sketching by Mahoney \cite{mahoney2011randomized} and Woodruff \cite{woodruff2014sketching}. The proposed method in this paper is different from these previous works in various aspects. First, many randomized sketching methods in the literature focus on relative approximation error \cite{mahoney2011randomized,woodruff2014sketching} and the sketching matrices are constructed only based on covariates \cite{drineas2012fast,drineas2010effective,kane2014sparser,pilanci2016iterative,raskutti2014statistical}. In contrast, we explicitly construct ``supervised" sketching matrices based on both the response $y_j$ and covariates $\bcX_j$ and obtain optimal bounds in mean square error under the statistical setting. Second, essentially speaking, our proposed importance sketching scheme reduces the number of columns (parameters) instead of the number of rows (samples) in the linear equation system. Third, different from the sketching on an overdetermined system of least squares \cite{clarkson2017low,diao2018sketching, nelson2013osnap, pilanci2016iterative,raskutti2014statistical}, we mainly focus on the high-dimensional setting where the number of samples can be significantly smaller than the number of coefficients.

\subsection{Organization} 
In Section~\ref{sec:notations} we introduce important notation; then we present our ISLET procedure under nonsparse and sparse settings in Sections \ref{sec:regular-islet-procedure} and \ref{sec:sparse-procedure}, respectively, and illustrate the procedure from a sketching perspective in Section \ref{sec:sketching-perspective}. In Section \ref{sec:oracle-inequality} we provide general theoretical guarantees for our procedure which make no assumptions on the design or the noise distribution; in Section~\ref{th:upper_bound_regression} we specialize our bounds to tensor regression with low Tucker rank and assume the design is independent Gaussian; a simulation study showing the substantial computational benefits of our algorithm is provided in Section~\ref{sec:numerical}. Additional notation, discussion on general-order ISLET, simulation results, an application to attention deficit hyperactivity disorder (ADHD) MRI imaging data analysis, and all technical proofs are provided in the supplementary materials \cite{ISLET-supplement}, linked from the main article webpage.

\section{Our Procedure: ISLET}\label{sec:procedure}

Here we introduce the general procedure of Importance Sketching Low-Rank Estimation for tensors (ISLET). Although for ease of presentation we will focus on order-3 tensors, the procedure for the general order-$d$ case can also be treated. Details of matrices and tensors greater than order 3 are provided in Section \ref{sec:general-order} of the supplementary materials \cite{ISLET-supplement}. 

\subsection{Notation and Preliminaries}\label{sec:notations}

The following notation will be used throughout this article. Additional definitions can be found in Section \ref{sec:row-permutation-operator} in the supplementary materials. Lowercase letters (e.g., $a, b$), lowercase boldface letters (e.g. $\u, \v$), uppercase boldface letters (e.g., $\U, \V$), and boldface calligraphic letters (e.g., $\bcA, \bcX$) are used to denote scalars, vectors, matrices, and order-3-or-higher tensors respectively. For simplicity, we denote $\bcX_j$ as the tensor indexed by $j$ in a sequence of tensors $\{\bcX_j\}$. For any two series of numbers, say $\{a_i\}$ and $\{b_i\}$, denote $a \asymp b$ if there exist uniform constants $c, C>0$ such that $ca_i \leq b_i\leq Ca_i, \forall i$ and $a  = \Omega (b)$ if there exists uniform constant $c > 0$ such that $a_i \geq c b_i, \forall i$. We use bracket subscripts to denote subvectors, submatrices, and subtensors. For example, $\v_{[2:r]}$ is the vector with the $2$nd to $r$th entries of $\v$; $\D_{[i_1,i_2]}$ is the entry of $\D$ on the $i_1$th row and $i_2$th column; $\D_{[(r+1):p_1, :]}$ contains the $(r+1)$th to the $p_1$th rows of $\D$; $\bcA_{[1:s_1, 1:s_2, 1:s_3]}$ is the $s_1$-by-$s_2$-by-$s_3$ subtensor of $\bcA$ with index set $\{(i_1, i_2, i_3): 1 \leq i_1\leq s_1, 1\leq i_2 \leq s_2, 1\leq i_3\leq s_3\}$. For any vector $\v\in \mathbb{R}^{p_1}$, define its $\ell_q$ norm as $\|\v\|_q = \left(\sum_i |v_i|^q\right)^{1/q}$. For any matrix $\D \in \mathbb{R}^{p_1\times p_2}$, let $\sigma_k(\D)$ be the $k$th singular value of $\D$. In particular, the least nontrivial singular value of $\D$, defined as $\sigma_{\min}(\D) = \sigma_{p_1\wedge p_2}(\D)$, will be extensively used in later analysis. We also denote $\SVD_r(\D) = [\u_1 ~ \cdots \u_r]$ and QR($\D$) as the subspace composed of the leading $r$ left singular vectors and the Q part of the QR orthogonalization of $\D$, respectively. The matrix Frobenius and spectral norms are defined as $ \|\D\|_F = \left(\sum_{i_1,i_2} \D_{[i_1,i_2]}^2\right)^{1/2} = (\sum_{i=1}^{p_1\wedge p_2}\sigma_i^2(\D))^{1/2} \quad \text{and} \quad\|\D\| = \max_{\u\in \mathbb{R}^{p_2}}\|\D \u\|_2/\|\u\|_2  = \sigma_1(\D).$ In addition, $\I_r$ represents the $r$-by-$r$ identity matrix. Let $\mathbb{O}_{p, r} = \{\U: \U^\top \U=\I_r\}$ be the set of all $p$-by-$r$ matrices with orthonormal columns. For any $\U\in \mathbb{O}_{p, r}$, $P_{\U} = \U\U^\top$ represents the projection matrix onto the column space of $\U$; we also use $\U_\perp\in \mathbb{O}_{p, p-r}$ to represent the orthonormal complement of $\U$. For any event $A$, let $\bbP(A)$ be the probability that $A$ occurs.

For any matrix $\D\in \mathbb{R}^{p_1\times p_2}$ and order-$d$ tensor $\bcA \in \mathbb{R}^{p_1\times \cdots\times p_d}$, let $\rmvec(\D)$ and $\rmvec(\bcA)$ be the vectorization of $\D$ and $\bcA$, respectively. The matricization $\mathcal{M}(\cdot)$ is the operation that unfolds or flattens the order-$d$ tensor $\bcA\in\mathbb{R}^{p_1\times \cdots \times p_d}$ into the matrix $\mathcal{M}_k(\bcA)\in \mathbb{R}^{p_k\times \prod_{j\neq k}p_j}$ for $k=1,\ldots, d$. Since the formal entrywise definitions of matricization and vectorization is rather tedious, we leave them to Section \ref{sec:row-permutation-operator} in the supplementary materials \cite{ISLET-supplement}. The Hilbert-Schmidt norm is defined as
$\|\bcA\|_{\tHS} = \left(\sum_{i_1,\ldots, i_d} \bcA_{[i_1,\ldots, i_d]}^2\right)^{1/2}.$ An order-$d$ tensor is rank-one if it can be written as the outer product of $d$ nonzero vectors. The CP rank of any tensor $\bcA$ is defined as the minimal number $r$ such that $\bcA$ can be decomposed as $\bcA = \sum_{i=1}^r \bcB_i$ for rank-1 tensors $\bcB_i$. The Tucker rank (or multilinear rank) of a tensor $\bcA$ is defined as a $d$-tuple $(r_1, \ldots, r_d)$, where $r_k = \text{rank}(\mathcal{M}_k(\bcA))$. The $k$-mode product of $\bcA \in \mathbb{R}^{p_1 \times \ldots \times p_d}$ with a matrix $\U\in \mathbb{R}^{p_k\times r_k}$ is denoted by $\bcA \times_k \U$ and is of size $p_1 \times \cdots \times p_{k-1}\times r_k \times p_{k+1}\times \cdots \times p_d$, such that 
$$(\bcA \times_k \U)_{[i_1, \ldots, i_{k-1}, j, i_{k+1}, \ldots, i_d]} = \sum_{i_k=1}^{p_k} \bcA_{[i_1, i_2, \ldots, i_d]} \U_{[i_k, j]}.$$
For convenience of presentation, all mode indices $(\cdot)_k$ of an order-3 tensor are in the sense of modulo-3, e.g., $r_1=r_4$, $s_2 = s_5$, $p_0=p_3$, $\bcX \times_4 \U_4 = \bcX \times_1 \U_1$.

For any matrices $\U\in \mathbb{R}^{p_1\times p_2}$ and $\V\in \mathbb{R}^{m_1\times m_2}$, let
$$\U\otimes \V = \begin{bmatrix}
\U_{[1,1]}\cdot \V & \cdots & \U_{[1, p_2]}\cdot \V\\
\vdots & & \vdots\\
\U_{[p_1,1]}\cdot \V & \cdots & \U_{[p_1, p_2]}\cdot \V\\
\end{bmatrix}\in \mathbb{R}^{(p_1m_1)\times(p_2m_2)}$$ 
be the Kronecker product. Some intrinsic identities among Kronecker product, vectorization, and matricization, which will be used later in this paper, are summarized in Lemma~\ref{lm:Kronecker-vectorization-matricization} in the supplementary materials \cite{ISLET-supplement}. Readers can refer to \cite{kolda2009tensor} for a more comprehensive introduction to tensor algebra. Finally, we use $C, C_1, C_2, c$ and other variations to represent the large and small constants, whose actual value may vary from line to line.

\subsection{Regular Low-rank Tensor Recovery}\label{sec:regular-islet-procedure}
We first consider the tensor regression model \eqref{eq:model}, where $\bcA$ is low-rank \eqref{eq:tucker-decomposition} without sparsity assumptions. The proposed algorithm of ISLET is divided into three steps and a pictorial illustration is provided in Figures \ref{fig:illu-1} - \ref{fig:illu-3} for readers' better understanding. The pseudocode is provided in Algorithm \ref{al:procedure_regular}.
\begin{enumerate}[leftmargin=*]
	\item[Step 1] (Probing importance sketching directions) We first probe the importance sketching directions. When the covariates satisfy $\mathbb{E}\rmvec(\bcX_j) \rmvec(\bcX_j)^\top = \I_{p_1p_2p_3}$, we evaluate
	\begin{equation}\label{eq:tilde-A}
	\widetilde{\bcA}  = \frac{1}{n}\sum_{j=1}^{n} y_j\bcX_j.
	\end{equation}
	$\widetilde\bcA$ is essentially the covariance tensor between $y$ and $\bcX$.
	Since $\bcA=\llbracket \bcS; \U_1, \U_2, \U_3\rrbracket$ has low Tucker rank, we perform the high-order orthogonal iterations (HOOI) on $\widetilde{\bcA}$ to obtain $\widetilde{\U}_k \in \mathbb{O}_{p_k, r_k}, k=1,2,3$ as initial estimates for $\U_k$. Here HOOI is a classic method for tensor decomposition that can be traced back to De Lathauwer, Moor, and Vandewalle \cite{de2000best}. The central idea of HOOI is the power iterated singular value thresholding. Then the outcome of HOOI $\{\widetilde{\U}_k\}_{k=1}^3$ yields the following low-rank approximation for $\bcA$:
	\begin{equation}\label{eq:bcA-approximation}
	\bcA \approx \llbracket \widetilde{\bcS}; \widetilde{\U}_1, \widetilde{\U}_2, \widetilde{\U}_3 \rrbracket,\quad \text{where}\quad \widetilde{\bcS} = \llbracket\widetilde{\bcA}; \widetilde{\U}_1^\top, \widetilde{\U}_2^\top, \widetilde{\U}_3^\top\rrbracket \in \mathbb{R}^{r_1\times r_2\times r_3}.
	\end{equation}
	We further evaluate
	\begin{equation*}
	\begin{split}
	\widetilde{\V}_k := {\rm QR}\left(\mathcal{M}_k^\top(\widetilde{\bcS})\right) \in \mathbb{O}_{r_{k+1}r_{k+2}, r_k}, \quad k=1,2,3.
	\end{split}
	\end{equation*}
	$\{\widetilde{\U}_k, \widetilde{\V}_k\}_{k=1}^3$ obtained here are regarded as the \emph{importance sketching directions}. As we will further illustrate in Section \ref{sec:oracle-regular}, the combinations of $\widetilde{\U}_k$ and $\widetilde{\V}_k$ provide approximations for singular subspaces of $\mathcal{M}_k(\bcA)$.
	\item[Step 2] (Linear regression on sketched covariates) Next, we perform sketching to reduce the dimension of the original regression model \eqref{eq:model}. To be specific, we project the original high-dimensional covariates onto the dimension-reduced subspace ``that is important in the covariance between $y$ and $\bcX$" and construct the following \emph{importance sketching covariates},
	\begin{equation}\label{eq:importance-sketching-covariates}
	\begin{split}
	& \widetilde{\X} = \left[\widetilde{\X}_\bcB ~~ \widetilde{\X}_{\D_1} ~~ \widetilde{\X}_{\D_2} ~~ \widetilde{\X}_{\D_3}\right] \in \mathbb{R}^{n\times m},\\
	& \widetilde{\X}_\bcB \in \mathbb{R}^{n\times m_{\bcB}},\quad \left(\widetilde{\X}_\bcB\right)_{[i,:]} = 
	\rmvec\left(\bcX_i \times_1 \widetilde{\U}_1^\top \times_2 \widetilde{\U}_2^\top \times_3 \widetilde{\U}_3^\top\right),\\
	& \widetilde{\X}_{\D_k} \in \mathbb{R}^{n\times m_{\D_k}}, \quad \left(\widetilde{\X}_{\D_k}\right)_{[i,:]} = \rmvec\left(\widetilde{\U}_{k\perp}^\top \mathcal{M}_k\left(\bcX_{i} \times_{k+1} \widetilde{\U}_{k+1}^\top \times_{k+2} \widetilde{\U}_{k+2}^\top\right) \widetilde{\V}_k\right),
	\end{split}
	\end{equation}
	where $m_{\bcB}=r_1r_2r_3$, $m_{\D_k} = (p_k-r_k)r_k$, $k=1,2,3$, and $m = m_{\bcB} + m_{\D_1} + m_{\D_2} + m_{\D_3}$.
	Then, we evaluate the least-squares estimator of the submodel with importance sketching covariates $\widetilde{\X}$,
	\begin{equation}\label{eq:partial-regression}
	\widehat{\bgamma} = \argmin_{\bgamma\in\mathbb{R}^m}\left\|y - \widetilde{\X}\bgamma\right\|_2^2.
	\end{equation}
	The dimension of sketching covariate regression \eqref{eq:partial-regression} is $m$, which is significantly smaller than the dimension of the original tensor regression model, $p_1p_2p_3$. Consequently, the computational cost can be significantly reduced. 
	\item[Step 3] (Assembling the final estimate) 	Then, $\widehat{\bgamma}$ is divided into four segments according to the blockwise structure of $\widetilde{\X} = [\widetilde{\X}_\bcB, \widetilde{\X}_{\D_1}, \widetilde{\X}_{\D_2}, \widetilde{\X}_{\D_3}]$,
	\begin{equation}\label{eq:def_hat_Beta_hat_D}
	\begin{split}
	& \rmvec(\widehat{\bcB}) = \widehat{\bgamma}_{[1:m_{\bcB}]},\\
	& \rmvec(\widehat{\D}_1) = \widehat{\bgamma}_{[(m_{\bcB}+1): (m_{\bcB}+m_{\D_1})]},\\
	& \rmvec(\widehat{\D}_2) = \widehat{\bgamma}_{[(m_{\bcB}+m_{\D_1}+1): (m_{\bcB}+m_{\D_1}+m_{\D_2})]},\\
	& \rmvec(\widehat{\D}_3) = \widehat{\bgamma}_{[(m_{\bcB}+m_{\D_1}+m_{\D_2}+1):(m_{\bcB}+m_{\D_1}+m_{\D_2}+m_{\D_3})]}.
	\end{split}
	\end{equation}
	Finally, we construct the regression estimator $\widehat{\bcA}$ for the original problem \eqref{eq:model} using the regression estimator $\widehat{\bgamma}$ for the submodel \eqref{eq:def_hat_Beta_hat_D}: let $\widehat{\B}_k = \mathcal{M}_k(\widehat{\bcB})$ and calculate
	\begin{equation}\label{eq:hat_A_non-sparse}
	\begin{split}
	& \widehat{\L}_k = \left(\widetilde{\U}_k\widehat{\B}_k\widetilde{\V}_k + \widetilde{\U}_{k\perp}\widehat{\D}_k\right)\left(\widehat{\B}_k\widetilde{\V}_k\right)^{-1},\quad k=1,2,3, \quad \widehat{\bcA} = \left\llbracket\widehat{\bcB}; \widehat{\L}_1, \widehat{\L}_2, \widehat{\L}_3 \right\rrbracket.
	\end{split}
	\end{equation}
	More interpretation of \eqref{eq:hat_A_non-sparse} is given in Section~\ref{sec:oracle-regular}.
\end{enumerate}

\begin{Remark}[Alternative Construction of $\widetilde{\bcA}$ in Step 1]
	When $\mathbb{E}\rmvec(\bcX)\rmvec(\bcX)^\top \neq \I_{p_1p_2p_3}$, we could consider the following alternative ways to construct the initial estimate $\widetilde{\bcA}$. First, in some cases we could do construction depending on the covariance structure of $\bcX$. For example, in the framework of tensor recovery via rank-one sketching (discussed in the introduction), we have $\bcX_j = \u_j\circ \u_j\circ \u_j$ and $\u_j \in \mathbb{R}^p$ has i.i.d entry $N(0,1)$. By the high-order Stein identity \cite{janzamin2014score}, one can show that
	$$\widetilde{\bcA} = \frac{1}{6}\left[\frac{1}{n}\sum_{j=1}^n  y_j \u_j \circ \u_j \circ \u_j - \sum_{j=1}^p \left(\w\circ \e_j \circ \e_j + \e_j \circ \w \circ \e_j +  \e_j \circ \e_j \circ \w \right)\right],$$ 
	is a proper initial unbiased estimator for $\bcA$ \cite[Lemma 4]{hao2018sparse}. Here, $\w = \frac{1}{n}\sum_{i=1}^n y_j \u_j$, $\e_j$ is the $j$th canonical basis in $\mathbb{R}^p$. Another commonly used setting in data analysis is the high-order Kronecker covariance structure: $\mathbb{E}(\rmvec(\bcX_j)\rmvec(\bcX_j)^\top) =  \bSigma_3\otimes \bSigma_2\otimes \bSigma_1$, where $\bSigma_k \in \mathbb{R}^{p_k\times p_k}, k=1,2,3$ are covariance matrices along three modes, respectively \cite{he2014graphical,lyu2019tensor,manceur2013maximum, pan2018covariate,zhou2014gemini}. Under this assumption, we can first apply existing approaches to obtain estimators $\widehat{\bSigma}_k$ for $\bSigma_k$, then whiten the covariates by replacing $\bcX_j$ by $\llbracket \bcX_j; \widehat{\bSigma}_1^{-1/2}, \widehat{\bSigma}_2^{-1/2}, \widehat{\bSigma}_3^{-1/2} \rrbracket$. After this preprocessing step, the other steps of ISLET still follow. Moreover, it still remains an open question how to perform initialization  if $\bcX$ has the more general, unstructured, and unknown design.
\end{Remark}

\begin{Remark}[Alternative Methods to HOOI]
In addition to high-order orthogonal iteration (HOOI), there are a variety of methods  proposed in the literature to compute the low-rank tensor approximation, such as Newton-type optimization methods on manifolds \cite{elden2009newton, ishteva2011best, ishteva2009differential, savas2010quasi}, black box approximation \cite{bebendorf2011adaptive, caiafa2010generalizing,mahoney2008tensor,oseledets2008tucker, oseledets2010cross, zhang2019cross}, generalizations of Krylov subspace method \cite{goreinov2012wedderburn,savas2013krylov}, greedy approximation method \cite{georgieva2019greedy}, among many others. Further, black box approximation methods \cite{bebendorf2011adaptive, caiafa2010generalizing,oseledets2008tucker, oseledets2010cross,zhang2019cross} can be applied even if the initial estimator $\widetilde{\bcA}$ does not fit into the core memory. When the tensor is further approximately CP low-rank, we can also apply the randomized compressing method \cite{sidiropoulos2012multi, sidiropoulos2014parallel} or randomized block sampling \cite{vervliet2015randomized} to obtain the CP low-rank tensor approximation. Although the rest of our discussion will focus on the HOOI procedure for initialization, these alternative methods can also be applied to obtain an initialization for the ISLET algorithm.

\end{Remark}

\begin{figure}[htbp]
	\centering
	\subfigure[Construct the covariance tensor $\widetilde\bcA$]{\includegraphics[height=2.5cm]{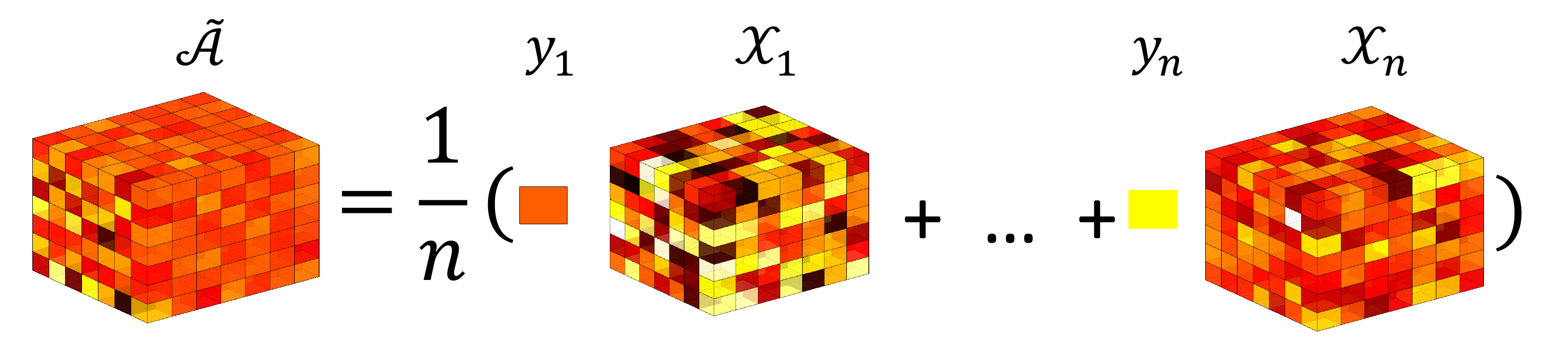}} \vskip1cm
	\subfigure[Perform HOOI on $\widetilde\bcA$ to obtain sketching directions]{~\quad\includegraphics[height=3cm]{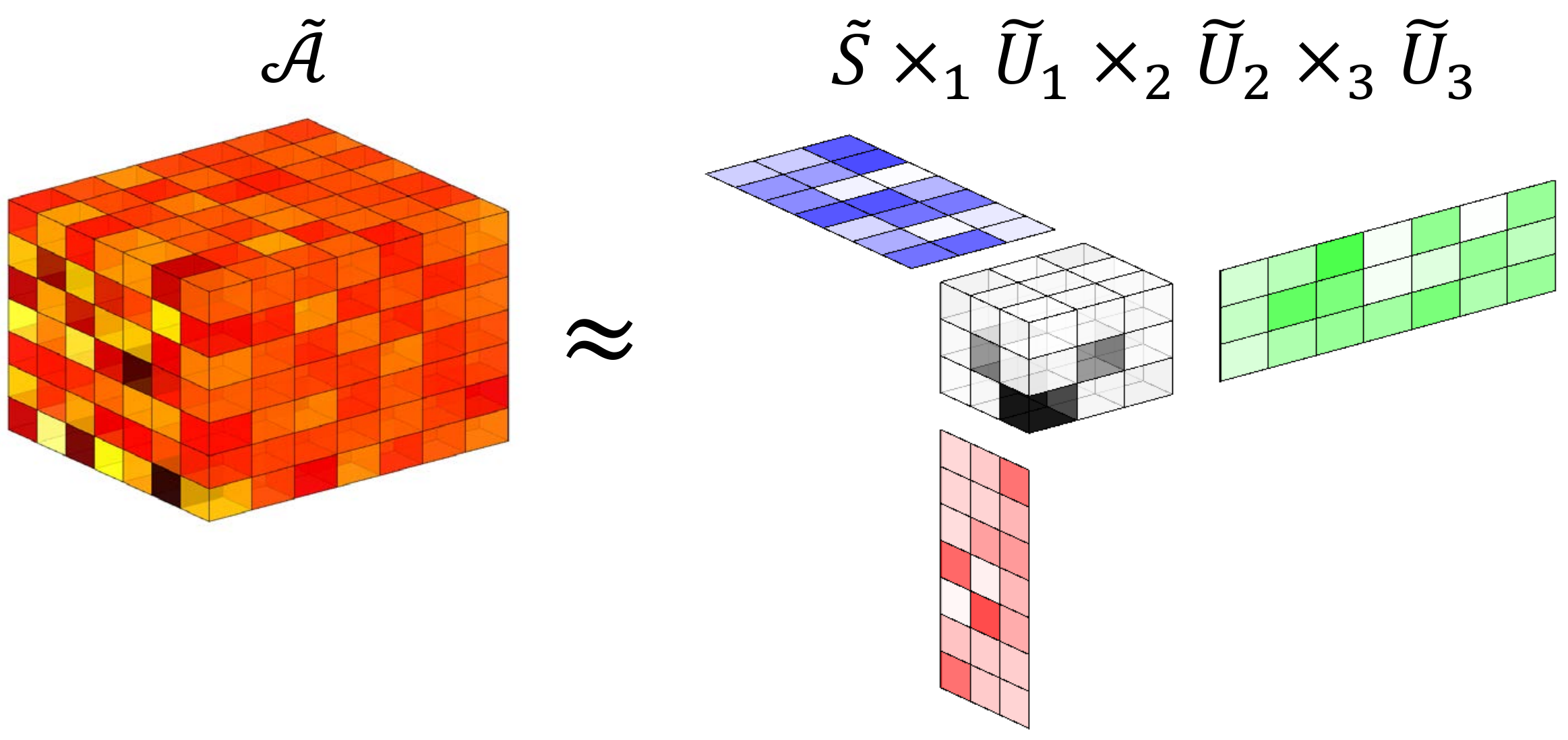}\quad~} \vskip1cm
	\subfigure[The sketching directions yield low-rank approximations for $\mathcal{M}_k(\widetilde{\bcA})$]{\includegraphics[height=5cm]{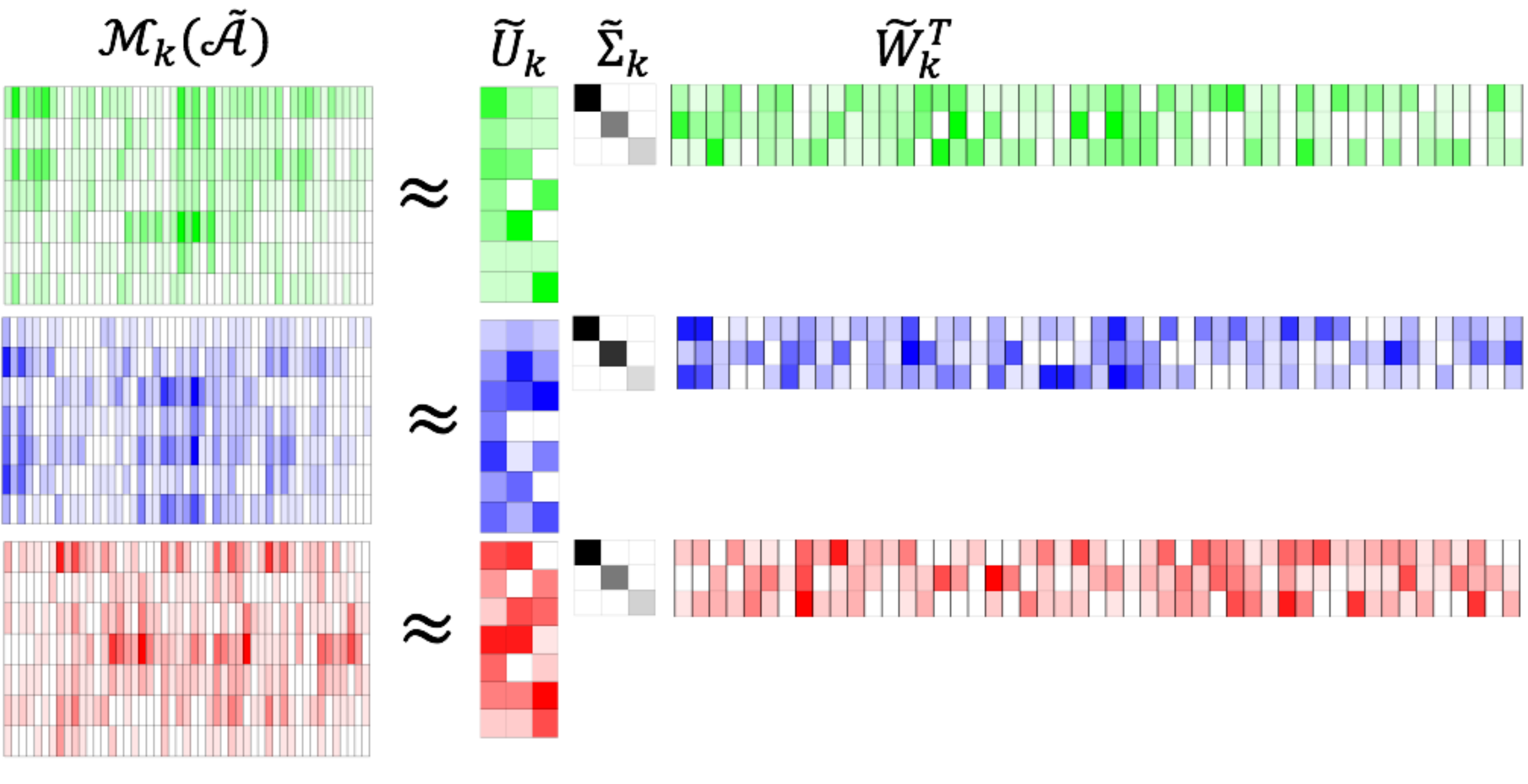}}
	\caption{Illustration for Step 1 of ISLET}
\label{fig:illu-1}
\end{figure}

\begin{figure}[htbp]
	\centering
	\subfigure[Construct importance sketching covariates by projections]{~\quad\includegraphics[width=.92\linewidth]{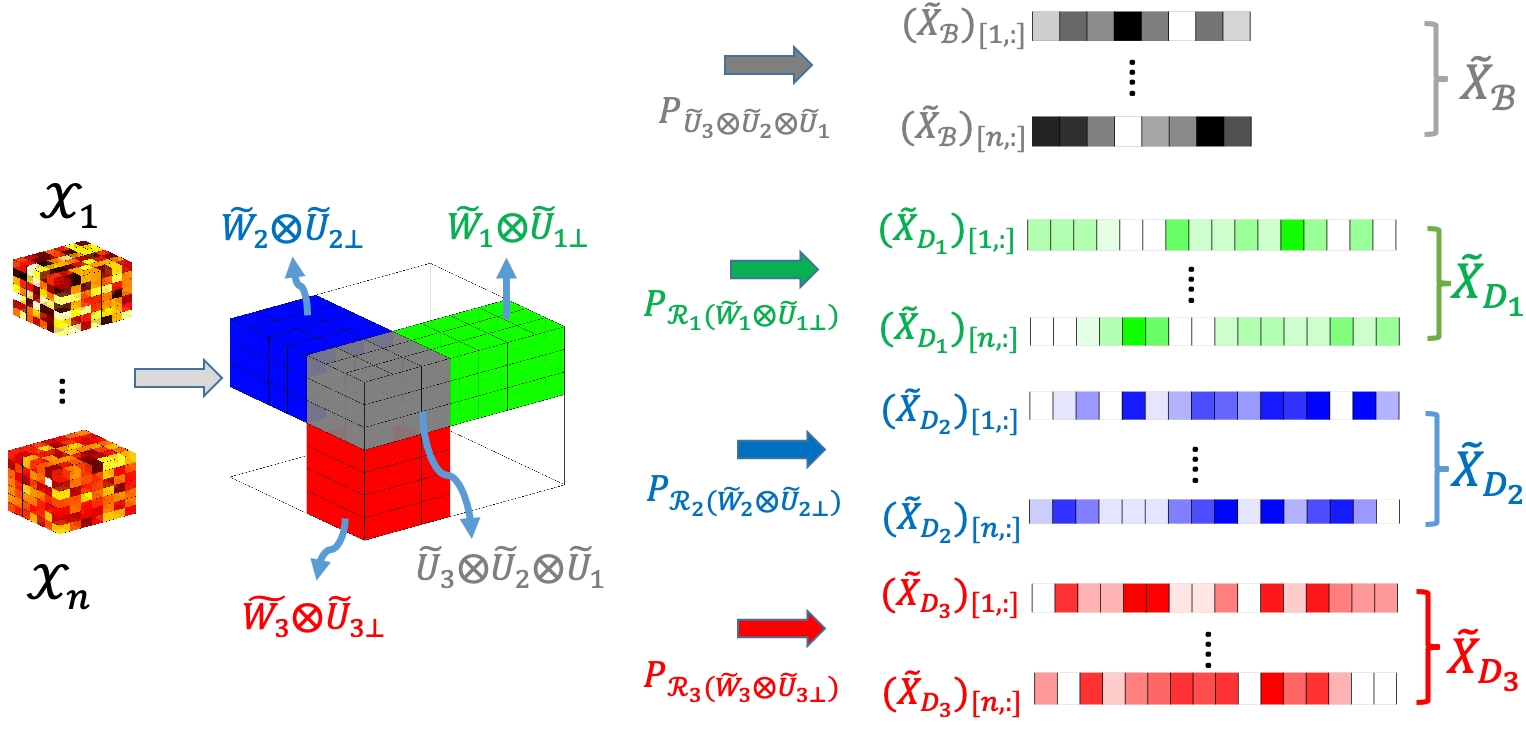}\quad~}\\\vskip.5cm
	\subfigure[Perform regression of submodel with importance sketching covariates]{~\qquad\includegraphics[width=.65\linewidth]{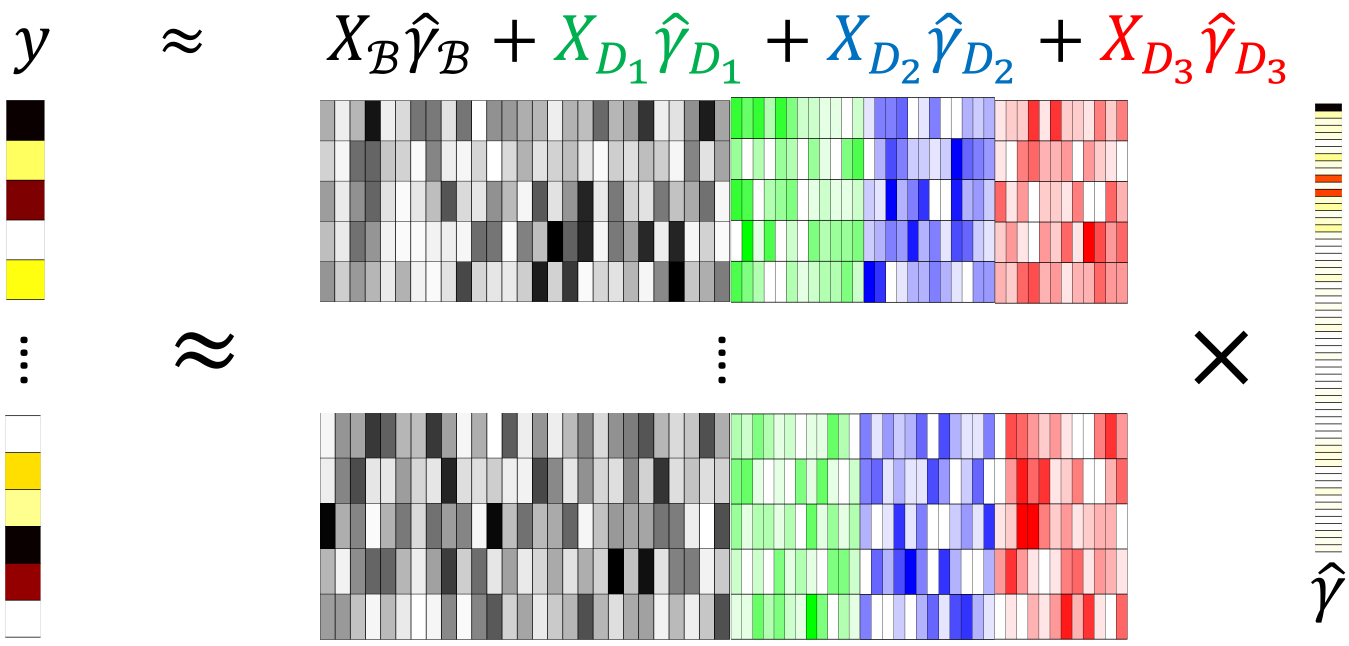}\qquad~}
	\caption{Illustration for Step 2 of ISLET}
	\label{fig:illu-2}
	\vskip1cm
	\includegraphics[width=.8\linewidth]{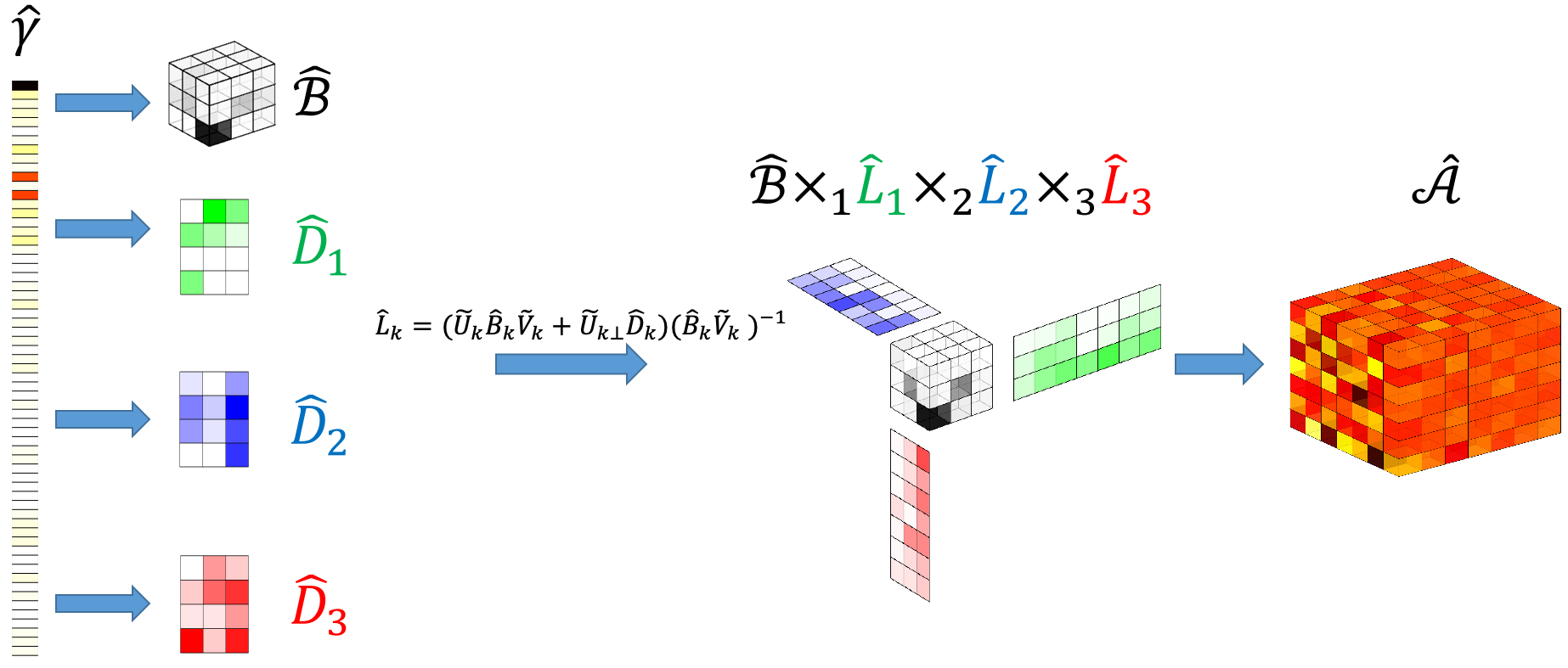}
	\caption{Illustration for Step 3 of ISLET}
	\label{fig:illu-3}
\end{figure}

\noindent{\bf Computation and implementation.} \quad  
We briefly discuss computational complexity and implementation aspects for the ISLET procedure here. It is noteworthy that ISLET accesses the sample only twice for constructing the covariance tensor (Step 1) and importance sketching covariates (Step 2), respectively. In large scale cases where it is difficult to store the whole dataset into random-access memory (RAM), this advantage can highly save the computational costs.
	
In addition, in the order-3 tensor case, when each mode shares the same dimension $p_k=p$ and rank $r_k=r$, the total number of observable values is $O(np^3)$ and the time complexity of ISLET is $O\left(np^3r + nr^6+Tp^4\right)$ where $T$ is the number of HOOI iterations. For general order-$d$ tensor regression, time complexity of ISLET is $O\left(np^dr + nr^{2d}+Tp^{d+1}\right)$. In contrast, the time complexity of the nonconvex PGD \cite{chen2016non} is $O\left(T'(np^d+rp^{d+1})\right)$, where $T'$ is the number of iterations of gradient descent; \cite{bousse2018linear} introduced an optimization based method with time complexity $O(T'dnp^d r)$ where $T'$ is the number of iterations in Gauss-Newton method. We can see if $T'\geq r$, a typical situation in practice, ISLET is significantly faster than these previous methods.

It is worth pointing out that the computing time of ISLET is still high when the tensor parameter has a large order $d$. In fact, without any structural assumption on the design tensors $\bcX_j$, such a time cost may be unavoidable since reading in all data requires $O(np^d)$ operations. If there is extra structure on the design tensor, e.g., Kronecker product \cite{ballani2013projection, hofreither2018black, hughes2005isogeometric,lynch1964tensor} and low separation rank \cite{beylkin2005algorithms, georgieva2019greedy}, the computing time can be significantly reduced by applying methods in this body of literature. Here, we mainly focus on the setting where $\bcX_j$ does not satisfy a clear structural assumption since in many real data applications, e.g., the neuroimaging data example studied in this and many other works \cite{allen2012regularized,li2013tucker,sun2017store,zhou2013tensor}, the design tensors $\bcX_j$ may not have a clear known structure. 
	
Moreover, in the order-3 tensor case, instead of storing all $\{ \bcX_j \}_{j=1}^n$ in the memory which requires $O(np^3)$ RAM, ISLET only requires $O(p^3 + n(pr+r^3))$ RAM space if one chooses to access the samples from hard disks but not to store to RAM. This makes large-scale computing possible. We empirically investigate the computation cost by simulation studies in Section \ref{sec:numerical}.

The proposed ISLET procedure also allows convenient parallel computing. Suppose we distribute all $n$ samples across $B$ machines: $\{(\bcX_{bi}, y_{bi})\}_{i=1}^{B_b}$, $b=1,\ldots, B$, where $B_b\approx n/B$. To evaluate the covariance tensor in Step 1, we can calculate $\widetilde{\bcA}_b = \sum_{i=1}^{B_i} y_{bi} \bcX_{bi}$ in each machine, then summarize them as $\widetilde{\bcA} = \frac{1}{n}\sum_{b=1}^B\widetilde{\bcA}_b$; to construct sketching covariates and perform partial regression in Step 2, we calculate
	\begin{equation}\label{eq:parallel-1}
	\y_b = (y_{b1},\ldots, y_{bB_b})^\top \in \mathbb{R}^{B_b},
	\end{equation}
	\begin{equation}\label{eq:parallel-2}
	\begin{split}
	& \widetilde{\X}_{bi} = \left[\widetilde{\X}_{\bcB, bi} ~~ \widetilde{\X}_{\D_1, bi} ~~ \widetilde{\X}_{\D_2, bi} ~~ \widetilde{\X}_{\D_3, bi}\right] \in \mathbb{R}^{m},\\
	& \widetilde{\X}_{\bcB, bi} =
	\rmvec\left(\bcX_{bi} \times_1 \widetilde{\U}_1^\top \times_2 \widetilde{\U}_2^\top \times_3 \widetilde{\U}_3^\top\right),\\
	& \widetilde{\X}_{\D_k,bi} = \rmvec\left(\widetilde{\U}_{k\perp}^\top \mathcal{M}_k\left(\bcX_{bi} \times_{k+1} \widetilde{\U}_{k+1}^\top \times_{k+2} \widetilde{\U}_{k+2}^\top\right) \widetilde{\V}_k\right),
	\end{split}
	\end{equation}
	\begin{equation}\label{eq:parallel-3}
	\widetilde{\G}_b = \sum_{i=1}^{B_b} \widetilde{\X}_{bi}^\top\widetilde{\X}_{bi} ,\quad \widetilde{\z}_b = \sum_{i=1}^{B_b}\widetilde{\X}_{bi}^\top y_{bi}
	\end{equation}
	in each machine. Then we combine the outcomes to
	$$\widehat{\bgamma} = \left(\sum_{b=1}^B \widetilde{\G}_b\right)^{-1} \left(\sum_{b=1}^B \widetilde{\z}_b\right).$$
	The computational complexity can be reduced to $O\left(\frac{np^3r+nr^6}{B} + Tp^4\right)$ via the parallel scheme. In the large-scale simulation we present in this article, we implement this parallel scheme for speed-up.

To implement the proposed procedure, the inputs of Tucker rank are required as tuning parameters. When they are unknown in practice, we can perform cross-validation or an adaptive rank selection scheme. A more detailed description and numerical results are postponed to Section \ref{sec:tuning} in the supplementary materials \cite{ISLET-supplement}.

\subsection{Sparse Low-rank Tensor Recovery}\label{sec:sparse-procedure}

When the target tensor $\bcA$ is simultaneously low-rank and sparse, in the sense that \eqref{eq:tucker-decomposition-sparse} holds for a subset $J_s\subseteq \{1,2,3\}$ known a priori, we introduce the following sparse ISLET procedure. The pseudocode for sparse ISLET is summarized in Algorithm \ref{al:procedure_sparse}.
\begin{enumerate}[leftmargin=*]
	\item[Step 1]  (Probing sketching directions) When $\mathbb{E} \rmvec(\bcX)\rmvec(\bcX)^\top = \I_{p_1p_2p_3}$, we still evaluate the covariance tensor $\widetilde{\bcA}$ as Equation \eqref{eq:tilde-A}.	
	Noting that $\bcA = \llbracket \bcS; \U_1, \U_2, \U_3\rrbracket$ and $\{\U_k\}_{k\in J_s}$ are row-wise sparse, we apply the sparse tensor alternating thresholding SVD (STAT-SVD) \cite{zhang2017optimal-statsvd} on $\widetilde{\bcA}$ to obtain $\widetilde{\U}_k \in \mathbb{O}_{p_k, r_k}, k=1,2,3$ as initial estimates for $\U_k$. Here, STAT-SVD is a sparse tensor decomposition method proposed by \cite{zhang2017optimal-statsvd} with central ideas of the double projection \& thresholding scheme and power iteration. Via STAT-SVD, we obtain the following sparse and low-rank approximation of $\bcA$,
	$$\bcA \approx \llbracket\widetilde{\bcS}; \widetilde{\U}_1, \widetilde{\U}_2, \widetilde{\U}_3\rrbracket, \quad \widetilde{\U}_k \in \mathbb{O}_{p_k, r_k}, \quad \widetilde{\bcS} = \llbracket \widetilde{\bcA}; \widetilde{\U}_1^\top, \widetilde{\U}_2^\top, \widetilde{\U}_3^\top \rrbracket \in \mathbb{R}^{r_1\times r_2\times r_3}.$$
	We further evaluate
	\begin{equation*}
	\begin{split}
	\widetilde{\V}_k = {\rm QR}\left(\mathcal{M}_k^\top(\widetilde{\bcS})\right) \in \mathbb{O}_{r_{k+1}r_{k+2}, r_k}.
	\end{split}
	\end{equation*}
	\item[Step 2] (Group Lasso on sketched covariates) We perform sketching and construct the following importance sketching covariates based on $\{\widetilde{\U}_k, \widetilde{\V}_k\}_{k=1}^3$, 
	\begin{equation}\label{eq:tilde-X_B-tilde-X_E}
	\begin{split}
	& \widetilde{\X}_\bcB \in \mathbb{R}^{n\times (r_1r_2r_3)},\quad (\widetilde{\X}_\bcB)_{[i,:]} = 
	\rmvec\left(\bcX_i \times_1 \widetilde{\U}_1^\top \times_2 \widetilde{\U}_2^\top \times_3 \widetilde{\U}_3^\top\right),\\
	& \widetilde{\X}_{\E_k} \in \mathbb{R}^{n\times p_kr_k}, \quad (\widetilde{\X}_{\E_k})_{[i,:]} = \rmvec\left(\mathcal{M}_k\left(\bcX_i\times_{k+1} \widetilde{\U}_{k+1}^\top\times_{k+2}\widetilde{\U}_{k+2}^\top\right)\widetilde{\V}_k\right).
	\end{split}
	\end{equation}
	Then we perform regression on sub-models with these reduced-dimensional covariates $\widetilde\X_{\bcB}$ and $\widetilde\X_{\E_k}$ respectively using least squares and group Lasso \cite{friedman2010note,yuan2006model},
	\begin{equation}\label{eq:regular-least-square}
	\widehat{\bcB}\in \mathbb{R}^{r_1\times r_2\times r_3}, \quad \rmvec(\widehat{\bcB}) = \argmin_{\bgamma\in\mathbb{R}^{r_1r_2r_3}} \|y - \widetilde{\X}_\bcB \bgamma\|_2^2,\\
	\end{equation}
		\begin{equation}\label{eq:E_k-formula}
		\widehat{\E}_k \in \mathbb{R}^{p_k\times r_k}, \rmvec(\widehat{\E}_k) = \left\{\begin{array}{ll}
		\argmin_{\bgamma} \|y - \widetilde{\X}_{\E_k}\bgamma\|_2^2, & \text{if } k \notin J_s;\\
		\argmin_{\bgamma} \|y - \widetilde{\X}_{\E_k}\bgamma\|_2^2 + \eta_k \sum_{j=1}^{p_k}\|\bgamma_{G_j^{k}}\|_2, & \text{if } k \in J_s.
		\end{array}\right.
	\end{equation}
	Here, $\{\eta_k\}_{k \in J_s}$ are the penalization level and
	\begin{equation}\label{eq:partition-G}
	G_j^k = \left\{j, j+p_k, \ldots, j+p_k(r_k-1)\right\}, \quad j=1,\ldots, p_k
	\end{equation}
	form a partition of $\{1,\ldots, p_kr_k\}$ that is induced by the construction of $\widetilde\X_{\E_k}$ (details for why to use group lasso can be found in Section \ref{sec:oracle-sparse}). 
	\item[Step 3] (Constructing the final estimator) $\widehat\bcA$ can be constructed using the regression coefficients $\widehat{\bcB}$ and $\widehat\E_k$'s in the submodels \eqref{eq:regular-least-square} and \eqref{eq:E_k-formula},
	\begin{equation}\label{eq:hat_A-sparse}
	\widehat{\bcA} = \left\llbracket \widehat{\bcB},  (\widehat{\E}_1(\widetilde{\U}_1^\top \widehat{\E}_1)^{-1}), (\widehat{\E}_2(\widetilde{\U}_2^\top \widehat{\E}_2)^{-1}), (\widehat{\E}_3(\widetilde{\U}_3^\top \widehat{\E}_3)^{-1})\right\rrbracket.
	\end{equation}
	More interpretation of \eqref{eq:hat_A-sparse} can be found in Section~\ref{sec:oracle-sparse}.
\end{enumerate}
\begin{algorithm}[h]
	\caption{Importance Sketching Low-rank Estimation for Tensors (ISLET): Order-3 Case}
	\begin{algorithmic}[1]
		\State Input: sample $\{y_j, \bcX_j\}_{j=1}^{n}$, Tucker rank $\br = (r_1, r_2, r_3)$.
		\State Calculate  $\widetilde{\bcA} = \frac{1}{n}\sum_{j=1}^{n} y_j \bcX_j.$
		\State Apply HOOI on $\widetilde{\bcA}$ and obtain initial estimates $\widetilde{\U}_1, \widetilde{\U}_2, \widetilde{\U}_3$. 
		\State Let $\widetilde{\bcS} = \llbracket\widetilde{\bcA}; \widetilde{\U}_1^\top, \widetilde{\U}_2^\top, \widetilde{\U}_3^\top\rrbracket$. Evaluate the sketching direction,
		$$\widetilde{\V}_k = {\rm QR}\left[\mathcal{M}_k(\widetilde{\bcS})^\top\right], \quad  k=1,2,3.$$
		\State Construct $\widetilde{\X} = \left[\widetilde{\X}_\bcB ~ \widetilde{\X}_{\D_1} ~ \widetilde{\X}_{\D_2} ~ \widetilde{\X}_{\D_3}\right] \in \mathbb{R}^{n\times m}$, where
		\begin{equation*}
		\begin{split}
		& \widetilde{\X}_\bcB \in \mathbb{R}^{n\times m_{\bcB}}, (\widetilde{\X}_\bcB)_{[i,:]} = 
		\rmvec\left(\bcX_i \times_1 \widetilde{\U}_1^\top \times_2 \widetilde{\U}_2^\top \times_3 \widetilde{\U}_3^\top\right),\\
		& \widetilde{\X}_{\D_k} \in \mathbb{R}^{n\times m_{\D_k}}, (\widetilde{\X}_{\D_k})_{[i,:]} = \rmvec\left(\widetilde{\U}_{k\perp}^\top \mathcal{M}_k\left(\bcX_{i} \times_{k+1}\widetilde{\U}_{k+1}^\top\times_{k+2} \widetilde{\U}_{k+2}^\top\right) \widetilde{\V}_k\right),
		\end{split}
		\end{equation*}
		for $m_{\bcB} = r_1r_2r_3, m_{\D_k} = (p_k-r_k)r_k$,  and $k=1,2,3$. 
		\item Solve $\widehat{\bgamma} = \argmin_{\bgamma\in\mathbb{R}^m}\|y - \widetilde{\X}\bgamma\|_2^2$. 
		\State Partition $\widehat{\bgamma}$ and assign each part to $\widehat{\bcB}, \widehat{\D}_1, \widehat{\D}_2, \widehat{\D}_3$, respectively,
		\begin{equation*}
		\begin{split}
		& \rmvec(\widehat{\bcB}) := \widehat{\bgamma}_\bcB = \widehat{\bgamma}_{[1:m_{\bcB}]},\\
		& \rmvec(\widehat{\D}_k) := \widehat{\bgamma}_{\D_k} = \widehat{\bgamma}_{\left[\left(m_{\bcB}+\sum_{k'=1}^{k-1} m_{\D_{k'}}+1\right): \left(m_{\bcB}+\sum_{k'=1}^{k}m_{\D_{k'}}\right)\right]},\quad k=1,2,3.\\
		\end{split}
		\end{equation*}
		\State Let $\widehat{\B}_k = \mathcal{M}_k(\widehat{\bcB})$. Evaluate
		$$\widehat{\bcA} = \llbracket\widehat{\bcB}; \widehat{\L}_1, \widehat{\L}_2, \widehat{\L}_3\rrbracket, \quad \widehat{\L}_k = \left(\widetilde{\U}_k\widehat{\B}_k\widetilde{\V}_k + \widetilde{\U}_{k\perp}\widehat{\D}_k\right)\left(\widehat{\B}_k\widetilde{\V}_k\right)^{-1}, \quad k=1,2,3.$$
	\end{algorithmic}\label{al:procedure_regular}
\end{algorithm}

\begin{algorithm}[h]
	\caption{Sparse Importance Sketching Low-Rank Estimation for Tensors (Sparse ISLET): Order-3 Case}
	\begin{algorithmic}[1]
		\State Input: sample $\{y_j, \bcX_j\}_{j=1}^{n}$, Tucker rank $\br = (r_1, r_2, r_3)$, sparsity index $J_s \subseteq \{1,2,3\}$.
		\State Evaluate $\widetilde{\bcA} = \frac{1}{n}\sum_{j=1}^{n} y_j \bcX_j.$
		\State Apply STAT-SVD on $\widetilde{\bcA}$ with sparsity index $J_s$. Let the outcome be $\widetilde{\U}_1, \widetilde{\U}_2, \widetilde{\U}_3$. 
		\State Let $\widetilde{\bcS} = \llbracket\widetilde{\bcA}; \widetilde{\U}_1^\top, \widetilde{\U}_2^\top, \widetilde{\U}_3^\top\rrbracket$ and evaluate the probing direction,
		$$\widetilde{\V}_k = {\rm QR}\left[\mathcal{M}_k(\widetilde{\bcS})^\top\right], \quad k=1,2,3.$$
		\State Construct
		\begin{equation*}
		\begin{split}
		& \widetilde{\X}_\bcB \in \mathbb{R}^{n\times (r_1r_2r_3)},\quad (\widetilde{\X}_\bcB)_{[i,:]} = 
		\rmvec(\bcX_i \times_1 \widetilde{\U}_1^\top \times_2 \widetilde{\U}_2^\top \times_3 \widetilde{\U}_3^\top),\\
		& \widetilde{\X}_{\E_k} \in \mathbb{R}^{n\times (p_kr_k)}, \quad (\widetilde{\X}_{\E_k})_{[i,:]} = \rmvec\left(\mathcal{M}_k\left(\bcX_i\times_{k+1}\widetilde{\U}_{k+1}^\top\times_{k+2} \widetilde{\U}_{k+2}^\top\right)\widetilde{\V}_k\right).
		\end{split}
		\end{equation*}
		\State Solve
		\begin{equation*}
		\begin{split}
		\widehat{\bcB}\in \mathbb{R}^{r_1r_2r_3},& \quad \rmvec(\widehat{\bcB}) = \argmin_{\bgamma\in\mathbb{R}^{r_1r_2r_3}} \|y - \widetilde{\X}_\B \bgamma\|_2^2;
		\end{split}
		\end{equation*}
		\begin{equation*}
		\begin{split}
		\widehat{\E}_k \in \mathbb{R}^{p_k\times r_k}, \rmvec(\widehat{\E}_k) =\left\{\begin{array}{ll}
		\argmin_{\bgamma} \|y - \widetilde{\X}_{\E_k}\bgamma\|_2^2 + \lambda_k \sum_{j=1}^{p_k}\|\bgamma_{G_j^{k}}\|_2, & k \in J_s;\\
		\argmin_{\bgamma} \|y - \widetilde{\X}_{\E_k}\bgamma\|_2^2, & k\notin J_s.
		\end{array} \right.
		\end{split}
		\end{equation*}
		\State Evaluate
		\begin{equation*}
		\widehat{\bcA} = \left\llbracket\widehat{\bcB};  (\widehat{\E}_1(\widetilde{\U}_1^\top \widehat{\E}_1)^{-1}), (\widehat{\E}_2(\widetilde{\U}_2^\top \widehat{\E}_2)^{-1}), (\widehat{\E}_3(\widetilde{\U}_3^\top \widehat{\E}_3)^{-1})\right\rrbracket.
		\end{equation*}
	\end{algorithmic}\label{al:procedure_sparse}
\end{algorithm}

\subsection{A Sketching Perspective of ISLET}\label{sec:sketching-perspective}

While one of the main focuses of this article is on low-rank tensor regression, from a sketching perspective, ISLET can be seen as a special case of a more general algorithm that broadly applies to high-dimensional statistical problems with dimension-reduced structure. In fact the three steps of the ISLET procedure are completely general and are summarized informally here:
\begin{enumerate}[leftmargin=*]
    \item [Step 1] (Probing projection directions) For the tensor regression problem, we use the HOOI~\cite{de2000best} or STAT-SVD \cite{zhang2017optimal-statsvd} approach for finding the informative low-rank subspaces along which we project/sketch. More generally, if we let $\widetilde{\bcA}  = \frac{1}{n}\sum_{j=1}^{n} y_j \bcX_j$, where $\bcX_j$ has ambient dimension $p$, we can define a general projection operator (with a slight abuse of notation) $\mathcal{P}_m(.) : \mathbb{R}^p \rightarrow \mathbb{R}^m$ indexed by low dimension $m$ and let $\mathcal{S}(\widetilde{\bcA})$ be the $m$-dimensional subspace of $\mathbb{R}^p$ determined by performing $\mathcal{P}_m(\widetilde{\bcA})$. 
    \item [Step 2] (Estimation in subspaces) The second step involves first projecting the data $\bcX$ on to the subspace $\mathcal{S}(\widetilde{\bcA})$, specifically $\widetilde{\X} = \mathcal{P}_{\mathcal{S}(\widetilde{\bcA})}(\bcX) \in \mathbb{R}^{n \times m}$. Then we perform regression or other procedures of choice using the sketched data $\widetilde{\X}$ onto determine the dimension-reduced parameter $\widehat{\bgamma} \in \mathbb{R}^m$.
    \item [Step 3] (Embedding to high-dimensional space) Finally, we need to project the estimator back to the high-dimensional space $\mathbb{R}^{p}$ by applying an equivalent to the inverse of the projection operator $\mathcal{P}^{-1}_{\mathcal{S}(\widetilde{\bcA})} : \mathbb{R}^m \rightarrow \mathbb{R}^p$. For low-rank tensor regression we require the formula \eqref{eq:hat_A_non-sparse}. 
\end{enumerate}
The description above illustrates that the idea of ISLET is applicable to more general high-dimensional problems with dimension-reduced structure. In fact, the well-regarded \emph{sure independence screening} in high-dimensional sparse linear regression~\cite{fan2008sure,xue2011sure} can be seen as a special case of this idea. To be specific, consider the high-dimensional linear regression model,
$$y_i = X_{[i,:]} \bbeta + \varepsilon_i,\quad i=1,\ldots, n,$$
where $\bbeta$ is the $m$-sparse vector of interests and $y_i\in \mathbb{R}$ and $X_{[i,:]}^\top\in \mathbb{R}^{p}$ are the observable response and covariate. Then the $m$-dimensional subspace $\mathcal{S}(\widetilde{\bbeta})$ in Step 1 can be the coordinates corresponding to the $m$ largest entries of $\widetilde{\bbeta} = \sum_{i=1}^n X_{[i,:]}^\top y_i$; Step 2 corresponds to the dimension reduced least squares in sure independence screening; the inverse operator in Step 3 is simply filling in $0$'s in the coordinates that do not correspond to $\mathcal{S}(\widetilde{\bbeta})$. In addition, this idea applies more broadly to problems such as matrix and tensor completion. One of the novel contributions of this article is finding suitable projection and inverse operators for low-rank tensors. 

We can also contrast this approach with prior approaches that involve randomized sketching~\cite{dobriban2018new,pilanci2015randomized,raskutti2014statistical}. These prior approaches showed that the randomized sketching may lose data substantially, increase the variance, and yield suboptimal result for many statistical problems. There are two key differences with how we exploit sketching in our context: (1) we sketch along the parameter directions of $\bcX$, reducing the data from $\mathbb{R}^{n \times p}$ to $\mathbb{R}^{n \times m}$; whereas approaches in~\cite{dobriban2018new,pilanci2015randomized,raskutti2014statistical} sketch along the sample directions, reducing the data from $\mathbb{R}^{n \times p}$ to $\mathbb{R}^{m \times p}$, which reduces the effective sample size from $n$ to $m$; (2) second and most importantly rather than using the randomized sketching that is \emph{unsupervised} without the response $y$, our importance sketching is \emph{supervised}, that is, obtained using both the response $y$ and covariates $\bcX$. Then we sketch along the subspace $\mathcal{S}(\widetilde{\bcA})$ which contains information on the low-dimensional structure of the parameter $\mathcal{A}$. This is why our general procedure has both desirable statistical and computational properties.

\section{Oracle Inequalities}\label{sec:oracle-inequality}

In this section, we provide general oracle inequalities without focusing on specific design, which provides a general guideline for the theoretical analyses of our ISLET procedure. We first introduce a quantification of the errors in sketching directions obtained in the first step of ISLET. Let $\V_k\in \mathbb{O}_{r_{k+1}r_{k+2}, r_k}$ be the right singular subspace of $\mathcal{M}_k(\bcS)$, where $\bcS$ is the core tensor in the Tucker decomposition of $\bcA$: $\bcA = \llbracket \bcS; \U_1, \U_2, \U_3\rrbracket$. By Lemma \ref{lm:Kronecker-vectorization-matricization} in the supplementary materials \cite{ISLET-supplement},
\begin{equation}\label{eq:def_W}
\begin{split}
\W_1 := (\U_{3}\otimes &\U_{2})\V_1 \in \mathbb{O}_{p_{2}p_{3}, r_1}, ~~ \W_2 := (\U_{3}\otimes \U_{1})\V_2 \in \mathbb{O}_{p_{1}p_{3}, r_2},\\
& \text{and}\quad \W_3 := (\U_{2}\otimes \U_{1})\V_3 \in \mathbb{O}_{p_{1}p_{2}, r_3}
\end{split}
\end{equation}
are the right singular subspaces of $\mathcal{M}_1(\bcA), \mathcal{M}_2(\bcA)$, and $\mathcal{M}_3(\bcA)$, respectively. Recall that we initially estimate $\U_k$ and $\V_k$ by $\widetilde{\U}_k$ and $\widetilde{\V}_k$, respectively in Step 1 of ISLET. Define
$$\widetilde{\W}_1 = (\widetilde{\U}_{3}\otimes \widetilde{\U}_{2}) \widetilde{\V}_1,\quad \widetilde{\W}_2 = (\widetilde{\U}_{3}\otimes \widetilde{\U}_{1}) \widetilde{\V}_2, \quad \text{and}\quad \widetilde{\W}_3 = (\widetilde{\U}_{2}\otimes \widetilde{\U}_{1}) \widetilde{\V}_3$$
in parallel to \eqref{eq:def_W}. Intuitively speaking, $\{\widetilde{\U}_k, \widetilde{\W}_k\}_{k=1}^3$ can be seen as the initial sample approximations for $\{\U_k, \W_k\}_{k=1}^3$. Therefore, we quantify the \emph{sketching direction error} by
\begin{equation}\label{eq:theta-sin-theta-distance}
\begin{split}
\theta := \max_{k=1,2,3}\left\{\|\sin\Theta(\widetilde{\U}_k, \U_k)\|, \|\sin\Theta(\widetilde{\W}_k, \W_k)\|\right\}.
\end{split}
\end{equation}
Next, we provide the oracle inequality via $\theta$ for ISLET under regular and sparse settings, respectively in the next two subsections. 

\subsection{Regular Tensor Regression and Oracle Inequality }\label{sec:oracle-regular}

In order to study the theoretical properties of the proposed procedure, we need to introduce another representation of the original model~\eqref{eq:model}. Decompose the vectorized parameter $\bcA$ as follows,
\begin{equation}\label{eq:decomposition-A}
\begin{split}
\rmvec(\bcA) = & P_{\widetilde{\U}}\rmvec(\bcA) + P_{\widetilde{\U}_\perp}\rmvec(\bcA)\\
= & P_{\widetilde{\U}_3\otimes\widetilde{\U}_2\otimes \widetilde{\U}_1} \rmvec(\bcA) + P_{\mathcal{R}_1(\widetilde{\W}_1\otimes \widetilde{\U}_{1\perp})} \rmvec(\bcA) + P_{\mathcal{R}_2(\widetilde{\W}_2\otimes \widetilde{\U}_{2\perp})} \rmvec(\bcA)\\
& + P_{\mathcal{R}_3(\widetilde{\W}_3\otimes \widetilde{\U}_{3\perp})} \rmvec(\bcA) + P_{\widetilde{\U}_\perp}\rmvec(\bcA)\\
= & (\widetilde{\U}_3\otimes\widetilde{\U}_2\otimes \widetilde{\U}_1) \rmvec(\widetilde{\bcB}) + \mathcal{R}_1(\widetilde{\W}_1\otimes \widetilde{\U}_{1\perp}) \rmvec(\widetilde\D_1) + \mathcal{R}_2(\widetilde{\W}_2\otimes \widetilde{\U}_{2\perp}) \rmvec(\widetilde\D_2) \\
& + \mathcal{R}_3(\widetilde{\W}_3\otimes \widetilde{\U}_{3\perp}) \rmvec(\widetilde\D_3) + P_{\widetilde{\U}_\perp} \rmvec(\bcA).
\end{split}
\end{equation}
(See the proof of Theorem \ref{th:upper_bound_general} for a detailed derivation of \eqref{eq:decomposition-A}). Here,
\begin{equation*}
\begin{split}
\qquad \widetilde{\U} = \left[\widetilde{\U}_3\otimes \widetilde{\U}_2\otimes \widetilde{\U}_1 ~~ \mathcal{R}_1(\widetilde{\W}_1\otimes \widetilde{\U}_{1\perp}) ~~ \mathcal{R}_2\left( \widetilde{\W}_2\otimes \widetilde{\U}_{2\perp}\right) ~~ \mathcal{R}_3\left( \widetilde{\W}_3\otimes \widetilde{\U}_{3\perp}\right)\right], 
\end{split}
\end{equation*}
\begin{equation*}
\begin{split}
\widetilde{\bcB} := & \left\llbracket \bcA; \widetilde{\U}_1^\top, \widetilde{\U}_2^\top, \widetilde{\U}_3^\top \right\rrbracket \in \mathbb{R}^{r_1r_2r_3} \quad \text{and}\quad \widetilde{\D}_k := \widetilde{\U}_{k\perp}^\top \mathcal{M}_k(\bcA) \widetilde{\W}_k \in \mathbb{R}^{(p_k-r_k)\times r_k}
\end{split}
\end{equation*}
are the singular subspace of the ``Cross structure" and the low-dimensional projections of $\bcA$ onto the ``body" and ``arms" formed by sketching directions $\{\widetilde{\U}_k,\widetilde{\V}_k\}_{k=1}^3$, respectively (See Figure \ref{fig:illustration_U} for an illustration of $\widetilde\U$, $\widetilde\bcB$, and $\widetilde\V_k$). Due to different alignments, the $i$th row of $\{\W_k\otimes \U_{k\perp}\}_{k=1}^3$ does not necessarily correspond to the $i$th entry of $\rmvec(\bcA)$ for all $1\leq i \leq p_1p_2p_3$. We thus permute the rows of $\{\widetilde\W_k\otimes \widetilde\U_{k\perp}\}_{k=1}^3$ to match each row of $\mathcal{R}_k(\widetilde\W_k\otimes \widetilde{\U}_{k\perp})$ to the corresponding entry in $\rmvec(\bcA)$. The formal definition of the rowwise permutation operator $\mathcal{R}_k$ is rather clunky and is postponed to Section \ref{sec:row-permutation-operator} in the supplementary materials. Intuitively speaking, $P_{\widetilde{\U}}\rmvec(\bcA)$ represents the projection of $\bcA$ onto to the Cross structure and $P_{\widetilde{\U}_\perp}\rmvec(\bcA)$ can be seen as a residual. If the estimates $\{\widetilde\U_k, \widetilde\W_k\}_{k=1}^3$ are close enough to $\{\U_k, \W_k\}_{k=1}^3$, i.e., $\theta$ defined in \eqref{eq:theta-sin-theta-distance} is small, we expect that the residual $P_{\widetilde\U_{\perp}}\rmvec(\bcA)$ has small amplitude. 

\begin{figure}
	\centering
	\includegraphics[height=4cm,width=.6\linewidth]{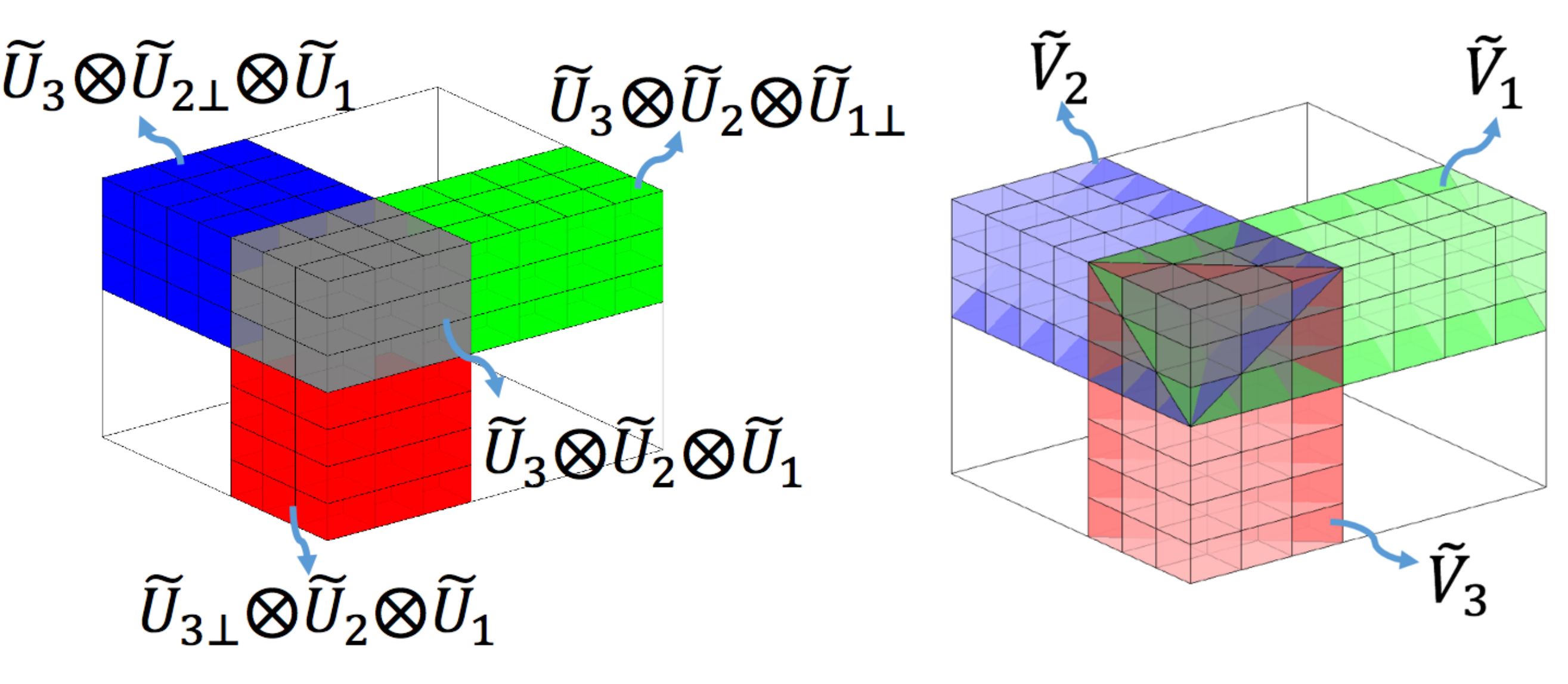}
	\caption{Illustration of decomposition \eqref{eq:decomposition-A}. Here we assume $\widetilde{\U}_k^\top = [\I_{r_k} ~ \boldsymbol{0}_{r_k \times (p_k-r_k)}]$, $k=1,2,3$, for a better visualization. The gray, green, blue, and red cubes represent the subspaces of $\widetilde{\U}_3\otimes \widetilde{\U}_2\otimes\widetilde{\U}_1$, $\widetilde{\U}_3\otimes \widetilde{\U}_2\otimes \widetilde{\U}_{1\perp}$, $\widetilde{\U}_3\otimes\widetilde{\U}_{2\perp}\otimes\widetilde{\U}_1, \widetilde{\U}_{3\perp}\otimes\widetilde{\U}_2\otimes \widetilde{\U}_1$. The gray cube also corresponds to the projected parameters $\widetilde{\bcB}$; matricizations of green, blue and red cubes correspond to the projected parameters $\widetilde\U_{1\perp}^\top \mathcal{M}_1(\bcA)(\widetilde\U_3\otimes \widetilde\U_2)$, $\widetilde\U_{2\perp}^\top \mathcal{M}_2(\bcA)(\widetilde\U_3\otimes \widetilde\U_1)$, and $\widetilde\U_{3\perp}^\top \mathcal{M}_3(\bcA)(\widetilde\U_2\otimes \widetilde\U_1)$, respectively. The three plains in the right panel correspond to the subspace of $\widetilde\V_1$, $\widetilde\V_2$, and $\widetilde\V_3$, respectively.}\label{fig:illustration_U}
\end{figure}

Based on \eqref{eq:decomposition-A}, we can rewrite the original regression model \eqref{eq:model} into the following partial regression model:
\begin{equation}\label{eq:low-dimensional-regression}
\begin{split}
y_j = & (\widetilde\X_{\bcB})_{[j,:]} \rmvec(\widetilde\bcB) + \sum_{k=1}^3 (\widetilde\X_{\D_k})_{[j,:]}\rmvec(\widetilde\D_k) + \rmvec(\bcX_j)^\top P_{\widetilde\U_\perp}\rmvec(\bcA) + \varepsilon_j\\
= & \widetilde\X_{[j,:]} \widetilde{\bgamma} + \widetilde{\varepsilon}_j, \quad j=1,\ldots, n.
\end{split}
\end{equation}
(See the proof of Theorem \ref{th:upper_bound_general} for a detailed derivation of \eqref{eq:low-dimensional-regression}.) Here,
\begin{itemize}[leftmargin=*]
	\item $\widetilde{\varepsilon}_j = \rmvec(\bcX_j)^\top P_{\widetilde{\U}_\perp}\rmvec(\bcA) + \varepsilon_j$ is the oracle noise; $\widetilde{\bvarepsilon} = (\widetilde{\varepsilon}_1,\ldots, \widetilde{\varepsilon}_{n})^\top$;
	\item $\widetilde\X_{\bcB}, \widetilde\X_{\D_k}$ are sketching covariates introduced in Equation \eqref{eq:importance-sketching-covariates};
	\item $\widetilde{\bgamma} = \left[\rmvec(\widetilde\bcB)^\top, \rmvec(\widetilde\D_1)^\top, \rmvec(\widetilde\D_2)^\top, \rmvec(\widetilde\D_3)^\top\right]^\top = \widetilde\U^\top \rmvec(\bcA) \in \mathbb{R}^m$ is the dimension-reduced parameter.
\end{itemize}
\eqref{eq:low-dimensional-regression} reveals the essence of the least squares estimator \eqref{eq:partial-regression} in the ISLET procedure -- the outcomes of \eqref{eq:partial-regression} and \eqref{eq:def_hat_Beta_hat_D}, i.e., $\widehat\bcB$ and $\widehat\D_k$, are sample-based estimates of $\widetilde\bcB$ and $\widetilde\D_k$. Finally, based on the detailed algebraic calculation in Step 3 and the proof of Theorem \ref{th:upper_bound_general},
\begin{equation}\label{eq:bcA-identity}
\begin{split}
& \bcA = \left\llbracket \widetilde{\bcB}; \widetilde\L_1, \widetilde\L_2, \widetilde\L_3\right\rrbracket, \quad  \widetilde\L_k = \left(\widetilde\U_k\widetilde\B_k \widetilde\V_k + \widetilde\U_{k\perp}\widetilde\D_k\right)\left(\widetilde\B_k\widetilde\V_k\right)^{-1}.
\end{split}
\end{equation}
\eqref{eq:bcA-identity} is essentially a higher-order version of the Schur complement formula (also see \cite{cai2016structured}). Finally, we apply the plug-in estimator to obtain the final estimator $\widehat\bcA$ (Equation \eqref{eq:hat_A_non-sparse} in Step 3 of the ISLET procedure).

Based on previous discussions, it can be seen that the estimation error of the original tensor regression is driven by the error of the least squares estimator $\widehat\bgamma$, i.e., $\|(\widetilde\X^\top\widetilde\X)^{-1}\widetilde\X^\top\widetilde\bvarepsilon\|_2^2$. We have the following oracle inequality for the proposed ISLET procedure.
\begin{Theorem}[Oracle Inequality of Regular Tensor Estimation: Order-3 Case]\label{th:upper_bound_general}
	Suppose $\bcA \in \mathbb{R}^{p_1\times p_2 \times p_3}$ has Tucker rank-$(r_1, r_2, r_3)$ tensor and $\widehat{\bcA}$ is the outcome of Algorithm \ref{al:procedure_regular}. Assume the sketching directions $\{\widetilde{\U}_k, \widetilde{\V}_k\}_{k=1}^3$ satisfy $\theta<1/2$ (see \eqref{eq:theta-sin-theta-distance} for the definition of $\theta$) and $\left\|\widehat{\D}_k(\widehat{\B}_k\widetilde{\V}_k)^{-1}\right\| \leq \rho$. We don't impose other specific assumptions on $\bcX_i$ and $\varepsilon_i$. Then, we have
	\begin{equation*}
	\left\|\widehat{\bcA} - \bcA \right\|_{\tHS}^2 \leq (1 + C(\theta+\rho))\left\|(\widetilde{\X}^\top\widetilde{\X})^{-1}\widetilde{\X}^\top\widetilde{\bvarepsilon}\right\|_2^2
	\end{equation*}
	for uniform constant $C>0$ that does not rely on any other parameters.
\end{Theorem}
\begin{proof}
See Appendix \ref{sec:proof_upper_bound_general} for a complete proof. In particular, the proof contains three major steps. After introducing a number of notations, we first transform the original regression model to the partial regression model \eqref{eq:low-dimensional-regression} and then rewrite the upper bound $\|(\widetilde\X^\top\widetilde\X)^{-1}\widetilde\X^\top\widetilde{\bvarepsilon}\|_2^2$ to $\|\widehat{\bcB} - \widetilde\bcB\|_{\tHS}^2 + \sum_{k=1}^3 \|\widehat{\D}_k-\widetilde\D_k\|_F^2$. Next, we introduce a factorization of $\bcA$ in parallel with the one of $\widehat\bcA$, based on which the loss $\|\widehat\bcA - \bcA\|_{\tHS}$ is decomposed into eight terms. Finally, we introduce a novel deterministic error bound for the ``Cross scheme" (Lemma \ref{lm:FGH} in the supplementary materials \cite{ISLET-supplement}; also see \cite{zhang2019cross}), carefully analyze each term in the decomposition of $\|\widehat\bcA - \bcA\|_{\tHS}$, and finalize the proof.
\end{proof} 
Theorem \ref{th:upper_bound_general} shows that once the sketching directions $\widetilde{\U}$ and $\widetilde{\V}$ are reasonably accurate, the estimation error for $\widehat{\bcA}$ will be close to the error of partial linear regression in Equation \eqref{eq:low-dimensional-regression}. This bound is general and deterministic, which can be used as a key step in more specific settings of low-rank tensor regression.

\subsection{Sparse Tensor Regression and Oracle Inequality}\label{sec:oracle-sparse}

Next, we study the oracle performance of the proposed procedure for sparse tensor regression, where $\bcA$ further satisfies the sparsity constraint \eqref{eq:tucker-decomposition-sparse}. As in the previous section, we decompose the vectorized parameter as
\begin{equation}\label{eq:decomposition-A-sparse1}
\begin{split}
\rmvec(\bcA) = & P_{\widetilde\U_3\otimes \widetilde\U_2\otimes \widetilde\U_1}\rmvec(\bcA) + P_{(\widetilde\U_3\otimes \widetilde\U_2\otimes \widetilde\U_1)_\perp}\rmvec(\bcA)\\
= & (\widetilde\U_3\otimes \widetilde\U_2 \otimes\widetilde\U_3) \rmvec(\widetilde\bcB) + P_{(\widetilde\U_3\otimes \widetilde\U_2\otimes \widetilde\U_1)_\perp}\rmvec(\bcA);
\end{split}
\end{equation}
\begin{equation}\label{eq:decomposition-A-sparse2}
\begin{split}
\rmvec(\bcA) = & P_{\mathcal{R}_k(\widetilde\W_k\otimes \I_{p_k})}\rmvec(\bcA) + P_{\mathcal{R}_k(\widetilde\W_k\otimes \I_{p_k})_\perp}\rmvec(\bcA)\\
= & \mathcal{R}_k(\widetilde\W_k\otimes \I_{p_k}) \rmvec(\widetilde\E_k) + P_{\mathcal{R}_k(\widetilde\W_k\otimes \I_{p_k})_\perp}\rmvec(\bcA), \quad k=1,2,3.
\end{split}
\end{equation}
Here,
\begin{equation}\label{eq:tilde-B-tilde-E}
\begin{split}
& \widetilde{\bcB} := \llbracket \bcA; \widetilde{\U}_1^\top, \widetilde{\U}_2^\top, \widetilde{\U}_3^\top\rrbracket \in \mathbb{R}^{r_1r_2r_3};\\
& \widetilde{\E}_k := \mathcal{M}_k\left(\bcA\times_{(k+1)}\widetilde{\U}_{k+1}^\top\times_{(k+2)}\widetilde{\U}_{k+2}^\top\right)\widetilde{\V}_{k}\in \mathbb{R}^{p_k\times r_k},\quad k=1,2,3,
\end{split}
\end{equation}
are the low-dimensional projections of $\bcA$ onto the importance sketching directions. Since $\{\U_k, \W_k\}$ are the left and right singular subspaces of $\mathcal{M}_k(\bcA)$, we can demonstrate that $P_{(\U_3\otimes \U_2\otimes \U_1)_\perp}\rmvec(\bcA)$ and $P_{\mathcal{R}_k(\W_k\otimes \I_{p_k})_\perp}\rmvec(\bcA)$ are zeros. Thus if the estimates $\{\widetilde{\U}_k, \widetilde{\W}_k\}_{k=1}^3$ are sufficiently accurate, i.e., $\theta$ defined in Eq.~\eqref{eq:theta-sin-theta-distance} is small, we can expect that the residuals $P_{(\widetilde\U_3\otimes \widetilde\U_2\otimes \widetilde\U_1)_\perp}\rmvec(\bcA)$ and $P_{\mathcal{R}_k(\widetilde\W_k\otimes \I_{p_k})_\perp}\rmvec(\bcA)$ have small amplitudes. Then, based on a more detailed calculation in the proof of Theorem \ref{th:upper_bound_sparse_general}, the model of sparse and low-rank tensor regression $y_j = \langle \bcX_j, \bcA\rangle + \varepsilon_j$ can be rewritten as the following partial linear regression,
\begin{equation}\label{eq:partial-linear-model-B}
y_j = (\widetilde\X_\bcB)_{[j,:]}\rmvec(\widetilde\B) + (\widetilde\bvarepsilon_\bcB)_j,
\end{equation}
\begin{equation}\label{eq:partial-linear-model-Ek}
y_j = (\widetilde\X_{\E_k})_{[j,:]}\rmvec(\widetilde\E_k) + (\widetilde\bvarepsilon_{\E_k})_j, \quad k=1,2,3.
\end{equation}
Here, $\widetilde\X_{\bcB}$ and $\widetilde\X_{\E_k}$ are the covariates defined in Equation \eqref{eq:tilde-X_B-tilde-X_E} and\\ $\widetilde{\bvarepsilon}_\bcB = ((\widetilde{\varepsilon}_\bcB)_1,\ldots, (\widetilde{\varepsilon}_\bcB)_{n})^\top$, $\widetilde{\bvarepsilon}_{\E_k} = ((\widetilde{\varepsilon}_{\E_k})_1,\ldots, (\widetilde{\varepsilon}_{\E_k})_{n})^\top$ are oracle noises defined as
	\begin{equation}\label{eq:epsilon_B-epsilon_E}
	\begin{split}
	& (\widetilde{\bvarepsilon}_\bcB)_j = \left\langle \rmvec(\bcX_j), P_{(\widetilde{\U}_3\otimes \widetilde{\U}_2\otimes \widetilde{\U}_1)_{\perp}}\rmvec(\bcA)\right\rangle + \varepsilon_j\\
	\text{and}\quad & (\widetilde{\bvarepsilon}_{\E_k})_j = \left\langle \rmvec(\bcX_{j}), P_{\left(\mathcal{R}_k(\widetilde{\W}_k\otimes\I_{p_k})\right)_\perp}\rmvec(\bcA)\right\rangle + \varepsilon_j.
	\end{split}
	\end{equation} 
Therefore, the Step 2 of sparse ISLET can be interpreted as the estimation of $\widetilde\bcB$ and $\widetilde\E_k$. 

We apply regular least squares to estimate $\widetilde\bcB$ and $\widetilde\E_k$ for $k\notin J_s$. For any sparse mode $k\in J_s$, $\widetilde{\E}_k$ are group sparse due to the definition \eqref{eq:tilde-B-tilde-E} and the assumption that $\U_k$ are row-wise sparse. Specifically, $\widetilde\E_k$ satisfies
\begin{equation}\label{eq:tilde-E-group-sparse}
\begin{split}
\left\|\rmvec(\widetilde{\E}_k)\right\|_{0, 2} := \sum_{i=1}^{p_k} 1_{\left\{\left(\rmvec(\widetilde\E_k)\right)_{G_i^k} \neq 0\right\}}\leq s_k,
\end{split}
\end{equation} 
where 
\begin{equation*}
G_i^k = \left\{i, i+p_k, \ldots, i+p_k(r_k-1)\right\}, \quad i=1,\ldots, p_k, \quad \forall k\in J_s,
\end{equation*}
is a partition of $\{1,\ldots, p_kr_k\}$ (see the proof for Theorem \ref{th:upper_bound_sparse_general} for a more detailed argument for \eqref{eq:tilde-E-group-sparse}). 
By detailed calculations in Step 3 of the proof for Theorem \ref{th:upper_bound_general}, one can verify that
$$\bcA = \left\llbracket \widetilde{\bcB},  (\widetilde{\E}_1(\widetilde{\U}_1^\top \widetilde{\E}_1)^{-1}), (\widetilde{\E}_2(\widetilde{\U}_2^\top \widetilde{\E}_2)^{-1}), (\widetilde{\E}_3(\widetilde{\U}_3^\top \widetilde{\E}_3)^{-1})\right\rrbracket.$$ Then the finally sparse ISLET estimator $\widehat\bcA$ in \eqref{eq:hat_A-sparse} can be seen as  the plug-in estimator.

To ensure that the group Lasso estimator in \eqref{eq:E_k-formula} 
provides a stable estimation for the proposed procedure, we introduce the following group restricted isometry condition, which can also be seen as an extension of restricted isometry property (RIP), a commonly used condition in compressed sensing and high-dimensional linear regression literature \cite{candes2005decoding}.
\begin{Condition}\label{con:group-RIP}
	We say a matrix $\X\in \mathbb{R}^{n\times p}$ satisfies the group restricted isometry property (GRIP) with respect to partition $G_1,\ldots, G_m \subseteq \{1,\ldots, p\}$, if there exists $\delta>0$ such that
	\begin{equation}
	n(1-\delta)\|\v\|_2^2 \leq \|\X \v\|_2^2 \leq n(1+\delta)\|\v\|_2^2
	\end{equation}
	for all groupwise sparse vector $v$ satisfying $\sum_{k=1}^m 1_{\{\v_{G_k} \neq 0\}}\leq s$.
\end{Condition}
We still use $\theta$ defined in Eq.~\eqref{eq:theta-sin-theta-distance} to characterize the sketching direction errors. The following oracle inequality holds for sparse tensor regression with importance sketching. 

\begin{Theorem}[Oracle Inequality for Sparse Tensor Regression: Order-3 Case]\label{th:upper_bound_sparse_general}
	Consider the sparse low-rank tensor regression \eqref{eq:model} \eqref{eq:tucker-decomposition-sparse}. Suppose $\theta<1/2$, the importance sketching covariates $\widetilde{\X}_\bcB$ and $\widetilde{\X}_{\E_k}$ ($k\notin J_s$) are nonsingular. For any $k\in J_s$, $\widetilde{\X}_{\E_k}$ satisfies group restricted isometry property (Condition \ref{con:group-RIP}) with respect to partition $G_1^k,\ldots, G_{p_k}^k$ in \eqref{eq:partition-G} and $\delta < 1/3$. We apply the proposed Algorithm \ref{al:procedure_sparse} with group Lasso penalty
	$$\eta_k = C_1\max_{i=1,\ldots, p_k}\left\|(\widetilde{\X}_{\E_k, [:, G_i^k]})^\top \widetilde{\bvarepsilon}_{\E_k}\right\|_2$$
	for $k\in J_s$ and some constant $C_1 \geq 3$. We also assume $\|\widetilde{\U}_{k\perp}^\top\widehat{\E}_k(\widetilde{\U}_k^\top \widehat{\E}_k)^{-1}\| \leq \rho$. Then,
	\begin{equation}\label{ineq:oracle-error-bound}
	\begin{split}
	& \left\|\widehat{\bcA}-\bcA\right\|_{\tHS}^2 \leq (1 +C_2s(\theta+ \rho)) \Bigg( \left\|(\widetilde{\X}_{\bcB}^\top\widetilde{\X}_{\bcB})^{-1}\widetilde{\X}_{\bcB}^\top \widetilde{\bvarepsilon}_\B\right\|_2^2 \\
	& + \sum_{k \notin J_s} \left\|(\widetilde{\X}_{\E_k}^\top\widetilde{\X}_{\E_k})^{-1}\widetilde{\X}_{\E_k}^\top \widetilde{\bvarepsilon}_{\E_k}\right\|_2^2 + C_3\sum_{k\in J_s} s_k \cdot\max_{i = 1,\ldots, p_k} \left\|(\widetilde{\X}_{\E_k, [:, G_i^k]})^\top\widetilde{\bvarepsilon}_{\E_k}/n\right\|_2^2\Bigg).
	\end{split}
	\end{equation}
\end{Theorem}
\begin{proof}
See Appendix \ref{sec:proof_upper_bound_sparse_general}.
\end{proof}
\begin{Remark}
In the oracle error bound \eqref{ineq:oracle-error-bound}, $\|(\widetilde{\X}_{\bcB}^\top\widetilde{\X}_{\bcB})^{-1}\widetilde{\X}_{\bcB}^\top \widetilde{\bvarepsilon}_\B\|_2^2$,\\ $ \left\|(\widetilde{\X}_{\E_k}^\top\widetilde{\X}_{\E_k})^{-1}\widetilde{\X}_{\E_k}^\top \widetilde{\bvarepsilon}_\B\right\|_2^2$, and $s_k \max_{i = 1,\ldots, p_k} \|(\widetilde{\X}_{\E_k, [:, G_i^k]})^\top\widetilde{\bvarepsilon}_{\E_k}/n\|_2^2$ correspond to the estimation errors of $\widehat{\bcB}$, $\widehat{\E}_k$ of the nonsparse mode, and $\widehat{\E}_k$ of sparse mode, respectively. When the group restricted isometry property (Condition \ref{con:group-RIP}) is replaced by group restricted eigenvalue condition (see, e.g., \cite{lounici2011oracle}), a similar result to Theorem \ref{th:upper_bound_sparse_general} can be derived. 
\end{Remark}

\section{Fast Low-rank Tensor Regression via ISLET}\label{sec:regression}

We further study the low-rank tensor regression with Gaussian ensemble design, i.e., $\bcX_i$ has i.i.d. standard normal entries. This has been considered a benchmark setting for low-rank tensor/matrix recovery literature  \cite{candes2011tight,chen2016non}. For convenience, we denote $\bp = (p_1,p_2,p_3), \br = (r_1,r_2,r_3)$, $p = \max\{p_1, p_2, p_3\}$, and $r = \max\{r_1, r_2, r_3\}$. We discuss the regular low-rank and sparse low-rank tensor regression in the next two subsections, respectively.

\subsection{Regular Low-rank Tensor Regression with ISLET}

We have the following theoretical guarantee for ISLET under Gaussian ensemble design.
\begin{Theorem}[Upper bound for tensor regression via ISLET]\label{th:upper_bound_regression}
	Consider the tensor regression model \eqref{eq:model}, where $\bcA \in \mathbb{R}^{p_1\times p_2 \times p_3}$ is Tucker rank-$(r_1, r_2, r_3)$, $\bcX_i$ has i.i.d. standard normal entries, and $\varepsilon\overset{i.i.d.}{\sim}N(0, \sigma^2)$. Denote $\widetilde{\sigma}^2 = \|\bcA\|_{\tHS}^2 + \sigma^2$, $\lambda_0 = \min_k \lambda_k, \lambda_k = \sigma_{r_k}(\mathcal{M}_k(\bcA))$, $\kappa = \max_k \|\mathcal{M}_k(\bcA)\|/\sigma_{r_k}(\mathcal{M}_k(\bcA))$, and $m = r_1r_2r_3 + \sum_{k=1}^3(p_k-r_k)r_k$. If $n_1\wedge n_2 \geq \frac{C\widetilde{\sigma}^2(p^{3/2} + \kappa pr)}{\lambda_0^2}$, then the sample-splitting ISLET estimator (see the forthcoming Remark \ref{rm:sample-splitting}) satisfies
	\begin{equation*}
	\left\|\widehat{\bcA}-\bcA\right\|_{\tHS}^2\leq \frac{m}{n_2}\left(\sigma^2 + \frac{C_1\widetilde{\sigma}^4mp}{n_1^2\lambda_0^2}\right)\left(1 + C_2\sqrt{\frac{\log p}{m}} + C_3\sqrt{\frac{m\widetilde{\sigma}^2}{(n_1\wedge n_2)\lambda_0^2}}\right)
	\end{equation*}
	with probability at least $1 - p^{-C_4}$.
\end{Theorem}

\begin{proof}
See Section \ref{sec:proof_upper_bound_regression} for details. Specifically, we first derive the estimation error upper bounds for sketching directions $\widetilde{\U}_k$ via the deterministic error bound of HOOI \cite{zhang2019HOOI}. Then we apply concentration inequalities to obtain upper bounds for $\left\|(\widetilde{\X}^\top\widetilde{\X})^{-1}\widetilde{\X}^\top\widetilde\bvarepsilon\right\|_2^2$ and $\|\widehat\D_k(\widehat\B_k\widetilde\V_k)^{-1}\|$ for $k=1,2,3$. Finally, the oracle inequality of Theorem \ref{th:upper_bound_general} leads to the desired upper bound.
\end{proof}

\begin{Remark}[Sample Complexity] 
In Theorem \ref{th:upper_bound_regression}, we show that as long as the sample size  $n=\Omega(p^{3/2}r + pr^2)$, ISLET achieves consistent estimation under regularity conditions. This sample complexity outperforms many computationally feasible algorithms in previous literature, e.g., $n= \Omega(p^2r\textrm{polylog}(p))$ in projected gradient descent \cite{chen2016non}, sum of nuclear norm minimization \cite{tomioka2011statistical}, and square norm minimization \cite{mu2014square}. To the best of our knowledge, ISLET is the first computationally efficient algorithm that achieves this sample complexity result.

On the other hand, \cite{mu2014square} showed that the direct nonconvex Tucker rank minimization, a computationally infeasible method, can do exact recovery with $O(pr + r^3)$ linear measurements in the noiseless setting. \cite{bousse2018linear} showed that if tensor parameter $\bcA$ is CP rank-$r$, the linear system $y_j = \langle \bcA, \bcX_j\rangle, j=1,\ldots,n$ has a unique solution with probability one if one has $O(pr)$ measurements. It remains an open question whether the sample complexity of $n = \Omega(p^{3/2}r + pr^2)$ is necessary for all computationally efficient procedures. 
\end{Remark}

\begin{Remark}[Sample splitting]\label{rm:sample-splitting} The direct analysis for the proposed ISLET in Algorithm \ref{al:procedure_regular} is technically involved, among which one major difficulty is the dependency between the sketching directions $\widetilde\U_k$ obtained in Step 1 and the regression noise $\widetilde{\bvarepsilon}$ in Step 2. To overcome this difficulty, we choose to analyze a modified procedure with the sample splitting scheme: we randomly split all $n$ samples into two sets with cardinalities $n_1$ and $n_2$, respectively. Then we use the first set of $n_1$ samples to construct the covariance tensor $\widetilde{\bcA}$ (Step 1) and use the second set of $n_2$ samples to evaluate the importance sketching covariates (Step 2). As illustrated by numerical studies in Section \ref{sec:numerical}, such a scheme is mainly for technical purposes and is not necessary in practice. Simulations suggest that it is preferable to use all samples $\{y_i, \bcX_i\}_{i=1}^n$ for both constructing the initial estimate $\widetilde{\bcA}$ and performing linear regression on sketching covariates.
\end{Remark}

We further consider the statistical limits for low-rank tensor regression with Gaussian ensemble. Consider the following class of general low-rank tensors,
\begin{equation}
\mathcal{A}_{\bp, \br} = \left\{\bcA \in \mathbb{R}^{p_1\times p_2\times p_3}: \text{Tucker rank}(\bcA) \leq (r_1,r_2,r_3) \right\}.
\end{equation}
The following minimax lower bound holds for all low-rank tensors in $\mathcal{A}_{\bp, \br}$.
\begin{Theorem}[Minimax Lower Bound]\label{th:lower-bound-regression} 
	If $n > m+1$, the following nonasymptotic lower bound in estimation error hold,
	\begin{equation}\label{ineq:hat-A-regression-lower}
	\inf_{\widehat{\bcA}}\sup_{\bcA \in \mathcal{A}_{\bp, \br}}\mathbb{E} \left\|\widehat{\bcA} - \bcA\right\|_{\tHS}^2 \geq \frac{m}{n-m-1}\cdot \sigma^2.
	\end{equation}
	If $n \leq m + 1$, 
	\begin{equation}
	\inf_{\widehat{\bcA}}\sup_{\bcA \in \mathcal{A}_{\bp, \br}}\mathbb{E} \left\|\widehat{\bcA} - \bcA\right\|_{\tHS}^2 = +\infty.
	\end{equation}
\end{Theorem}
\begin{proof}
See Appendix \ref{sec:proof_lower-bound-regression}.
\end{proof}
Combining Theorems \ref{th:upper_bound_regression} and \ref{th:lower-bound-regression}, we can see that as long as the sample size satisfies $\frac{m\widetilde\sigma^2}{n_1\lambda_0^2} = o(1)$, $\frac{m(p_1+p_2+p_3)\widetilde{\sigma}^4}{n_1n_2\lambda_0^2}=o(\sigma^2)$, and $n_2 = (1+o(1))n$, the statistical loss of the proposed method is sharp with matching constant to the lower bound.

\begin{Remark}[Matrix ISLET vs. Previous Matrix Recovery Methods]
If the order of tensor reduces to two, the tensor regression becomes the well-regarded \emph{low-rank matrix recovery} in literature \cite{candes2011tight,recht2010guaranteed}:
$$y_i = \left\langle\X_i, \A\right\rangle + \varepsilon_i, \quad i=1,\ldots, n.$$
Here, $\A \in \mathbb{R}^{p_1\times p_2}$ is the unknown rank-$r$ target matrix, $\{\X_i\}_{i=1}^n$ are design matrices, and $\varepsilon_i \sim N(0, \sigma^2)$ are noises. The low-rank matrix recovery, including its instances such as phase retrieval \cite{candes2015phase}, has been widely considered in recent literature. Various methods, such as nuclear norm minimization \cite{candes2010matrix,recht2010guaranteed}, projected gradient descent \cite{toh2010accelerated}, singular value thresholding \cite{cai2010singular}, Procrustes flow \cite{tu2016low}, etc, have been introduced and both the theoretical and computational performances have been extensively studied. By similar proof of Theorem \ref{th:upper_bound_regression}, the following upper bound for matrix ISLET estimator $\widehat{\A}$ (Algorithm \ref{al:procedure_regular_matrix} in the supplementary materials \cite{ISLET-supplement})
$$\left\|\widehat{\A} - \A\right\|_F^2 \leq \frac{m}{n_2}\left(\sigma^2 + \frac{C_1\widetilde{\sigma}^4mp}{n_1^2\lambda_0^2}\right)\left(1 + C_2\sqrt{\frac{\log p}{m}} + C_3\sqrt{\frac{m\widetilde{\sigma}^2}{(n_1\wedge n_2)\lambda_0^2}}\right) $$
can be established with high probability. Here, $m=(p_1+p_2-r)r$, $\lambda_0 = \sigma_r(\A), \widetilde{\sigma}^2 = \|\A\|_F^2+\sigma^2$. The lower bound similarly to Theorem \ref{th:lower-bound-regression} also holds.  
\end{Remark}

\subsection{Sparse Tensor Regression with Importance Sketching}\label{sec:sparse-regression}

We further consider the simultaneously sparse and low-rank tensor regression with Gaussian ensemble design. We have the following theoretical guarantee for sparse ISLET. Due to the same reason as for regular ISLET (see Remark \ref{rm:sample-splitting}), the sample splitting scheme is introduced in our technical analysis.
\begin{Theorem}[Upper Bounds for Sparse Tensor Regression via ISLET]\label{th:upper_bound_sparse_tensor_regression}
	Consider the tensor regression model \eqref{eq:model}, where $\bcA$ is simultaneously low-rank and sparse \eqref{eq:tucker-decomposition-sparse}, $\bcX_i$ has i.i.d. standard Gaussian entries, and $\varepsilon_i\overset{i.i.d.}{\sim}N(0, \sigma^2)$. Denote $\lambda_0 = \min_k \sigma_{r_k}(\mathcal{M}_k(\bcA))$, $s_k = p_k$ if $k\notin J_s$, $m_s = r_1r_2r_3 + \sum_{k\in J_s} s_k(r_k + \log p_k) + \sum_{k \notin J_s} p_kr_k$, and $\kappa = \max_k \|\mathcal{M}_k(\bcA)\|/\sigma_{r_k}(\mathcal{M}_k(\bcA))$. We apply the proposed Algorithm \ref{al:procedure_sparse} with sample splitting scheme (see Remark \ref{rm:sample-splitting}) and group Lasso penalty $\eta_k = C_0 \widetilde{\sigma}\sqrt{n_2(r_k + \log(p_k))}$. If $\log(p_1)\asymp \log(p_2)\asymp \log(p_3)\asymp\log (p)$,
	$$n_1 \geq \frac{C_1\kappa^2\widetilde{\sigma}^2}{\lambda_0^2}\left(s_1s_2s_3\log(p) + \sum_{k=1}^3(s_k^2r_k^2 + r_{k+1}^2r_{k+2}^2)\right),\quad n_2 \geq \frac{C_2 m_s\kappa^2\widetilde{\sigma}^2}{\lambda_0^2},$$
	the output $\widehat{\bcA}$ of sparse ISLET satisfies
	\begin{equation}\label{ineq:sparse-tensor-regression}
	\left\|\widehat{\bcA} - \bcA\right\|_{\tHS}^2 \leq \frac{C_3 m_s}{n_2}\left(\sigma^2 + \frac{C_4 m_s\kappa^2\widetilde{\sigma}^2}{n_1}\right)
	\end{equation}
	with probability at least $1 - p^{-C}$. 
\end{Theorem}
\begin{proof}
See Appendix \ref{sec:proof_upper_bound_sparse_tensor_regression}.
\end{proof}
We further consider the following class of simultaneously sparse and low-rank tensors,
\begin{equation}
\mathcal{A}_{\bp, \br, \bs} = \left\{\bcA = \llbracket \bcS; \U_1, \U_2, \U_3\rrbracket: \U_k\in \mathbb{O}_{p_k, r_k}, \|\U_k\|_{0,2} \leq s_k, k\in J_s \right\}.
\end{equation}
The following minimax lower bound of the estimation risk holds in this class. 
\begin{Theorem}[Lower Bounds]\label{th:lower_bound_sparse_tensor_regression}
	There exists constant $C>0$ such that whenever $m_s\geq C$, the following lower bound holds for any arbitrary estimator $\widehat{\bcA}$ based on $\{\bcX_i, y_i\}_{i=1}^n$, 
	\begin{equation}
	\inf_{\bcA}\sup_{\bcA\in \mathcal{A}_{\bp, \br, \bs}}\mathbb{E}\left\|\widehat{\bcA} - \bcA\right\|_{\tHS}^2 \geq \frac{cm_s}{n}\sigma^2.
	\end{equation}
\end{Theorem}
\begin{proof}
See Appendix \ref{sec:proof-lower_bound_sparse_tensor_regression}.
\end{proof}
Combining Theorems \ref{th:upper_bound_sparse_tensor_regression} and \ref{th:lower_bound_sparse_tensor_regression}, we can see the proposed procedure achieves optimal rate of convergence if $\frac{m_s\|\bcA\|_{\tHS}^2}{n_1\sigma^2} = O(1)$ and $n_2\asymp n$.

\section{Numerical Analysis}\label{sec:numerical}

In this section, we conduct a simulation study to investigate the numerical performance of ISLET. In each study, we construct sensing tensors $\boldsymbol{\mathcal{X}}_j\in \mathbb{R}^{p\times p \times p}$ with independent standard normal entries. In the nonsparse settings, using the Tucker decomposition we generate the core tensor $\bcS \in \mathbb{R}^{r\times r\times r}$ and $\E_k \in \mathbb{R}_{p, r}$ with i.i.d. Gaussian entries, the coefficient tensor $\bcA = \llbracket \boldsymbol{\mathcal{S}}; \E_1; \E_2; \E_3\rrbracket$; in the sparse settings, we construct $\bcS$ and $\bcA$ in the same way and generate $\E_k$ as
$$\quad (\E_k)_{[i, :]}
 = \left\{\begin{array}{ll}
(\bar{\E}_k)_{[j,:]}, & i\in \Omega_k, \text{ and $i$ is the $j$-th element of $\Omega_k$};\\
0, & i \notin \Omega_k,
\end{array}\right. $$ 
where $\Omega_k$ is a uniform random subset of $\{1,\ldots, p\}$ with cardinality $s_k$ and $\bar{\E}_k$ has $s_k$-by-$r$ i.i.d. Gaussian entries. Finally, let the response $y_j = \langle \bcX_j, \bcA \rangle + \varepsilon_j, j = 1, 2, \ldots, n$, where $\varepsilon_j \overset{i.i.d.}\sim N(0, \sigma^2)$. We report both the average root mean-squared error (RMSE) $\|\widehat{\bcA} - \bcA||_{\tHS}/ ||\bcA\|_{\tHS}$ and the run time for each setting. Unless otherwise noted, the reported results are based on the average of 100 repeats and on a computer with Intel Xeon E5-2680 2.50GHz CPU. 
Additional simulation results of tuning-free ISLET and approximate low-rank tensor regression are collected in Sections \ref{sec:tuning} and \ref{sec:additional-simu} in the supplementary materials \cite{ISLET-supplement}.

Since we proposed to evaluate sketching directions and dimension-reduced regression (Steps 1 and 2 of Algorithm \ref{al:procedure_regular}) both using the complete sample, but introduced a sample splitting scheme (Remark \ref{rm:sample-splitting}) to prove Theorems \ref{th:upper_bound_regression} and \ref{th:upper_bound_sparse_tensor_regression}, we investigate how the sample splitting scheme affects the numerical performance of ISLET in this simulation setting. Let $n$ vary from 1000 to 4000, $p = 10$, $r = 3, 5$, $\sigma = 5$. In addition to the original ISLET without splitting, we also implement sample-splitting ISLET, where a random $n_1 \approx \{\frac{3}{10}n, \frac{4}{10}n, \frac{5}{10}n\}$ samples are allocated for importance direction estimation (Step 1 of ISLET) and $n-n_1$ are allocated for dimension-reduced regression (Step 2 of ISLET). The results plotted in Figure \ref{fig: 1} clearly show that the no-sample-splitting scheme yields much smaller estimation error than all sample-splitting approaches. Although the sample splitting scheme brings advantages for our theoretical analyses for ISLET, it is not necessary in practice. Therefore, we will only perform ISLET without sample splitting for the rest of the simulation studies. 
\begin{figure}[h!]
	\centering
	\subfigure[$r = 3$]{
		\includegraphics[width=0.45\textwidth,height=2in]{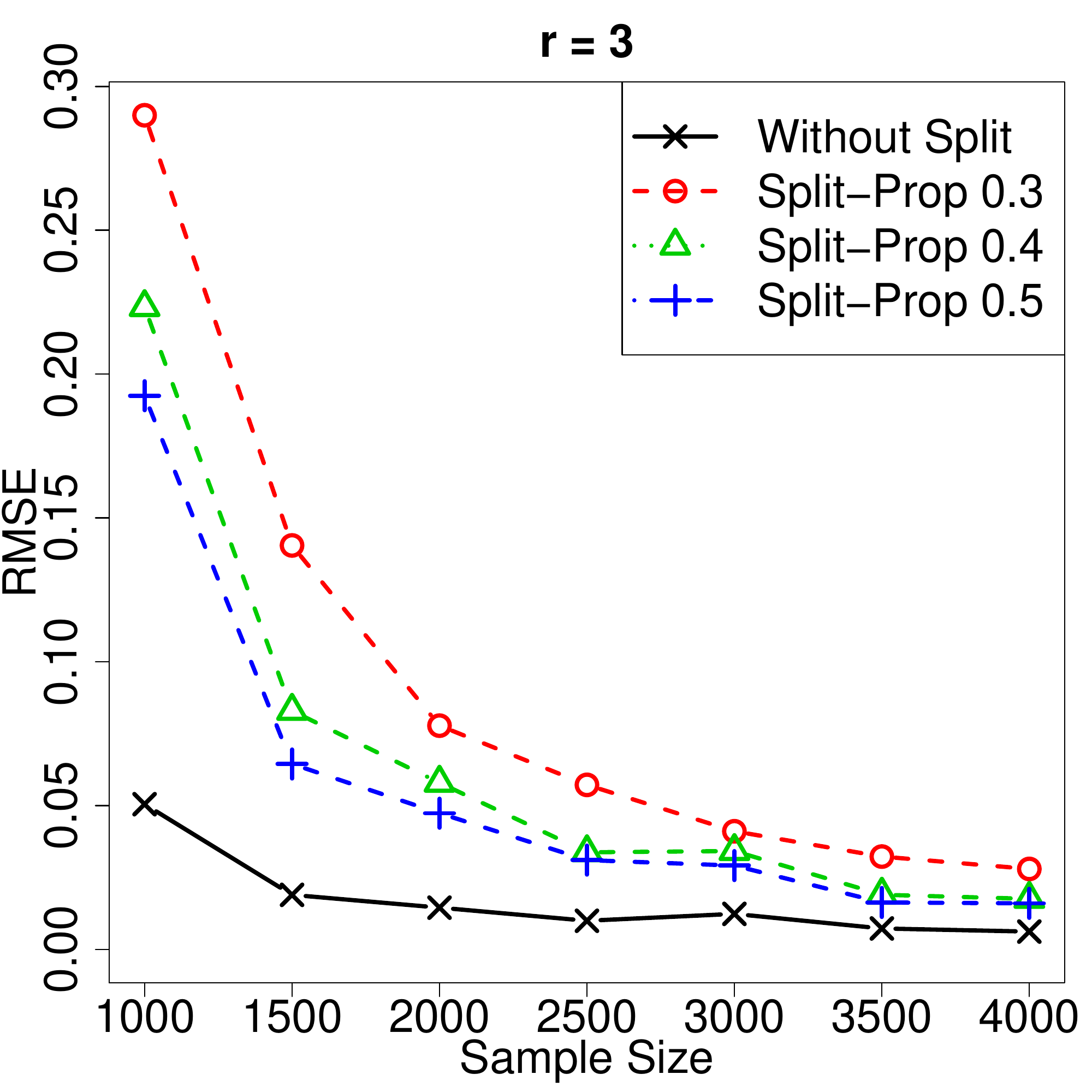}}
	\subfigure[$r = 5$]{
		\includegraphics[width=0.45\textwidth,height=2in]{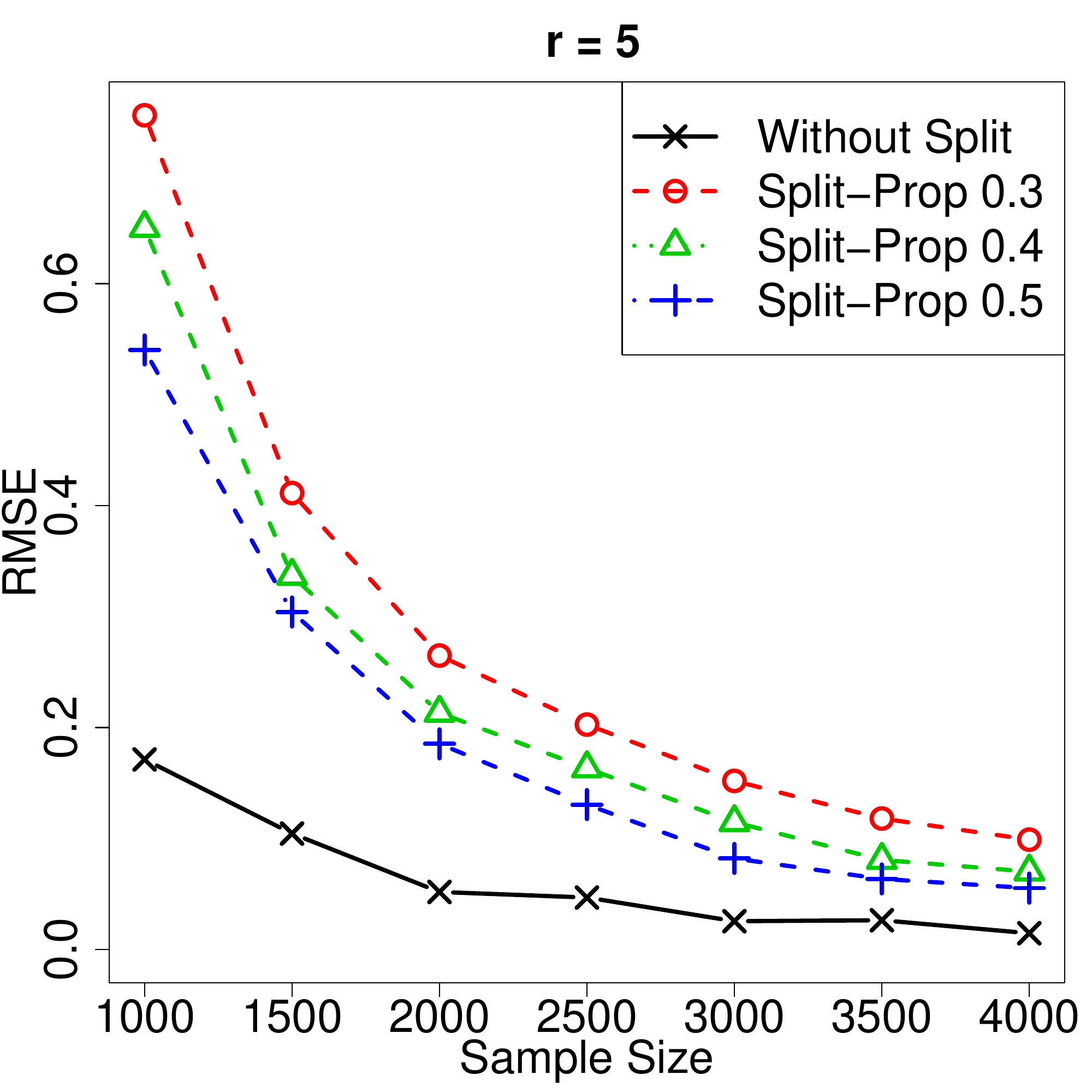}}
	\caption{No-splitting vs. splitting ISLET: $n$ varies from 1000 to 4000, $p = 10$, $r = 3, 5$, $\sigma = 5$.}\label{fig: 1}
\end{figure}

We also compare the performance of nonsparse ISLET with a number of contemporary methods, including nonconvex projected gradient descent (nonconvex PGD) \cite{chen2016non}, Tucker low-rank regression via alternating gradient descent (Tucker regression)\footnote{Software package downloaded at \url{https://hua-zhou.github.io/TensorReg/}} \cite{li2013tucker,zhou2013tensor}, and convex regularization low-rank tensor recovery (convex regularization)\footnote{The convex regularization aims to minimize the following objective function
	$$\sum_{i}^n \frac{1}{2n} (y_i - \langle \boldsymbol{\mathcal{X}}_i, \boldsymbol{\mathcal{A}} \rangle  )^2 + \lambda \sum_{k = 1}^3 ||\mathcal{M}_{k}(\boldsymbol{\mathcal{A}})||_*.$$
	Here, $\|\cdot\|_\ast$ is the matrix nuclear norm.} \cite{liu2013tensor,raskutti2015convex,tomioka2011statistical}. 
We implement all four methods for $p=10$, but only the ISLET and nonconvex projected PGD for $p=50$, as the time cost of Tucker regression and convex regularization are beyond our computational limit if $p=50$. Results for $p=10$ and $p=50$ are respectively plotted in Panels (a)(b) and Panels (c)(d) of Fig.~\ref{fig: 4}. Plots in Fig.~\ref{fig: 4} (a) and (c) show that the RMSEs of ISLET, tucker tensor regression and nonconvex PGD are close, and all of them are slightly better than the convex regularization method; Figure \ref{fig: 4} (b) and (d) further indicate that ISLET is much faster than other methods -- the advantage significantly increases as $n$ and $p$ grow. In particular, ISLET is about 10 times faster than nonconvex PGD when $p = 50, n = 12000$. In summary, the proposed ISLET achieves similar statistical performance within in a significantly shorter time period comparing to the other state-or-the-art methods. 
\begin{figure}[h!]
	\centering
	\subfigure[RMSE]{
		\includegraphics[width=0.45\textwidth,height=2in]{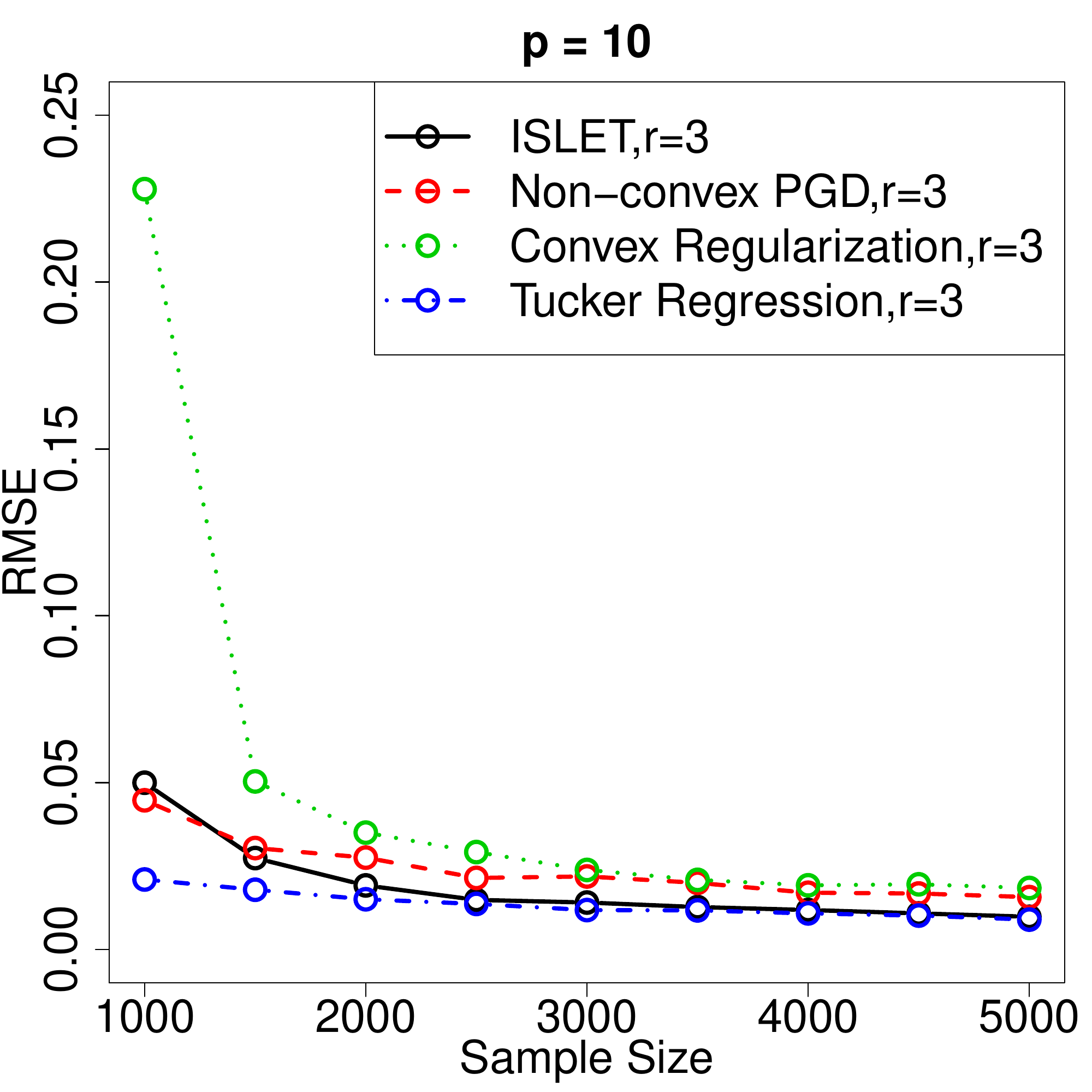}}
	\subfigure[Run Time]{
		\includegraphics[width=0.45\textwidth,height=2in]{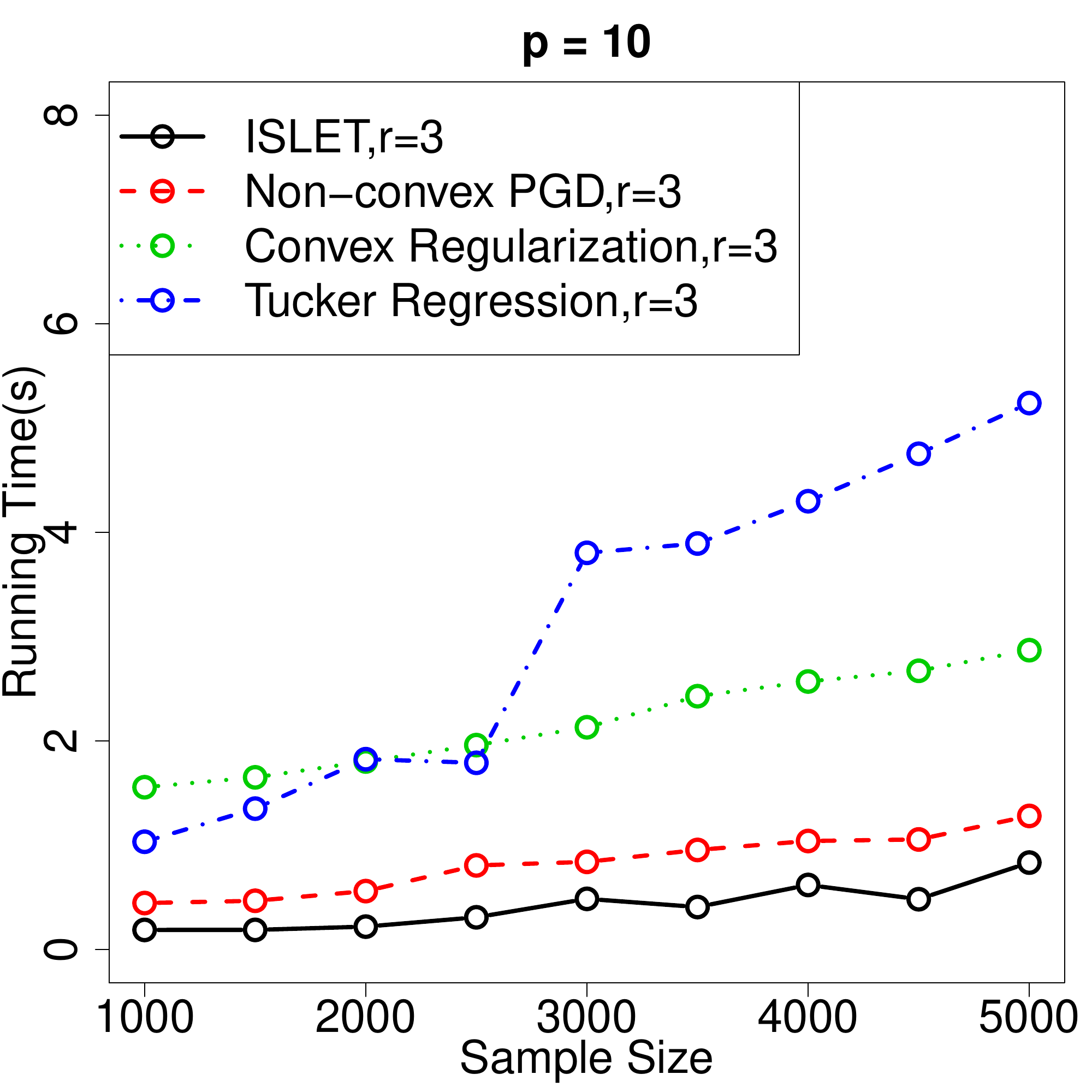}}
	\subfigure[RMSE]{
		\includegraphics[width=0.45\textwidth,height=2in]{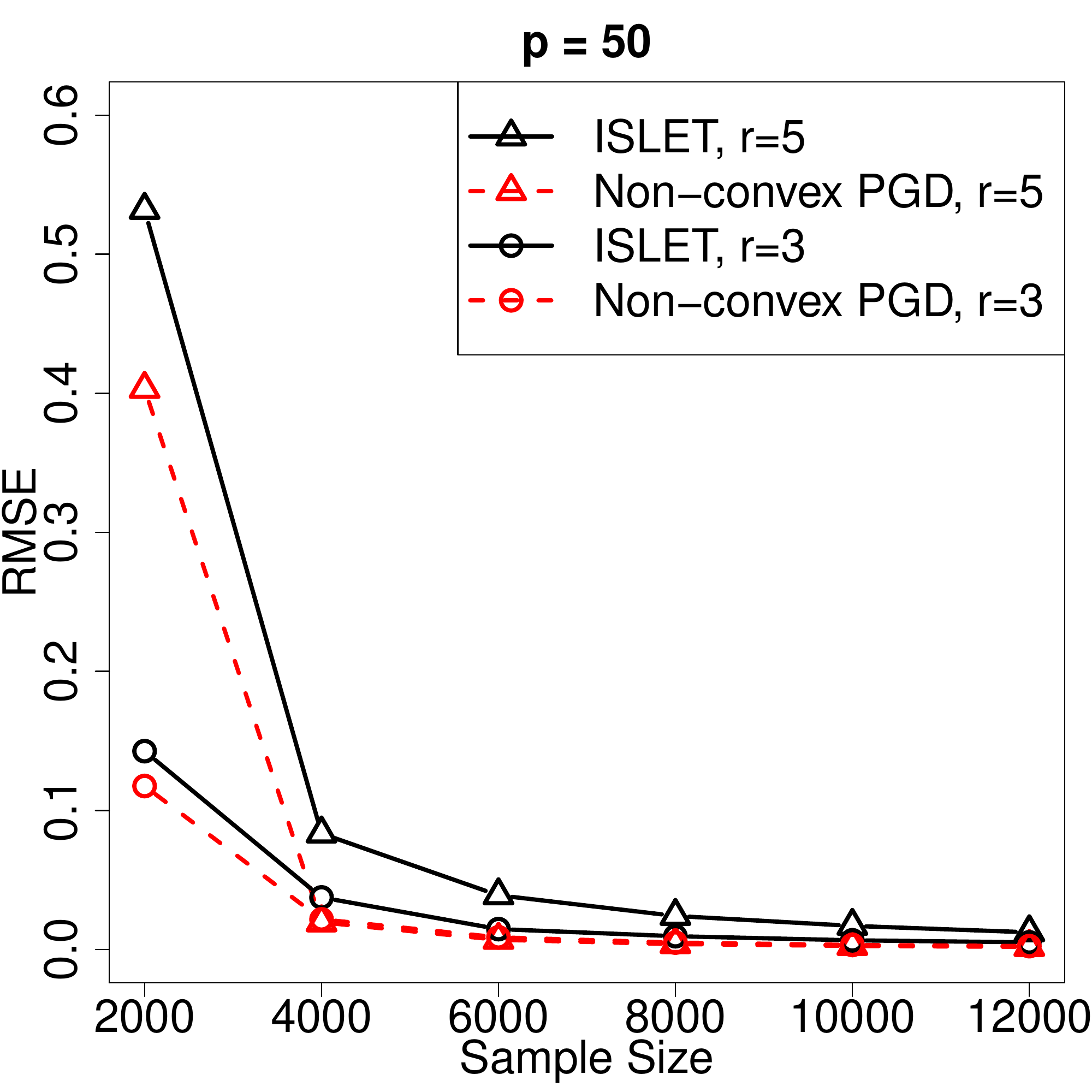}}
	\subfigure[Run Time]{
		\includegraphics[width=0.45\textwidth,height=2in]{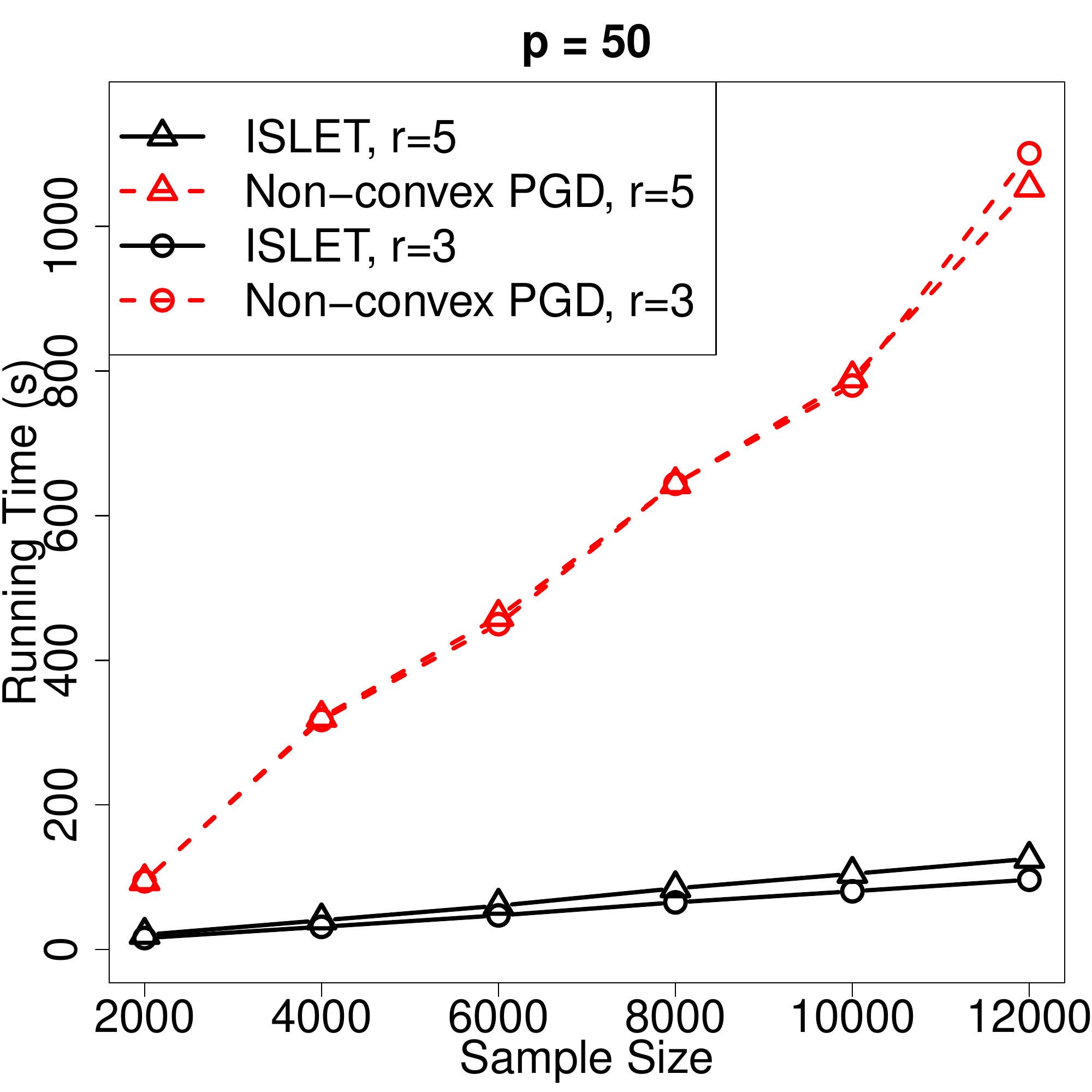}}
	\caption{ISLET vs. nonconvex PGD, Tucker regression, convex regularization. Here, $\sigma = 5$; Panels (a)(b): $p = 10$; Panels (c)(d): $p = 50$.}\label{fig: 4}
\end{figure}

Next, we investigate the performance of ISLET when $p$ and $n$ substantially grow. Let $p = 100, 150, 200$, $r = 3, 5$, $n\in [8000, 20000]$. The results in RMSE and run time are shown in Fig.~\ref{fig: 9} (a), (b), (c), and (d), respectively. We can see that the estimation error significantly decays as the sample size $n$ grows, the dimension $p$ decreases, or the Tucker rank $r$ decreases. 

We further fix $r = 2, n=30000$ and let $p$ grow to 400. Now the space cost for storing $\{\bcX_i\}_{i=1}^n$ reaches $400^3\times 30000\times 4\text{bytes} = 7.68$ terabytes, which is far beyond the volume of most personal computing devices. Since each sample is used only twice in ISLET, we perform this experiment in a parallel way. 
To be specific, in each machine $b = 1,\ldots, 40$, we store the random seed, draw pseudo random tensor $\bcX_{bi}$, evaluate $y_{bi}$ and $\widetilde{\bcA}_b$ by the procedure in Section \ref{sec:regular-islet-procedure}, and clean up the memory of $\bcX_{bi}$. After synchronizing the outcomes and obtaining the importance sketching directions, for each machine $b=1,\ldots, 40$, we generate pseudorandom covariates $\bcX_{bi}$ again using the stored random seeds, evaluate $\widetilde{\G}_b$ and $\widetilde{\X}_{bi}$ by \eqref{eq:parallel-2}-\eqref{eq:parallel-3}, and clean up the memory of $\bcX_{bi}$ again. The rest of the procedure follows from Section \ref{sec:regular-islet-procedure} and the original ISLET in Algorithm \ref{al:procedure_regular}. The average RMSE and run time for five repeats are shown in Figure \ref{fig: 14}. We clearly see that ISLET yields good statistical performance within a reasonable amount of time, while the other contemporary methods can hardly do so in such an ultrahigh-dimensional setting.
\begin{figure}[h!]
	\centering
	\subfigure[RMSE]{
		\includegraphics[width=0.45\textwidth,height=2in]{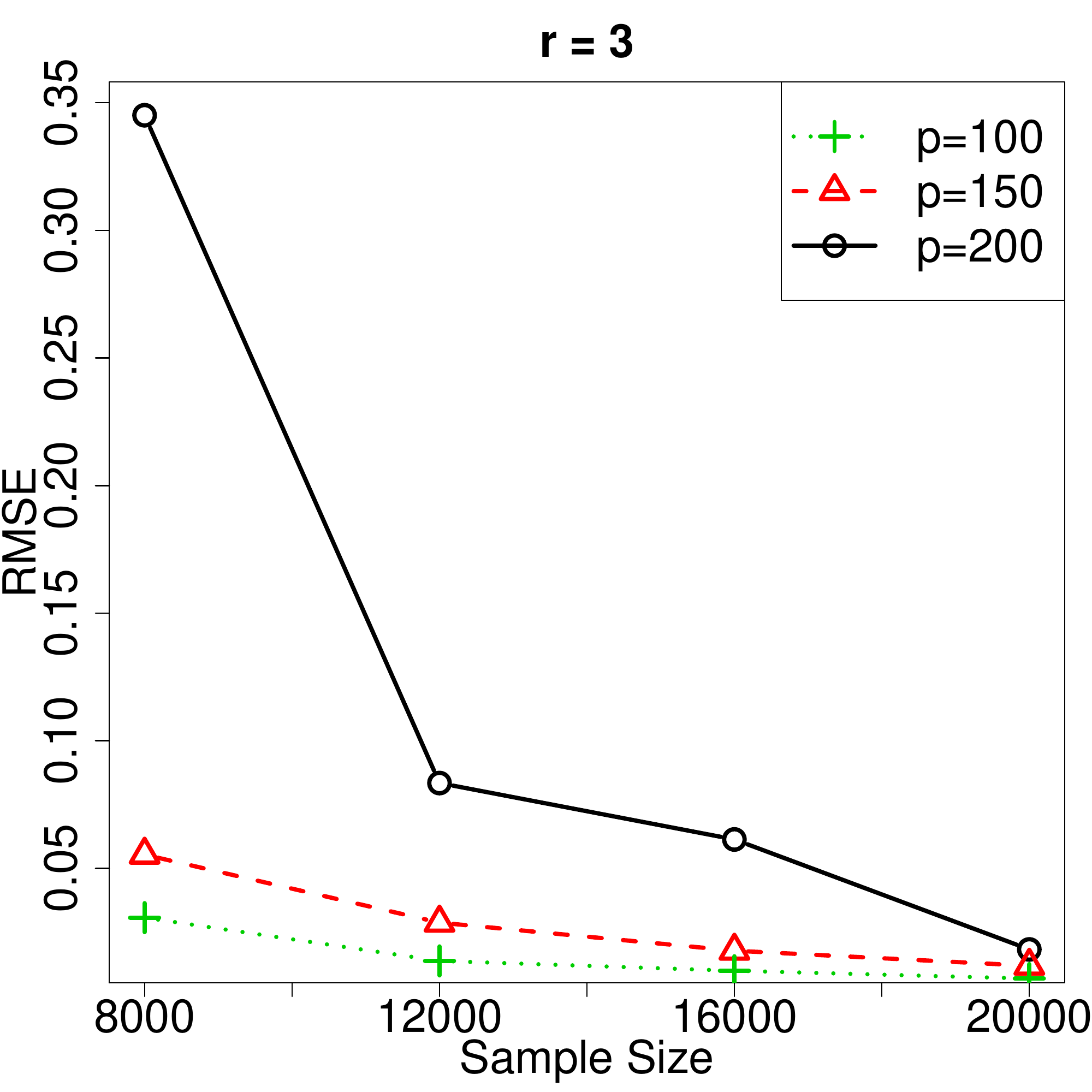}}
	\subfigure[Run Time (Unit: hours)]{
		\includegraphics[width=0.45\textwidth,height=2in]{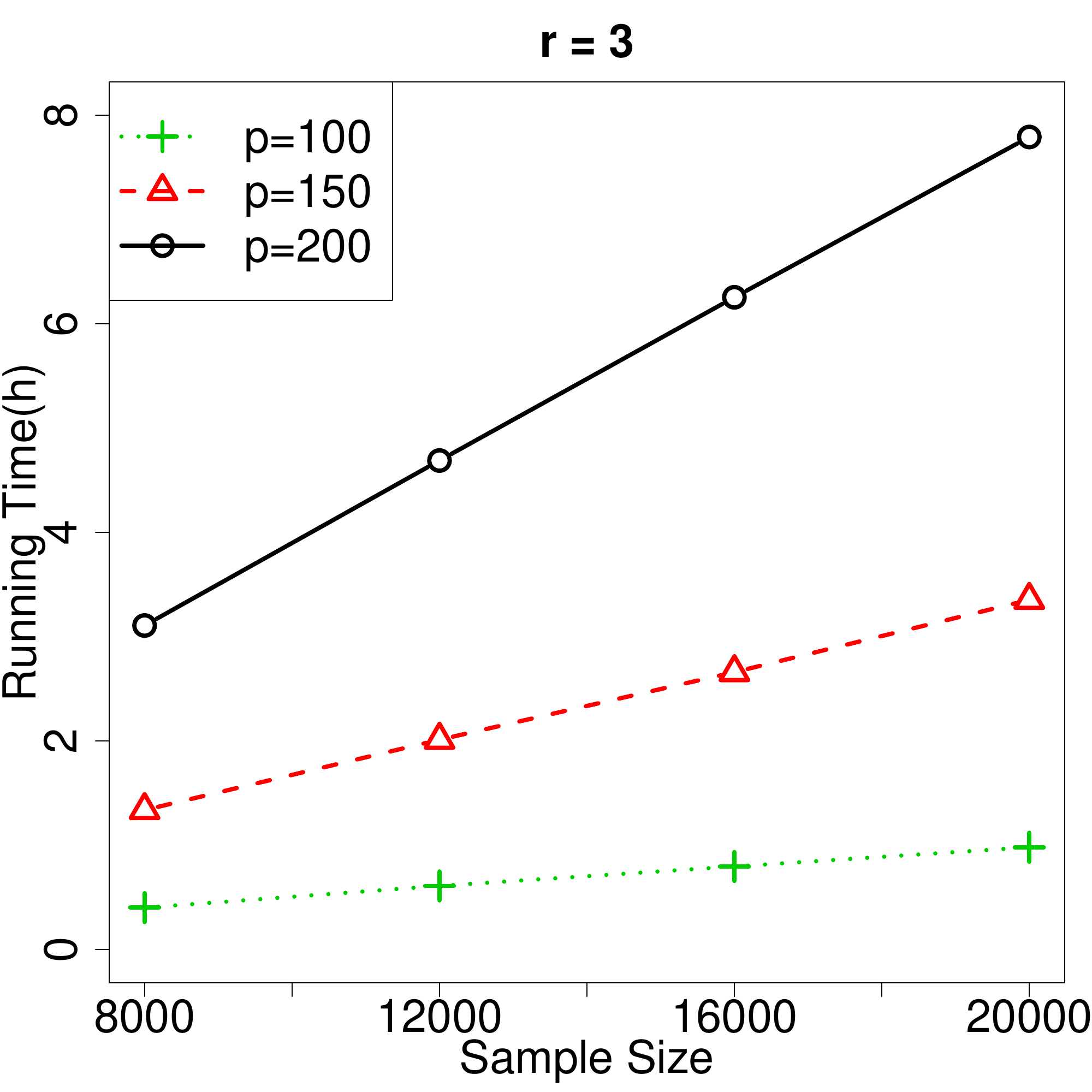}}
	\subfigure[RMSE]{
		\includegraphics[width=0.45\textwidth,height=2in]{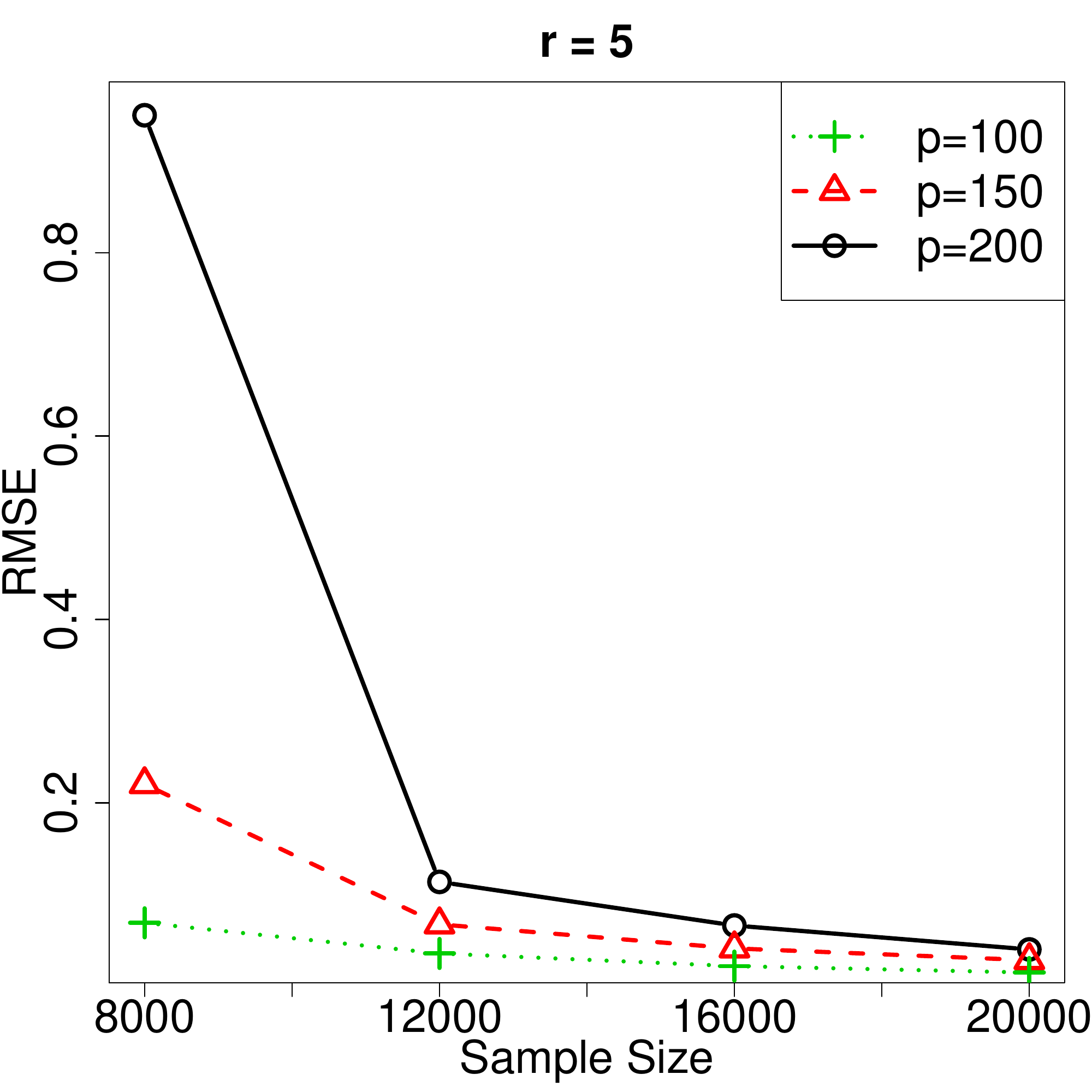}}
	\subfigure[Run Time (Unit: hours)]{
		\includegraphics[width=0.45\textwidth,height=2in]{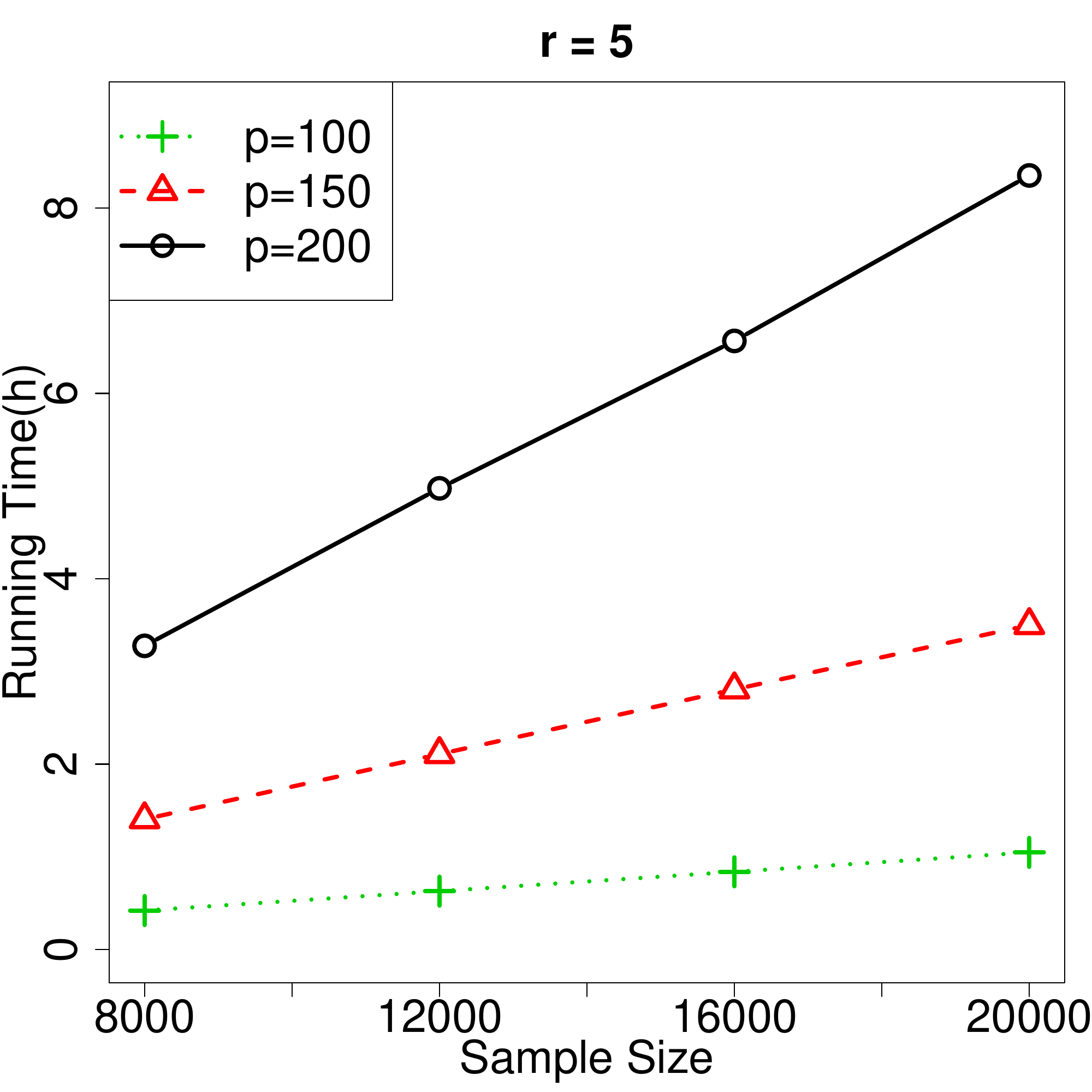}}
	\caption{Performance of ISLET when $p$ and $n$ significantly grow.}\label{fig: 9}
	\centering
	\subfigure[RMSE]{
		\includegraphics[width=0.45\textwidth,height=2in]{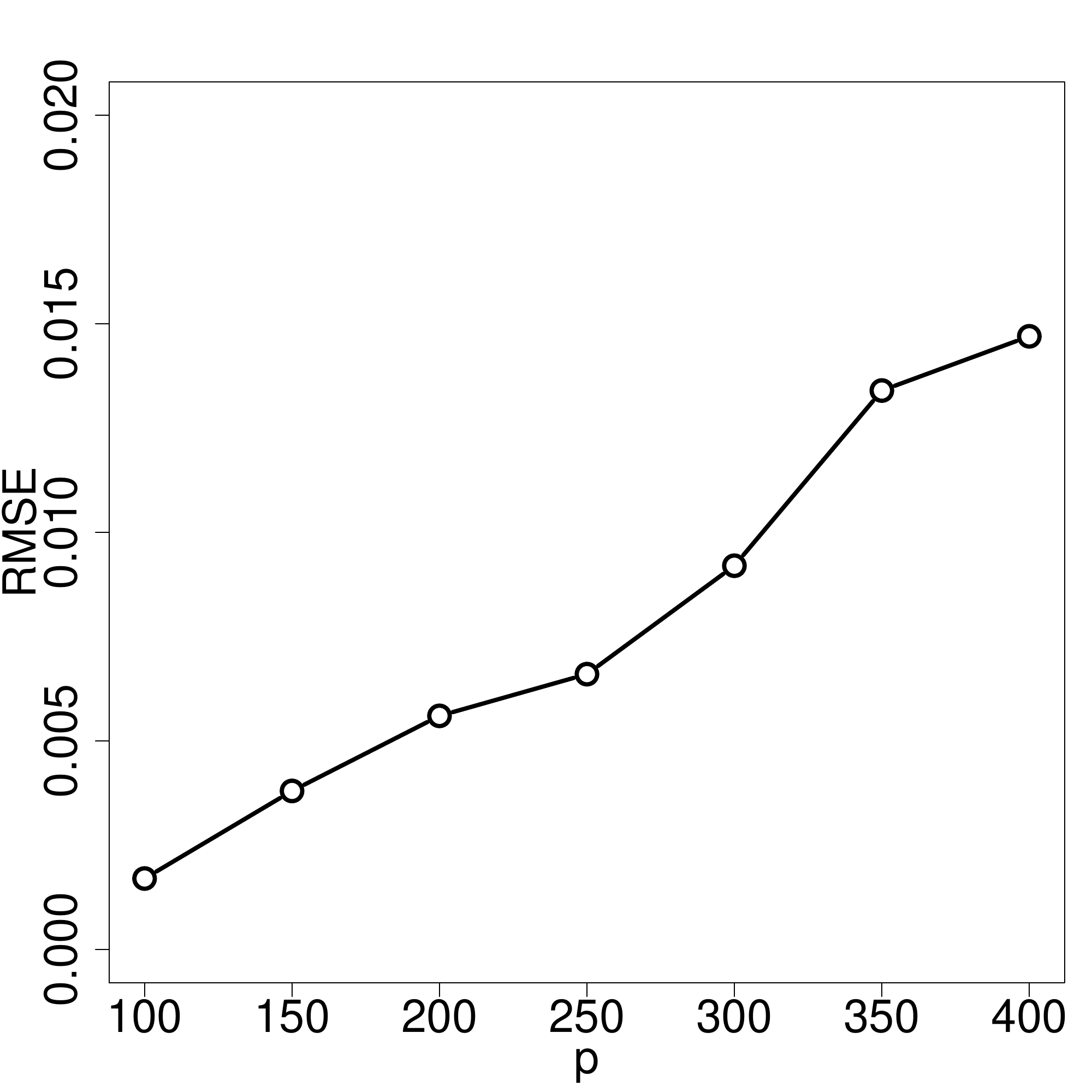}}
	\subfigure[Run Time]{
		\includegraphics[width=0.45\textwidth,height=2in]{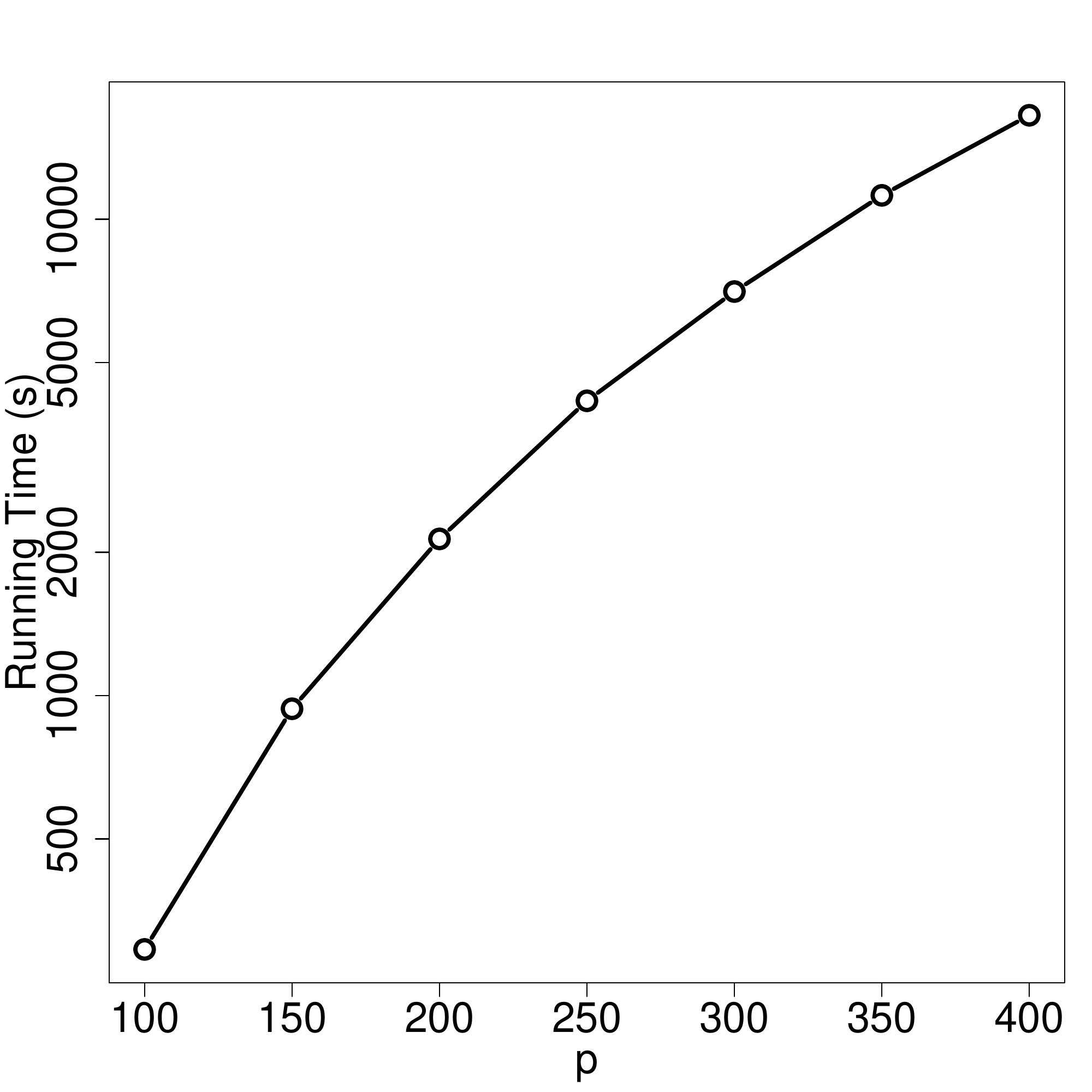}}
	\caption{Performance of ISLET in ultrahigh-dimensional setting. $p$ grows up to $400$, $n=30000$.}\label{fig: 14}
\end{figure}

In addition, we explore the numerical performance of ISLET for simultaneously sparse and low-rank tensor regression. To perform sparse ISLET (Algorithm \ref{al:procedure_sparse}), we apply the \emph{gglasso} package\footnote{Available online at: \url{https://cran.r-project.org/web/packages/gglasso/index.html}.} \cite{yang2015fast} for group Lasso and penalty level selection. Let $n$ vary from 1500 to 4000, $p = 20, 25, 30$, $r = 3, 5$, $\sigma = 5$, $s = s_1 = s_2 = s_3 = 8$. The result is shown in Fig.~\ref{fig: 10}. Similar to the nonsparse ISLET, as sample size $n$ increases or Tucker rank $r$ decreases, the average estimation errors decrease. 

We also compare sparse ISLET with slice-sparse nonconvex PGD proposed by \cite{chen2016non}. Let $n \in [5000, 12000]$, $p = 50$, $r = 3, 5$, $\sigma = 5$, $s_1 = s_2 = s_3 = 15$. 
From Fig.~\ref{fig: 12}, we can see that ISLET yields much smaller estimation error with significantly shorter time than nonconvex PGD -- the difference between two algorithms becomes more significant as $n$ grows.
\begin{figure}[h!]
	\centering
	\subfigure[$r = 3$]{
		\includegraphics[width=0.45\textwidth,height=2in]{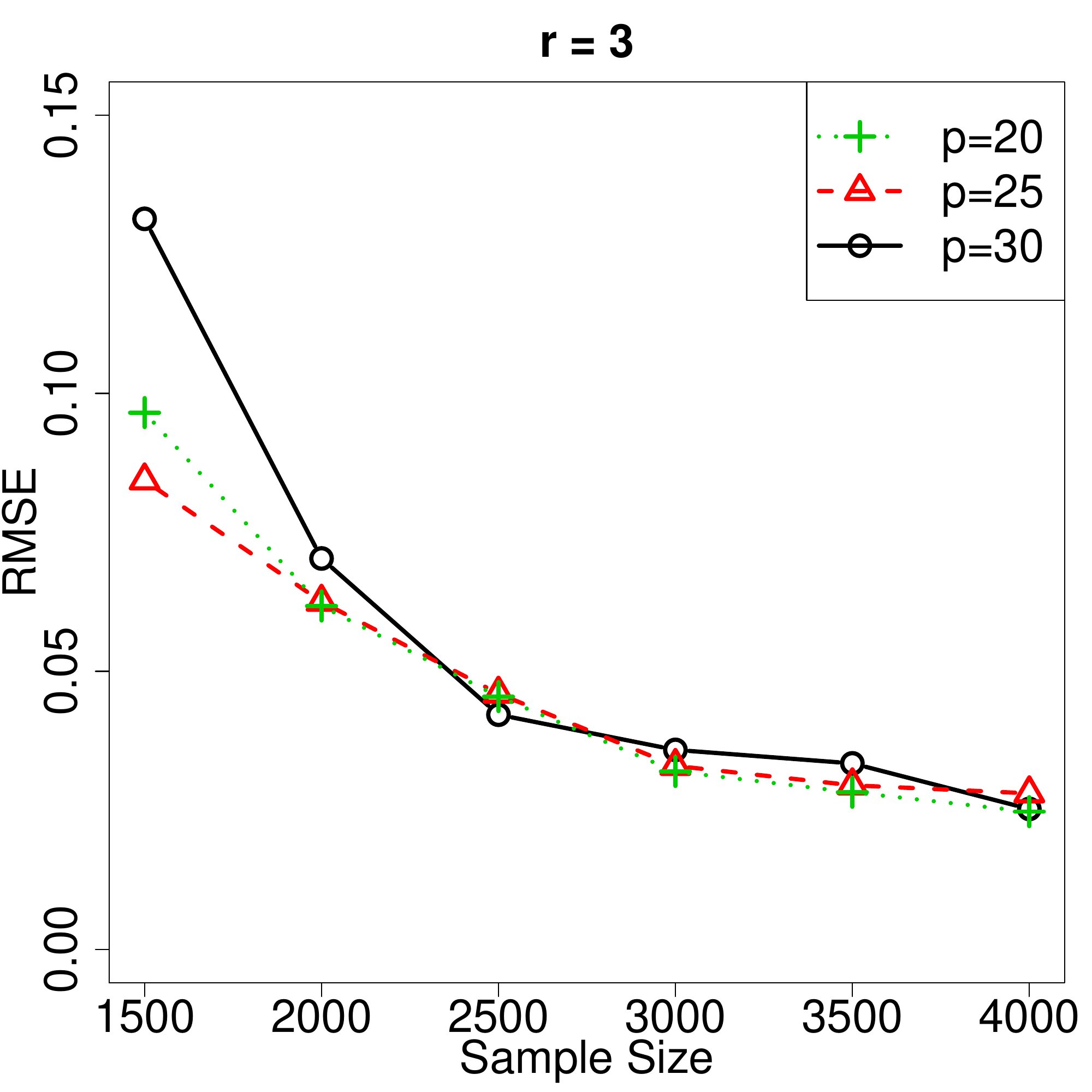}}
	\subfigure[$r = 5$]{
		\includegraphics[width=0.45\textwidth,height=2in]{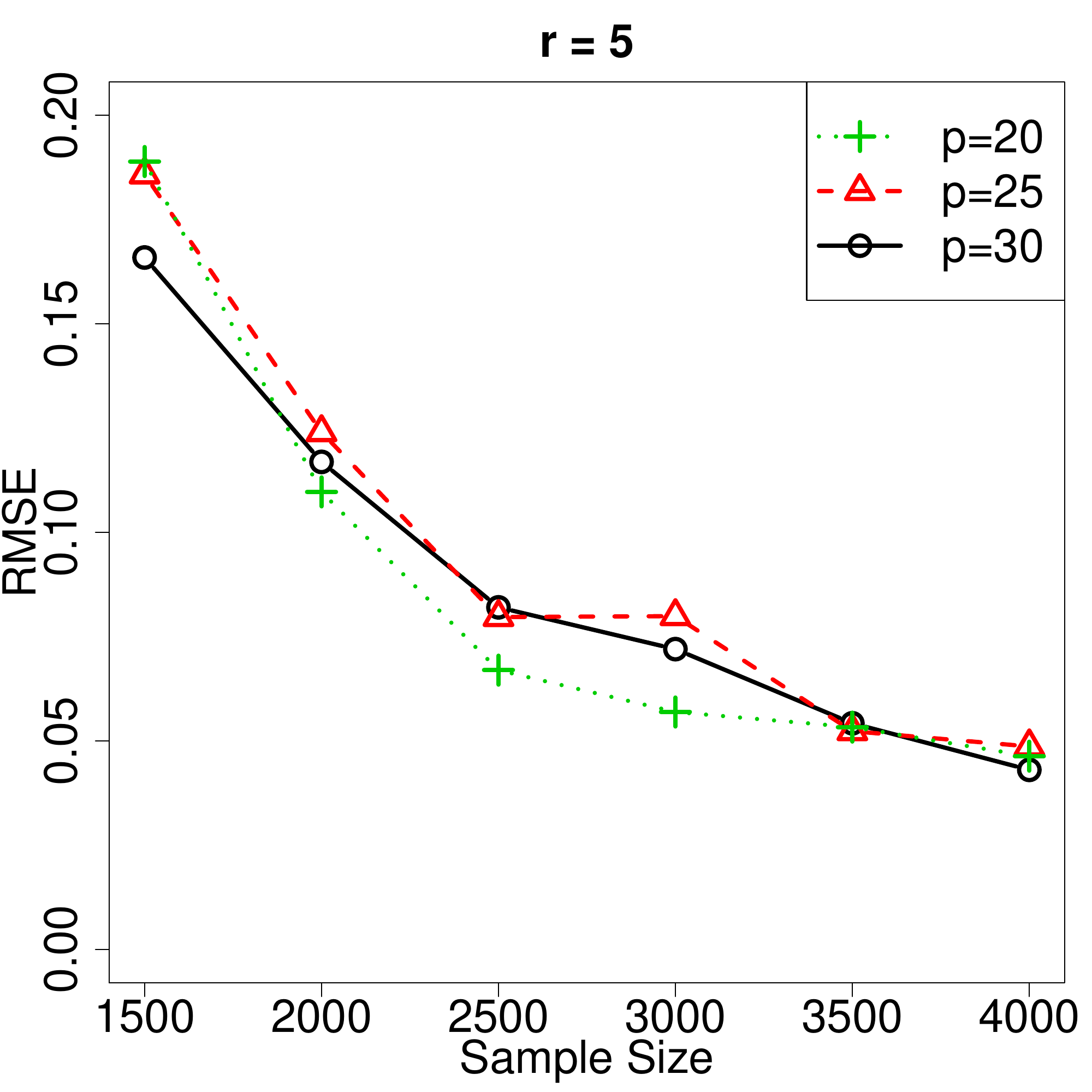}}
	\caption{RMSE of ISLET for sparse and low-rank tensor recovery}\label{fig: 10}
	\centering
	\subfigure[RMSE]{
		\includegraphics[width=0.45\textwidth,height=2in]{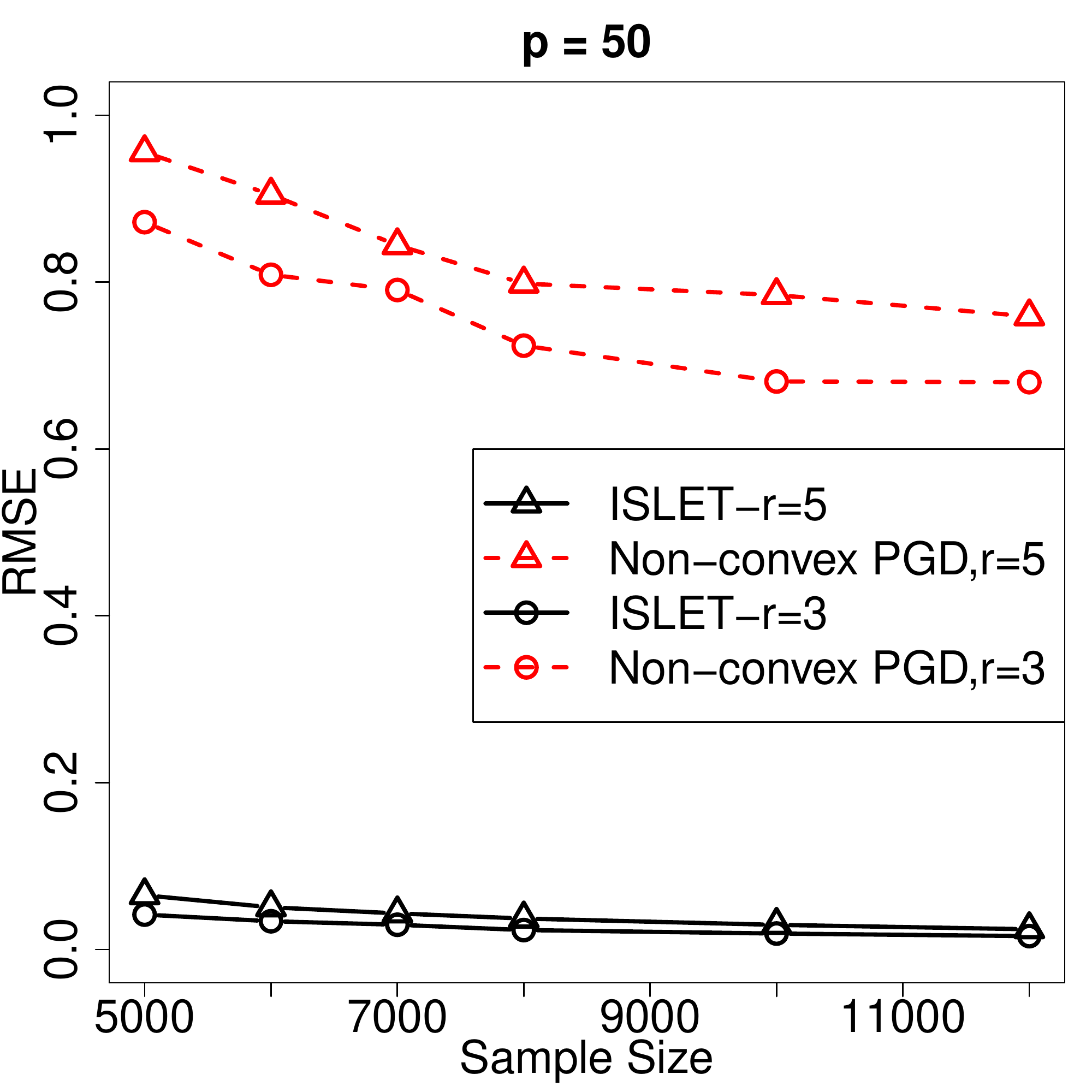}}
	\subfigure[Run Time]{
		\includegraphics[width=0.45\textwidth,height=2in]{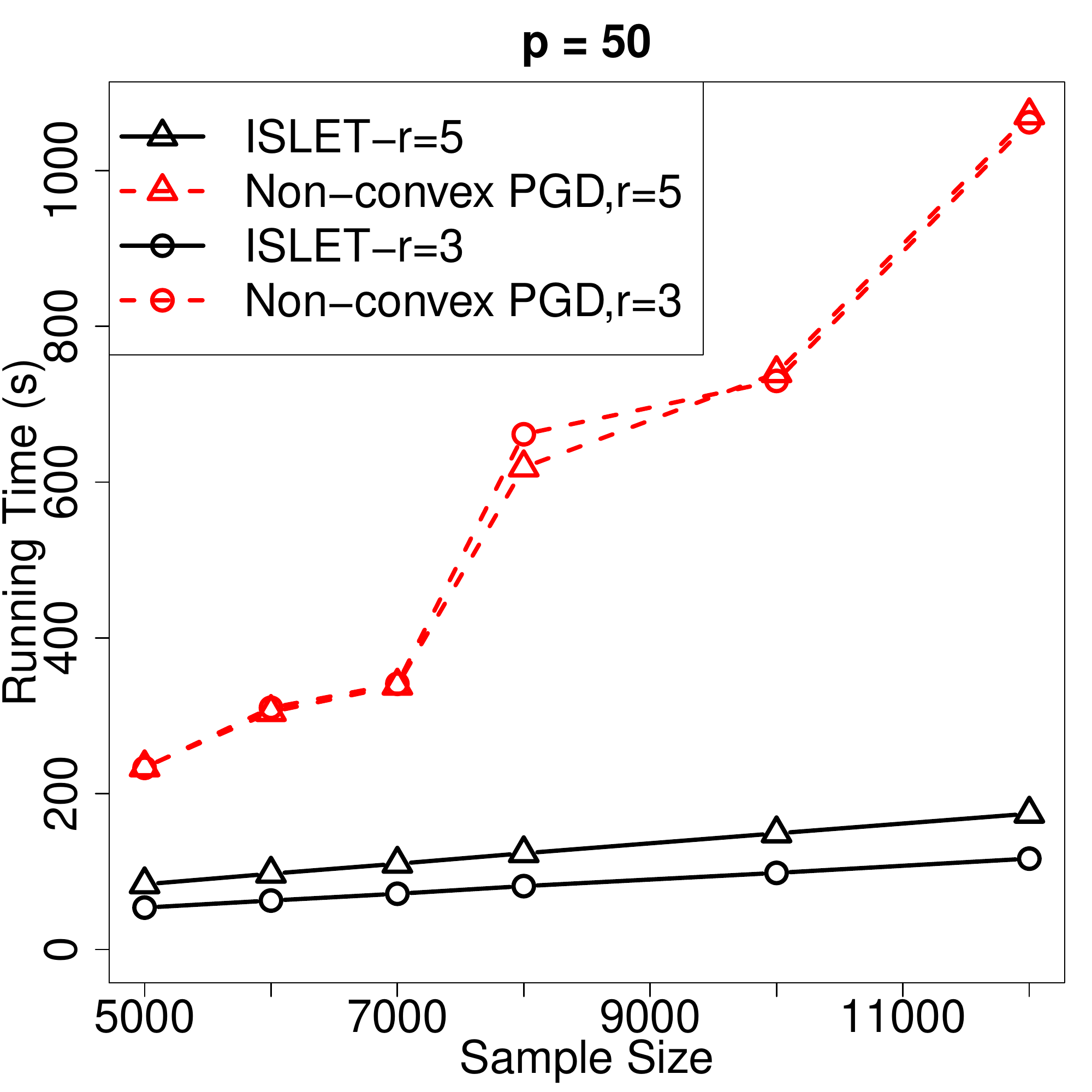}}
	\caption{ISLET vs. nonconvex PGD for sparse tensor regression}\label{fig: 12}
\end{figure}

Finally, if the tensor is of order $2$, tensor regression becomes the classic \emph{low-rank matrix recovery} problem \cite{candes2011tight,recht2010guaranteed}. Among existing approaches for low-rank matrix recovery, the nuclear norm minimization (NNM) has been proposed and extensively studied in recent literature. We compare the numerical performance of matrix ISLET (see Algorithm \ref{al:procedure_regular_matrix} in Section \ref{sec:general-order} for implementation details) and NNM that aims to solve \footnote{The optimization of NNM is implemented by accelerated proximal gradient method  \cite{toh2010accelerated} using the software package available online at \url{https://blog.nus.edu.sg/mattohkc/softwares/nnls/}.}
$$\sum_{i=1}^n (y_i - \langle \X_i, \A \rangle )^2 + \lambda ||\A||_\ast,$$ 
where $\|\A\|_\ast = \sum_i \sigma_i(\A)$ is the matrix nuclear norm. We consider two specific settings: (1) $p_1 = p_2 = 50$, $r = 2$, $\sigma = 10$, $n \in [2000, 16000]$; (2) $p_1 = p_2 = 100, r = 4, \sigma = 10, n \in [2000, 28000]$. From Figure \ref{fig: 5}, we find that ISLET has similar, or sometimes even better performance than NNM in estimation error. On the other hand, the run time of ISLET is negligibly small compared to NNM. 
\begin{figure}[h!]
	\centering
	\subfigure[RMSE, $p=50$]{
		\includegraphics[width=0.45\textwidth,height=2in]{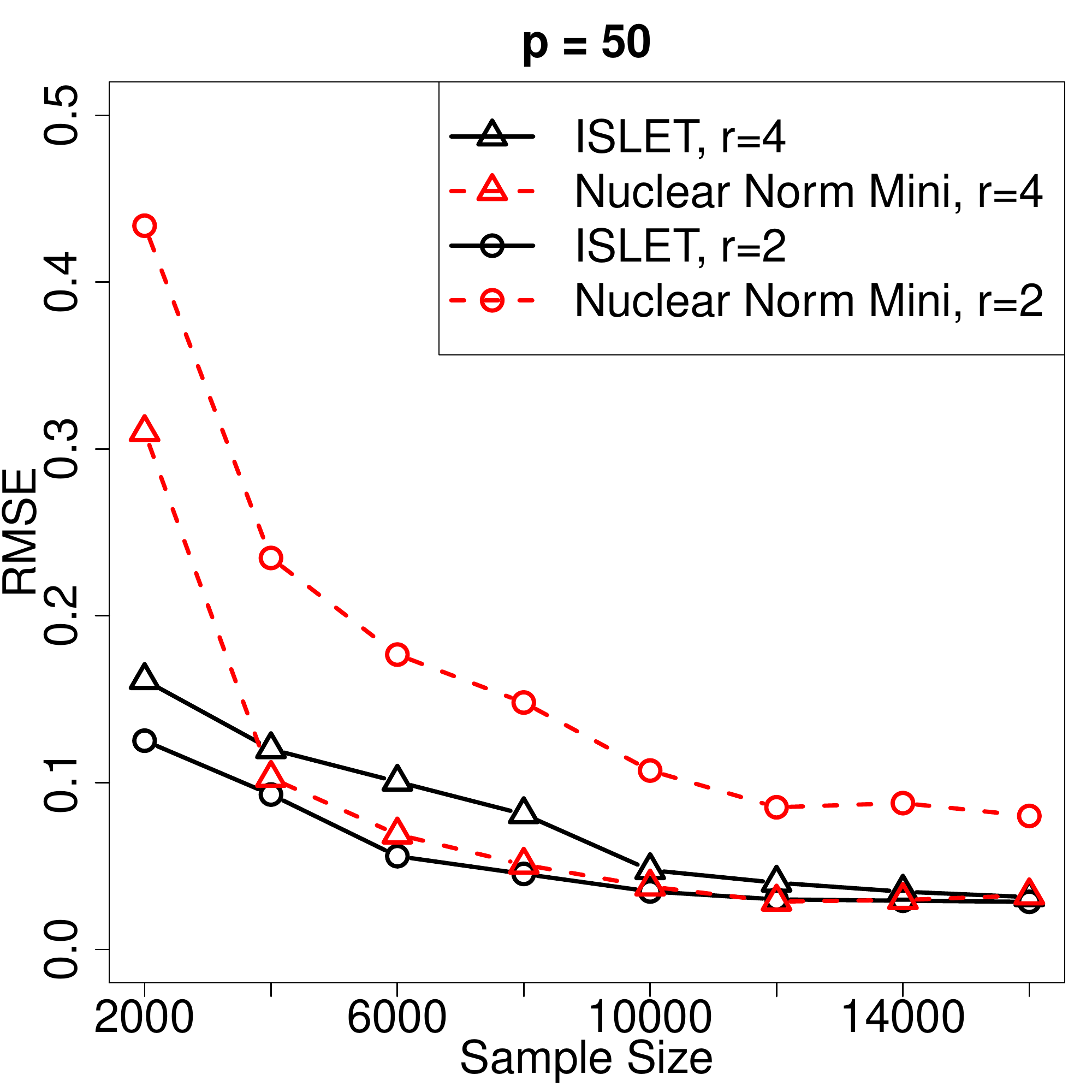}}
	\subfigure[Run Time, $p=50$]{
		\includegraphics[width=0.45\textwidth,height=2in]{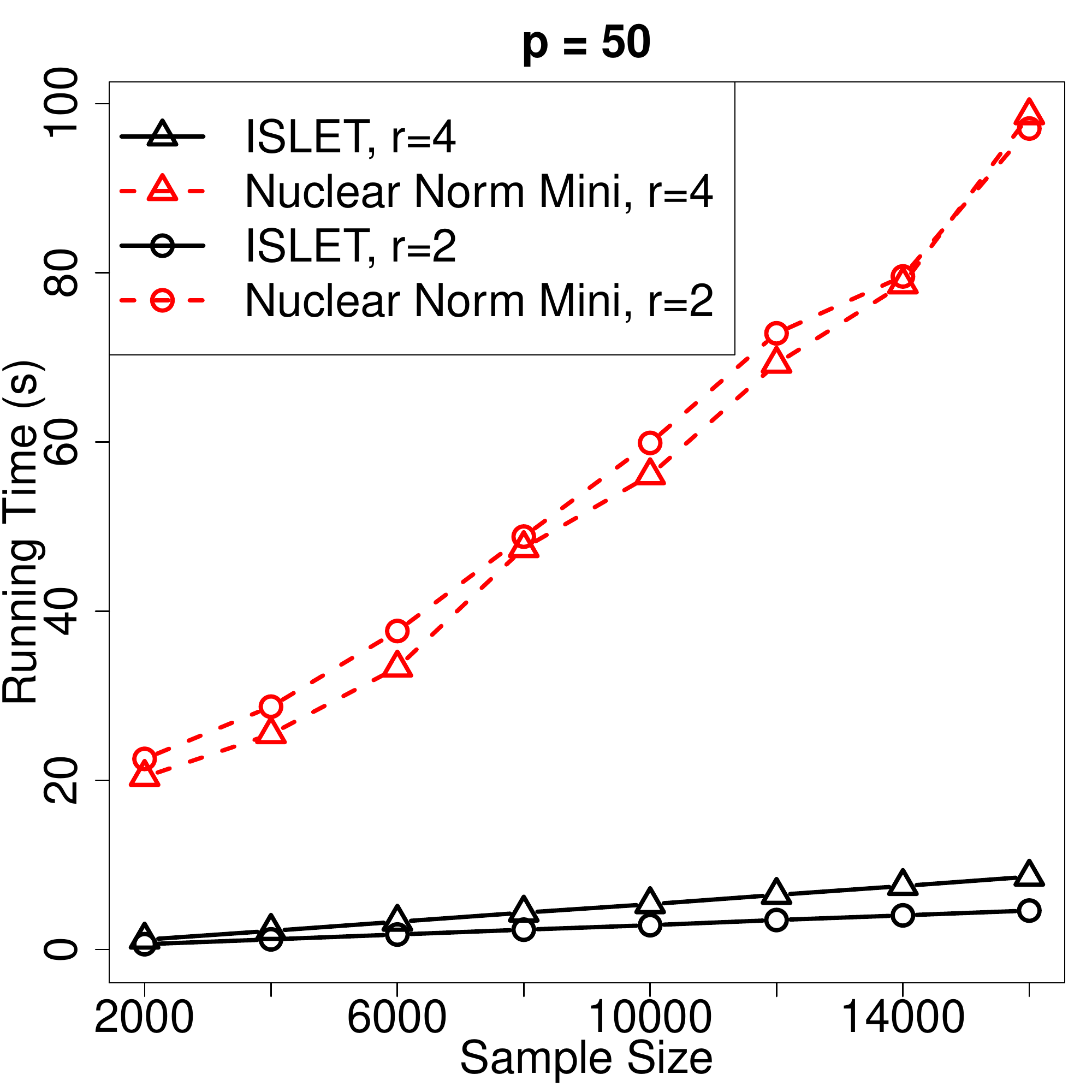}}
	\subfigure[RMSE, $p=100$]{
		\includegraphics[width=0.45\textwidth,height=2in]{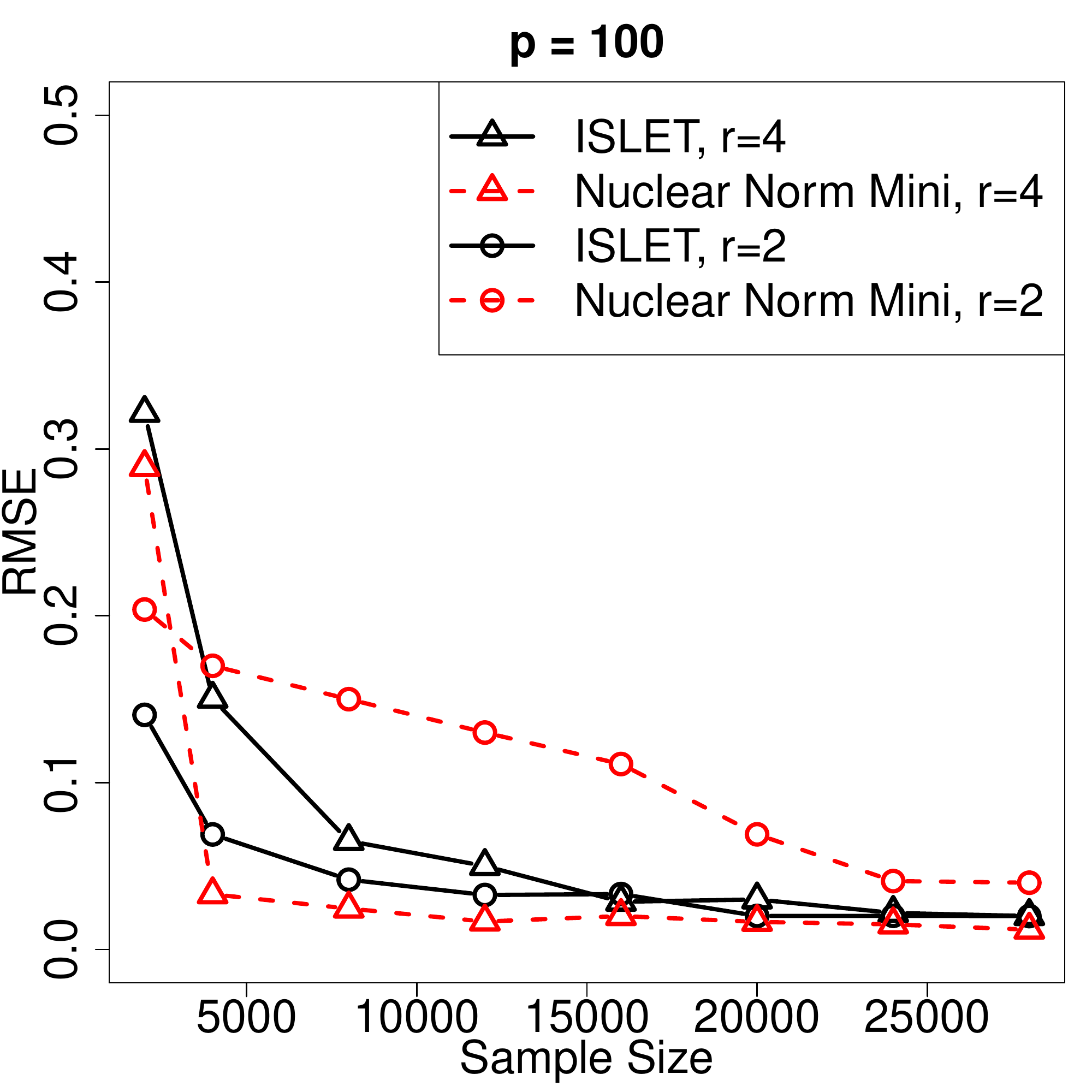}}
	\subfigure[Run Time, $p=100$]{
		\includegraphics[width=0.45\textwidth,height=2in]{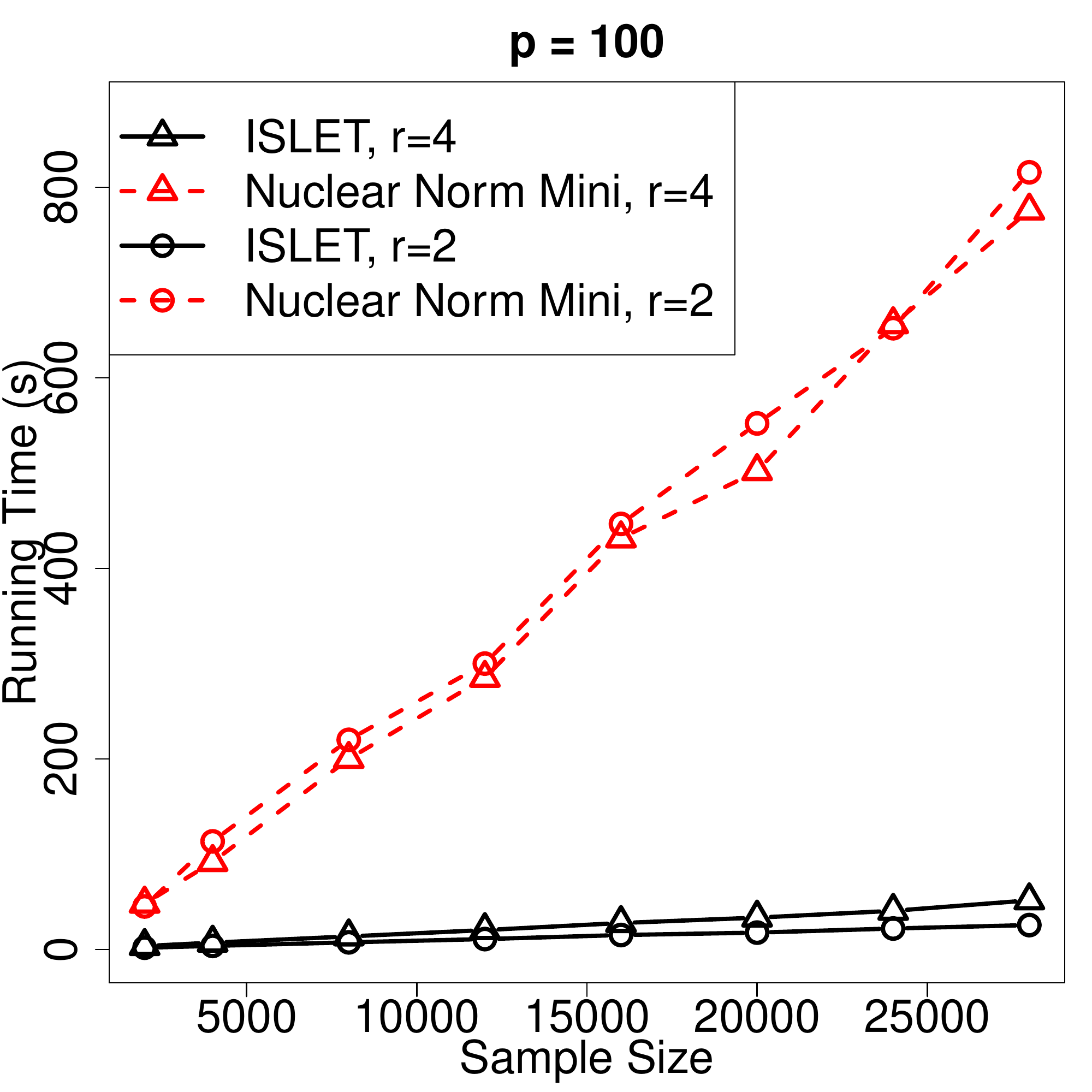}}
	\caption{ISLET vs. nuclear norm minimization for low-rank matrix recovery}\label{fig: 5}
\end{figure}

\section{Discussion}\label{sec:discussion}

In this article, we develop a general importance sketching algorithm for high-dimensional low-rank tensor regression. In particular, to sufficiently reduce the dimension of the higher-order structure, we propose a fast algorithm named \emph{\underline{i}mportance \underline{s}ketching \underline{l}ow-rank \underline{e}stimation for \underline{t}ensors} (ISLET). The proposed algorithm includes three major steps: we first apply tensor decomposition approaches, such as HOOI and STAT-SVD, to obtain importance sketching directions; then we perform regression using the sketched tensor/matrices (in the sparse case, we add group-sparsity regularizers); finally we assemble the final estimator. We establish deterministic oracle inequalities for the proposed procedure under general design and noise distributions. We also prove that ISLET achieve optimal mean-squared error rate under Gaussian ensemble design -- regular ISLET can further achieves the optimal constant for mean-squared error. As illustrated in simulation studies, the proposed procedure is computationally efficient comparing to contemporary methods. Although the presentation mainly focuses on order-3 tensors here, the method and theory for the general order-$d$ tensors can be elaborated similarly.

It is also noteworthy that the storage cost for Tucker decomposition in the proposed procedure grows exponentially with the order $d$. Thus, if the target tensor has a large order, it is more desirable to consider other low-rank approximation methods than Tucker, such as the CP decomposition \cite{bousse2017linear,bousse2018linear}, Hierarchical Tucker (HT) decomposition \cite{ballani2013projection,grasedyck2010hierarchical, hackbusch2009new}, and Tensor Train (TT) decomposition \cite{oseledets2011tensor, oseledets2009breaking}, etc. The ISLET framework can be adapted to these structures as long as there are two key components: there exists a sketching approach for dimension reduction and a computational inversion step for embedding the low-dimensional estimate back to the high-dimensional space (also see Section \ref{sec:sketching-perspective}). Whether these components hold for the previously described methods remains an interesting open question.

In addition to low-rank tensor regression, the idea of ISLET can be applied to various other high-dimensional problems. First, \emph{high-order interaction pursuit} is an important topic in high-dimensional statistics that aims at the interaction among three or more variables in the regression setting. This problem can be transformed to the tensor estimation based on a number of rank-1 projections by the argument in \cite{hao2018sparse}. Similarly to analysis on tensor regression in this paper, the idea of ISLET can be used to develop an optimal and efficient procedure for high-order interaction pursuit with provable advantages over other baseline methods.

In addition, \emph{matrix/tensor completion} has attracted significant attention in the recent literature \cite{candes2010power,liu2013tensor,xia2017polynomial,xia2017statistically,yuan2014tensor}. The central task of matrix/tensor completion is to complete the low-rank matrix/tensor based on a limited number of observable entries. Since each observable entry in matrix/tensor completion can be seen as a special rank-one projection of the original matrix/tensor, the idea behind ISLET can be used to achieve a more efficient algorithm in matrix/tensor completion with theoretical guarantees. It will be an interesting future topic to further investigate the performance of ISLET on other high-dimensional problems. 

\section*{Acknowledgment}
The authors would like to thank the editors and anonymous referees for the helpful suggestions that helped to improve the presentation of this paper.

\bibliographystyle{plain}

\newpage
	\appendix
	
	\setcounter{page}{1}
	\setcounter{section}{0}
	
	\begin{center}
		{\LARGE Supplement to ``ISLET: Fast and Optimal Low-rank Tensor}
		\medskip
		{\LARGE  Regression via Importance Sketching"	

		}		
		\bigskip\medskip
		
		{\large Anru Zhang, ~ Yuetian Luo, ~ Garvesh Raskutti, ~ and Ming Yuan}
	\end{center}
	
	\begin{abstract}
		In this supplement, we provide additional notation, preliminaries, ISLET procedure for general order tensor estimations, more details on tuning parameter selection, and all proofs for the main results of the paper.
	\end{abstract}

\section{Additional Notation and Preliminaries}\label{sec:row-permutation-operator}

To conveniently specify the dimensions of tensors, for an order-$d$ tensor $\bcA$ with dimensions $p_1\times \cdots \times p_d$, we denote $p_{-k} = p_1\cdots p_d/p_k$ for $k=1,\ldots, d$. Then the mode-$k$ matricization of $\bcA$, denoted as $\mathcal{M}_k(\bcA)$, has dimension $p_k\times p_{-k}$. For any matrix $\D\in \mathbb{R}^{p_1\times p_2}$ and order-$d$ tensor $\bcA$, we formally define the vectorization as
\begin{equation*}
\begin{split}
& \rmvec(\D) \in \mathbb{R}^{(p_1p_2)}, \quad \rmvec(\D)_{[i_1 + (i_2-1)p_1]} = \D_{[i_1,i_2]}; \\
& \rmvec(\bcA)\in \mathbb{R}^{(p_1\cdots p_d)}, \quad \rmvec(\bcA)_{[i_1+p_1(i_2-1)+\cdots + (i_d-1)p_1\cdots p_d]} = \bcA_{[i_1, \ldots, i_d]}.
\end{split}
\end{equation*}
For any tensor $\bcA\in \mathbb{R}^{p_1\times \cdots \times p_d}$, the Mode-$k$ matricization is formally defined as 
\begin{equation*}
    \mathcal{M}_k(\bcA) \in \mathbb{R}^{p_k\times p_{-k}}, \quad \bcA_{[i_1,\ldots, i_d]} = \left(\mathcal{M}_k(\bcA)\right)_{\left[i_k, j\right]}, \quad j = 1 + \sum_{\substack{l=1\\l\neq k}}^d\left\{(i_l-1)\prod_{\substack{m=1\\m\neq k}}^{l-1}p_m\right\}
\end{equation*}
for any $1\leq i_l \leq p_l, l=1,\ldots, d$. Also see \cite[Section 2.4]{kolda2009tensor} for more discussions on tensor matricizations.

In order to better illustrate the proposed procedure, we have introduced a row-permutation operator $\mathcal{R}_k$ that matches the index of $\W_k\otimes \V_k$ to $\rmvec(\bcA)$. In particular if $\bcA\in \mathbb{R}^{p_1\times p_2\times p_3}, \W_k \in \mathbb{R}^{p_{-k}\times r_k}, \V_k\in \mathbb{R}^{p_k\times r_k}$, $\mathcal{R}_k$ is defined as follows:
\begin{equation*}
\begin{split}
& \left(\mathcal{R}_1\left(\W_1 \otimes \V_1\right)\right)_{[i_1 + (i_2-1)p_1+(i_3-1)p_1p_2, :]} = \left(\W_1\otimes \V_1\right)_{[i_1 + (i_2-1)p_1+(i_3-1)p_1p_2,:]},\\
& \left(\mathcal{R}_2\left(\W_2 \otimes \V_2\right)\right)_{[i_1 + (i_2-1)p_1+(i_3-1)p_1p_2, :]} = \left(\W_2\otimes \V_2\right)_{[i_2+(i_1-1)p_2+(i_3-1)p_2p_1, :]},\\
& \left(\mathcal{R}_3\left(\W_3 \otimes \V_3\right)\right)_{[i_1 + (i_2-1)p_1+(i_3-1)p_1p_2, :]} = \left(\W_3\otimes \V_3\right)_{[i_3+(i_1-1)p_3+(i_2-1)p_1p_3, :]}
\end{split}
\end{equation*}
for $1\leq i_1\leq p_1, 1\leq i_2\leq p_2, 1\leq i_3\leq p_3$.

\section{ADHD MRI Imaging Data Analysis}\label{sec:real_data}

In this section, we display the value of our method on predicting attention deficit hyperactivity disorder (ADHD) with magnetic resonance imaging (MRI) dataset provided by Neuro Bureau\footnote{Link: \url{http://neurobureau.projects.nitrc.org/ADHD200/Data.html}}. The dataset involves 973 subjects, where each subject is associated with a $121$-by-$145$-by-$121$ MRI image and several demographic variables. After removing the missing values, we obtain 930 samples, among which 356 and 574 are diagnosed and control subjects, respectively. 

We aim to do prediction based on the association between the diagnosis label $y_i$ of $i^{th}$ observation and its covariates with MRI imaging $\bcX_i$, demographic variables age $x^1_{i}$, gender $x^2_{i}$, and handedness $x^3_{i}$. To better cope the job of predicting binary response $y_i$ and incorporate the demographic information in addition to tensor image covariates, we apply importance sketching, the central idea of ISLET, for dimension reduction. The 5-fold cross-validation is applied to examine the prediction power. Specifically for $l=1,\ldots, 50$, we randomly partition all 930 subjects into 5 uniform subsets $\{\Omega_j^{(l)}\}_{j=1,\ldots, 5} \subseteq \{1,\ldots, 930\}$. For $j=1,\ldots, 5$, we assign one fold $\Omega_j^{(l)}$ and the other four folds $\Omega_{-j}^{(l)} = \cup_{j'\neq j}\Omega_{j'}$ as the testing and training sets, respectively. We apply Step 1 of sparse ISLET (described in Section \ref{sec:sparse-procedure}) on $\{y_i, \bcX_i\}_{i\in \Omega_{(-j)}^{(l)}}$ to obtain $\widetilde{\U}_1, \widetilde{\U}_2, \widetilde{\U}_3$ and construct the importance sketching covariates $\widetilde{\x}_i = \rmvec(\bcX_i\times_1\widetilde{U}_1^\top,\times_1\widetilde{U}_1^\top,\times_1\widetilde{U}_1^\top)$, perform logistic regression for $y_{i}$ versus the combined covariates $\left[\widetilde{\x}_i, x^{1}_{i}, x^{2}_{i}, x^{3}_{i}\right]$, $i\in \Omega_{-j}^{(l)}$ and possible $\ell_1$ regularizer to get the estimates. Then we use estimates and $\left[\widetilde{\x}_i, x_{1}^{i}, x_{2}^{i}, x_{3}^{i}\right], i\in \Omega_{j}^{(l)}$ to predict the labels of samples in the testing set $\Omega_{j}^{(l)}$. For comparison, we also perform Tucker regression and Tucker regression with regularizer proposed by \cite{li2013tucker,zhou2013tensor} under the same setting. Since it is computationally intensive to perform full Tucker regression on complete tensor covariates of dimension $121\times 145 \times 121$, we follow the procedure described in \cite{li2013tucker,zhou2013tensor} and apply the discrete cosine transformation to downsize the MRI data to $12\times 14\times 12$ using the code available at the authors' website \cite{zhou2017matlab}. For all methods, we input Tucker rank $(r,r,r)$ for $r = 3,4,5$ and other regularization tuning parameters selected via cross validation. We repeat experiments for $l=1,\ldots, 50, j=1,\ldots, 5$ and take average to ensure stable estimations of the prediction accuracy for both procedures.

The average prediction accuracy with standard deviation in the parenthesis and runtime for both methods are shown in Table \ref{tab: ADHD realdata}. We can see the importance sketching method performs significantly better than Tucker regression in both the prediction accuracy and runtime for all different Tucker rank choices. Particularly for the importance sketching, adding $\ell_1$ regularizer provides more accurate prediction but costs more time. In addition, compared to the downsizing method by \cite{zhou2013tensor,li2013tucker} that deterministically relies on external information, our importance sketching is fully data-driven. We can also see downsizing the tensor covariates to 3-by-3-by-3 by importance sketching provides more prediction power than downsizing to 12-by-14-by-12 by deterministic methods. This reveals the runtime advantage and immediately demonstrates the advantage of the proposed method over other state-of-the-art approaches.

\begin{table}
\centering
\begin{tabular}{l c| c c c c}
\hline 
 & \multirow{2}{2em}{Rank}& \multicolumn{4}{c}{Methods} \\
 & & IS & IS  & Tucker Reg. & Tucker Reg.\\
  & & & + regularizer & & + regularizer\\
 \hline
Prediction & 3 & 0.684(0.010) & \textbf{0.686}(0.009) & 0.624(0.014) & 0.647(0.009)\\
Accuracy & 4 & 0.673(0.009) & \textbf{0.682}(0.008) & 0.609(0.014) & 0.648(0.007)\\
 & 5 & 0.653(0.009) & \textbf{0.674}(0.007) & 0.591(0.015) & 0.644(0.007)\\
 \hline
Runtime & 3 & \textbf{0.008} & 0.392 & 14.291 & 3.03 \\
Unit: & 4 & \textbf{0.024} & 1.003 & 22.088 & 5.761\\
seconds & 5 & \textbf{0.064} & 3.339 & 33.392 & 13.710\\
 \hline
\end{tabular}
\caption{Importance sketching (IS) vs. Tucker regression in prediction accuracy and runtime}
\label{tab: ADHD realdata}
\end{table}

\section{ISLET for General Order Tensor Estimation}\label{sec:general-order}

For completeness, we provide the ISLET procedure for general order-$d$ low-rank tensor estimation in this section. The procedure for $d\geq 3$ is provided in Algorithms \ref{al:procedure_regular_general_order} and the one for $d = 2$ (i.e., the low-rank matrix estimation) is provided in Algorithm \ref{al:procedure_regular_matrix}. The sparse versions for $d\geq 3$ and $d=2$ are provided in Algorithms \ref{al:procedure_sparse_general_order} and \ref{al:procedure_sparse_matrix}, respectively.
\begin{algorithm}[htbp]
	\caption{Order-$d$ ISLET ($d\geq 3$)}
	\begin{algorithmic}[1]
		\State Input: $y_1,\ldots, y_n\in \mathbb{R}, \bcX_1,\ldots, \bcX_n \in \mathbb{R}^{p_1\times \cdots \times p_d}$, rank $\br = (r_1,\ldots, r_d)$.
		\State Evaluate $\widetilde{\bcA} = \frac{1}{n}\sum_{j=1}^{n} y_j \bcX_j.$
		\State Apply order-$d$ HOOI on $\widetilde{\bcA}$ to obtain initial estimates $\widetilde{\U}_k, k =1,\ldots, d$.
		\State Let $\widetilde{\bcS} = \llbracket\widetilde{\bcA}; \widetilde{\U}_1^\top, \ldots, \widetilde{\U}_d^\top\rrbracket$. Evaluate the sketching directions,
		$$\widetilde{\V}_k = {\rm QR}\left[\mathcal{M}_k(\widetilde{\bcS})^\top\right], \quad  k=1,\ldots, d.$$
		\State Construct $\widetilde{\X} = \left[\widetilde{\X}_\bcB ~ \widetilde{\X}_{\D_1} ~ \cdots ~ \widetilde{\X}_{\D_d}\right] \in \mathbb{R}^{n\times m}$, where
		\begin{equation*}
		\begin{split}
		& \widetilde{\X}_\bcB \in \mathbb{R}^{n\times m_{\bcB}},\quad (\widetilde{\X}_\bcB)_{[i,:]} = 
		\rmvec\left(\bcX_i \times_{l=1}^d \widetilde{\U}_l^\top\right),\\
		& \widetilde{\X}_{\D_k} \in \mathbb{R}^{n\times m_{\D_k}}, \quad (\widetilde{\X}_{\D_k})_{[i,:]} = \rmvec\left(\widetilde{\U}_{k\perp}^\top \mathcal{M}_k\left(\bcX_{i} \times_{\substack{l=1\\l\neq k}}^d\widetilde{\U}_{l}^\top\right) \widetilde{\V}_k\right)
		\end{split}
		\end{equation*}
		for $m_{\bcB} = r_1\cdots r_d, m_{\D_k} = (p_k-r_k)r_k$, $k=1,\ldots, d$, and $m = m_{\bcB} + m_{\D_1}+\cdots +m_{\D_d}$.
		\item Solve $\widehat{\bgamma} = \argmin_{\bgamma\in\mathbb{R}^m}\|y - \widetilde{\X}\bgamma\|_2^2$. Partition $\widehat{\bgamma}$ to $\widehat{\bcB}, \widehat{\D}_1, \ldots, \widehat{\D}_d$,
		\begin{equation*}
		\begin{split}
		& \rmvec(\widehat{\bcB}) := \widehat{\bgamma}_\bcB = \widehat{\bgamma}_{[1:m_{\bcB}]},\\
		& \rmvec(\widehat{\D}_k) := \widehat{\bgamma}_{\D_k} = \widehat{\bgamma}_{\left[\left(m_{\bcB}+\sum_{k'=1}^{k-1} m_{\D_{k'}} +1\right): \left(m_{\bcB}+\sum_{k'=1}^k m_{\D_{k'}}\right)\right]}, \quad k=1,\ldots, d.
		\end{split}
		\end{equation*}
		\State Let $\widehat{\B}_k = \mathcal{M}_k(\widehat{\bcB})$, evaluate
		$$\widehat{\bcA} = \llbracket\widehat{\bcB}; \widehat{\L}_1, \ldots, \widehat{\L}_d\rrbracket, \quad \widehat{\L}_k = \left(\widetilde{\U}_k\widehat{\B}_k\widetilde{\V}_k + \widetilde{\U}_{k\perp}\widehat{\D}_k\right)\left(\widehat{\B}_k\widetilde{\V}_k\right)^{-1}, \quad k=1,\ldots, d.$$
	\end{algorithmic}\label{al:procedure_regular_general_order}
\end{algorithm}
\begin{algorithm}[htbp]
	\caption{Matrix ISLET}
	\begin{algorithmic}[1]
		\State Input: $y_1,\ldots, y_n\in \mathbb{R}, \X_1,\ldots, \X_n \in \mathbb{R}^{p_1\times p_2}$, rank $r$.
		\State Evaluate $\widetilde{\A} = \frac{1}{n}\sum_{j=1}^{n} y_j \X_j.$ and let $\widetilde{\U}_1 = \SVD_r(\widetilde{\A}), \widetilde{\U}_2 = \SVD_r(\widetilde{\A}^\top)$.
		\State Construct $\widetilde{\X} = \left[\widetilde{\X}_\B ~ \widetilde{\X}_{\D_1} \widetilde{\X}_{\D_2}\right] \in \mathbb{R}^{n\times r(p_1+p_2-r)}$, where
		\begin{equation*}
		\begin{split}
		& \widetilde{\X}_\B \in \mathbb{R}^{n\times r^2},\quad (\widetilde{\X}_\B)_{[i,:]} = 
		\rmvec\left(\widetilde{\U}_1^\top \X_i \widetilde{\U}_2\right),\\
		\widetilde{\X}_{\D_k} \in \mathbb{R}^{n\times (p_k-r)r}, &\quad (\widetilde{\X}_{\D_1})_{[i,:]} = \rmvec\left(\widetilde{\U}_{1\perp}^\top \X_{i} \widetilde{\U}_2\right),\quad (\widetilde{\X}_{\D_2})_{[i,:]} = \rmvec\left(\widetilde{\U}_{2\perp}^\top \X_{i}^\top \widetilde{\U}_1\right).
		\end{split}
		\end{equation*}
		\item Solve $\widehat{\bgamma} = \argmin_{\bgamma\in\mathbb{R}^m}\|y - \widetilde{\X}\bgamma\|_2^2$. Partition $\widehat{\bgamma}$ and assign to $\widehat{\B}, \widehat{\D}_1, \widehat{\D}_2$,
		\begin{equation*}
		\begin{split}
		\rmvec(\widehat{\B}) := \widehat{\bgamma}_{[1:r^2]},\quad
		\rmvec(\widehat{\D}_1) := \widehat{\bgamma}_{\left[\left(r^2+1\right): rp_1\right]}, \quad  \rmvec(\widehat{\D}_2) := \widehat{\bgamma}_{\left[\left(rp_1 + 1\right): \left(r(p_1+p_2-r)\right)\right]}.
		\end{split}
		\end{equation*}
		\State Evaluate
		$$\widehat{\bcA} = \widehat{\L}_1 \widehat{\B}\widehat{\L}_2^\top, \quad \widehat{\L}_1 = \left(\widetilde{\U}_1\widehat{\B} + \widetilde{\U}_{1\perp}\widehat{\D}_1\right)\widehat{\B}^{-1}, \quad \widehat{\L}_2 = \left(\widetilde{\U}_2\widehat{\B}^\top + \widetilde{\U}_{2\perp}\widehat{\D}_2\right)\left(\widehat{\B}^\top\right)^{-1}.$$
	\end{algorithmic}\label{al:procedure_regular_matrix}
\end{algorithm}

\begin{algorithm}[htbp]
	\caption{Order-$d$ Sparse ISLET}
	\begin{algorithmic}[1]
		\State Input: $y_1,\ldots, y_n \in\mathbb{R}$, $\bcX_1,\ldots, \bcX_n\in \mathbb{R}^{p_1\times \cdots \times p_d}$, rank $\br = (r_1, r_2, \ldots, r_d)$, sparsity index $J_s$.
		\State Evaluate $\widetilde{\bcA} = \frac{1}{n}\sum_{j=1}^{n} y_j \bcX_j.$
		\State Apply STAT-SVD on $\widetilde{\bcA}$ with sparsity index $J_s$. Let the outcome be $\widetilde{\U}_1, \widetilde{\U}_2, \widetilde{\U}_3, \ldots ,\widetilde{\U}_d $. 
		\State Let $\widetilde{\bcS} = \llbracket\widetilde{\bcA}; \widetilde{\U}_1^\top, \ldots, \widetilde{\U}_d^\top\rrbracket$ and evaluate the probing directions $\widetilde{\V}_k = {\rm QR}\left[\mathcal{M}_k(\widetilde{\bcS})^\top\right], k=1,\ldots, d.$
		\State Construct
		\begin{equation*}
		\begin{split}
		& \widetilde{\X}_\bcB \in \mathbb{R}^{n\times (r_1\cdots r_d)},\quad (\widetilde{\X}_\bcB)_{[i,:]} = 
		\rmvec\left(\bcX_i \times_{l=1}^d \widetilde{\U}_l^\top\right),\\
		& \widetilde{\X}_{\E_k} \in \mathbb{R}^{n\times (p_kr_k)}, \quad (\widetilde{\X}_{\E_k})_{[i,:]} = \rmvec\left(\mathcal{M}_k\left(\bcX_i\times_{\substack{l=1\\l\neq k}}^{d}\widetilde{\U}_{l}^\top\right)\widetilde{\V}_k\right), \quad k=1,\ldots, d.
		\end{split}
		\end{equation*}
		\State Solve
		\begin{equation*}
		\begin{split}
		& \widehat{\bcB}\in \mathbb{R}^{r_1\cdots r_d}, \quad \rmvec(\widehat{\bcB}) = \argmin_{\bgamma\in\mathbb{R}^{r_1\cdots r_d}} \|y - \widetilde{\X}_\bcB \bgamma\|_2^2;\\
		\widehat{\E}_k \in \mathbb{R}^{p_k\times r_k}, \quad & \rmvec(\widehat{\E}_k) =\left\{\begin{array}{ll}
		\argmin_{\bgamma\in\mathbb{R}^{p_kr_k}} \|y - \widetilde{\X}_{\E_k}\bgamma\|_2^2 + \lambda_k \sum_{j=1}^{p_k}\|\bgamma_{G_j^{k}}\|_2, & k \in J_s;\\
		\argmin_{\bgamma\in\mathbb{R}^{p_kr_k}} \|y - \widetilde{\X}_{\E_k}\bgamma\|_2^2, & k\notin J_s.
		\end{array} \right.
		\end{split}
		\end{equation*}
		\State Evaluate
		\begin{equation*}
		\widehat{\bcA} = \llbracket\widehat{\bcB}; (\widehat{\E}_1(\widetilde{\U}_1^\top \widehat{\E}_1)^{-1}), \ldots,  (\widehat{\E}_d(\widetilde{\U}_d^\top \widehat{\E}_d)^{-1})\rrbracket
		\end{equation*}
	\end{algorithmic}\label{al:procedure_sparse_general_order}
\end{algorithm}

\begin{algorithm}[htbp]
	\caption{Matrix Sparse ISLET}
	\begin{algorithmic}[1]
		\State Input: $y_1,\ldots, y_n \in\mathbb{R}$, $\X_1,\ldots, \X_n\in \mathbb{R}^{p_1 \times p_2}$, rank $r$, sparsity index $J_s \subseteq \{1,2\}$.
		\State Evaluate $\widetilde{\A} = \frac{1}{n_1}\sum_{j=1}^{n} y_j \X_j.$ Apply sparse matrix SVD (the Two-Way Iterative Thresholding in \cite{yang2014sparse} or the order-2 version of STAT-SVD in \cite{zhang2017optimal-statsvd}) on $\widetilde{\A}$ with sparsity index $J_s$. Let the estimated left and right subspaces be $\widetilde{\U}_1, \widetilde{\U}_2$.
		\State Construct
		\begin{equation*}
		\begin{split}
		& \widetilde{\X}_\B \in \mathbb{R}^{n\times (r^2)},\quad (\widetilde{\X}_\B)_{[i,:]} = 
		\rmvec(\widetilde{\U}_1^\top\X_i \widetilde{\U}_2),\\
		& \widetilde{\X}_{\E_k} \in \mathbb{R}^{n\times (p_k r)}, \quad (\widetilde{\X}_{\E_1})_{[i,:]} = \rmvec\left(\X_i\widetilde{\U}_{2}\right), \quad (\widetilde{\X}_{\E_2})_{[i,:]} = \rmvec\left(\widetilde{\U}_1^\top \X_i\right).
		\end{split}
		\end{equation*}
		\State Solve $ \widehat{\B}\in \mathbb{R}^{r\times r}, \rmvec(\widehat{\B}) = \argmin_{\bgamma\in\mathbb{R}^{r^2}} \|y - \widetilde{\X}_\B \bgamma\|_2^2$;
		\begin{equation*}
		\begin{split}
		\widehat{\E}_k \in \mathbb{R}^{p_k\times r_k}, \quad \rmvec(\widehat{\E}_k) =\left\{\begin{array}{ll}
		\argmin_{\bgamma\in\mathbb{R}^{p_k r}} \|y - \widetilde{\X}_{\E_k}\bgamma\|_2^2 + \lambda_k \sum_{j=1}^{p_k}\|\bgamma_{G_j^{k}}\|_2, & k \in J_s;\\
		\argmin_{\bgamma\in\mathbb{R}^{p_k r}} \|y - \widetilde{\X}_{\E_k}\bgamma\|_2^2, & k\notin J_s.
		\end{array} \right.
		\end{split}
		\end{equation*}
		\State Evaluate
		\begin{equation*}
		\widehat{\A} =\widehat{\E}_1(\widetilde{\U}_1^\top \widehat{\E}_1)^{-1}\widehat{\B}(\widetilde{\U}_2^\top \widehat{\E}_2)^{-\top}\widehat{\E}_2^\top.
		\end{equation*}
	\end{algorithmic}\label{al:procedure_sparse_matrix}
\end{algorithm}

\section{More Details on Tuning Parameter Selection}\label{sec:tuning}

The implementation of ISLET requires the rank $\br$ as inputs. When $\br$ is unknown in practice, we propose a two-stage-scheme for adaptive low-rank tensor regression. First, we input a conservatively large value of $\br_{ini}$ into ISLET to obtain $\widehat{\bcB}, \widehat{\D}_k$ (regular case) or $\widehat{\bcB}, \widehat{\E}_k$ (sparse case), based on which we estimate the rank $\widehat{\br}$ by the ``Cross scheme" introduced recently by \cite{zhang2019cross}. Then, we run ISLET again with $\widehat{\br}$ to obtain the final estimates. The pseudo-codes for regular and sparse order-$d$ tensor regression are provided in Algorithms \ref{al:procedure_regular_general_order_unknown_r} and \ref{al:procedure_sparse_general_order_unknown_r}, respectively. 

\begin{algorithm}[htbp]
	\caption{Order-$d$ ISLET, unknown $r$}
	\begin{algorithmic}[1]
		\State Input: $y_1,\ldots, y_n\in \mathbb{R}, \bcX_1,\ldots, \bcX_n \in \mathbb{R}^{p_1\times \cdots \times p_d}$, rank $\br_{ini} = (r_{1,ini},\ldots, r_{d,ini})$.
        \State Apply Algorithms \ref{al:procedure_regular}, \ref{al:procedure_regular_general_order}, \ref{al:procedure_regular_matrix} with rank $\br_{ini}$ to obtain $\widetilde{\U}_k, \widetilde{\V}_k$, $\widehat{\bcB}$, and $\widehat{\D}_k$ for $k=1,\ldots, d$. 
        \State Denote $\widehat\B_k = \mathcal{M}_k(\widehat\bcB)$. Evaluate $\U_k^{(B)}$ and $\V_k^{(A)}$ via SVDs. Then rotate,
        $$\U_k^{(B)} \in \mathbb{O}_{r_{k,ini}}, \text{ as the left singular vectors of }\widehat{\B}_k,$$
        $$\V_k^{(A)} \in \mathbb{O}_{r_{k,ini}}, \text{ as the right singular vectors of }\left(\widetilde{\U}_k\widehat{\B}_k\widetilde{\V}_k + \widetilde{\U}_{k\perp}\widehat{\D}_k\right);$$
		$$\A_k = \left(\widetilde{\U}_k\widehat{\B}_k\widetilde{\V}_k + \widetilde{\U}_{k\perp}\widehat{\D}_k\right)\V_k^{(A)}\in\mathbb{R}^{p_k\times r_{k, ini}},$$ $$\boldsymbol{J}_k = (\U_k^{(B)})^\top \cdot \left(\widehat\B_k\widetilde\V_k\right) \cdot \V_k^{(A)} \in \mathbb{R}^{r_{k,ini}\times r_{k, ini}}.$$
		\For {$k = 1, \ldots, d$}
		\For {$s = r_{k, ini} : -1: 1$}
		\State {\bf if} $\boldsymbol{J}_{k, [1:s, 1:s]}$ is not singular and $\|\A_{k, [:, 1:s]} \boldsymbol{J}_{k, [1:s, 1:s]}^{-1}\| \leq 3$ {\bf then}
		\State \quad\quad $\widehat r_k = s$; {\bf break} from the loop;
		\State {\bf end if}
		\EndFor		
		\State{\bf If} $\widehat{r}_k$ is still unassigned {\bf then} $\widehat{r}_k = 0$.
		\EndFor
		\State Apply Algorithm \ref{al:procedure_regular} again with rank $\widehat{\br} = (\widehat{r}_1,\ldots, \widehat{r}_d)$. Let the final output be $\widehat{\bcA}$.
	\end{algorithmic}\label{al:procedure_regular_general_order_unknown_r}
\end{algorithm}
\begin{algorithm}[htbp]
	\caption{Order-$d$ Sparse ISLET, unknown $r$}
	\begin{algorithmic}[1]
		\State Input: $y_1,\ldots, y_n\in \mathbb{R}, \bcX_1,\ldots, \bcX_n \in \mathbb{R}^{p_1\times \cdots \times p_d}$, rank $\br_{ini}$, sparsity index $J_s$.
        \State Apply Algorithms \ref{al:procedure_sparse}, \ref{al:procedure_sparse_general_order}, or \ref{al:procedure_sparse_matrix} with rank $\br_{ini}$ to obtain $\widetilde{\U}_k, \widetilde{\V}_k$, $\widehat{\bcB}$, and $\widehat{\E}_k$ for $k=1,\ldots, d$. 
        \State Denote $\widehat\B_k = \mathcal{M}_k(\widehat\bcB)$. Evaluate $\U_k^{(B)}$ and $\V_k^{(A)}$ via SVDs, then rotate,
        $$\U_k^{(B)} \in \mathbb{O}_{r_{k,ini}}, \text{ as the left singular vectors of }\widehat{\B}_k,$$
        $$\V_k^{(A)} \in \mathbb{O}_{r_{k,ini}}, \text{ as the right singular vectors of }\widehat{\E}_k;$$
		$$\A_k = \widehat{\E}_k\V_k^{(A)}\in\mathbb{R}^{p_k\times r_{k, ini}},\quad \boldsymbol{J}_k = (\U_k^{(B)})^\top \cdot \left(\widehat\B_k\widetilde\V_k\right) \cdot \V_k^{(A)} \in \mathbb{R}^{r_{k,ini}\times r_{k, ini}}.$$
		\For {$k = 1, \ldots, d$}
		\For {$s = r_{k, ini} : -1: 1$}
		\State {\bf if} $\boldsymbol{J}_{k, [1:s, 1:s]}$ is not singular and $\|\A_{k, [:, 1:s]} \boldsymbol{J}_{k, [1:s, 1:s]}^{-1}\| \leq 3$ {\bf then}
		\State \quad\quad $\widehat r_k = s$; {\bf break} from the loop;
		\State {\bf end if}
		\EndFor		
		\State{\bf If} $\widehat{r}_k$ is still unassigned {\bf then} $\widehat{r}_k = 0$.
		\EndFor
		\State Apply Algorithm \ref{al:procedure_sparse} again with rank $\widehat{\br} = (\widehat{r}_1,\ldots, \widehat{r}_d)$. Let the final output be $\widehat{\bcA}$.
	\end{algorithmic}\label{al:procedure_sparse_general_order_unknown_r}
\end{algorithm}

Next, we perform simulation studies to verify the proposed rank selection scheme in both the regular and sparse cases. In particular, let $p = 20, 30$, $\br_{ini} = \lfloor \bp/3 \rfloor$, $n \in[2000, 5000]$, $\sigma = 5$, $s= 12$, and the actual rank $r = 3, 5$. We randomly generate the regular and sparse regression settings as described in Section \ref{sec:numerical}, then perform Algorithms \ref{al:procedure_regular_general_order_unknown_r} and \ref{al:procedure_sparse_general_order_unknown_r}. The average estimation error results are plots in Figures \ref{fig:unknown-rank-regular} and \ref{fig:unknown-rank-sparse} respectively for the regular and sparse cases. We can see from both cases that the estimation errors with known rank are close to the one without known rank and the difference decreases when the sample size gets larger.

\begin{figure}[htbp]
	\centering
	\subfigure[r = 3]{
		\includegraphics[width=0.45\textwidth,height=2in]{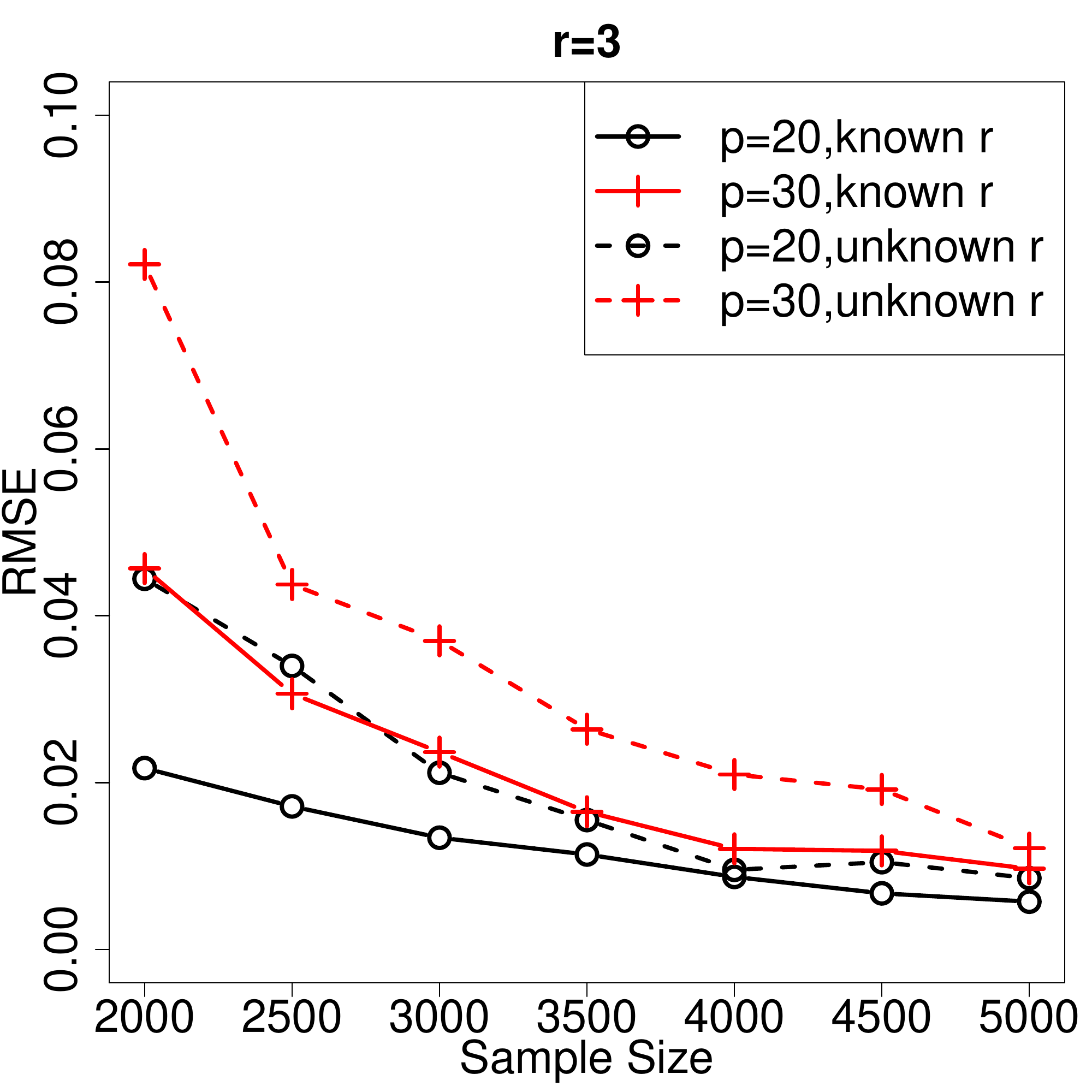}}
	\subfigure[r = 5]{
		\includegraphics[width=0.45\textwidth,height=2in]{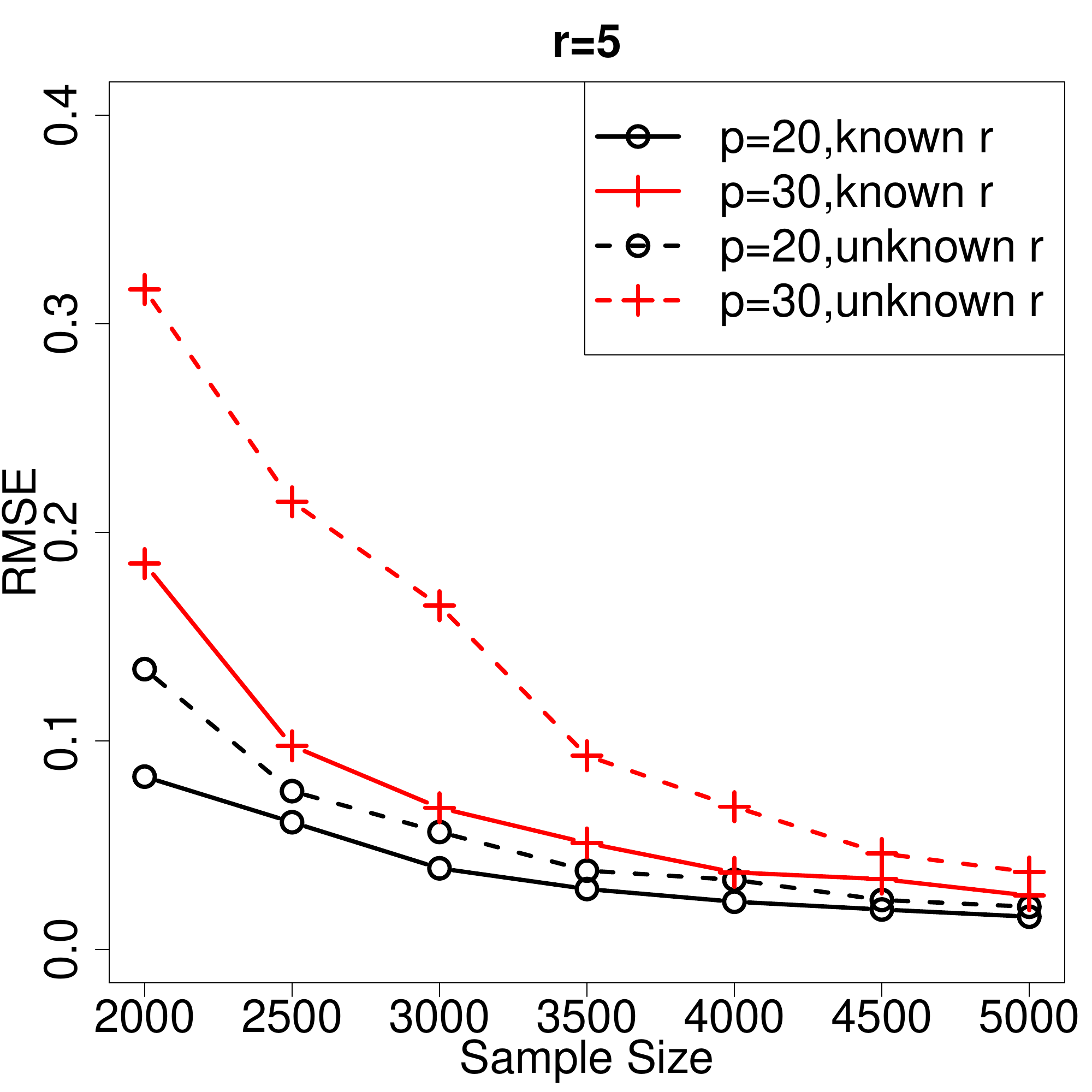}}
	\caption{ISLET: known rank vs unknown rank. Here, $\sigma = 5$, $\br_{ini} = \lfloor \bp/3 \rfloor$.}\label{fig:unknown-rank-regular}
\end{figure}

\begin{figure}[h!]
	\centering
	\subfigure[r = 3]{
		\includegraphics[width=0.45\textwidth,height=2in]{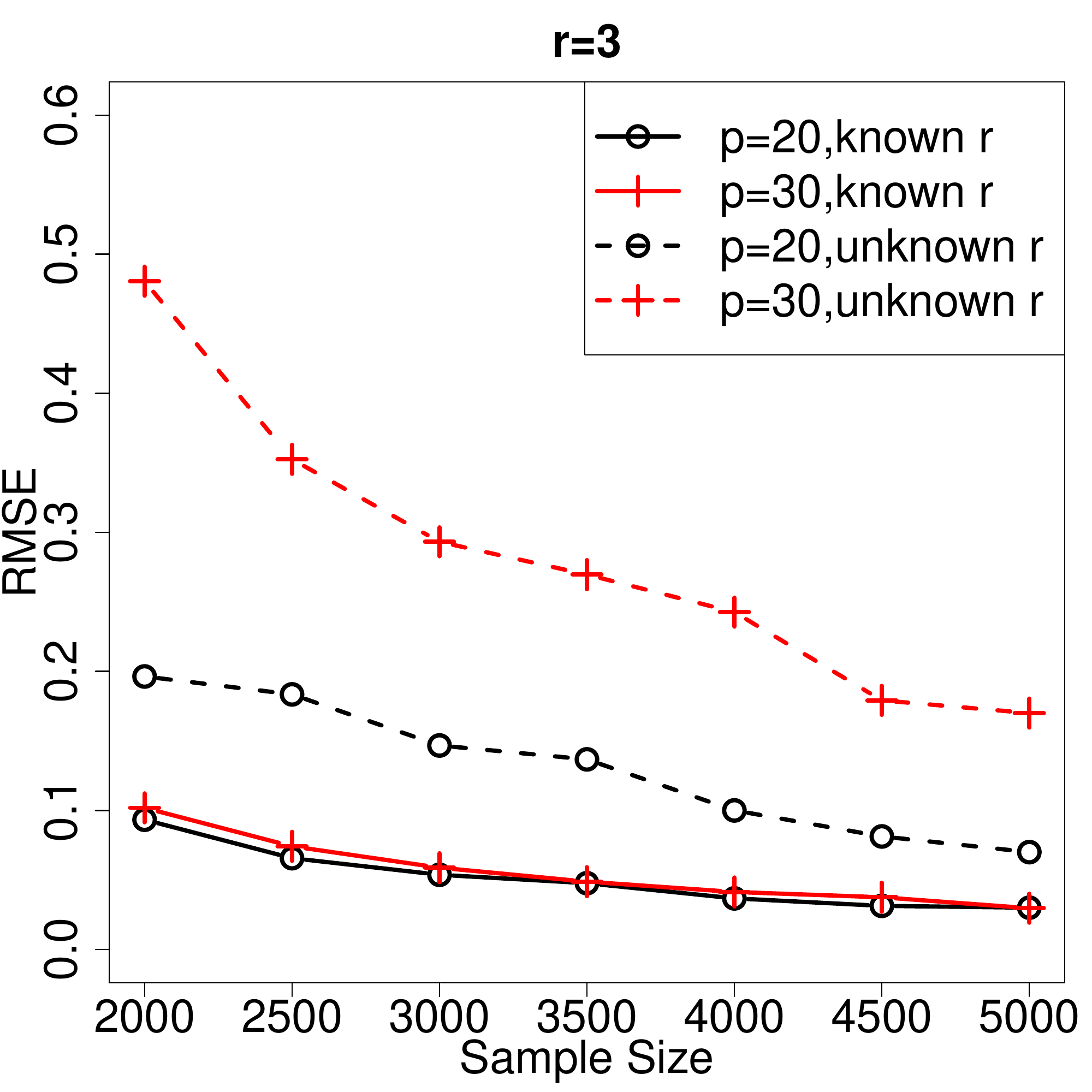}}
	\subfigure[r = 5]{
		\includegraphics[width=0.45\textwidth,height=2in]{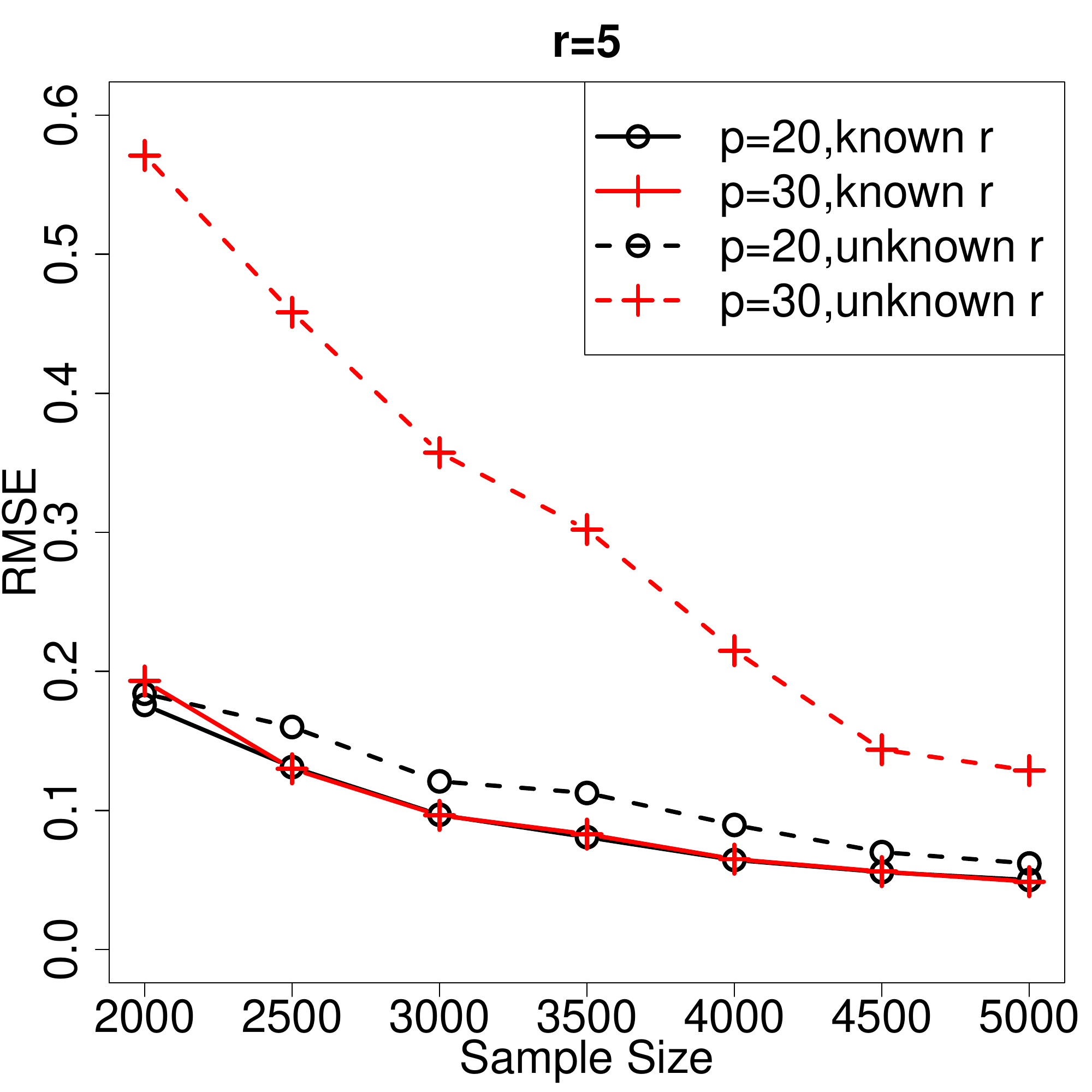}}
	\caption{Sparse ISLET: known rank vs unknown rank. Here, $\sigma = 5$; $\br_{ini} = \lfloor \bp/3 \rfloor$, $s = 12$}\label{fig:unknown-rank-sparse}
\end{figure}

\section{Simulation Study on Approximate Low-rank Tensor Regression}\label{sec:additional-simu}
We provide simulation results on the performance of ISLET when the parameter $\bcA$ is approximately low rank. Specifically, we first simulate the exact low Tucker rank tensor $\bcA_0$ in the same way as the one in previous settings and simulate $\bcZ$ as the perturbation tensor with i.i.d. standard normal entries. Then we set $\bcA = \bcA + \frac{\tau \|\bcA\|_F \bcZ}{p^3}$. The response $y_j$ and covariate $\bcX_j$ are generated the same to previous settings. Let $\sigma = 5, p = 20, n = [2000,8000], s_1 = s_2 = s_3 = 12, \tau = 0,0.1,0.3,0.5$. $\tau$ here characterizes how close $\bcA$ is to the exact low-rank tensor -- $\bcA$ is exact low rank if $\tau = 0$. We apply ISLET in both the regular and sparse regimes with the tuning parameter selection scheme described in Algorithms \ref{al:procedure_regular_general_order_unknown_r} and \ref{al:procedure_sparse_general_order_unknown_r}.
The results are collected in the Figure \ref{fig:approx-low-rank}. We can see that the estimation error decreases as $\tau$ decreases or $n$ increases; generally speaking, ISLET achieve good performance under both the regular and sparse regime when the true parameter $\bcA$ is only approximately low rank. 

\begin{figure}[htbp]
	\centering
	\subfigure[Regular ISLET]{
		\includegraphics[width=0.45\textwidth,height=2in]{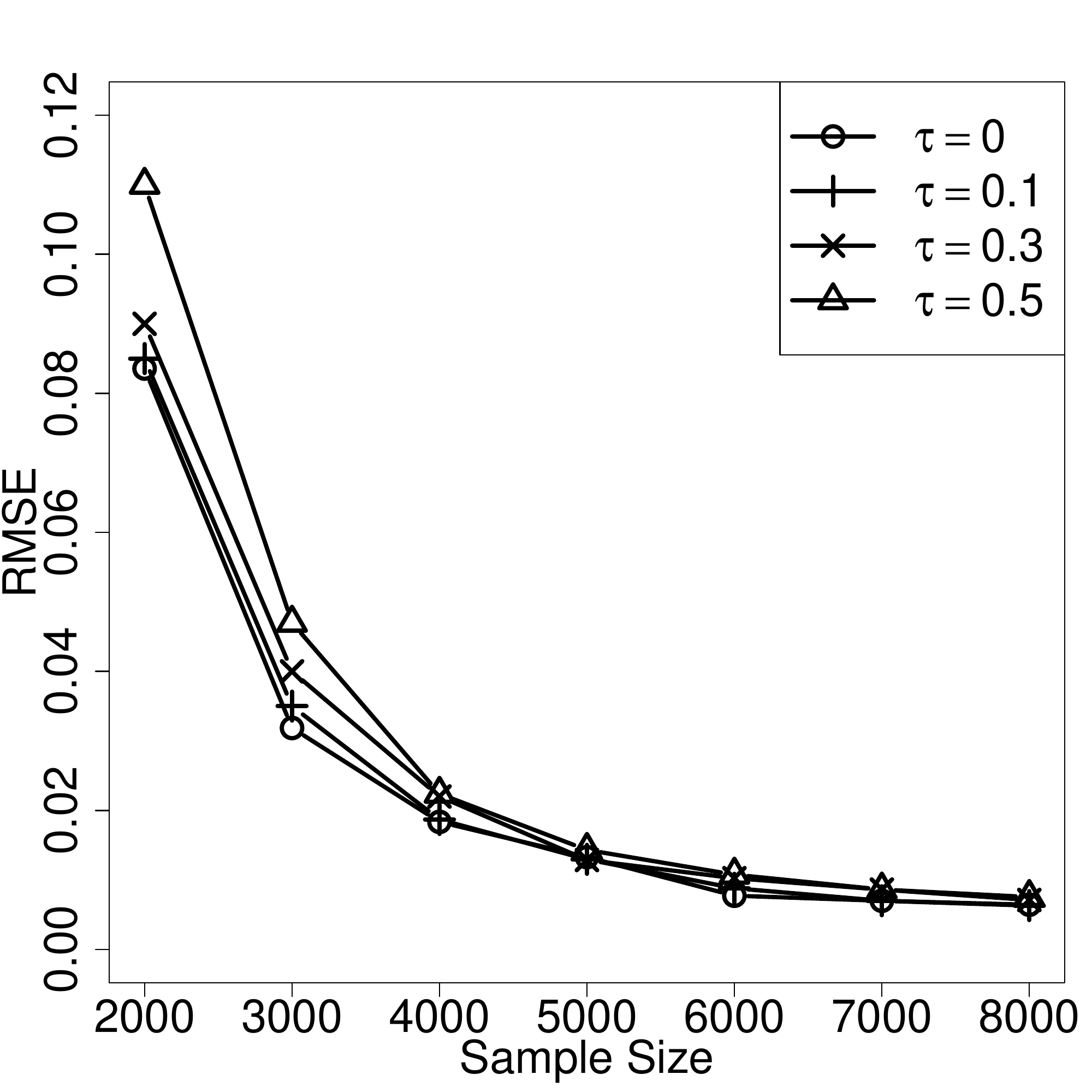}}
	\subfigure[Sparse ISLET]{
		\includegraphics[width=0.45\textwidth,height=2in]{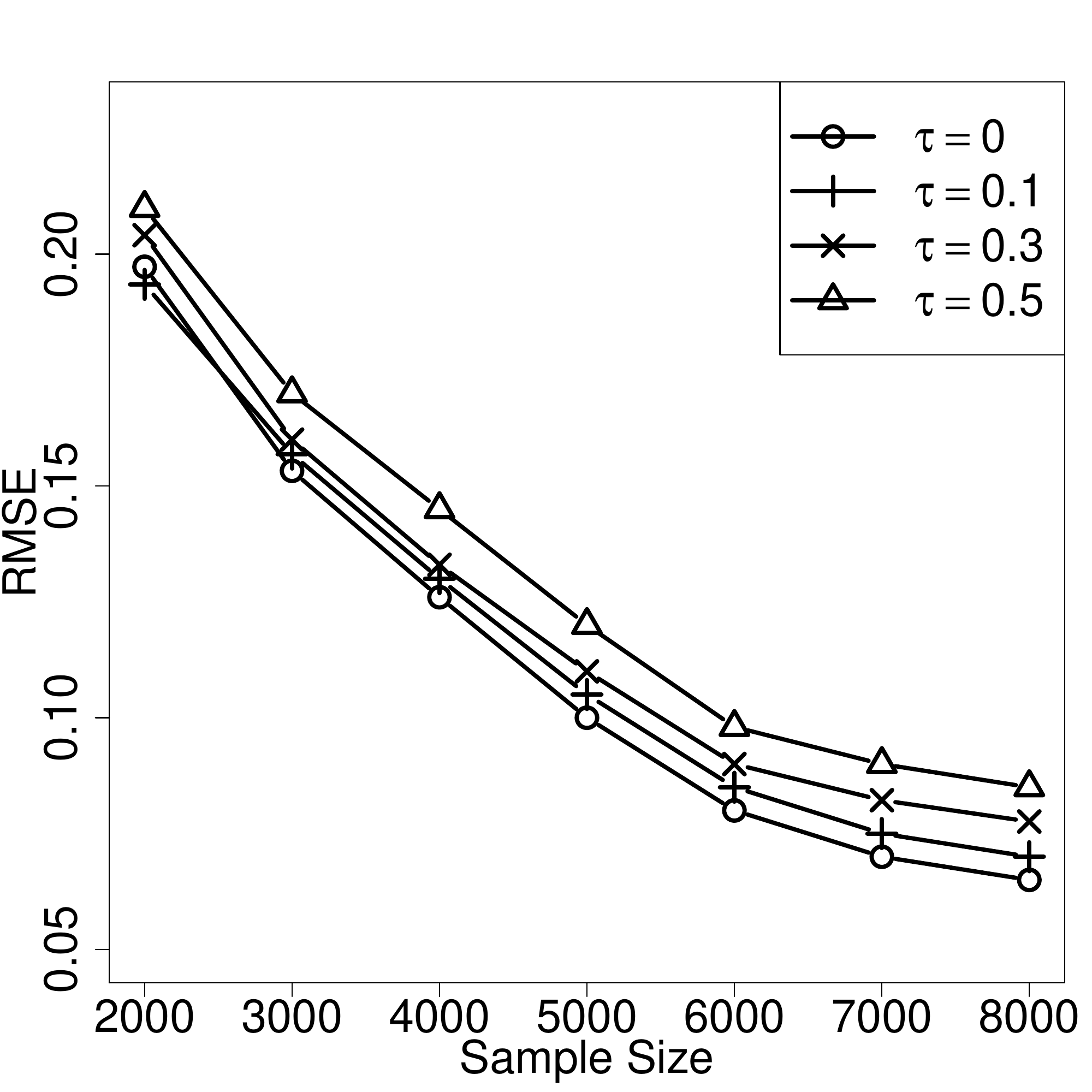}}
	\caption{Average estimation error of ISLET under approximate low Tucker rank case. Left panel: regular case; right panel: sparse case. Here, $\sigma = 5, p = 20,
	n = [2000,8000], s_1 = s_2 = s_3 = 12, \tau = 0,0.1,0.3,0.5$.}\label{fig:approx-low-rank}
\end{figure}


\section{Proofs}\label{sec:proofs}

We collect all proofs of the main technical results in this section.

\subsection{Proof of Theorem \ref{th:upper_bound_general}}\label{sec:proof_upper_bound_general}

This theorem aims to develop a deterministic error bound for $\|\widehat{\bcA} - \bcA\|_{\tHS}^2$ in terms of the sketching direction error $\theta$, $\rho$, and error term $\|(\widetilde{\X}^\top\widetilde{\X})^{-1}\widetilde{\X}^\top \widetilde{\bvarepsilon}\|_2^2$. Since the proof is long and technically challenging, we divide the whole argument into six steps for a better presentation. In Step 1, we introduce the notation to be used throughout the proof. In Step 2, we transform the original high-dimensional low-rank tensor regression model to dimension-reduced one \eqref{eq:low-dimensional-regression}. We also rewrite the key quantities in the upper bound $\|(\widetilde\X^\top\widetilde\X)^{-1}\widetilde\X^\top\widetilde{\bvarepsilon}\|_2^2$ to $\|\widehat{\bcB} - \widetilde\bcB\|_{\tHS}^2 + \sum_{k=1}^3 \|\widehat{\D}_k-\widetilde\D_k\|_F^2$. In step 3, we introduce the factorization for $\bcA$ and $\widehat\bcA$. 
Based on this factorization and the property of orthogonal projection, in step 4, we decompose the loss $\|\widehat\bcA - \bcA\|_{\tHS}$ into eight terms. In step 5, we bound some intermediate error terms in terms of $\theta$ and $\rho$ using properties of the spectral norm and least singular value. In the last Step 6, we finish the proof by bounding each of the eight terms in Step 4 using the results in Step 2, 5, and Lemma \ref{lm:FGH}.
\begin{enumerate}[leftmargin=*]
	\item[Step 1] For simplicity, we denote 
	$$\x_j = \rmvec(\bcX_j)\in \mathbb{R}^{p_1 p_2 p_3}, \quad \X_{jk} = \mathcal{M}_k(\bcX_j)\in \mathbb{R}^{p_k \times (p_{k+1}p_{k+2})},$$
	$$\a = \rmvec(\bcA)\in \mathbb{R}^{p_1p_2p_3},\quad \A_k = \mathcal{M}_k(\bcA) \in \mathbb{R}^{p_k\times(p_{k+1}p_{k+2})}$$ 
	as the vectorized and matricized tensor covariates and parameter. (Note that $\X_{jk}$ is a matrix rather than the $(j,k)$-th entry of $\X$. Instead, we use $\X_{[j,k]}$ to denote the specific $(i,j)$-th entry of the matrix $\X$ in our notation system.) All mode indices $(\cdot)_k$ are in module-3, e.g., $p_{4} = p_1$, $\A_4 = \A_1$, $\X_{j5} = \X_{j2}$, etc. Recall 
	$$\W_1 = ({\U}_3\otimes {\U}_2)\V_1, \quad\W_2 = ({\U}_3\otimes {\U}_1)\V_2, \quad \W_3 = ({\U}_2\otimes {\U}_1)\V_3,$$ 
	$$\widetilde\W_1 = (\widetilde{\U}_3\otimes \widetilde{\U}_2)\widetilde\V_1, \quad \widetilde\W_2 = (\widetilde{\U}_3\otimes \widetilde{\U}_1)\widetilde\V_2, \quad \widetilde{\W}_3 = (\widetilde{\U}_2\otimes \widetilde{\U}_1)\widetilde\V_3.$$ 
	Define
	\begin{equation}\label{eq:def-tilde-B-D_k}
	\begin{split}
	\widetilde{\bcB} = & \left\llbracket \bcA; \widetilde{\U}_1^\top, \widetilde{\U}_2^\top, \widetilde{\U}_3^\top \right\rrbracket = \left\llbracket \bcS\times_1 \U_1\times_2 \U_2\times_3 \U_3; \widetilde{\U}_1^\top, \widetilde{\U}_2^\top , \widetilde{\U}_3^\top \right\rrbracket \in \mathbb{R}^{r_1\times r_2\times r_3};\\
	\widetilde{\D}_1 = & \widetilde{\U}_{1\perp}^\top \mathcal{M}_1(\bcA\times_2 \widetilde{\U}_2^\top \times_3 \widetilde{\U}_3) \widetilde{\V}_1 \overset{\text{Lemma \ref{lm:Kronecker-vectorization-matricization}}}{=} \widetilde{\U}_{1\perp}^\top \A_1 \widetilde{\W}_1 \in \mathbb{R}^{(p_1-r_1)\times r_1},\\
	\widetilde{\D}_2 = & \widetilde{\U}_{2\perp}^\top \mathcal{M}_2(\bcA\times_1 \widetilde{\U}_1^\top \times_3 \widetilde{\U}_3) \widetilde{\V}_2 = \widetilde{\U}_{2\perp}^\top \A_2 \widetilde{\W}_2 \in \mathbb{R}^{(p_2-r_2)\times r_2},\\
	\widetilde{\D}_3 = & \widetilde{\U}_{3\perp}^\top \mathcal{M}_3(\bcA\times_1 \widetilde{\U}_1^\top \times_2 \widetilde{\U}_2) \widetilde{\V}_3 = \widetilde{\U}_{3\perp}^\top \A_3 \widetilde{\W}_3 \in \mathbb{R}^{(p_3-r_3)\times r_3}.
	\end{split}
	\end{equation}
	Intuitively speaking, $\widetilde\B$ is the parameter core tensor lying in the singular subspaces $\widetilde{\U}_3\otimes \widetilde{\U}_2\otimes \widetilde{\U}_1$ and $\widetilde{\D}_1, \widetilde{\D}_2, \widetilde{\D}_3$ are the parameter matrices corresponding to the arm-minus-body part lying in the singular subspace of $\mathcal{R}_1\left(\widetilde{\W}_1\otimes \widetilde{\U}_{1\perp}\right)$, $\mathcal{R}_2\left(\widetilde{\W}_2\otimes \widetilde{\U}_{2\perp}\right)$, $\mathcal{R}_3\left(\widetilde{\W}_3\otimes \widetilde{\U}_{3\perp}\right)$. 
	\item[Step 2] In this step, we introduce an important decomposition for $y_j$ and the error term $\|(\widetilde{\X}^\top\widetilde\X)^{-1}\widetilde\X^\top \widetilde\bvarepsilon\|_2^2$. In correspondence to $\widehat\bgamma$ \eqref{eq:partial-regression}, we construct $\widetilde{\bgamma}$ as
	\begin{equation}\label{eq:gamma_tilde}
	\widetilde{\bgamma} = \left(\rmvec(\widetilde{\bcB})^\top, \rmvec(\widetilde{\D}_1)^\top, \rmvec(\widetilde{\D}_2)^\top, \rmvec(\widetilde{\D}_3)^\top\right)^\top \in \mathbb{R}^{m}.
	\end{equation}
	
	Then for $j=1,\ldots, n$, the response $y_j$ can be decomposed as
	\begin{equation}\label{ineq:y_j^(2)-decompose}
	\begin{split}
	y_j = & \langle \bcX_j, \bcA \rangle +\varepsilon_j = \left\langle \x_j, \a \right\rangle +\varepsilon_j\\
	= & \left\langle \x_j, P_{\widetilde{\U}} \a \right\rangle +\varepsilon_j + \left\langle \x_j, P_{\widetilde{\U}_{\perp}} \a \right\rangle \\
	= & \left\langle \x_j, P_{\widetilde{\U}_1\otimes \widetilde{\U}_2\otimes \widetilde{\U}_3} \a \right\rangle + \sum_{k=1}^3\left\langle \x_j, P_{\mathcal{R}_k\left(\widetilde{\U}_{k\perp}\otimes \widetilde{\W}_k\right)} \a \right\rangle + \widetilde{\varepsilon}_j\\
	\overset{\eqref{eq:def-tilde-B-D_k}}{=} & \left\langle (\widetilde{\U}_3\otimes \widetilde{\U}_2\otimes \widetilde{\U}_1)^\top \x_j, ~ (\widetilde{\U}_3\otimes \widetilde{\U}_2\otimes \widetilde{\U}_1)^\top \a \right\rangle\\
	& + \sum_{k=1}^3 \left\langle \widetilde{\U}_{k\perp}^\top \X_{jk}\widetilde{\W}_k, \widetilde{\U}_{k\perp}^\top \A_k \widetilde{\W}_k \right\rangle + \widetilde{\varepsilon}_j \\
	\overset{\eqref{eq:def-tilde-B-D_k}}{=} & (\widetilde{\X}_\bcB)_{[j, :]} \rmvec(\widetilde{\bcB}) + \sum_{k=1}^3(\widetilde{\X}_{\D_k})_{[j,:]} \rmvec(\widetilde{\D}_k) + \widetilde{\varepsilon}_j = \widetilde{\X}_{[j,:]}\cdot \widetilde{\bgamma} + \widetilde{\varepsilon}_j.
	\end{split}
	\end{equation} 
	Given the definitions of $\widehat{\D}_k$, $\widehat{\bcB}$ \eqref{eq:def-tilde-B-D_k} and $\widehat{\bgamma}$ \eqref{eq:partial-regression} and the fact that $\widetilde{\X}$ is non-singular, $\widehat\bgamma$ can be rewritten into the following vectorized form,
	\begin{equation*}
	\begin{split}
	\widehat{\bgamma} = & \argmin_{\bgamma\in\mathbb{R}^{m}} \sum_{i=1}^{n}\left(y_i - \widetilde{\X}_{[i, :]}\bgamma\right)^2 = \argmin_{\bgamma\in\mathbb{R}^{m}}\left\|y - \widetilde{\X}\bgamma\right\|_2^2\\
	= & \left(\widetilde{\X}^\top\widetilde{\X}\right)^{-1} \widetilde{\X}^\top y = \left(\widetilde{\X}^\top\widetilde{\X}\right)^{-1} \widetilde{\X}^\top \left(\widetilde{\X}\widetilde{\bgamma} + \widetilde{\bvarepsilon}\right)\\
	= & \widetilde{\bgamma} + \left(\widetilde{\X}^\top\widetilde{\X}\right)^{-1} \widetilde{\X}^\top\widetilde{\bvarepsilon}.
	\end{split}
	\end{equation*}
	where $m = r_1 r_2 r_3 + \sum_{k=1}^3 (p_k -r_k)r_k $. Thus, by the definition of $\widetilde{\bgamma}$ \eqref{eq:gamma_tilde}, $\widehat{\bgamma}$ \eqref{eq:partial-regression}, $\widehat{\bcB}$ and $\widehat\D_k$ \eqref{eq:def_hat_Beta_hat_D}, we have
	\begin{equation}\label{th:hat_B-B}
	\begin{split}
	& \|\widehat{\bcB} - \widetilde{\bcB}\|_{\tHS}^2 + \sum_{k=1}^3\|\widehat{\D}_k - \widetilde{\D}_k\|_F^2 = \left\|\widehat{\bgamma} - \widetilde{\bgamma}\right\|_2^2 = \left\|(\widetilde{\X}^\top\widetilde{\X})^{-1}\widetilde{\X}^\top \widetilde{\bvarepsilon}\right\|_2^2 := \kappa^2.
	\end{split}
	\end{equation}
	\item[Step 3] In this step, we introduce the factorization for $\bcA$ \eqref{eq:factor-bcA}. Since the left and right singular subspaces of $\A_k$ are $\U_k$ and $\W_k$, respectively, 
	\begin{equation}\label{ineq:sigma_min-tildeUAW}
	\begin{split}
	& \sigma_{r_k} \left(\widetilde{\U}_k^\top \A_k \widetilde{\W}_k\right) = \sigma_{r_k} \left(\widetilde{\U}_k^\top P_{\U_k} \A_k P_{\W_k}\widetilde{\W}_k\right) =  \sigma_{r_k} \left((\widetilde{\U}_k^\top \U_k) \U_k^\top \A_k \W_k (\W_k^\top\widetilde{\W}_k)\right) \\
	\geq & \sigma_{\min}(\widetilde{\U}_k^\top \U_k) \cdot \sigma_{\min}(\U_k^\top \A_k\W_k) \cdot \sigma_{\min}(\W_k^\top\widetilde{\W}_k)\\
	= & \sqrt{1 - \|\sin\Theta(\widetilde{\U}_k, \U_k)\|^2} \cdot \sigma_{r_k}(\A_k) \cdot \sqrt{1 - \|\sin\Theta(\widetilde{\W}_k, \W_k)\|^2}\\
	\geq & \sigma_{r_k}(\A_k)(1-\theta^2)>0.
	\end{split}
	\end{equation}
	Here, the last but one equality is due to the property of $\sin\Theta$ distance (c.f., Lemma 1 in \cite{cai2018rate}). Thus, $\rank(\widetilde{\U}_k^\top \A_k \widetilde{\W}_k) = r_k$, which is a full rank matrix.  Thus,
	\begin{equation}\label{eq:factor-bcA}
	\begin{split}
	& \bcA = \left\llbracket \bcB; \U_1, \U_2, \U_3 \right \rrbracket\\
	= & \left\llbracket \llbracket  \bcB; \U_1, \U_2, \U_3  \rrbracket; \U_1 (\widetilde{\U}_1^\top \U_1)^{-1}\widetilde{\U}_1^\top ,  \U_2 (\widetilde{\U}_2^\top \U_2)^{-1}\widetilde{\U}_2^\top,  \U_3 (\widetilde{\U}_3^\top \U_3)^{-1}\widetilde{\U}_3^\top \right\rrbracket\\
	= & \left\llbracket \bcA; \U_1 (\widetilde{\U}_1^\top \U_1)^{-1}\widetilde{\U}_1^\top ,  \U_2 (\widetilde{\U}_2^\top \U_2)^{-1}\widetilde{\U}_2^\top,  \U_3 (\widetilde{\U}_3^\top \U_3)^{-1}\widetilde{\U}_3^\top  \right \rrbracket\\
	= & \left\llbracket \bcA;  \A_1\widetilde{\W}_1(\widetilde{\U}_1^\top \A_1\widetilde{\W}_1)^{-1}\widetilde{\U}_1^\top, \A_2\widetilde{\W}_2(\widetilde{\U}_2^\top \A_2\widetilde{\W}_2)^{-1}\widetilde{\U}_2^\top , \A_3\widetilde{\W}_3(\widetilde{\U}_3^\top \A_3\widetilde{\W}_3)^{-1}\widetilde{\U}_3^\top   \right \rrbracket\\
	\end{split}
	\end{equation}
	The fourth equality is because the left singular space and right singular space of $\A_k$ is $\U_k$ and $\W_k$.
	
	Recall
	\begin{equation*}
	\widehat{\bcA} = \left\llbracket \widehat{\bcB}; \widehat{\L}_1, \widehat{\L}_2, \widehat{\L}_3 \right\rrbracket, \quad \widehat{\L}_k = (\widetilde{\U}_k\widehat{\B}_k\widetilde{\V}_k + \widetilde{\U}_{k\perp}\widehat{\D}_k)(\widehat{\B}_k\widetilde{\V}_k)^{-1},\quad k=1,2,3.
	\end{equation*}
	Denote $\widetilde{\B}_k = \mathcal{M}_k(\widetilde{\bcB})$, $\widehat{\B}_k = \mathcal{M}_k(\widehat{\bcB})$. In parallel to the definition of $\widehat{\L}_k$, we define
	\begin{equation}\label{eq:tilde-L_k}
	\begin{split}
	\widetilde{\L}_1 = & (\widetilde{\U}_1\widetilde{\B}_1\widetilde{\V}_1 + \widetilde{\U}_{1\perp}\widetilde{\D}_1)(\widetilde{\B}_1\widetilde{\V}_1)^{-1}, \\
	= & \left(\widetilde{\U}_1\widetilde{\U}_1^\top \A_1 (\widetilde{\U}_{3}\otimes \widetilde{\U}_{2})\widetilde{\V}_1 + \widetilde{\U}_{1\perp}\widetilde{\U}_{1\perp}^\top \A_1(\widetilde{\U}_3\otimes \widetilde{\U}_2)\widetilde{\V}_1\right)\\
	& ~~ \cdot \left(\widetilde{\U}_1^\top \A_1(\widetilde{\U}_3\otimes \widetilde{\U}_2)\widetilde{\V}_1\right)^{-1}\\
	= & \A_1\widetilde{\W}_1\left(\widetilde{\U}_1^\top \A_1 \widetilde{\W}_1\right)^{-1}.
	\end{split}
	\end{equation}
	Similarly,
	\begin{equation*}
	\begin{split}
	\widetilde{\L}_2 = & (\widetilde{\U}_2\widetilde{\B}_2\widetilde{\V}_2 + \widetilde{\U}_{2\perp}\widetilde{\D}_2)(\widetilde{\B}_2\widetilde{\V}_2)^{-1} = \A_2\widetilde{\W}_2\left(\widetilde{\U}_2^\top \A_2 \widetilde{\W}_2\right)^{-1},\\
	\widetilde{\L}_3 = & (\widetilde{\U}_3\widetilde{\B}_3\widetilde{\V}_3 + \widetilde{\U}_{3\perp}\widetilde{\D}_3)(\widetilde{\B}_3\widetilde{\V}_3)^{-1} = \A_3\widetilde{\W}_3\left(\widetilde{\U}_3^\top \A_3 \widetilde{\W}_3\right)^{-1}.
	\end{split}
	\end{equation*}
	Thus, in addition to $\widehat\bcA = \llbracket\widehat\bcB; \widehat\L_1, \widehat\L_2, \widehat\L_3\rrbracket$, we have
	\begin{equation}\label{eq: tensorA-indentity}
	    \bcA = \llbracket\widetilde{\bcB}; \widetilde{\L}_1, \widetilde{\L}_2, \widetilde{\L}_3 \rrbracket
	\end{equation}
	\item[Step 4] Next, we analyze the estimation error of $\widehat{\bcA}$. First, the error bound of $\widehat{\bcA} - \bcA$ can be decomposed into eight parts,
	\begin{equation}\label{eq:error-decomposition}
	\begin{split}
	& \|\widehat{\bcA}-\bcA\|_{\tHS}^2 = \left\|\llbracket\widehat{\bcA}-\bcA; P_{\widetilde{\U}_1} + P_{\widetilde{\U}_{1\perp}}, P_{\widetilde{\U}_2} + P_{\widetilde{\U}_{2\perp}}, P_{\widetilde{\U}_3} + P_{\widetilde{\U}_{3\perp}} \rrbracket\right\|_{\tHS}^2\\
	= & \left\|\llbracket (\widehat{\bcA} - \bcA); \widetilde{\U}_1^\top, \widetilde{\U}_2^\top, \widetilde{\U}_3^\top \rrbracket\right\|_{\tHS}^2 + \left\|\llbracket (\widehat{\bcA} - \bcA); \widetilde{\U}_{1\perp}^\top, \widetilde{\U}_2^\top, \widetilde{\U}_3^\top \rrbracket\right\|_{\tHS}^2 \\
	& + \left\|\llbracket (\widehat{\bcA} - \bcA); \widetilde{\U}_1^\top, \widetilde{\U}_{2\perp}^\top, \widetilde{\U}_3^\top \rrbracket\right\|_{\tHS}^2 + \left\|\llbracket (\widehat{\bcA} - \bcA); \widetilde{\U}_1^\top, \widetilde{\U}_2^\top, \widetilde{\U}_{3\perp}^\top \rrbracket\right\|_{\tHS}^2 \\
	& + \left\|\llbracket (\widehat{\bcA} - \bcA); \widetilde{\U}_1^\top, \widetilde{\U}_{2\perp}^\top, \widetilde{\U}_{3\perp}^\top \rrbracket\right\|_{\tHS}^2 + \left\|\llbracket (\widehat{\bcA} - \bcA); \widetilde{\U}_{1\perp}^\top, \widetilde{\U}_2^\top, \widetilde{\U}_{3\perp}^\top \rrbracket\right\|_{\tHS}^2\\
	& + \left\|\llbracket (\widehat{\bcA} - \bcA); \widetilde{\U}_{1\perp}^\top, \widetilde{\U}_{2\perp}^\top, \widetilde{\U}_3^\top \rrbracket\right\|_{\tHS}^2 + \left\|\llbracket (\widehat{\bcA} - \bcA); \widetilde{\U}_{1\perp}^\top, \widetilde{\U}_{2\perp}^\top, \widetilde{\U}_{3\perp}^\top \rrbracket\right\|_{\tHS}^2.
	\end{split}
	\end{equation}
	Here we used the fact that $P_{\widetilde{\U}_1}$ and $P_{\widetilde{\U}_{1\perp}}$ are orthogonal complementary. We aim to apply Lemma \ref{lm:FGH} to analyze each term above in the next two steps. 
	
	\item[Step 5] Before giving the upper bounds for each term of \eqref{eq:error-decomposition}, we denote
	\begin{equation}
	\begin{split}
	\lambda_k = & \max\left\{\left\|\widehat{\D}_k(\widehat{\B}_k\widetilde{\V}_k)^{-1}\right\|, \left\|\widetilde{\D}_k(\widetilde{\B}_k\widetilde{\V}_k)^{-1}\right\|\right\},\\
	\pi_k = & \|(\widetilde{\B}_k\widetilde{\V}_k)^{-1}\widetilde{\B}_k\|,\quad k=1, 2, 3
	\end{split}
	\end{equation}
	and aim to provide upper bounds for $\lambda_k, \pi_k$ in this step. By definition of $\widetilde\B_k$ and the fact that the right singular vector of $\A_k$ is $\W_k$,
	\begin{equation}\label{ineq:pi_k}
	\begin{split}
	\pi_1 = & \left\|(\widetilde{\B}_1\widetilde{\V}_1)^{-1}\widetilde{\B}_1\right\| = \left\|(\widetilde{\U}_1^\top \A_1 \widetilde{\W}_1)^{-1} \widetilde{\U}_1^\top \A_1(\widetilde{\U}_{3}\otimes \widetilde{\U}_{2})\right\|\\
	\leq & \left\|\left(\widetilde{\U}_1^\top \A_1 \widetilde{\W}_1\right)^{-1}\widetilde{\U}_1^\top \A_1\right\| = \left\|\left(\widetilde{\U}_1^\top \A_1 \W_1\W_1^\top \widetilde{\W}_1\right)^{-1}\widetilde{\U}_1^\top \A_1\W_1\right\|\\
	\leq & \left\|(\W_1^\top \widetilde{\W}_1)^{-1}\right\| = \sigma_{\min}^{-1}(\widetilde{\W}_1^\top \W_1) = \left(1-\|\sin\Theta(\widetilde{\W}_k, \W_k)\|^2\right)^{-1/2} \\
	\leq & \frac{1}{(1-\theta^2)^{1/2}}.
	\end{split}
	\end{equation}
	Similarly, the same upper bounds also applies to $\pi_2$ and $\pi_3$.
	
	Based on definitions of $\widetilde{\D}_k$ and $\widetilde{\B}_k$ and the fact that the left singular subspace of $\A_k$ is $\U_k$, we have 
	\begin{equation}\label{ineq:tilde-D-tilde-B-tilde-V}
	\begin{split}
	& \|\widetilde{\D}_k(\widetilde{\B}_k\widetilde{\V}_k)^{-1}\|^2+1 = \left\|\widetilde{\U}_{k\perp}^\top \A_k\widetilde{\W}_k(\widetilde{\U}_k^\top \A_k\widetilde{\W}_k)^{-1}\right\|^2 + 1\\
	= & \left\|\begin{bmatrix}
	\I_{r_k} \\
	\widetilde{\U}_{k\perp}^\top \A_k\widetilde{\W}_k(\widetilde{\U}_k^\top \A_k\widetilde{\W}_k)^{-1}
	\end{bmatrix}\right\|^2 =   \left\|\begin{bmatrix}
	\widetilde{\U}_k^\top \A_k \widetilde{\W}_k(\widetilde{\U}_k^\top \A_k \widetilde{\W}_k)^{-1}\\
	\widetilde{\U}_{k\perp}^\top \A_k\widetilde{\W}_k(\widetilde{\U}_k^\top \A_k\widetilde{\W}_k)^{-1}
	\end{bmatrix}\right\|^2 \\
	= & \left\|\A_k\widetilde{\W}_k \left(\widetilde{\U}_k^\top \A_k\widetilde{\W}_k\right)^{-1}\right\|^2 = \left\|\U_k^\top \A_k\widetilde{\W}_k \left(\widetilde{\U}_k^\top \U_k \U_k^\top \A_k\widetilde{\W}_k\right)^{-1}\right\|^2 \\
	= & \left\|\U_k^\top \A_k\widetilde{\W}_k \left(\U_k^\top \A_k\widetilde{\W}_k\right)^{-1} \left(\widetilde{\U}_k^\top \U_k\right)^{-1}\right\|^2\\
	= & \left\|\left(\widetilde{\U}_1^\top \U_1\right)^{-1}\right\|^2 = \sigma_{\min}^{-2}\left(\widetilde{\U}_1^\top \U_1\right) = \left(1 - \|\sin\Theta(\widetilde{\U}_1, \U_1)\|^2\right)^{-1} \leq \frac{1}{1-\theta^2},
	\end{split}
	\end{equation}
	which implies
	$$\|\widetilde\D_k (\widetilde\B_k\widetilde\V_k)^{-1}\| \leq \sqrt{\frac{1}{1-\theta^2} - 1} = \sqrt{\frac{\theta^2}{1-\theta^2}}.$$
	By the assumption of the theorem that $\|\widehat{\D}_1(\widehat{\B}_1\widetilde{\V}_1)^{-1}\| \leq \rho$ and $\theta\leq 1/2$, we have
	\begin{equation}\label{ineq:lambda_k}
	\lambda_k \leq \max\left\{\rho, \frac{\theta}{\sqrt{1-\theta^2}}\right\} \leq \rho + \frac{2}{\sqrt{3}}\theta, \quad k=1,2,3.
	\end{equation}
	\item[Step 6] Now we are ready to give upper bounds for all terms in \eqref{eq:error-decomposition}. 
	\begin{itemize}[leftmargin=*]
		\item First, by definition of $\widehat{\bcB}$, $\widehat{\bcA}$ \eqref{eq:hat_A_non-sparse},
		\begin{equation}\label{eq:hat-A-project}
		\begin{split}
		& \llbracket \widehat{\bcA}; \widetilde{\U}_1^\top, \widetilde{\U}_2^\top, \widetilde{\U}_3^\top \rrbracket = \left\llbracket \llbracket \widehat\bcB; \widehat{\L}_1, \widehat{\L}_2, \widehat{\L}_3\rrbracket; \widetilde{\U}_1^\top, \widetilde{\U}_2^\top, \widetilde{\U}_3^\top \right\rrbracket\\
		= & \left\llbracket \widehat{\bcB}; \widetilde{\U}_1^\top \widehat\L_1, \widetilde{\U}_2^\top \widehat\L_2, \widetilde{\U}_3^\top \widehat\L_3 \right\rrbracket.
		\end{split}
		\end{equation}
		Here, 
		\begin{equation*}
		\widetilde\U_k^\top\widehat\L_k = \widetilde{\U}_k^\top \left((\widetilde{\U}_k\widehat{\B}_k\widetilde{\V}_k + \widetilde{\U}_{k\perp}\widehat{\D}_k)(\widehat{\B}_k\widetilde{\V}_k)^{-1}\right) = (\widehat{\B}_k\widetilde{\V}_k) (\widehat{\B}_k\widetilde{\V}_k)^{-1} = \I_{r_k}.
		\end{equation*}
		Similarly, we have $\widetilde\U_k^\top\widetilde\L_k = \I_{r_k}$.
		
		Thus, $\llbracket \widehat{\bcA}; \widetilde{\U}_1^\top, \widetilde{\U}_2^\top, \widetilde{\U}_3^\top \rrbracket = \widehat\bcB$. By definition of $\widetilde{\bcB}$ \eqref{eq:def-tilde-B-D_k}, we have
		\begin{equation}\label{ineq:thm1-term1}
		\begin{split}
		\left\|\llbracket (\widehat{\bcA} - \bcA); \widetilde{\U}_1^\top, \widetilde{\U}_2^\top, \widetilde{\U}_3^\top \rrbracket\right\|_{\tHS}^2 = & \left\|\llbracket \widehat{\bcA}; \widetilde{\U}_1^\top, \widetilde{\U}_2^\top, \widetilde{\U}_3^\top \rrbracket - \llbracket \bcA; \widetilde{\U}_1^\top, \widetilde{\U}_2^\top, \widetilde{\U}_3^\top \rrbracket \right\|_{\tHS}^2 \\
		= & \|\widehat{\bcB} - \widetilde{\bcB}\|_{\tHS}^2.
		\end{split}
		\end{equation}
		
		\item Note that
		\begin{equation}\label{ineq:thm1-term1.5}
		\begin{split}
		& \left\|\llbracket(\widehat{\bcA} - \bcA); \widetilde{\U}_{1\perp}^\top, \widetilde{\U}_2^\top, \widetilde{\U}_3^\top \rrbracket\right\|_{\tHS}^2 \\
		\overset{ \eqref{eq: tensorA-indentity},\eqref{eq:hat-A-project}}{=} & \left\|\llbracket \widehat{\bcB}; \widetilde{\U}_{1\perp}^\top \widehat\L_1, \widetilde{\U}_2^\top\widehat{\L}_2, \widetilde{\U}_3^\top\widehat{\L}_3 \rrbracket - \llbracket\widetilde\bcB; \widetilde\U_{1\perp}^\top\widetilde{\L}_1, \widetilde\U_{2}^\top\widetilde{\L}_2, \widetilde\U_{3}^\top\widetilde{\L}_3\rrbracket\right\|_{\tHS}^2\\
		\overset{\eqref{eq:hat_A_non-sparse}\eqref{eq:tilde-L_k}}{=} & \left\|\llbracket \widehat{\bcB}; \widehat{\D}_1(\widehat{\B}_1\widetilde{\V}_1)^{-1}, \I, \I\rrbracket - \llbracket \widetilde{\bcB}; \widetilde{\D}_1(\widetilde{\B}_1\widetilde{\V}_1)^{-1}, \I, \I\rrbracket\right\|_{\tHS}^2\\
		\overset{\text{Lemma \ref{lm:Kronecker-vectorization-matricization}}}{=} & \left\|\widehat{\D}_1(\widehat{\B}_1\widetilde{\V}_1)^{-1}\widehat{\B}_1 - \widetilde{\D}_1(\widetilde{\B}_1\widetilde{\V}_1)^{-1}\widetilde{\B}_1 \right\|_F^2\\
		\end{split}
		\end{equation}
		By the first part of Lemma \ref{lm:FGH},
		\begin{equation*}
		\begin{split}
		& \left\|\widehat{\D}_1(\widehat{\B}_1\widetilde{\V}_1)^{-1}\widehat{\B}_1 - \widetilde{\D}_1(\widetilde{\B}_1\widetilde{\V}_1)^{-1}\widetilde{\B}_1 \right\|_F^2\\
		\leq & \left(\pi_1\|\widehat{\D}_1-\widetilde{\D}_1\|_F + \lambda_1\|\widehat{\B}_1- \widetilde{\B}_1\|_F + \pi_1\lambda_1 \|\widehat{\B}_1\widetilde{\V}_1 - \widetilde{\B}_1\widetilde{\V}_1\|_F\right)^2\\
		\overset{\eqref{ineq:pi_k}\eqref{ineq:lambda_k}}{\leq} & \left(\frac{1}{\sqrt{1-\theta^2}}\|\widehat{\D}_1-\widetilde{\D}_1\|_F + (\rho+\frac{2}{\sqrt{3}}\theta)\kappa + (\rho+\frac{2}{\sqrt{3}}\theta)\frac{1}{\sqrt{1-\theta^2}}\kappa\right)^2\\
		\leq & \frac{1}{1-\theta^2}\|\widehat{\D}_1 - \widetilde{\D}_1\|_F^2 + C_1(\rho+\theta)\|\widehat{\D}_1-\widetilde{\D}_1\|_F\kappa + C_2(\rho+\theta)^2\kappa^2\\
		\leq & \|\widehat\D_1 - \widetilde\D_1\|_F^2 + 2\theta^2\|\widehat\D_1-\widetilde\D_1\|_F^2 + C_1(\rho+\theta)\|\widehat\D_1 - \widetilde\D_1\|_F\kappa + C_2(\rho+\theta)^2\kappa^2\\
		\leq & \|\widehat{\D}_1 - \widetilde{\D}_1\|_F^2 + C(\rho+\theta)\kappa^2.
		\end{split}
		\end{equation*}
		Here, the last inequality is due to the fact that $\|\widehat\D_1 - \widetilde\D_1\|_F\leq \kappa$. Therefore,
		\begin{equation}\label{ineq:thm1-term2}
		\begin{split}
		& \left\|\left\llbracket(\widehat{\bcA} - \bcA); \widetilde{\U}_{1\perp}^\top, \widetilde{\U}_{2}^\top, \widetilde{\U}_3\right\rrbracket\right\|_{\tHS}^2 \leq \|\widehat\D_1 - \widetilde\D_1\|_F^2 + C(\rho+\theta)\kappa^2;\\
		\text{similarly} \quad & \left\|\left\llbracket(\widehat{\bcA} - \bcA); \widetilde{\U}_{1}^\top, \widetilde{\U}_{2\perp}^\top, \widetilde{\U}_3\right\rrbracket\right\|_{\tHS}^2 \leq \|\widehat{\D}_2 - \widetilde{\D}_2\|_F^2 + C(\rho+\theta)\kappa^2,\\
		& \left\|\left\llbracket(\widehat{\bcA} - \bcA); \widetilde{\U}_{1}^\top, \widetilde{\U}_{2}^\top, \widetilde{\U}_{3\perp}\right\rrbracket\right\|_{\tHS}^2 \leq \|\widehat{\D}_3-\widetilde{\D}_3\|_F^2 + C(\rho+\theta)\kappa^2.
		\end{split}
		\end{equation}

		\item By similar argument as \eqref{ineq:thm1-term1.5}, we have
		\begin{equation*}
		\begin{split}
		& \left\|\llbracket(\widehat{\bcA} - \bcA); \widetilde{\U}_{1\perp}^\top, \widetilde{\U}_{2\perp}^\top, \widetilde{\U}_3\rrbracket\right\|_F^2 \\
		= & \left\|\llbracket \widehat{\bcB}; \widehat{\D}_1(\widehat{\B}_1\widetilde{\V}_1)^{-1}, \widehat{\D}_2(\widehat{\B}_2\widetilde{\V}_2)^{-1}, \I\rrbracket - \llbracket \widetilde{\bcB}; \widetilde{\D}_1(\widetilde{\B}_1\widetilde{\V}_1)^{-1}, \widetilde{\D}_2(\widetilde{\B}_2\widetilde{\V}_2)^{-1}, \I\rrbracket\right\|_F^2\\
		\end{split}
		\end{equation*}
		By the second part of Lemma \ref{lm:FGH},
		\begin{equation*}
		\begin{split}
		& \left\|\llbracket \widehat{\bcB}; \widehat{\D}_1(\widehat{\B}_1\widetilde{\V}_1)^{-1}, \widehat{\D}_2(\widehat{\B}_2\widetilde{\V}_2)^{-1}, \I\rrbracket - \llbracket \widetilde{\bcB}; \widetilde{\D}_1(\widetilde{\B}_1\widetilde{\V}_1)^{-1}, \widetilde{\D}_2(\widetilde{\B}_2\widetilde{\V}_2)^{-1}, \I\rrbracket\right\|_F^2\\
		\leq & \left(\lambda_1\lambda_2\|\widehat{\bcB}-\widetilde{\bcB}\|_F + \sum_{k=1,2}\pi_k \lambda_1\lambda_2/\lambda_k \|\widehat{\D}_k-\widetilde{\D}_k\|_F + \sum_{k=1,2}\pi_k\lambda_1\lambda_2 \|\widehat{\B}_k\widetilde{\V}_k - \widetilde{\B}_k\widetilde{\V}_k\|_F\right)^2\\
		\overset{\eqref{th:hat_B-B}}{\leq} & (\lambda_1\lambda_2 + \pi_1\lambda_2+\pi_2\lambda_1+\pi_1\lambda_1\lambda_2 + \pi_2\lambda_1\lambda_2)^2\kappa^2 \overset{\eqref{ineq:pi_k}}{\leq} C(\rho+\theta)^2\kappa^2.
		\end{split}
		\end{equation*}
		Therefore,
		\begin{equation}\label{ineq:thm1-term3}
		\begin{split}
		& \left\|\llbracket(\widehat{\bcA} - \bcA); \widetilde{\U}_{1\perp}^\top, \widetilde{\U}_{2\perp}^\top, \widetilde{\U}_3^\top\rrbracket\right\|_F^2 \leq C(\rho+\theta)^2\kappa^2;\\
		\text{similarly,}\quad & \left\|\llbracket(\widehat{\bcA} - \bcA); \widetilde{\U}_{1\perp}^\top, \widetilde{\U}_2^\top, \widetilde{\U}_{3\perp}^\top\rrbracket\right\|_F^2 \leq C(\rho+\theta)^2\kappa^2,\\
		& \left\|\llbracket(\widehat{\bcA} - \bcA); \widetilde{\U}_1^\top, \widetilde{\U}_{2\perp}^\top, \widetilde{\U}_{3\perp}^\top\rrbracket\right\|_F^2 \leq C(\rho+\theta)^2\kappa^2.
		\end{split}
		\end{equation}
		
		\item By the second part of Lemma \ref{lm:FGH},
		\begin{equation}\label{ineq:thm1-term4}
		\begin{split}
		& \left\|\llbracket(\widehat{\bcA} - \bcA); \widetilde{\U}_{1\perp}^\top, \widetilde{\U}_{2\perp}^\top, \widetilde{\U}_{3\perp}\rrbracket\right\|_F^2 \\
		= & \Big\|\llbracket \widehat{\bcB}; \widehat{\D}_1(\widehat{\B}_1\widetilde{\V}_1)^{-1}, \widehat{\D}_2(\widehat{\B}_2\widetilde{\V}_2)^{-1}, \widehat{\D}_3(\widehat{\B}_3\widetilde{\V}_3)^{-1}\rrbracket\\
		& - \llbracket \widetilde{\bcB}; \widetilde{\D}_1(\widetilde{\B}_1\widetilde{\V}_1)^{-1}, \widetilde{\D}_2(\widetilde{\B}_2\widetilde{\V}_2)^{-1}, \widetilde{\D}_3(\widetilde{\B}_3\widetilde{\V}_3)^{-1}\rrbracket\Big\|_F^2\\
		\leq & \Big(\lambda_1\lambda_2\lambda_3\|\widehat{\bcB}-\widetilde{\bcB}\|_F + \sum_{k=1,2,3}\pi_k \lambda_1\lambda_2\lambda_3/\lambda_k \|\widehat{\D}_k-\widetilde{\D}_k\|_F \\
		& \quad + \sum_{k=1,2,3}\pi_k\lambda_1\lambda_2\lambda_3 \|\widehat{\B}_k\widetilde{\V}_k - \widetilde{\B}_k\widetilde{\V}_k\|_F\Big)^2\\
		\overset{{\eqref{th:hat_B-B}}\eqref{ineq:pi_k}}{\leq} & C(\rho+\theta)^4\kappa^2.
		\end{split}
		\end{equation}
	\end{itemize}
	Combining \eqref{eq:error-decomposition}, \eqref{ineq:thm1-term1}, \eqref{ineq:thm1-term2}, \eqref{ineq:thm1-term3} and \eqref{ineq:thm1-term4}, we finally have
	\begin{equation*}
	\begin{split}
	& \left\|\widehat{\bcA} - \bcA\right\|_{\tHS}^2 \leq \|\widehat{\bcB}-\bcB\|_F^2 + \sum_{k=1} \|\widehat{\D}_k-\widetilde{\D}_k\|_F^2 + C(\rho+\theta)\kappa^2 = (1+C(\rho+\theta))\kappa^2.
	\end{split}
	\end{equation*}
\end{enumerate}
In summary, we have finished the proof of this theorem.\quad $\square$

\subsection{Proof of Theorem \ref{th:upper_bound_sparse_general}}\label{sec:proof_upper_bound_sparse_general}
This theorem gives a deterministic error bound of $\|\widehat{\bcA} - \bcA \|_{\tHS}^2$ in terms of $\theta, \rho$ and $\|(\widetilde{\X}_{\bcB}^\top\widetilde{\X}_{\bcB})^{-1}\widetilde{\X}_{\bcB}^\top \widetilde{\bvarepsilon}_\B\|_2^2$, $\|(\widetilde{\X}_{\E_k}^\top\widetilde{\X}_{\E_k})^{-1}\widetilde{\X}_{\E_k}^\top \widetilde{\bvarepsilon}_{\E_k}\|_2^2$, $\|(\widetilde{\X}_{\E_k, [:, G_i^k]})^\top\widetilde{\bvarepsilon}_{\E_k}/n\|_2^2$ for the sparse ISLET estimator $\widehat{\bcA}$ in the sparse low-rank tensor regression model. To prove this theorem, we first rewrite the original high-dimensional regression model to four dimension-reduced ones \eqref{eq:B-regression-model}, \eqref{eq:E_k-regression-model}. Then we derive error bounds for the least square estimator or group Lasso estimator in terms of $\|\widehat\bcB - \bcB\|_{\tHS}^2$ or $\|\widehat\E_k-\widetilde\E_k\|_F^2$ for each of these dimension-reduced regression models. The rest of the proof aims to assemble the upper bound for $\|\widehat\bcA - \bcA\|_{\tHS}^2$, which essentially follows from Steps 3-6 in the proof of Theorem \ref{th:upper_bound_general}.

Denote
$$\A_k = \mathcal{M}_k(\bcA), \quad \a = \rmvec(\bcA), \quad \X_{jk} = \mathcal{M}_k(\bcX_j), \quad \x_j=\rmvec(\bcX_j), \quad 1\leq j \leq n,\quad k=1,2,3;$$
\begin{equation}\label{eq:B-E}
\begin{split}
& \widetilde{\bcB} = \llbracket \bcA; \widetilde{\U}_1^\top, \widetilde{\U}_2^\top, \widetilde{\U}_3^\top\rrbracket;\\
& \widetilde{\E}_k = \mathcal{M}_k(\bcA\times_{k+1}\widetilde{\U}_{k+1}^\top\times_{k+2}\widetilde{\U}_{k+2}^\top)\widetilde{\V}_{k} = \A_k \widetilde{\W}_k \in \mathbb{R}^{p_k\times r_k},\quad k=1,2,3;
\end{split}
\end{equation}
\begin{equation}\label{eq:gamma_B-gamma_E}
\widetilde{\bgamma}_\bcB = \rmvec(\widetilde{\bcB}) \in \mathbb{R}^{p_1p_2p_3}, \quad \widetilde{\bgamma}_{\E_k} = \rmvec(\widetilde{\E}_k)\in \mathbb{R}^{p_kr_k},\quad k=1,2,3.
\end{equation}
Then similarly as the argument \eqref{ineq:y_j^(2)-decompose} in the proof of Theorem \ref{th:upper_bound_general}, we can write down the following partial regression formulas that relate $y_j$ and $(\bcX_j, \bcA)$,
\begin{equation}\label{eq:B-regression-model}
\begin{split}
y_j = & \langle \bcX_j, \bcA\rangle +\varepsilon_j = \langle \x_j, \a\rangle +\varepsilon_j \\
= & \left\langle \x_j, P_{\widetilde{\U}_3\otimes \widetilde{\U}_2 \otimes \widetilde{\U}_1}\a\right\rangle + \varepsilon_j + \langle \x_j, P_{(\widetilde{\U}_3\otimes \widetilde{\U}_2\otimes \widetilde{\U}_1)_{\perp}} \a\rangle \\
= & \left\langle (\widetilde{\U}_3\otimes\widetilde{\U}_2\otimes\widetilde{\U}_1)^\top \x_j, (\widetilde{\U}_3\otimes\widetilde{\U}_2\otimes\widetilde{\U}_1)^\top \a\right\rangle + (\widetilde{\bvarepsilon}_\bcB)_j\\
\overset{\eqref{eq:B-E}\eqref{eq:gamma_B-gamma_E}}{=} & (\widetilde{\X}_\bcB)_{[j, :]} \widetilde{\bgamma}_\bcB + (\widetilde{\bvarepsilon}_\bcB)_j,
\end{split}
\end{equation}
\begin{equation}\label{eq:E_k-regression-model}
\begin{split}
y_j = & \langle \bcX_j, \bcA\rangle +\varepsilon_j \\
= & \left\langle \bcX_j, P_{\mathcal{R}_k\left(\widetilde{\W}_k\otimes \I_{p_k}\right)}[\bcA] \right\rangle + \varepsilon_j + \left\langle \bcX_j, P_{\left(\mathcal{R}_k(\widetilde{\W}_k\otimes \I_{p_k})\right)_\perp}[\bcA] \right\rangle \\
= & \left\langle\X_{jk}\widetilde{\W}_k, ~ \A_{k}\widetilde{\W}_k \right\rangle + (\widetilde{\bvarepsilon}_{\E_k})_j\\
\overset{\eqref{eq:B-E}\eqref{eq:gamma_B-gamma_E}}{=} & (\widetilde{\X}_{\E_k})_{[j, :]} \widetilde{\bgamma}_{\E_k} + (\widetilde{\bvarepsilon}_{\E_k})_j
\end{split}
\end{equation}
for $j=1,\ldots, n$ and $k=1,2,3$.  We discuss the estimation errors of $\widehat{\bgamma}_{\E_k}$ ($k\in J_s$), $\widehat{\bgamma}_{\E_k}$ ($k\notin J_s$), and $\widehat{\bcB}$ separately as below.
\begin{itemize}[leftmargin=*] 
	\item For any $k\in J_s$, due to the definition that
	$$\widetilde{\bgamma}_{\E_k}=\rmvec(\widetilde\E_k), \quad \widetilde\E_k = \A_k\widetilde\W_k,$$
	and the left singular vectors of $\A_k$ is $\U_k$ that satisfying $\|\U_k\|_0 = \sum_{i=1}^{p_k} 1_{\{(\U_k)_{[i, :]}\neq 0\}} \leq s_k$, $\widetilde{\bgamma}_{\E_k}$ is correspondingly group-wise sparse. More specifically, let $G_k^i = \{i, i+p_k, \ldots, i+p_k(r_k-1)\}$ with $i=1,\ldots, p_k$ be a partition of $\{1,\ldots, p_kr_k\}$. Then
	\begin{equation}\label{eq:group-wise-sparsity}
	\widetilde{\bgamma}_{\E_k}^i := (\widetilde{\bgamma}_{\E_k})_{G_k^i} \in \mathbb{R}^{r_k},\quad \sum_{i=1}^{p_k} 1_{\{\widetilde{\bgamma}_{\E_k}^i \neq 0\}} \leq s_k.
	\end{equation}
	Accordingly, $\widetilde{\X}_{\E_k}\in \mathbb{R}^{n_2\times (p_kr_k)}$ are with grouped covariates with respect to $\{G_k^1,\ldots, G_k^{p_k}\}$: 
	\begin{equation}\label{eq:tilde-X-E-k}
	\widetilde{\X}_{\E_k}^i = (\widetilde{\X}_{\E_k})_{[:, G_k^i]} \in \mathbb{R}^{n\times r_k}, \quad i=1,\ldots, p_k.
	\end{equation}
	Recall $\widehat{\bgamma}_{\E_k}$ is the group Lasso estimator,
	$$\widehat{\bgamma}_{\E_k} = \argmin_{\bgamma\in\mathbb{R}^{(p_kr_k)}} \|y - \widetilde{\X}_{\E_k} \bgamma\|_2^2 + \eta_k \sum_{i=1}^{p_k} \|\bgamma_{G_k^i}\|_2. $$
	By the group-wise sparsity structure \eqref{eq:group-wise-sparsity}\eqref{eq:tilde-X-E-k}, the partial linear regression model \eqref{eq:E_k-regression-model}, the assumption that $\widetilde{\X}_{\E_k}\in \mathbb{R}^{n_2\times (p_k r_k)}$ satisfies GRIP assumption with $\delta<1/4$, and $\eta_k = C\max_{1\leq i\leq p_k} \|(\widetilde{\X}_{\E_k}^i)^\top \widetilde{\bvarepsilon}_{\E_k}\|_2$ for constant $C\geq 3$, Lemma \ref{lm:group-oracle-RIP} yields
	\begin{equation}\label{ineq:hat-D_k-D_k}
	\|\widehat{\E}_k - \widetilde{\E}_k\|_F = \|\widehat{\bgamma}_{\E_k} - \widetilde{\bgamma}_{\E_k}\|_2 \leq \frac{C\sqrt{s_k} \eta_k}{n} \leq C\sqrt{s_k}\max_{1\leq i\leq p_k} \|(\widetilde{\X}_{\E_k}^i)^\top \widetilde{\bvarepsilon}_{\E_k}/n\|_2,\quad \forall k\in J_s.
	\end{equation}
	\item For $k\notin J_s$, recall $\widehat{\E}_k$ is evaluated via the least square estimator,
	\begin{equation*}
	\rmvec(\widehat{\E}_k) = \widehat{\bgamma}_{\E_k},\quad \widehat{\bgamma}_{\E_k} = \argmin_{\bgamma\in\mathbb{R}^{(p_kr_k)}} \left\|y - \widetilde{\X}_{\E_k}\bgamma\right\|_2^2.
	\end{equation*}
	By linear regression model \eqref{eq:E_k-regression-model} and the definition of the least square estimator,
	\begin{equation}\label{eq:hat_Ek-Ek-identity}
	\|\widehat{\E}_k - \widetilde{\E}_k\|_F = \|\widehat{\bgamma}_{\E_k} - \widetilde{\bgamma}_{\E_k}\|_2 = \left\|(\widetilde{\X}_{\E_k}^\top\widetilde{\X}_{\E_k})^{-1}\widetilde{\X}_{\E_k}^\top\widetilde{\bvarepsilon}_{\E_k}\right\|_2^2.
	\end{equation}
	\item In addition, recall 
	$$\rmvec(\widehat{\bcB}) = \widehat{\bgamma}_\bcB,\quad \widehat{\bgamma}_\bcB = \argmin_{\bgamma\in\mathbb{R}^{r_1r_2r_3}} \|y - \widetilde{\X}_\bcB\bgamma\|_2^2.$$
	By linear regression model \eqref{eq:B-regression-model} and the definition of the least square estimator $\widehat{\bgamma}_{\bcB}$,
	\begin{equation}\label{eq:hat_B-B-identity}
	\|\widehat{\bcB} - \bcB\|_{\tHS}^2 = \|\widehat{\bgamma}_\bcB - \bgamma_\bcB\|_2^2 = \|(\widetilde{\X}_\bcB^\top\widetilde{\X}_\bcB)^{-1}\widetilde{\X}_\bcB^\top \widetilde{\bvarepsilon}_\bcB\|_2^2.
	\end{equation}
\end{itemize}
Given $\theta = \max\{\|\sin\Theta(\widetilde{\U}_k, \U_k)\|, \|\sin\Theta(\widetilde{\W}_k, \W_k)\|\} \leq 1/2$, similarly as the proof of Theorem \ref{th:upper_bound_general}, one can show $\widetilde{\U}_k^\top \widetilde{\E}_k$ is non-singular. Therefore,
\begin{equation*}
\begin{split}
\|\widehat\bcB - \bcB\|_{\tHS}^2 + \sum_{k=1}^3 \|\widehat\E_k-\widetilde\E_k\|_F^2 \leq &  \left\|(\widetilde{\X}_\bcB^\top\widetilde{\X}_\bcB)^{-1}\widetilde{\X}_\bcB^\top \widetilde{\bvarepsilon}_\bcB\right\|_2^2 + C\sum_{k\in J_s} s_k\max_{1\leq i\leq p_k} \left\|(\widetilde{\X}_{\E_k}^i)^\top \widetilde{\bvarepsilon}_{\E_k}/n_2\right\|_2^2\\
& + \sum_{k \notin J_s} \left\|(\widetilde{\X}_{\E_k}^\top\widetilde{\X}_{\E_k})^{-1}\widetilde{\X}_{\E_k}^\top\widetilde{\bvarepsilon}_{\E_k}\right\|_2^2.
\end{split}
\end{equation*}
The rest of the proof directly follows from Steps 3 - 6 in Theorem \ref{th:upper_bound_sparse_general}.\quad $\square$ 

\subsection{Proof of Theorem \ref{th:upper_bound_regression}}\label{sec:proof_upper_bound_regression}

The goal of Theorem \ref{th:upper_bound_regression} is to give a probabilistic error bound for regular tensor regression via ISLET. 
The high level idea is to 
first derive the error bound for importance sketching regression by a perturbation bound of the HOOI outcome (Theorem 1 in \cite{zhang2019HOOI}), and then apply the oracle inequality in Theorem \ref{th:upper_bound_general} to obtain the final estimation error rate.  For a better presentation, we divide the long proof into six steps. First in Step 1, we bound the initialization error of $\widetilde{\U}_k^{(0)}$ using perturbation theory \cite{cai2018rate} and concentration inequality (Lemmas \ref{lm:concatenation-singular-value} and \ref{lm:concentration-Gaussian-ensemble}). Then in Step 2, we aim to apply Theorem 1 in \cite{zhang2019HOOI} to get an error bound for the importance sketching directions $\widetilde{\U}_k$. The central goal of Step 3 is to prove an error bound for $\theta$. In Steps 4, we move on to the second batch of sample and derive error bounds for a few intermediate terms. In step 5, we evaluate key quantities $\rho$ and $\left\|(\widetilde{\X}^\top\widetilde{\X})^{-1}\widetilde{\X}^\top \widetilde{\bvarepsilon}\right\|_2^2$ in the context of Theorem \ref{th:upper_bound_general}. Finally, we plug in all quantities to Theorem \ref{th:upper_bound_general} and finish the proof. 

We begin the proof by introducing some notations. Throughout the proof, the mode indices $(\cdot)_k$ are presented in modulo 3: e.g., $\U_4 = \U_1$, $\V_5 = \V_2$. For convenience, we denote 
$$\widetilde{\sigma}^2 = \|\bcA\|_{\tHS}^2 + \sigma^2, \quad \A_k = \mathcal{M}_k(\bcA),\quad  \widetilde{\A}_k = \mathcal{M}_k(\widetilde{\bcA}), \quad \X_{ik} = \mathcal{M}_k(\bcX_i)$$
for $k=1,2,3$. $p=\max\{p_1,p_2,p_3\}$, $r = \max\{r_1, r_2, r_3\}$. To avoid repeating similar notations consecutively, throughout the proof of this theorem we slightly abuse the notation and denote
$$\U_{k+2}\otimes \U_{k+1} = \left\{\begin{array}{ll}
\U_3\otimes \U_2, & k=1;\\
\U_3 \otimes \U_1, & k=2;\\
\U_2 \otimes \U_1, & k=3
\end{array}\right.$$
without ambiguity. Other related notations, e.g., $(\U_{k+2 \perp}\V) \otimes \U_{k+1}$, are defined in a similar fashion. 

The rest of the proof for Theorem \ref{th:upper_bound_regression} is divided into 6 steps.
\begin{enumerate}[leftmargin=*]
	\item[Step 1] We first develop the error bound for $\widetilde{\U}_1^{(0)}$, $\widetilde{\U}_2^{(0)}$, and $\widetilde{\U}_3^{(0)}$. Particularly, we aim to show that
	\begin{equation}\label{ineq:U_k^{(0)}-upper-bound}
	\bbP\left(\left\|\sin\Theta(\widetilde{\U}_k^{(0)}, \U_k)\right\| \leq \left(\frac{C\widetilde{\sigma}\sqrt{p_k/n_1}}{\lambda_k} + \frac{\widetilde{\sigma}^2 \sqrt{p_1p_2p_3}/n_1}{\lambda_k^2}\right)\wedge 1, k=1,2,3\right) \geq 1 - p^{-C}.
	\end{equation}
	We only focus on $\widetilde{\U}_1^{(0)}$ as the conclusions for $\widetilde{\U}_2^{(0)}$ and $\widetilde{\U}_3^{(0)}$ similarly follow. Recall the baseline unbiased estimator
	\begin{equation*}
	\widetilde{\bcA} = \frac{1}{n_1} \sum_{i=1}^{n_1} y_i^{(1)} \bcX_i^{(1)} =\frac{1}{n_1}\sum_{i=1}^{n_1} \left(\langle \bcX_i^{(1)}, \bcA \rangle + \varepsilon_i^{(1)}\right)\bcX_i^{(1)} \in \mathbb{R}^{p_1\times p_2\times p_3}.
	\end{equation*}
	Since the left and right singular subspaces of $\A_1$ are $\U_1$ and $\W_1$, respectively, we further have $\widetilde{\A}_1\in \mathbb{R}^{p_1\times (p_2p_3)}$ and
	\begin{equation*}
	\begin{split}
	\widetilde{\A}_1 = & \mathcal{M}_1\left(\widetilde{\bcA}\right) = \frac{1}{n_1} \sum_{i=1}^n y_i^{(1)} \X_{i1}^{(1)} = \frac{1}{n_1}\sum_{i=1}^{n_1} \left(\langle \X_{i1}^{(1)}, \A_1 \rangle +\varepsilon_i^{(1)}\right) \X_{i1}^{(1)} \\
	= & \frac{1}{n_1}\sum_{i=1}^{n_1} \left(\langle \X_{i1}^{(1)}, P_{\U_1}\A_1 P_{\W_1}\rangle +\varepsilon_i^{(1)}\right)\X_{i1}^{(1)}\\
	= & \frac{1}{n_1}\sum_{i=1}^{n_1} \left(\tr\left((\X_{i1}^{(1)})^\top \U_1\U_1^\top \A_1 \W_1 \W_1^\top\right) +\varepsilon_i^{(1)}\right)\X_{i1}^{(1)}\\
	= & \frac{1}{n_1}\sum_{i=1}^{n_1} \left(\langle \U_1^\top \X_{i1}^{(1)} \W_1, \U_1^\top \A_1 \W_1\rangle + \varepsilon_i^{(1)}\right)\X_{i1}^{(1)}.
	\end{split}
	\end{equation*}
	Since $\widetilde{\U}^{(0)}_1 = \SVD_{r_1}(\widetilde{\A}_1)$, the one-sided perturbation bound \cite[Proposition 1]{cai2018rate} yields
	\begin{equation}\label{ineq:th2-initialization-perturbation}
	\begin{split}
	\left\|\sin\Theta\left(\widetilde{\U}_1^{(0)}, \U_1\right)\right\| \leq \frac{\sigma_{r_1}(\U_1^\top \widetilde{\A}_1)\|\U_{1\perp}^\top \widetilde{\A}_1 P_{(\U_1^\top \widetilde{\A}_1)^\top}\|}{\sigma_{r_1}^2(\U_1^\top \widetilde{\A}_1) - \sigma_{r_1+1}^2(\widetilde{\A}_1)}\wedge 1
	\end{split}
	\end{equation}
	To proceed, we analyze $\sigma_{\min}^2\left(\U_1^\top \widetilde{\A}_1\right)$, $\sigma_{r_1+1}(\widetilde\A_1)$, and $\|\U_{1\perp}^\top \widetilde{\A}_1 P_{(\U_1^\top \widetilde{\A}_1)^\top}\|$, respectively. 
	\begin{itemize}[leftmargin=*]
		\item 
		\begin{equation*}
		\begin{split}
		& \sigma_{\min}^2\left(\U_1^\top \widetilde{\A}_1\right) \overset{\text{Lemma \ref{lm:concatenation-singular-value}}}{\geq} \sigma_{\min}^2\left(\U_1^\top \widetilde{\A}_1 \W_1\right) + \sigma_{\min}^2\left(\U_1^\top \widetilde{\A}_1 (\W_1)_{\perp}\right)\\
		= & \sigma_{\min}^2\left(\frac{1}{n_1}\sum_{i=1}^{n_1} \left(\langle \U_1^\top \X_{i1}^{(1)} \W_1, \U_1^\top \A_1 \W_1\rangle + \varepsilon_i^{(1)}\right)\U_1^\top \X_{i1}^{(1)}\W_1\right)\\
		& + \sigma_{\min}^2\left(\frac{1}{n_1}\sum_{i=1}^{n_1} \left(\langle \U_1^\top \X_{i1}^{(1)} \W_1, \U_1^\top \A_1 \W_1\rangle + \varepsilon_i^{(1)}\right)\U_1^\top \X_{i1}^{(1)}(\W_1)_{\perp}\right).
		\end{split}
		\end{equation*}
		By Lemma \ref{lm:concentration-Gaussian-ensemble}, $\U_1^\top \A_1\W_1\in \mathbb{R}^{r_1\times r_1}$, and $n_1 \geq Cp^{3/2}r_1$, we have
		\begin{equation*}
		\begin{split}
		& \sigma_{\min}\left(\frac{1}{n_1}\sum_{i=1}^{n_1} \left(\langle \U_1^\top \X_{i1}^{(1)} \W_1, \U_1^\top \A_1 \W_1\rangle + \varepsilon_i^{(1)}\right)\U_1^\top \X_{i1}^{(1)}\W_1\right) \\
		\geq & \sigma_{\min}(\U_1^\top \A_1 \W_1) - \left\|\frac{1}{n_1}\sum_{i=1}^{n_1} \left(\langle \U_1^\top \X_{i1}^{(1)} \W_1, \U_1^\top \A_1 \W_1\rangle + \varepsilon_i^{(1)}\right)\U_1^\top \X_{i1}^{(1)} \W_1 - \U_1^\top \A_1 \W_1\right\|\\
		\overset{\text{Lemma \ref{lm:concentration-Gaussian-ensemble}}}{\geq} & \sigma_{r_1}(\A_1) - C\sqrt{\frac{\log p }{n_1} \left(2r_1\|\A_1\|_F^2 + \sigma^2\right)} \geq (1-c)\sigma_{r_1}(\A_1)
		\end{split}
		\end{equation*}
		with probability at least $1 - p^{-c}$. When $\X_{i1}^{(1)}$ has i.i.d. Gaussian entries and $\W_1$ is fixed orthogonal matrix, $\U_1^\top \X_{i1}^{(1)}(\W_1)_{\perp}\in \mathbb{R}^{r_1\times (p_{-1}-r_1)}$ and $\left(\langle \U_1^\top \X_{i1}^{(1)} \W_1, \U_1^\top \A_1 \W_1\rangle + \varepsilon_i\right)\in \mathbb{R}$ are independently Gaussian distributed and 
		$$\left\langle \U_1^\top \X_{i1}^{(1)} \W_1, \U_1^\top \A_1 \W_1\right\rangle + \varepsilon_i^{(1)}\sim N(0, \widetilde{\sigma}^2).$$ 
		By Lemma \ref{lm:concentration-independent},
		\begin{equation*}
		\begin{split}
		& \sigma_{\min}^2\left(\frac{1}{n_1}\sum_{i=1}^{n_1} \left(\langle \U_1^\top \X_{i1}^{(1)} \W_1, \U_1^\top \A_1 \W_1\rangle + \varepsilon_i^{(1)}\right)\U_1^\top \X_{i1}^{(1)}\W_{1\perp}\right) \\
		\geq & \widetilde{\sigma}^2 \cdot \frac{n_1-C_1\sqrt{n_1\log p}}{n_1^2}\cdot \left(\sqrt{p_{-1}-r_1} - \sqrt{r_1} - C_2\sqrt{\log p}\right)^2\\
		\geq & \frac{\widetilde{\sigma}^2}{n_1} \cdot \left(1 - C_1\sqrt{\frac{\log p}{n_1}}\right)\cdot \left(p_{-1} - C_3\sqrt{p_{-1}r_1}-C_2\sqrt{p_{-1}\log p}\right)\\
		\geq & \frac{\widetilde{\sigma}^2}{n_1} \left(p_{-1} - C_4\sqrt{p_{-1}r_1} - C_5\sqrt{p_{-1}\log p}\right)
		\end{split}
		\end{equation*}
		with probability at least $1-p^{-c}$. To sum up,
		\begin{equation}\label{ineq:thm3-1}
		\sigma_{\min}^2\left(\U_1^\top \widetilde\A_1\right) \geq (1-c)\sigma_{r_1}^2(\A_1) + \frac{\widetilde{\sigma}^2}{n_1} \cdot \left(p_{-1} - C_1 \sqrt{p_{-1}r_1}-C_2 \sqrt{p_{-1}\log p}\right)
		\end{equation}
		with probability at least $1 - p^{-c}$. 
		\item Next, we consider $\sigma_{r_1+1}(\widetilde\A_1)$, note that
		\begin{equation*}
		\begin{split}
		\sigma_{r_1+1}(\widetilde\A_1) = & \min_{\rank(M)\leq r_1} \left\|\widetilde\A_1 - \M\right\| \leq \left\|\widetilde{\A}_1 - P_{\U_1}\widetilde\A_1\right\| \leq \|\U_{1\perp}^\top \widetilde\A_1\|\\
		= & \left\|\frac{1}{n_1} \sum_{i=1}^{n_1} \left(\langle \U_1^\top \X_{i1}^{(1)}, \U_1^\top \A_1 \rangle +\varepsilon_i^{(1)}\right)\U_{1\perp}^\top \X_{i1}^{(1)}\right\|.
		\end{split}
		\end{equation*}
		Since
		\begin{equation*}
		\left(\langle \U_1^\top \X_{i1}^{(1)}\W_1, \U_1^\top \A_1 \W_1\rangle +\varepsilon_i^{(1)}\right) \sim N\left(0, \widetilde{\sigma}^2\right),
		\end{equation*}
		which is also independent of $\U_{1\perp}^\top \X_{i1}^{(1)}$. Thus,
		\begin{equation}\label{ineq:thm3-2}
		\begin{split}
		\sigma_{r_1+1}^2(\widetilde\A_1) = & \left\|\frac{1}{n_1} \sum_{i=1}^{n_1} \left(\langle \U_1^\top \X_{i1}^{(1)}, \U_1^\top \A_1 \rangle +\varepsilon_i^{(1)}\right)\U_{1\perp}^\top \X_{i1}^{(1)}\right\|^2 \\
		\leq & \widetilde{\sigma}^2\cdot \frac{n_1+C(\sqrt{n_1\log p} + \log p)}{n_1^2}\cdot \left(\sqrt{p_1 -r_1} + \sqrt{p_{-1}} + C\sqrt{\log p}\right)^2\\
		\leq & \frac{\widetilde\sigma^2}{n_1}\left(1 + C\sqrt{\frac{\log p}{n_1}}\right)\left(p_{-1}+C\sqrt{p_{-1}p_1} + C\sqrt{p_{-1}\log p} + Cp_1 + C\log p\right)\\
		\leq & \frac{\widetilde{\sigma}^2}{n_1} \cdot \left(p_{-1} +  C\sqrt{p_{-1}p_1} + C\sqrt{p_{-1}\log p} + Cp_1 + C\log p\right)
		\end{split}
		\end{equation}
		with probability at least $1 - p^{-c}$. 
		\item Then we consider $\left\|\U_{1\perp}^\top \widetilde\A_1 P_{(\U_1^\top \widetilde\A_1)^\top}\right\|$. Note that
		\begin{equation*}
		\U_{1\perp}^\top \widetilde\A_1 P_{(\U_1^\top \widetilde\A_1)^\top} = \frac{1}{n_1}\sum_{i=1}^{n_1} \left(\langle \U_1^\top \X_{i1}^{(1)}\W_1, \U_1^\top \A_1\W_1 \rangle +\varepsilon_i^{(1)}\right) \U_{1\perp}^\top \X_{i1}^{(1)} P_{(\U_1^\top \widetilde\A_1)^\top},
		\end{equation*}
		Here, $\left(\langle \U_1^\top \X_{i1}^{(1)} \W_1, \U_1^\top \A_1\W_1\rangle +\varepsilon_i\right)\sim N(0, \widetilde{\sigma}^2)$; by independence, conditioning on fixed value of $\U_1^\top \X_{i1}^{(1)}$, $\U_{1\perp}^\top \X_{i1}^{(1)}$ is still standard normal, and then 
		$$\U_{1\perp}^\top \X_{i1}^{(1)} P_{(\U_1^\top \widetilde\A_1)^\top}\Big| \U_1^\top\X_{i1}^{(1)}$$ 
		is a $(p_1-r_1)$-by-$r_1$ i.i.d. standard Gaussian matrix. By Lemma \ref{lm:concentration-independent}, we have
		\begin{equation}\label{ineq:thm3-3}
		\begin{split}
		\left\|\U_{1\perp}^\top \widetilde\A_1 P_{(\U_1^\top \widetilde\A_1)^\top}\right\| \leq &  \widetilde{\sigma} \sqrt{\frac{n_1+C_1\sqrt{n_1\log p} + C_2\log p}{n_1^2}}\cdot \left(\sqrt{p_1-r_1} + \sqrt{r_1} + C_3\sqrt{\log p}\right) \\
		\leq & C_4\widetilde{\sigma} \cdot \sqrt{\frac{p_1}{n_1}}
		\end{split}
		\end{equation}
		with probability at least $1 - p^{-C}$.
	\end{itemize}
	Combining \eqref{ineq:thm3-1}-\eqref{ineq:thm3-3} with \eqref{ineq:th2-initialization-perturbation}, we have the following inequality holds with probability at least $1 - p^{-C}$,
	\begin{equation*}
	\begin{split}
	& \left\|\sin\Theta\left(\widetilde{\U}_1^{(0)}, \U_1\right)\right\|\\
	\leq & \frac{\sigma_{r_1}(\U_1^\top \widetilde\A_1)\|\U_{1\perp}^\top \widetilde\A_1 P_{(\U_1^\top \widetilde\A_1)^\top}\|}{\sigma_{r_1}^2(\U_1^\top \widetilde\A_1) - \sigma_{r_1+1}^2(\widetilde\A_1)}\wedge 1\\ 
	\leq & \frac{\left((1-c)\sigma_{r_1}(\A_1) + \widetilde{\sigma}\sqrt{p_{-1}/n_1}\right)\cdot C_1\widetilde{\sigma} \sqrt{p_1/n_1}}{\left((1-c)\sigma_{r_1}(\A_1) + \widetilde{\sigma}\sqrt{p_{-1}/n_1}\right)^2 - \frac{\widetilde{\sigma}^2}{n_1} \cdot \left(p_{-1} +  C_2\sqrt{p_{-1}p_1} + C_3\sqrt{p_{-1}\log p} + C_4p_1 + C_5\log p\right)}\wedge 1\\
	\end{split}
	\end{equation*}
	Since $n_1 \geq Cp^{3/2}\widetilde\sigma^2/\lambda_0^2$ for large constant $C>0$, we have
	\begin{equation*}
	\begin{split}
	& \left((1-c)\sigma_{r_1}(\A_1) + \widetilde{\sigma}\sqrt{p_{-1}/n_1}\right)^2 - \frac{\widetilde{\sigma}^2}{n_1} \cdot \left(p_{-1} +  C_1\sqrt{p_{-1}p_1} + C_2\sqrt{p_{-1}\log p} + C_3p_1 + C_4\log p\right) \\
	\geq & (1-c)^2\sigma_{r_1}^2(\A_1) + 2(1-c)\sigma_{r_1}(\A_1)\widetilde\sigma\sqrt{p_{-1}/n_1} - \frac{C_2\widetilde\sigma^2}{n_1}\left(\sqrt{p_1p_2p_3} + \sqrt{p_{-1}\log p} + C_3p_1 + C_4\log p\right)\\
	\geq & c\sigma_{r_1}^2(\A_1)
	\end{split}
	\end{equation*}
	and additionally,
	\begin{equation*}
	\begin{split}
	\left\|\sin\Theta\left(\widetilde{\U}_1^{(0)}, \U_1\right)\right\| \leq & \left(\frac{C_1\widetilde{\sigma}\sqrt{p_1/n_1} \cdot \sigma_{r_1}(\A_1) + \widetilde{\sigma}^2 \sqrt{p_1p_2p_3}/n_1}{\sigma_{r_1}^2(\A_1)}\right) \wedge 1.
	\end{split}
	\end{equation*}
	with probability at least $1 - p^{-C}$. Similar inequalities also hold for $\left\|\sin\Theta\left(\widetilde{\U}_2^{(0)}, \U_2\right)\right\|$ and $\left\|\sin\Theta\left(\widetilde{\U}_3^{(0)}, \U_3\right)\right\|$.
	Based on these arguments, we conclude that \eqref{ineq:U_k^{(0)}-upper-bound} holds.  \eqref{ineq:U_k^{(0)}-upper-bound} further implies that
	\begin{equation}\label{ineq:regression-prob-0}
	\begin{split}
	e_0 := &  \max_k\left\|\widetilde{\U}^{(0)\top}_{k\perp}\mathcal{M}_k(\bcA)\right\| = \max_k \left\|\widetilde{\U}_{k\perp}^{(0)\top} \U_k\U_k^\top \mathcal{M}_k(\bcA)\right\| \\
	\leq & \max_k \|\widetilde{\U}^{(0)\top}_{k\perp} \U_k\|\cdot \|\U_k^\top\mathcal{M}_k(\bcA)\| \leq \max_k \|\sin\Theta(\widetilde{\U}_k^{(0)}, \U_k)\| \cdot \|\U_k^\top \mathcal{M}_k(\bcA)\|\\
	\leq & \max_k C\|\A_k\|\left(\frac{\widetilde{\sigma}\sqrt{p_k/n_1}}{\sigma_{r_k}(\A_k)} + \frac{\widetilde{\sigma}^2\sqrt{p_1p_2p_3}/n_1}{\sigma_{r_k}^2(\A_k)}\right)\\
	\leq & C_1\kappa \left(\frac{\widetilde{\sigma}p^{1/2}}{n_1^{1/2}} + \frac{\widetilde{\sigma}^2p^{3/2}}{\lambda_0n_1}\right)
	\end{split}
	\end{equation}
	with probability at least $1 - p^{-C}$.
	\item[Step 2] Then we develop the error bound for $\widetilde{\U}_k$ after enough number of iterations in this step. In particular, we aim to apply Theorem 1 in \cite{zhang2019HOOI} to give an error bound for the output $\widetilde{\U}_k$ from the high-order order orthogonal iteration (HOOI). To this end, we verify the conditions in Theorem 1 in \cite{zhang2019HOOI} in this step. Defining
	\begin{equation}\label{eq:def-Z-T-tilde-T}
	\bcZ = \widetilde{\bcA} - \bcA,\quad \bcT = \bcA + \bcZ \times_1 P_{\U_1}\times_2 P_{\U_2}\times_3 P_{\U_3},\quad \widetilde{\bcT} = \widetilde{\bcA}.
	\end{equation}
	Then,
	\begin{equation}\label{eq:tilde-T-T}
	\widetilde{\bcT} - \bcT = \bcZ - \bcZ \times_1 P_{\U_1} \times_2 P_{\U_2} \times_3 P_{\U_3}.
	\end{equation}
	In order to apply Theorem 1 in \cite{zhang2019HOOI}, we develop the following upper bounds under the assumptions of Theorem \ref{th:upper_bound_regression}.
	\begin{itemize}[leftmargin=*]
		\item Since $\mathcal{M}_1\left((\widetilde\bcA - \bcA)\times_1\U_1^\top\times_2\U_2^\top\times_3\U_3^\top\right)$ is a $r_1$-by-$(r_2r_3)$ matrix, Lemma \ref{lm:concentration-Gaussian-ensemble} implies
		\begin{equation}
		\begin{split}
		& \left\|\mathcal{M}_1\left((\widetilde{\bcA} - \bcA)\times_1 \U_1^\top \times_2 \U_2^\top \times_3 \U_3^\top\right)\right\|\\
		= & \left\|\U_1^\top \mathcal{M}_1\left(\widetilde{\bcA} - \bcA\right)(\U_3 \otimes \U_2)\right\|\\
		= & \Bigg\|\frac{1}{n_1} \sum_{i=1}^{n_1} \left(\left\langle \U_1^\top \X_{i1}^{(1)}(\U_3 \otimes \U_2), \U_1^\top \A_{1}(\U_3 \otimes \U_2) \right\rangle +\varepsilon_i^{(1)}\right) \U_1^\top \X_{i1}^{(1)}\left(\U_3 \otimes \U_2\right)\\
		& - \U_1^\top \A_1 (\U_3\otimes \U_2)\Bigg\|\\
		\overset{\text{Lemma \ref{lm:concentration-Gaussian-ensemble}}}{\leq} & C_1\sqrt{\frac{\log p \cdot (r_1+r_2r_{3})\widetilde{\sigma}^2}{n_1}}\\
		\end{split}
		\end{equation}
		with probability at least $1 - p^{-C}$. Similar results also hold for $\mathcal{M}_2(\cdot)$ and $\mathcal{M}_3(\cdot)$. Then
		\begin{equation}
		\begin{split}
		\lambda_k(\bcT) := &  \sigma_{r_k}\left(\mathcal{M}_k(\bcT)\right)\\
		\overset{\eqref{eq:tilde-T-T}}{\geq} & \sigma_{r_k} \left(\mathcal{M}_k(\bcA)\right) - \left\|\mathcal{M}_k\left( (\widetilde{\bcA} - \bcA) \times_1 P_{\U_1} \times_2 P_{\U_2} \times_3 P_{\U_3} \right)\right\|\\
		\geq & \lambda_k - C_1\sqrt{\frac{\log p \cdot (r_k+r_{k+1}r_{k+2})\widetilde{\sigma}^2}{n_1}}\geq (1-c)\lambda_0
		\end{split}
		\end{equation}
		with probability at least $1 - p^{-C}$.
		\item Next, we consider
		$$\tau_{0k}: = \left\|\mathcal{M}_k(\widetilde{\bcT} - \bcT) \left(\U_{k+2}\otimes \U_{k+1}\right)\right\|,\quad k=1,2,3.$$
		In particular,
		\begin{equation}
		\begin{split}
		& \left\|\mathcal{M}_1(\widetilde{\bcT} - \bcT) \left(\U_{3}\otimes \U_{2}\right)\right\| \\
		\overset{\eqref{eq:tilde-T-T}}{=} & \left\|\mathcal{M}_1(\bcZ - \llbracket\bcZ; P_{\U_1}, P_{\U_2}, P_{\U_3}\rrbracket)(\U_3\otimes \U_2)\right\|\\
		= & \left\|\mathcal{M}_1\left(\left(\widetilde{\bcA} - \bcA - \llbracket\widetilde{\bcA}-\bcA; P_{\U_1}, P_{\U_2}, P_{\U_3}\rrbracket\right)\times_{2} \U_{2}^\top\times_{3} \U_{3}^\top\right)\right\|\\
		= & \Big\|\mathcal{M}_1\left((\widetilde{\bcA} - \bcA)\times_1(P_{\U_1}+P_{\U_{1\perp}}) \times_{2} \U_{2}^\top \times_{3} \U_{3}^\top\right)\\
		&  - \mathcal{M}_1\left((\widetilde{\bcA}-\bcA) \times_1 P_{\U_1}\times_2 \U_{2}^\top \times_3 \U_{3}^\top\right)\Big\|\\
		= & \left\|\mathcal{M}_1\left((\widetilde{\bcA} - \bcA)\times_1P_{\U_{1\perp}} \times_2 \U_2^\top \times_3 \U_3^\top\right)\right\|\\
		= & \left\|\U_{1\perp}^\top (\widetilde{\A}_1 - \A_1) \cdot (\U_3\otimes\U_2)\right\|\\
		\leq & \Big\|\frac{1}{n_1}\sum_{i=1}^{n_1}\left(\langle \U_1^\top \X_{i1}^{(1)} (\U_{3}\otimes \U_{2}), \U_1^\top \A_1 (\U_{3}\otimes \U_{2})\rangle +\varepsilon_i^{(1)}\right)\U_{1\perp}^\top\X_{i1}^{(1)}(\U_3\otimes \U_2)\\
		& - \U_{1\perp}^\top \A_1 (\U_{3}\otimes \U_{2}) \Big\|\\
		\overset{\text{Lemma \ref{lm:concentration-independent}}}{\leq} & \widetilde{\sigma}\sqrt{\frac{n_1+C_1\sqrt{n_1\log p}}{n_1^2}}\left(\sqrt{p_1-r_1} + \sqrt{r_2 r_3} + C_2\sqrt{\log p}\right) \leq C_3\widetilde{\sigma}\sqrt{\frac{p_1}{n_1}},
		\end{split}
		\end{equation}
		with probability at least $1 - p^{-C}$. Thus,
		\begin{equation}\label{ineq:regression-prob-1}
		\begin{split}
		\bbP\left(\tau_{0k} \leq C_1\widetilde{\sigma}  \sqrt{p_k/n_1},~~ k=1,2,3\right) \geq 1 - p^{-C}.
		\end{split}
		\end{equation}
		\item Next we consider the upper bound of
		\begin{equation}\label{eq:def-tau_1}
		\begin{split}
		\tau_1 := & \max_{k}\Big\{\max_{\substack{\V\in \mathbb{R}^{(p_{k+1}-r_{k+1})\times r_{k+1}}\\\|\V\|\leq 1}}\left\|\mathcal{M}_k(\widetilde{\bcT}-\bcT) \cdot \left\{(\U_{k+2, \perp} \V) \otimes \U_{k+1}\right\}\right\|,\\
		& \quad\quad\quad\quad  \max_{\substack{\V\in \mathbb{R}^{(p_{k+2}-r_{k+2})\times r_{k+2}}\\\|\V\|\leq 1}}\left\|\mathcal{M}_k(\widetilde\bcT - \bcT) \cdot \left\{\U_{k+2} \otimes (\U_{k+1, \perp}\V)\right\}\right\| \Big\}.\\
		\end{split}
		\end{equation}
		Note that
		\begin{equation*}
		\begin{split}
		& \mathcal{M}_1\left(\widetilde{\bcT}-\bcT\right) (\U_{3\perp}\V)\otimes \U_2\\ 
		= & \left(\mathcal{M}_1(\bcZ) - \mathcal{M}_1\left(\bcZ \times_1 P_{\U_1} \times_2 P_{\U_2} \times_3 P_{\U_3}\right)\right) (\U_{3\perp}\V)\otimes \U_2 \\
		= & \mathcal{M}_1(\bcZ) (\U_{3\perp} \V)\otimes \U_2 = \frac{1}{n_1}\sum_{i=1}^{n_1} y_i^{(1)} \X_{i1}^{(1)}((\U_{3\perp}\V)\otimes \U_2),
		\end{split}
		\end{equation*}
		$$y_i^{(1)} = \langle\bcX_i^{(1)}, \bcA\rangle+\varepsilon_i^{(1)} = \langle\U_1^\top \X_{i1}^{(1)}(\U_3\otimes \U_2), \U_1^\top\A_1(\U_3\otimes \U_2) \rangle+\varepsilon_i^{(1)}.$$ 
		Since $\U_{3\perp}$ and $\U_3$ are orthogonal, $y_i^{(1)}$ and $\X_{i1}^{(1)}(\U_{3\perp}\otimes \U_2)$ are independently Gaussian distributed. Thus, conditioning on fixed values of $\{y_i^{(1)}\}_{i=1}^{n_1}$,  
		$$\frac{1}{n_1}\sum_{i=1}^{n_1} y_i^{(1)} \X_{i1}^{(1)}(\U_{3\perp}\otimes \U_2)\bigg| \|\y^{(1)}\|_2^2$$
		is a $p_1$-by-$((p_2-r_2)r_3)$ random matrix with i.i.d. Gaussian entries with mean zero and variance $\|\y^{(1)}\|_2^2/n_1^2$. By Lemma 5 in \cite{zhang2017tensor},
		\begin{equation}\label{ineq:regression-prob-4}
		\begin{split}
		& \bbP\Bigg(\max_{\V\in \mathbb{R}^{(p_2-r_2)\times r_2}}\left\|\mathcal{M}_1\left(\bcZ (\U_{3\perp}\V \otimes \U_2)\right)\right\| \\
		& \quad \quad \geq \frac{C\|\y^{(1)}\|_2}{n_1}\left(\sqrt{p_1} + \sqrt{r_2r_3} + \sqrt{1+t}(\sqrt{p_2r_2}+\sqrt{p_3r_3})\right)\Bigg|\|\y^{(1)}\|_2^2\Bigg) \\
		\leq & C\exp\left(-Ct(p_2r_2+p_3r_3)\right).
		\end{split}
		\end{equation}
		Note that $\|\y^{(1)}\|_2^2\sim \widetilde{\sigma}^2\chi^2_{n_1}$, we have
		\begin{equation}\label{ineq:regression-prob-5}
		\begin{split}
		\bbP\left(\|\y^{(1)}\|_2^2 \geq \widetilde{\sigma}^2(n_1 + 2\sqrt{n_1t} + 2t)\right) \leq \exp(-t).
		\end{split}
		\end{equation}
		Combining \eqref{ineq:regression-prob-4} (with $t = pr/(p_2r_2+p_3r_3)$), \eqref{ineq:regression-prob-5} (with $t = Cpr$), and the fact that $n_1\geq Cpr$ for large constant $C>0$, we have
		\begin{equation*}
		\begin{split}
		& \bbP\left(\max_{\substack{\V\in \mathbb{R}^{(p_3-r_2)\times r_1}\\\|\V\|\leq 1}}\left\|\mathcal{M}_1\left(\widetilde{\bcT}-\bcT\right) (\U_{3\perp}\V)\otimes \U_2\right\| \geq C\widetilde{\sigma}\sqrt{\frac{pr}{n_1}}\right)	\leq C\exp\left(-cpr\right).
		\end{split}
		\end{equation*}
		By symmetry, we have similar results for other terms in the right hand side of \eqref{eq:def-tau_1} and the following conclusion,
		\begin{equation}\label{ineq:regression-prob-2}
		\begin{split}
		& \bbP\left(\tau_1 \geq C\widetilde{\sigma}\sqrt{\frac{pr}{n_1}}\right) \leq C\exp(-cpr).
		\end{split}
		\end{equation}
		\item Based on essentially the same argument as the previous step, we can also show
		\begin{equation}\label{ineq:regression-prob-3}
		\begin{split}
		\tau_2 := & \max_k \max_{\substack{\V\in \mathbb{R}^{(p_{k+1}-r_{k+1})\times r_{k+1}}: \|\V\|\leq 1;\\ \V' \in \mathbb{R}^{(p_{k+2}-r_{k+2})\times r_{k+2}}: \|\V'\|\leq 1}} \left\|\mathcal{M}_k(\bcZ)\left\{(\U_{k+1 \perp}\V) \otimes (\U_{k+2 \perp}\V')\right\}\right\|\\
		\leq & C\widetilde{\sigma}\sqrt{\frac{pr}{n_1}}
		\end{split}
		\end{equation}
		with probability at least $1 - C\exp(-cpr)$.
	\end{itemize}
	Now, when the statements in \eqref{ineq:regression-prob-1}, \eqref{ineq:regression-prob-2}, \eqref{ineq:regression-prob-3} all hold, given $n_1\geq \frac{\widetilde{\sigma}^2}{\lambda_0^2}(\kappa pr\vee p^{3/2})$ for large enough constant $C>0$, we have $n_1 \geq \frac{C\widetilde{\sigma}^2}{\lambda_0^2}p^{4/3}r^{1/3}$ (by H\"older's inequality) and the condition 
	\begin{equation*}
	\begin{split}
	& \frac{\tau_1}{\lambda(\bcT)} + \max_k\frac{4\tau_2(4\tau_{0k}+e_0)}{\lambda^2(\bcT)} \\
	\leq & \frac{C_1\widetilde{\sigma}\sqrt{pr/n_1}}{\lambda_0} + \frac{C_2\widetilde{\sigma}\sqrt{pr/n_1}\left(\widetilde{\sigma}\sqrt{p/n_1}+\kappa\widetilde{\sigma}\sqrt{p/n_1} + \kappa\widetilde{\sigma}^2 p^{3/2}/(\lambda_0n_1)\right)}{\lambda_0^2}\\
	\leq & \frac{C_1\widetilde{\sigma}p^{1/2}r^{1/2}}{\lambda_0n_1^{1/2}} + \frac{C_2\widetilde{\sigma}^2\kappa pr^{1/2}}{\lambda_0^2n_1} + \frac{C_3\kappa\widetilde{\sigma}^3p^2r^{1/2}}{\lambda_0^3n_1^{3/2}}\leq 1
	\end{split}
	\end{equation*} 
	holds. Namely, the condition in Theorem 1 in \cite{zhang2019HOOI} holds when the events of \eqref{ineq:regression-prob-1}, \eqref{ineq:regression-prob-2}, \eqref{ineq:regression-prob-3} occur.
	\item[Step 3] In this step, we try to establish the estimation errors for $\widetilde{\U}_k$ and $\widetilde{\W}_k$.
	First, Theorem 1 in \cite{zhang2019HOOI} and \eqref{ineq:regression-prob-1}, \eqref{ineq:regression-prob-2}, \eqref{ineq:regression-prob-3} imply
	\begin{equation*}
	\begin{split}
	& \left\|\sin\Theta\left(\widetilde{\U}_k, \U_k\right)\right\| \leq \frac{C\tau_{0k}}{\sigma_{r_k}(\mathcal{M}_k(\bcT))}\leq \frac{C\widetilde{\sigma} \sqrt{p_k/n_1}}{\lambda_k},\quad k=1,2,3,
	\end{split}
	\end{equation*}
	\begin{equation*}
	\begin{split}
	\text{and} \quad & \left\|\llbracket\widetilde{\bcT}; P_{\widetilde{\U}_1}, P_{\widetilde{\U}_2}, P_{\widetilde{\U}_3}\rrbracket - \bcT\right\|_{\tHS} \leq C\widetilde{\sigma}\sqrt{\frac{p_1r_1+p_2r_2+p_3r_3+r_1r_2r_3}{n_1}}
	\end{split}
	\end{equation*}
	with probability at least $1 - p^{-C}$.	Moreover,
	\begin{equation*}
	\begin{split}
	& \left\|\bcT  - \bcA\right\|_{\tHS} \overset{\eqref{eq:def-Z-T-tilde-T}}{=} \left\|\left(\widetilde{\bcA} -\bcA\right)\times_1 \U_1^\top \times_2 \U_2^\top \times_3 \U_3^\top\right\|_{\tHS}\\
	= & \Big\|\frac{1}{n_1}\sum_{i=1}^{n_1}\left(\left\langle \rmvec(\bcX_i\times_1 \U_1^\top \times_2 \U_2^\top \times_3 \U_3^\top), \rmvec(\bcA\times_1 \U_1^\top \times_2 \U_2^\top \times_3 \U_3^\top) \right\rangle+\varepsilon_i\right)\\
	& \quad \cdot \rmvec(\bcX_i\times_1 \U_1^\top \times_2 \U_2^\top \times_3 \U_3^\top) - \rmvec(\bcA\times_1 \U_1^\top \times_2 \U_2^\top \times_3 \U_3^\top)\Big\|_2\\
	\overset{\text{Lemma \ref{lm:concentration-Gaussian-ensemble}}}{\leq} & C\sqrt{\frac{\widetilde{\sigma}^2}{n_1}}\left(\sqrt{r_1r_2r_3} + \sqrt{\log p}\right)
	\end{split}
	\end{equation*}
	with probability at least $1 - p^{-C}$. Combing the previous two inequalities, we have
	\begin{equation}\label{ineq:regression-prob-6}
	\begin{split}
	& \left\|\llbracket\widetilde{\bcA}; P_{\widetilde{\U}_1}, P_{\widetilde{\U}_2}, P_{\widetilde{\U}_3}\rrbracket - \bcA\right\|_{\tHS} \\
	\leq & \left\|\llbracket\widetilde{\bcT}; P_{\widetilde{\U}_1}, P_{\widetilde{\U}_2}, P_{\widetilde{\U}_3}\rrbracket - \bcT\right\|_{\tHS} + \left\|\bcA - \bcT\right\|_{\tHS}\\
	\leq & C\widetilde{\sigma}\sqrt{\frac{p_1r_1+p_2r_2+p_3r_3 + r_1r_2r_3}{n_1}}\asymp C\widetilde\sigma \sqrt{m/n_1}
	\end{split}
	\end{equation}
	with probability at least $1 - p^{-C}$. Then, for $k=1,2,3$,
	\begin{equation*}
	\begin{split}
	& \|\widetilde{\U}_{k\perp}^\top \A_k\|_F \leq  \left\|\widetilde{\U}_{k\perp}^\top \left(P_{\widetilde{\U}_k}\widetilde{\A}_k(P_{\widetilde{\U}_{k+2}}\otimes P_{\widetilde{\U}_{k+1}}) -\A_k\right)\right\|_F\\
	\leq & \left\|P_{\widetilde{\U}_k}\widetilde{\A}_k(P_{\widetilde{\U}_{k+2}}\otimes P_{\widetilde{\U}_{k+1}}) -\A_k\right\| = \left\|\llbracket\widetilde{\bcA}; P_{\widetilde{\U}_1}, P_{\widetilde{\U}_2}, P_{\widetilde{\U}_3}\rrbracket - \bcA\right\|_{\tHS} \leq C\widetilde{\sigma}\sqrt{m/n_1}
	\end{split}
	\end{equation*}
	with probability at least $1 - p^{-C}$. 
	
	Next, we are in the position of evaluating the estimation errors of $\widetilde{\W}_k$. Denote $\widetilde{\bcS} = \widetilde{\bcA}\times_1 \widetilde{\U}_1^\top \times_2 \widetilde{\U}_2^\top \times_3 \widetilde{\U}_3^\top$, $\widetilde{\V}_k = \SVD_{r_k}\left(\mathcal{M}_k(\widetilde{\bcS})^\top\right)$, we know
	\begin{equation*}
	\begin{split}
	\widetilde{\W}_k = & (\widetilde{\U}_{k+2}\otimes \widetilde{\U}_{k+1})\widetilde{\V}_k = \SVD_{r_k}\left((\widetilde{\U}_{k+2}\otimes \widetilde{\U}_{k+1})\mathcal{M}_k(\widetilde{\bcS})^\top\right)\\
	= & \SVD_{r_k}\left(\mathcal{M}_k\left(\widetilde{\bcS}\times_{(k+1)}\widetilde{\U}_{k+1}\times_{(k+2)}\widetilde{\U}_{k+2}\right)^\top\right)\\
	= & \SVD_{r_k}\left(\mathcal{M}_k\left(\widetilde{\bcS}\times_{(k+1)}\widetilde{\U}_{k+1}\times_{(k+2)}\widetilde{\U}_{k+2}\right)^\top \widetilde{\U}_k^\top\right)\\
	= & \SVD_{r_k}\left(\mathcal{M}_k\left(\widetilde{\bcS}\times_{k} \widetilde{\U}_k\times_{(k+1)}\widetilde{\U}_{k+1}\times_{(k+2)}\widetilde{\U}_{k+2}\right)^\top\right)  \\
	= & \SVD_{r_k}\left(\mathcal{M}_k\left(\llbracket\widetilde{\bcA}; P_{\widetilde{\U}_1}, P_{\widetilde{\U}_2}, P_{\widetilde{\U}_3}\rrbracket\right)^\top\right).
	\end{split}
	\end{equation*}
	On the other hand, $\W_k = \SVD_{r_k}(\A_k^\top) = \SVD_{r_k}\left(\mathcal{M}_k(\bcA)^\top\right)$. By Lemma \ref{lm:SVD-projection},
	\begin{equation}
	\begin{split}
	\|\A_k \widetilde{\W}_{k\perp}\|_F \leq & 2\left\|\mathcal{M}_{k}(\llbracket\widetilde{\bcA}; P_{\widetilde{\U}_1}, P_{\widetilde{\U}_2}, P_{\widetilde{\U}_3}\rrbracket) - \mathcal{M}_k(\bcA)\right\|_F \\
	= & 2\left\|\llbracket\widetilde{\bcA}; P_{\widetilde{\U}_1}, P_{\widetilde{\U}_2}, P_{\widetilde{\U}_3}\rrbracket - \bcA \right\|_{\tHS} \overset{\eqref{ineq:regression-prob-6}}{\leq} C\widetilde{\sigma}\sqrt{\frac{m}{n_1}}
	\end{split}
	\end{equation}
	with probability at least $1 - p^{-C}$. Therefore, we also have
	\begin{equation}
	\left\|\sin\Theta(\widetilde{\W}_k, \W_k)\right\|_F \leq \|\widetilde{\W}_{k\perp}^\top\W_k\|_F \leq \frac{\|\widetilde{\W}_{k\perp}^\top\W_k\widetilde{\W}_{k\perp}^\top \A_k^\top\|_F}{\sigma_{r_k}(\widetilde{\W}_{k\perp}^\top \A_k^\top)} \leq C\sqrt{\frac{\widetilde{\sigma}^2m}{\lambda_k^2n_1}}
	\end{equation}
	with probability at least $1 - p^{-C}$.
	
	To summarize the progress in this step, we have established the following probabilistic inequalities for $\widetilde{\U}_1, \widetilde{\U}_2, \widetilde{\U}_3$ and $\widetilde{\W}_1, \widetilde{\W}_2, \widetilde{\W}_3$,
	\begin{equation}\label{ineq:conclustion-HOOI}
	\begin{split}
	& \left\|\sin\Theta\left(\widetilde{\U}_k, \U_k\right)\right\| \leq \frac{C\widetilde{\sigma} \sqrt{p_k/n_1}}{\lambda_k}, \quad \left\|\sin\Theta\left(\widetilde{\W}_k, \W_k\right)\right\|_F \leq \frac{C\widetilde{\sigma} \sqrt{m/n_1}}{\lambda_k},\quad k=1,2,3,
	\end{split}
	\end{equation}
	\begin{equation}\label{ineq:conclustion-HOOI-2}
	\begin{split}
	& \left\|\widetilde{\U}_k^\top \A_k\right\|_F \leq C\widetilde{\sigma}\sqrt{m/n_1},\quad \left\|\A_k\widetilde{\W}_{k\perp}\right\|_F \leq C\widetilde{\sigma}\sqrt{m/n_1},\quad k=1,2,3,
	\end{split}
	\end{equation}
	with probability at least $1 - p^{-C}$.
	
	\item[Step 4] For the rest of the proof, we assume \eqref{ineq:conclustion-HOOI} and \eqref{ineq:conclustion-HOOI-2} hold. Next, we move on to evaluate the estimation error bound for $\widehat{\bcA}$. The focus now shifts from the first batch of samples $(\bcX^{(1)}, \y^{(1)})$ to the second one $(\bcX^{(2)}, y^{(2)})$. Denote
	\begin{equation}\label{ineq:regression-theta}
	\begin{split}
	\theta_k := \left\|\sin\Theta\left(\widetilde{\U}_k, \U_k\right)\right\|  \overset{\eqref{ineq:conclustion-HOOI}}{\leq} \frac{C\widetilde{\sigma}\sqrt{p_k/n_1}}{\lambda_k},\quad k=1,2,3;\\
	\end{split}
	\end{equation}
	\begin{equation}\label{ineq:regression-xi}
	\begin{split}
	\xi_k := \|\A_k\widetilde{\W}_{k\perp}\|_F \overset{\eqref{ineq:conclustion-HOOI-2}}{\leq} C\widetilde{\sigma}\sqrt{m/n_1}, \quad k=1,2,3;
	\end{split}
	\end{equation}
	\begin{equation}\label{ineq:regression-eta}
	\begin{split}
	\eta_k := & \left\|\widetilde{\U}_k^\top \A_k\right\|_F  \overset{\eqref{ineq:conclustion-HOOI-2}}{\leq} C\widetilde{\sigma}\sqrt{m/n_1},\quad k=1,2,3;
	\end{split}
	\end{equation}
	\begin{equation}\label{ineq:regression-hat-sigma}
	\widehat{\sigma}^2 := \left\|P_{\widetilde{\U}_\perp}\rmvec(\bcA)\right\|_2^2 + \sigma^2.
	\end{equation}
	By Lemma \ref{lm:projection_remainder},
	\begin{equation*}
	\begin{split}
	\|P_{\widetilde{\U}_{\perp}}\rmvec(\bcA)\|_2^2\leq & \frac{C\widetilde{\sigma}^4mp}{n_1^2\lambda_0^2}  + \frac{C_1\widetilde{\sigma}^6mp^2}{\lambda_0^4n_1^3}.
	\end{split}
	\end{equation*}
	Provided that $m = r_1r_2r_3+\sum_k(p_k-r_k)r_k$ and $n_1\geq \frac{C\widetilde{\sigma}^2p}{\lambda_0^2}$, we know
	\begin{equation}\label{ineq:hat-sigma}
	\begin{split}
	& \|P_{\widetilde{\U}_{\perp}}\rmvec(\bcA)\|_2^2\leq \frac{C\widetilde{\sigma}^4mp}{n_1^2\lambda_0^2}, \quad \widehat{\sigma}^2\leq \sigma^2 + \frac{C\widetilde{\sigma}^4 mp}{n_1^2\lambda_0^2}.
	\end{split}
	\end{equation}
	
	\item[Step 5] In this step, we evaluate two crucial quantities for applying the oracle inequality (Theorem \ref{th:upper_bound_general}). Recall the importance sketching covariates \eqref{eq:importance-sketching-covariates} are defined as
	\begin{equation*}
	\begin{split}
	& \widetilde{\X} = \left[\widetilde{\X}_\bcB ~~ \widetilde{\X}_{\D_1} ~~ \widetilde{\X}_{\D_2} ~~ \widetilde{\X}_{\D_3}\right] \in \mathbb{R}^{n_2\times m}, \\
	& \widetilde{\X}_\bcB \in \mathbb{R}^{n\times (r_1r_2r_3)},\quad \left(\widetilde{\X}_\bcB\right)_{[i,:]} = 
	\rmvec\left(\bcX_i^{(2)} \times_1 \widetilde{\U}_1^\top \times_2 \widetilde{\U}_2^\top \times_3 \widetilde{\U}_3^\top\right),\\
	& \widetilde{\X}_{\D_k} \in \mathbb{R}^{n\times (p_k-r_k)r_k}, \quad \left(\widetilde{\X}_{\D_k}\right)_{[i,:]} = \rmvec\left(\widetilde{\U}_{k\perp}^\top \mathcal{M}_k\left(\bcX_{i}^{(2)} \times_{k+1} \widetilde{\U}_{k+1}^\top \times_{k+2} \widetilde{\U}_{k+2}^\top\right) \widetilde{\V}_k\right).
	\end{split}
	\end{equation*}
	When $\bcX_i^{(2)}$ are i.i.d. Gaussian matrices and independent of $\widetilde{\U}_k$, $\widetilde{\V}_k$, $\widetilde{\W}_k$, $\widetilde\X$ can be seen as an orthogonal projection of $\bcX_i^{(2)}$ and has i.i.d. Gaussian entries. 
	Thus, by Proposition 5.35 in \cite{vershynin2010introduction}, 
	\begin{equation*}
	\begin{split}
	\bbP\left(\sigma_{\min}(\widetilde\X^\top\widetilde\X) = \sigma_{\min}^2(\widetilde{\X}) \geq \left(\sqrt{n_2}-\sqrt{m}-t\right)^2\right) \geq 1 - \exp(-t^2/2).
	\end{split}
	\end{equation*}
	By definition, $\widetilde{\varepsilon}\in \mathbb{R}^{n_2}$ is independent of $\widetilde\X$, and
	\begin{equation*}
	\begin{split}
	\widetilde\varepsilon_j = \langle\bcX_j^{(2)}, P_{\widetilde\U_{\perp}}\bcA\rangle + \varepsilon_j \sim N\left(0, \left\|P_{\widetilde{\U}_{\perp}}\rmvec(\bcA) \right\|_{2}^2+\sigma^2\right) = N(0, \widehat\sigma^2).
	\end{split}
	\end{equation*}
	Then, $\|\widetilde{\bvarepsilon}\|_2^2\sim \widehat\sigma^2\chi^2_{n_2}$ and $\|\widetilde\X^\top\widetilde\bvarepsilon\|_2^2 \Big| \|\bvarepsilon\|_2^2 \sim \|\bvarepsilon\|_2^2\chi^2_{m}$.
	Based on $\chi^2$ distribution tail bound \cite[Lemma 1]{laurent2000adaptive} and $n_2\geq C(p^{3/2} + r^3) \geq Cm$,
	\begin{equation}\label{ineq:XXXepsilon-bound}
	\begin{split}
	& \left\|(\widetilde{\X}^\top\widetilde{\X})^{-1}\widetilde{\X}^\top\widetilde{\bvarepsilon}\right\|_2^2 \\
	\leq &  \frac{\widehat{\sigma}^2\left(n_2+2\sqrt{n_2C_1\log(p)} + 2C_2\log(p)\right)\left(m+2\sqrt{mC_3\log(p)}+2C\log(p)\right)}{\left(\sqrt{n_2}-\sqrt{m}-C_4\log(p)\right)^4}\\
	\leq & \frac{\widehat{\sigma}^2m}{n_2}\frac{\left(1 + 2\sqrt{\frac{C\log p}{n_2}}+2\frac{\log p}{n_2}\right)\left(1+2\sqrt{\frac{t}{m}}+2\frac{t}{m}\right)}{\left(1 - \sqrt{\frac{m}{n_2}} - \frac{C_1\log(p)}{\sqrt{n_2}}\right)^4}\\
	= & \frac{\widehat{\sigma}^2m}{n_2}\left(1 + C_1\sqrt{\frac{m}{n_2}} + C_2\sqrt{\frac{\log p}{m}}\right).
	\end{split}
	\end{equation}
	with probability at least $1 - p^{-C}$.

	We assume \eqref{ineq:XXXepsilon-bound} holds. It remains to check $\left\|\widehat\D_k(\widehat{\B}_k\widetilde\V_k)^{-1}\right\|$. Similarly as the proof of Theorem \ref{th:upper_bound_general}, we define
	\begin{equation*}
	\begin{split}
	\widetilde{\bcB} = & \left\llbracket \bcA; \widetilde{\U}_1^\top, \widetilde{\U}_2^\top, \widetilde{\U}_3^\top \right\rrbracket = \left\llbracket \bcS\times_1 \U_1\times_2 \U_2\times_3 \U_3; \widetilde{\U}_1^\top, \widetilde{\U}_2^\top , \widetilde{\U}_3^\top \right\rrbracket \in \mathbb{R}^{r_1\times r_2\times r_3};\\
	\widetilde{\B}_k = & \mathcal{M}_k(\widetilde{\bcB}) \in \mathbb{R}^{r_k\times (r_{k+1}r_{k+2})}, \quad k=1,2,3,\\
	\widetilde{\D}_1 = & \widetilde{\U}_{1\perp}^\top \mathcal{M}_1(\bcA\times_2 \widetilde{\U}_2^\top \times_3 \widetilde{\U}_3) \widetilde{\V}_1 \overset{\text{Lemma \ref{lm:Kronecker-vectorization-matricization}}}{=} \widetilde{\U}_{1\perp}^\top \A_1 \widetilde{\W}_1 \in \mathbb{R}^{(p_1-r_1)\times r_1},\\
	\widetilde{\D}_2 = & \widetilde{\U}_{2\perp}^\top \mathcal{M}_2(\bcA\times_1 \widetilde{\U}_1^\top \times_3 \widetilde{\U}_3) \widetilde{\V}_2 = \widetilde{\U}_{2\perp}^\top \A_2 \widetilde{\W}_2 \in \mathbb{R}^{(p_2-r_2)\times r_2},\\
	\widetilde{\D}_3 = & \widetilde{\U}_{3\perp}^\top \mathcal{M}_3(\bcA\times_1 \widetilde{\U}_1^\top \times_2 \widetilde{\U}_2) \widetilde{\V}_3 = \widetilde{\U}_{3\perp}^\top \A_3 \widetilde{\W}_3 \in \mathbb{R}^{(p_3-r_3)\times r_3}.
	\end{split}
	\end{equation*}
	By the proof of Theorem \ref{th:upper_bound_general}, we have
	\begin{equation}\label{ineq:B-B+D-D}
	\begin{split}
	& \left\|\widehat\bcB - \widetilde{\bcB}\right\|_{\tHS}^2 + \sum_{k=1}^3\left\|\widehat\D_k-\widetilde\D_k\right\|_F^2 \overset{\eqref{th:hat_B-B}}{\leq} \left\|(\widetilde\X^\top\widetilde\X)^{-1}\widetilde\X^\top\widetilde\varepsilon\right\|_2^2 \\
	\overset{\eqref{ineq:XXXepsilon-bound}}{\leq} &  \frac{\widehat{\sigma}^2m}{n_2}\left(1 + C_1s\sqrt{\frac{\log m}{n_2}} + C_2\sqrt{\frac{\log p}{m}}\right),
	\end{split}
	\end{equation}
	\begin{equation}\label{ineq:tilde-DBV}
	\|\widetilde\D_k(\widetilde\B_k\widetilde\V_k)^{-1}\| \overset{\eqref{ineq:tilde-D-tilde-B-tilde-V}
	}{\leq} C\max_k\left\{\|\sin\Theta(\widetilde{\U}_k, \U_k)\|, \|\sin\Theta(\W_k, \W_k)\|\right\} \leq \frac{C\widetilde{\sigma}\sqrt{m/n_1}}{\lambda_k},
	\end{equation}
	\begin{equation*}
	\sigma_{\min}(\widetilde{\B}_k\widetilde{\V}_k) = \sigma_{\min}(\widetilde{\U}_k^\top \A_k\widetilde{\W}_k) \overset{\eqref{ineq:sigma_min-tildeUAW}}{\geq} \lambda_k\left(1 - \frac{C\widetilde{\sigma}^2m}{\lambda_k^2n_1}\right) \geq \lambda_k(1-c)
	\end{equation*}
	for some constant $0<c<1$. This additionally means
	\begin{equation}\label{ineq:hat-B-tilde-V}
	\sigma_{\min}\left(\widehat{\B}_k\widetilde{\V}_k\right) \geq \sigma_{\min}(\widetilde{\B}_k\widetilde{\V}_k) - \|\widehat{\B}_k - \B_k\| \overset{\eqref{ineq:B-B+D-D}}{\geq} \lambda_k\left(1 - \frac{C\widetilde{\sigma}^2m}{\lambda_k^2n_1}\right) - \frac{C\widehat{\sigma}^2m}{n_2} \geq (1-c)\lambda_k.
	\end{equation}
	It is easy to check that the following equality,
	\begin{equation*}
	(\widehat{\B}_k\widetilde{\V}_k)^{-1} = (\widetilde{\B}_k\widetilde{\V}_k)^{-1} + (\widetilde{\B}_k\widetilde{\V}_k)^{-1} \left(\widetilde{\B}_k\widetilde{\V}_k - \widehat{\B}_k\widetilde{\V}_k\right) (\widehat{\B}_k\widetilde{\V}_k)^{-1}. 
	\end{equation*}
	Thus,
	\begin{equation}\label{ineq:delta-upper-bound}
	\begin{split}
	\rho := \left\|\widehat{\D}_k(\widehat{\B}_k\widetilde{\V}_k)^{-1}\right\| \leq & \left\|(\widehat{\D}_k-\widetilde{\D}_k)(\widehat{\B}_k\widetilde{\V}_k)^{-1}\right\| + \left\|\widetilde{\D}_k(\widehat{\B}_k\widetilde{\V}_k)^{-1}\right\|\\
	\leq & \frac{C\left\|\widehat{\D}_k -\widetilde{\D}_k\right\|}{\lambda_k} + \left\|\widetilde{\D}_k(\widetilde{\B}_k\widetilde{\V}_k)^{-1}\right\| \\
	& + \left\|\widetilde{\D}_k(\widetilde{\B}_k\widetilde{\V}_k)^{-1}\right\| \cdot \left\|(\widetilde{\B}_k-\widehat{\B}_k)\widetilde{\V}_k\right\| \cdot \|(\widehat{\B}_k\widetilde{\V}_k)^{-1}\|\\
	\overset{\eqref{ineq:B-B+D-D}\eqref{ineq:tilde-DBV}\eqref{ineq:hat-B-tilde-V}}{\leq} & \frac{C\widetilde{\sigma}}{\lambda_k} \sqrt{\frac{m}{n_1}} + \frac{C\widehat{\sigma}}{\lambda_k}\sqrt{\frac{m}{n_2}}.
	\end{split}
	\end{equation}
	
	\item[Step 6] Finally, we apply the oracle inequality, i.e., Theorem \ref{th:upper_bound_general}, and obtain the final upper bound for $\widehat\bcA$. We have shown that the conditions of Theorem \ref{th:upper_bound_general} holds if \eqref{ineq:conclustion-HOOI}, \eqref{ineq:conclustion-HOOI-2}, and \eqref{ineq:XXXepsilon-bound} hold. Then Theorem \ref{th:upper_bound_general} implies
	\begin{equation*}
	\begin{split}
	\left\|\widehat{\bcA} - \bcA\right\|_{\tHS}^2 \leq &  (1+C\theta+C\rho)\left\|(\widetilde{\X}^\top\widetilde{\X})^{-1}\widetilde{\X}^\top\widetilde{\bvarepsilon}\right\|_2^2\\
	\overset{\eqref{ineq:regression-theta}\eqref{ineq:XXXepsilon-bound}\eqref{ineq:delta-upper-bound}}{\leq} & \frac{\widehat{\sigma}m}{n_2}\left(1 + C_1\sqrt{\frac{m}{n_2}} + C_2\sqrt{\frac{\log p}{m}} + \frac{C_3\widetilde{\sigma}}{\lambda_0}\sqrt{\frac{m}{n_1}} + \frac{C_4\widehat{\sigma}}{\lambda_0}\sqrt{\frac{m}{n_2}}\right)\\
	\overset{\eqref{ineq:hat-sigma}}{\leq} & \frac{m}{n_2}\left(\sigma^2+\frac{C_1\widetilde{\sigma}^4mp}{n_1^2\lambda_0^2}\right)\left(1 + C_2\sqrt{\frac{m}{n_2}} + C_3\sqrt{\frac{\log p}{m}} + \frac{C_4\widetilde{\sigma}}{\lambda_0}\sqrt{\frac{m}{n_1}} + \frac{C_5\widehat{\sigma}}{\lambda_0}\sqrt{\frac{m}{n_2}}\right)\\
	\leq & \frac{m}{n_2}\left(\sigma^2 + \frac{C_1\widetilde{\sigma}^4mp}{n_1^2\lambda_0^2}\right)\left(1 + C_2\sqrt{\frac{\log p}{m}} + C_3\sqrt{\frac{m\widetilde{\sigma}^2}{(n_1\wedge n_2)\lambda_0^2}}\right)
	\end{split}
	\end{equation*}
	with probability at least $1 - p^{-C}$. Here, the last inequality is due to $n_1\wedge n_2\geq C\widetilde{\sigma}^2(p^{3/2}+r^3)/\lambda_0^2$ and $\widehat{\sigma}=\|\bcA\|_{\tHS}^2+\sigma^2 \geq \lambda_0$. \quad $\square$
\end{enumerate}

\subsection{Proof of Theorem \ref{th:lower-bound-regression}}\label{sec:proof_lower-bound-regression}

In this theorem, we provide an estimation error lower bound for low-rank tensor regression. The central idea is to carefully transform the original high-dimensional low-rank tensor regression model to the unconstrained dimension-reduced linear regression model \eqref{eq:dimension-reduced-regression}, then apply the classic Bayes risk of linear regression (Lemma \ref{lm:linear-regression-lower-bound}) to finalize the desired lower bound on estimation error.

Since $r_1, r_2$, and $r_3$ satisfy $r_k \leq r_{k+1}r_{k+2}$ for $k=1,2,3$, the $r_1$-by-$r_2$-by-$r_3$ tensor with i.i.d. normal entries has full Tucker rank with probability 1. Thus, we can set $\bcS_0\in \mathbb{R}^{r_1\times r_2\times r_3}$ as a fixed tensor with full Tucker rank, i.e., $\rank(\bcS_0) = (r_1, r_2, r_3)$. Let $T>0$ be a large to-be-specified constant. Define
\begin{equation}
\bcA_0\in \mathbb{R}^{p_1\times p_2\times p_3}, \quad (\bcA_0)_{[1:r_1, 1:r_2, 1:r_3]} = T\bcS_0, \quad (\bcA_0)_{[1:r_1, 1:r_2, 1:r_3]^c} = 0.
\end{equation}
Suppose $\U_k\in \mathbb{O}_{p_k, r_k}$ and $\W_k\in \mathbb{O}_{p_{-k}, r_k}$ are the left and right singular subspaces of $\mathcal{M}_k(\bcA_0)$, respectively; $\V_k\in \mathbb{O}_{r_{k+1}r_{k+2}, r_k}$ is the right singular subspace of $\mathcal{M}_k(\bcS_0)$.  Then by definition of $\bcA_0$,
$$\U_k = \begin{bmatrix}
\I_{r_k}\\
\boldsymbol{0}_{(p_k-r_k) \times r_k}
\end{bmatrix}, \quad k=1,2,3.$$
Next, for to-be-specified values $\tau, T>0$, we introduce a prior distribution $\bar{P}_{\tau, T}$ on the class of  $\mathcal{A}_{\bp, \br}$: the $p_1$-by-$p_2$-by-$p_3$ random tensor $\bar{\bcA} \sim \bar{P}_{\tau,T}$ if and only if it can be generated based on the following process.
\begin{enumerate}[leftmargin=*]
	\item Generate an $r_1$-by-$r_2$-by-$r_3$ tensor $\bcB \overset{iid}{\sim} N(0, \tau^2)$ and assign $\bar{\bcA}_{[1:r_1, 1:r_2, 1:r_3]} = T\bcS_0 + \bcB$.
	\item Suppose $\mathcal{M}_k(\bar{\bcA}_{[1:r_1, 1:r_2, 1:r_3]}) = \bar{\A}_{0k}\in \mathbb{R}^{r_k\times r_{-k}}$ and $\bar{\V}_k = \SVD_{r_k}(\bar{\A}_{0k}^\top)\in \mathbb{O}_{r_{-k}, r_k}$. Assign
	$$\mathcal{M}_1\left(\bar{\bcA}_{[(r_1+1):p_1, 1:r_2, 1:r_3]}\right) = \B_{1}\cdot \bar{\V}_1^\top, $$
	$$\mathcal{M}_2\left(\bar{\bcA}_{[1:r_1, (r_2+1):p_2, 1:r_3]}\right) = \B_{2}\cdot \bar{\V}_2^\top, $$
	$$\mathcal{M}_3\left(\bar{\bcA}_{[1:r_1, 1:r_2, (r_3+1):p_3]}\right) = \B_{3}\cdot \bar{\V}_3^\top, $$
	where all entries of $\B_{1}\in \mathbb{R}^{(p_1-r_1)\times r_1}, \B_{2}\in\mathbb{R}^{(p_2-r_2)\times r_2}, \B_{3}\in\mathbb{R}^{(p_3-r_3)\times r_3}$ are independently drawn from $N(0,\tau^2)$.
	\item The other blocks of $\bar{\bcA}$ are calculated as follows,
	\begin{equation}
	\begin{split}
	& \bar{\bcA}_{[(r_1+1):p_1, (r_2+1):p_2, 1:r_3]} = \bar{\bcA}_{[1:r_1, 1:r_2, 1:r_3]} \times_1 \left(\B_{1}(\bar{\A}_{01}\bar{\V}_1)^{-1}\right) \times_2 \left(\B_{2}(\bar{\A}_{02}\bar{\V}_2)^{-1}\right), \\
	& \bar{\bcA}_{[(r_1+1):p_1, 1:r_2, (r_3+1):p_3]} = \bar{\bcA}_{[1:r_1, 1:r_2, 1:r_3]} \times_1 \left(\B_{1}(\bar{\A}_{01}\bar{\V}_1)^{-1}\right) \times_3 \left(\B_{3}(\bar{\A}_{03}\bar{\V}_3)^{-1}\right), \\
	& \bar{\bcA}_{[1:r_1, (r_2+1):p_2, (r_3+1):p_3]} = \bar{\bcA}_{[1:r_1, 1:r_2, 1:r_3]} \times_2 \left(\B_{2}(\bar{\A}_{02}\bar{\V}_2)^{-1}\right) \times_3 \left(\B_{3}(\bar{\A}_{03}\bar{\V}_3)^{-1}\right), \\
	& \bar{\bcA}_{[(r_1+1):p_1, (r_2+1):p_2, (r_3+1):p_3]} \\
	& \quad = \bar{\bcA}_{[1:r_1, 1:r_2, 1:r_3]} \times_1 \left(\B_{1}(\bar{\A}_{01}\bar{\V}_1)^{-1}\right) \times_2 \left(\B_{2}(\bar{\A}_{02}\bar{\V}_2)^{-1}\right) \times_3 \left(\B_{3}(\bar{\A}_{03}\bar{\V}_3)^{-1}\right).
	\end{split}
	\end{equation}
\end{enumerate}
One can check by comparing each block that $\bar{\bcA}$ satisfies
\begin{equation}
\begin{split}
& \bar{\bcA} = \left\llbracket T\bcS_0 + \bcB; \bar{\L}_1, \bar{\L}_2, \bar{\L}_3\right\rrbracket, \quad \text{where}\quad  \bar{\L}_k = \begin{bmatrix}
\I_{r_k}\\
\B_{k}(\bar{\A}_{0k}\bar{\V}_k)^{-1}
\end{bmatrix},\quad k=1,2,3.
\end{split}
\end{equation}
Thus, $\rank(\bar{\bcA}) \leq (r_1, r_2, r_3)$ and $\bar{\bcA}\in \mathcal{A}_{\bp, \br}$. Then we consider another distribution $P^\ast_{\tau, T}$ on the whole tensor space $\mathbb{R}^{p_1\times p_2\times p_3}$, 
\begin{equation}
\begin{split}
\bcA^\ast \sim P^\ast_{\tau,T}, \quad \text{such that}\quad & \bcA^\ast_{[1:r_1, 1:r_2, 1:r_3]} = T\bcS_0 + \bcB, \\
& \mathcal{M}_1\left(\bcA^\ast_{[(r_1+1):p_1, 1:r_2, 1:r_3]}\right) = \B_{1}\cdot \V_{1}^\top; \\
& \mathcal{M}_2\left(\bcA^\ast_{[1:r_1, (r_2+1):p_2, 1:r_3]}\right) = \B_{2}\cdot \V_{2}^\top; \\
& \mathcal{M}_3\left(\bcA^\ast_{[(r_1+1):p_1, 1:r_2, 1:r_3]}\right) = \B_{3}\cdot \V_{3}^\top; \\
& \text{the other blocks of $\bcA^\ast$ are set to zero}.
\end{split}
\end{equation}
Here, $\bcB, \B_{1}, \B_{2}, \B_{3} \overset{iid}{\sim}N(0, \tau^2)$. Suppose $\bar{\bcA}\sim \bar{P}_{\tau,T}$ and $\bcA^\ast\sim P^\ast_{\tau,T}$. Recall that $\V_k = \SVD_{r_k}(\mathcal{M}_k(\bcS_0)^\top))$ and $\bar{\V}_k = \SVD_{r_k}(\mathcal{M}_k(\bcS_0+\bcB/T)^\top)$. As $T\to \infty$, we must have
\begin{equation}
\bar{\V}_k \overset{d}{\to} \V_k \quad \text{and}\quad (\bar{\bcA} - \bcA_0) \overset{d}{\to} (\bcA^\ast - \bcA_0).
\end{equation}

Next, we move on to the regular tensor regression model 
$$y_i = \langle \bcX_i, \bcA\rangle + \varepsilon_i, \quad i=1,\ldots, n.$$
For convenience, we divide $\bcX_i$ and $\bcA$ into eight blocks and denote them separately as
\begin{equation*}
\begin{split}
& \bcX_{i, s_1s_2s_3} = (\bcX_i)_{[I_{1, s_1}, I_{2, s_2}, I_{3, s_3}]},\quad  \bcA_{s_1s_2s_3} = \bcA_{[I_{1, s_1}, I_{2, s_2}, I_{3, s_3}]}, \text{ for } s_1, s_2, s_3 \in \{1, 2\}, \\
\text{where}\quad & I_{k, 1} = \{1,\ldots, r_k\}, \quad I_{k, 2} = \{r_k+1,\ldots, p_k\}, \quad k=1,2,3.
\end{split}
\end{equation*}
If $\bcA^\ast \sim P_{\tau, T}^\ast$, $\bcA^\ast_{122}, \bcA^\ast_{212}, \bcA^\ast_{221}, \bcA^\ast_{222}$ are all zeros. Then,
\begin{equation*}
\begin{split}
y_i = & \langle \bcX_i, \bcA^\ast\rangle + \varepsilon_i = \sum_{s_1, s_2, s_3=1}^2\langle\bcX_{i,s_1s_2s_3}, \bcA^\ast_{s_1s_2s_3}\rangle + \varepsilon_i\\ 
= & \langle (\bcX_{i, 111}, T\bcS_0 + \bcB\rangle + \langle\mathcal{M}_1(\bcX_{i, 211}), \B_{1}\V_1^\top\rangle\\
& + \langle\mathcal{M}_2(\bcX_{i, 121}), \B_{2}\V_2^\top\rangle + \langle\mathcal{M}_3(\bcX_{i,112}), \B_{3} \V_3^\top\rangle + \varepsilon_i\\
= & \langle \bcX_i, \bcA_0 \rangle + \varepsilon_i + \langle \rmvec(\bcX_{i, 111}), \rmvec(\bcB)\rangle + \langle \mathcal{M}_1(\bcX_{i, 211})\V_1, \B_{1}\rangle \\
& + \langle \mathcal{M}_2(\bcX_{i, 121})\V_2, \B_{2}\rangle +  \langle \mathcal{M}_3(\bcX_{i, 112})\V_3, \B_{3}\rangle\\
:= & \langle \bcX_i, \bcA_0\rangle + \langle \bar{\X}_i, \b\rangle + \varepsilon_i,
\end{split}
\end{equation*}
where
\begin{equation*}
\begin{split}
\bar{\X}_i = & \begin{bmatrix}
\rmvec\left(\bcX_{i,111}\right)\\ 
\rmvec\left(\mathcal{M}_1(\X_{i, 211})\V_1\right)\\ \rmvec\left(\mathcal{M}_2(\X_{i, 121})\V_2\right)\\ \rmvec\left(\mathcal{M}_3(\X_{i, 112})\V_3\right)
\end{bmatrix} \in \mathbb{R}^{m},\quad \bar{\X} = \begin{bmatrix}
\bar{\X}_1^\top\\
\vdots \\
\bar{\X}_n^\top
\end{bmatrix}\in \mathbb{R}^{n\times m},\quad
\b = \begin{bmatrix}
\rmvec(\bcB)\\  \rmvec(\B_{1})\\  \rmvec(\B_{2})\\  \rmvec(\B_{3})
\end{bmatrix} \in \mathbb{R}^{m}.
\end{split}
\end{equation*}
Suppose the parameter $\bcA^\ast$ is drawn from the prior distribution $P_{\tau, T}^\ast$. Then, $\b\overset{iid}{\sim}N(0, \tau^2)$. Note that $\bar\X_i$ is an orthogonal projection of $\bcX_i$, so $\bar{\X}_i\overset{iid}{\sim}N(0, 1)$. Now, $y_i, \bar{\X}_i, \bar{\b}$ can be related by the following regression model,
\begin{equation}\label{eq:dimension-reduced-regression}
\begin{split}
& y_i - \langle \bcX_i, \bcA_0\rangle = \bar{\X}_i^\top \b + \varepsilon_i,\quad i=1,\ldots, n;\\
& \b\overset{iid}{\sim}N(0, \tau^2),\quad \varepsilon \overset{iid}{\sim}N(0,\sigma^2).
\end{split}
\end{equation}
By the construction of $\bcA^\ast$ and the setting that $\bcS_0$ is fixed, the estimation of $\bcA^\ast$ is equivalent to the estimation $\b$. By Lemma \ref{lm:linear-regression-lower-bound}, the Bayes risk of estimating $\b$ (and the Bayes risk of estimating $\bcA^\ast$ if $\bcA^\ast\sim P_{\tau, T}$) is
\begin{equation*}
\left\|\widehat\bcA^\ast - \bcA^\ast\right\|_{\tHS}^2\Big|\{\bar\X_i\}_{i=1}^n = \left\|\widehat{\b} - \b\right\|_2^2\Big|\{\bar\X_i\}_{i=1}^n = \tr\left(\left(\frac{\I_m}{\tau^2}+\frac{\bar\X^\top\bar\X}{\sigma^2}\right)^{-1}\right).
\end{equation*}
Here, $\widehat\bcA^\ast$ and $\widehat{\b}$ are the posterior mean of $\bcA^\ast$ and $\b$, respectively.

Since $\bar{P}_{\tau, T}\to P_{\tau, T}$ and $\bar{\bcA} - \bcA_0 \to \bcA^\ast - \bcA_0$ as $T\to \infty$, we have
\begin{equation*}
\mathbb{E}\left\|\widehat{\bcA} - \bar{\bcA}\right\|_{\tHS}^2 \Big| \{\bar\X_i\}_{i=1}^n \to     \mathbb{E}\left\|\widehat{\bcA}^\ast - \bcA^\ast\right\|_{\tHS}^2 \Big| \{\bar\X_i\}_{i=1}^n = \tr\left(\left(\frac{\I_m}{\tau^2} + \frac{\bar{\X}^\top\bar\X}{\sigma^2}\right)^{-1}\right),
\end{equation*}
where $\widehat\bcA$ is the posterior mean of $\bar{\bcA}$ if $\bar\bcA \sim \bar{P}_{\tau, T}$. Since $\bar{\bcA} \sim \bar{P}_{\tau, T}$ and $\bar{P}_{\tau, T}$ is the distribution on $\mathcal{A}_{\bp, \br}$, we have the following estimation lower bound,
\begin{equation*}
\inf_{\widehat{\bcA}}\sup_{\bcA \in \mathcal{A}_{\bp, \br}} \left\|\widehat{\bcA} - \bcA\right\|_{\tHS}^2\Big| \{\bar\X_i\}_{i=1}^n \geq \tr\left(\left(\frac{\I_m}{\tau^2} + \frac{\bar\X^\top \bar\X}{\sigma^2}\right)^{-1}\right).
\end{equation*}

Finally, since $(\bar{\X}^\top \bar{\X})^{-1}$ is inverse Wishart distributed and\footnote{See \url{https://en.wikipedia.org/wiki/Inverse-Wishart_distribution} for expectation of inverse Wishart distribution.} $$\tr(\mathbb{E}(\bar{\X}^\top \bar{\X})^{-1}) =\left\{\begin{array}{ll} \frac{1}{n-m-1}\tr(\I_m) = \frac{m}{n-m-1} & n> m+1;\\
\infty & n\leq m+1.
\end{array}\right.$$ 
By letting $\tau \to \infty$, we finally obtain
\begin{equation*}
\begin{split}
& \inf_{\widehat{\bcA}}\sup_{\bcA \in \mathcal{A}_{\bp, \br}} \left\|\widehat{\bcA} - \bcA\right\|_{\tHS}^2 \geq \limsup_{\tau\to \infty}\mathbb{E}\tr\left(\left(\frac{\I_m}{\tau^2} + \frac{\bar\X^\top \bar\X}{\sigma^2}\right)^{-1}\right) \\
= & \tr\left(\frac{\sigma^2 \I_m}{n-m-1}\right) = \left\{\begin{array}{ll}
\frac{m\sigma^2}{n-m-1}, & \text{if } n>m+1;\\ 
+\infty & \text{if }n\leq m+1.
\end{array}\right.
\end{split}
\end{equation*}
\quad $\square$

\subsection{Proof of Theorem \ref{th:upper_bound_sparse_tensor_regression}}\label{sec:proof_upper_bound_sparse_tensor_regression}

In this theorem, we aim to establish an estimation error upper bound for sparse ISLET in sparse low-rank tensor regression problem. After introducing some necessary notations, we develop the estimation error bounds for sketching directions $\widetilde{\U}_k$ and $\widetilde{\W}_k$ in Steps 1 and 2. In Step 3, 
we give error bounds for a number of intermediate terms. In Step 4, we prove upper bounds for key quantities $\rho, \left\|(\widetilde{\X}_{\bcB}^\top\widetilde{\X}_{\bcB})^{-1}\widetilde{\X}_{\bcB}^\top \widetilde{\bvarepsilon}_\B\right\|_2^2, \left\|(\widetilde{\X}_{\E_k}^\top\widetilde{\X}_{\E_k})^{-1}\widetilde{\X}_{\E_k}^\top \widetilde{\bvarepsilon}_{\E_k}\right\|_2^2$, and $\max_{i = 1,\ldots, p_k} \left\|(\widetilde{\X}_{\E_k, [:, G_i^k]})^\top\widetilde{\bvarepsilon}_{\E_k}/n\right\|_2^2$. Finally, we plug in these values to Theorem \ref{th:upper_bound_sparse_general} to finalize the proof.

We first introduce a number of notations that will be used in the proof. Similarly as the proof of Theorem \ref{th:upper_bound_regression}, denote 
$$\A_k = \mathcal{M}_k(\bcA), \quad \S_k = \mathcal{M}_k(\bcS),$$
$$\widetilde{\A}_k = \mathcal{M}_k(\widetilde{\bcA}), \quad \widetilde{\bcS} = \llbracket\widetilde{\bcA}; \widetilde{\U}_1^\top, \widetilde{\U}_2^\top, \widetilde{\U}_3^\top\rrbracket, \quad \widetilde{\S}_k = \mathcal{M}_k(\widetilde{\bcS}), \quad \X_{jk} = \mathcal{M}_k(\bcX_j), \quad k=1,2,3.$$
Recall 
$$\widetilde{\sigma}^2 = \|\bcA\|_{\tHS}^2 + \sigma^2,\quad \lambda_k = \sigma_{r_k}(\mathcal{M}_k(\bcA)),$$
\begin{equation}\label{eq:m_s}
m_s = r_1r_2r_3 + \sum_{k\in J_s} s_k(r_k + \log (p_k)) + \sum_{k \notin J_s} p_kr_k,
\end{equation}
and $\widetilde{\U}_1, \widetilde{\U}_2, \widetilde{\U}_2$ are the output from Step 1. We also denote
\begin{equation*}
I_k = \left\{i: \U_{k, [i,:]} \neq 0\right\},\quad k=1, 2, 3,
\end{equation*}
\begin{equation}\label{eq:zeta}
\zeta_j = (\U_3\otimes \U_2\otimes \U_1)^\top\rmvec(\bcX_j^{(1)}) = \rmvec(\llbracket\bcX_j^{(1)}; \U_1^\top, \U_2^\top, \U_3^\top\rrbracket) \in \mathbb{R}^{r_1r_2r_3},\quad j=1,\ldots, n_1,
\end{equation}
\begin{equation}\label{eq:sigma_zeta}
\widetilde{\sigma}^2_\zeta = \frac{1}{n_1}\sum_{j=1}^{n_1}\left(\varepsilon_j^{(1)} + \zeta_j^\top \rmvec(\bcS)\right)^2.
\end{equation}
\begin{enumerate}[leftmargin=*]
	\item[Step 1] In this first step, we develop the perturbation bound for $\widetilde{\U}_k$ and $\widetilde{\W}_k$. First, $\widetilde{\A}$ can be decomposed as
	\begin{equation}\label{eq:tilde-A-decomposition}
	\begin{split}
	\widetilde{\bcA} = & \frac{1}{n_1}\sum_{j=1}^{n_1} y_j^{(1)} \bcX_j^{(1)} = \frac{1}{n_1}\sum_{j=1}^{n_1} \left(\varepsilon_j^{(1)} + \langle \bcX_j^{(1)}, \bcA \rangle \right) \bcX_j^{(1)} \\
	= & \frac{1}{n_1}\sum_{j=1}^{n_1} \left(\varepsilon_j^{(1)} + \langle\llbracket\bcX_j^{(1)}; \U_1^\top, \U_2^\top, \U_3^\top\rrbracket, \bcS\rangle \right) \bcX_j^{(1)}\\
	= & \frac{1}{n_1}\sum_{j=1}^{n_1}\left(\varepsilon_j^{(1)}+\langle \llbracket\bcX_j^{(1)}; \U_1^\top, \U_2^\top, \U_3^\top \rrbracket, \bcS \rangle\right) \llbracket \bcX_j^{(1)}; P_{\U_1}, P_{\U_2}, P_{\U_3} \rrbracket\\
	& ~~ + \frac{1}{n_1}\sum_{j=1}^{n_1}\left(\varepsilon_j^{(1)}+\langle \llbracket\bcX_j^{(1)}; \U_1^\top, \U_2^\top, \U_3^\top \rrbracket, \bcS \rangle\right) P_{(\U_3\otimes \U_2\otimes \U_1)_\perp}[\bcX_j^{(1)}]\\
	:= & \bcH + \bcR.
	\end{split}
	\end{equation}
	In particular, $\bcH$ is fully determined by $\zeta_j$ and $\varepsilon_j^{(1)}$; $\bcH$ is of Tucker rank-$(p_1,p_2,p_3)$ and has loadings $\U_1,\U_2, \U_3$. By Lemma \ref{lm:concentration-Gaussian-ensemble},
	\begin{equation}\label{ineq:sparse-regression-0}
	\begin{split}
	& \left\|\mathcal{M}_1(\bcH) - \A_1\right\| = \left\|\U_1^\top\mathcal{M}_1(\bcH) (\U_3\otimes \U_2) - \U_1^\top\A_1(\U_3\otimes \U_2)\right\|\\
	= & \left\|\frac{1}{n_1}\sum_{j=1}^{n_1}\left(\varepsilon_j^{(1)}+\langle \llbracket\bcX_j^{(1)}; \U_1^\top, \U_2^\top, \U_3^\top \rrbracket, \bcS \rangle\right) \U_1^\top\X_{jk}^{(1)} (\U_3\otimes \U_2) - \U_1^\top\A_1(\U_3\otimes \U_2)\right\|\\
	= & \left\|\frac{1}{n_1}\sum_{j=1}^{n_1}\left(\varepsilon_j^{(1)}+\left\langle \U_1^\top\X_{j1}^{(1)}(\U_3\otimes \U_2), \S_1 \right\rangle\right) \U_1^\top\X_{j1}^{(1)} (\U_3\otimes \U_2) - \S_1\right\|\\
	\leq & \sqrt{\frac{(r_1+r_2r_3)\widetilde{\sigma}^2\log p}{n_1}}
	\end{split}
	\end{equation}
	with probability at least $1 - p^{-C}$. Similar inequalities also hold for $\|\mathcal{M}_2(\bcH) - \A_2\|$ and $\|\mathcal{M}_3(\bcH) - \A_3\|$. Provided that $\lambda_0 = \min_{k=1,2,3} \sigma_{r_k}(\A_k)$ satisfies $\lambda_0^2 \geq C\widetilde{\sigma}^2(r_1r_2+r_2r_3+r_3r_1)/n_1$, we have
	\begin{equation}\label{ineq:sparse-regression-1}
	\sigma_{r_k}(\mathcal{M}_k(\bcH)) \geq \sigma_{r_k}(\mathcal{M}_k(\bcA)) - \left\|\mathcal{M}_k(\bcH) - \A_k\right\| \geq  (1-c)\lambda_k
	\end{equation}
	with probability at least $1 - p^{-C}$. 
	
	Recall the definition of $\zeta_j$ and $\widetilde{\sigma}^2_\zeta$ in \eqref{eq:zeta} \eqref{eq:sigma_zeta}. For any $j = 1,\ldots, n_1$, $\varepsilon_j^{(1)} + \zeta_j^\top \rmvec(\bcS) \sim N(0, \sigma^2 + \|\bcS\|_{\tHS}^2) \sim N(0, \sigma^2 + \|\bcA\|_{\tHS}^2) \sim N(0, \widetilde{\sigma}^2)$, which means $\widetilde{\sigma}_\zeta^2 \sim \frac{\widetilde{\sigma}^2}{n_1}\chi^2_{n_1}.$ 
	By the tail bound of $\chi^2$ distribution \cite[Lemma 1]{laurent2000adaptive},
	\begin{equation}\label{ineq:tilde-sigma-eta-sigma}
	\left|\widetilde{\sigma}_{\zeta}^2 - \widetilde{\sigma}^2\right| \leq C\widetilde{\sigma}^2\left(\sqrt{\frac{\log p}{n_1}} + \frac{\log p}{n_1}\right)\leq C\widetilde{\sigma}^2\sqrt{\frac{\log p}{n_1}}
	\end{equation} 
	with probability at least $1 - p^{-C}$.

	Since $\rmvec(\bcX_j^{(1)})$ has i.i.d. Gaussian entries and $(\U_3\otimes \U_2\otimes \U_1)$ is orthogonal to $(\U_3\otimes \U_2\otimes \U_1)_\perp$, we have that $(\U_3\otimes \U_2\otimes \U_1)^\top \rmvec(\bcX_j^{(1)})$ is independent of $(\U_3\otimes \U_2\otimes \U_1)^\top_\perp \rmvec(\bcX_j^{(1)})$ and $\bcR$ (defined in \eqref{eq:tilde-A-decomposition}) is Gaussian distributed conditioning on fixed values of $\zeta_j$ and $\varepsilon_j^{(1)}$:
	\begin{equation}
	\begin{split}
	& \rmvec(\bcR) \bigg| \{\varepsilon_j^{(1)}, \zeta_j \}_{j=1}^{n_1} \text{ has same distribution as } P_{(\U_3\otimes \U_2\otimes \U_1)_\perp} \rmvec(\bcR_0), \\ 
	& \text{ where } \bcR_0 \in \mathbb{R}^{p_1\times p_2\times p_3}, \quad \bcR_0 \overset{iid}{\sim}N\left(0, \frac{\widetilde{\sigma}^2_\zeta}{n_1}\right).
	\end{split}
	\end{equation}
	Particularly, $\bcR_{[I_1, I_2, I_3]^c}\Big|\{\varepsilon_j^{(1)}, \zeta_j\}_{j=1}^{n_1}\overset{iid}{\sim} N(0, \widetilde{\sigma}_\zeta)^2$, i.e., $\bcR$ is i.i.d. Gaussian outside of the support of $\bcA$.
	
	\item[Step 2] The rest of this proof will be conditioning on the fixed value of $\{\varepsilon_j^{(1)}, \zeta_j\}_{j=1}^{n_1}$ that satisfies \eqref{ineq:sparse-regression-0}, \eqref{ineq:sparse-regression-1}, and \eqref{ineq:tilde-sigma-eta-sigma}. Provided \eqref{ineq:sparse-regression-1}, \eqref{ineq:tilde-sigma-eta-sigma}, and
	$$n_1 \geq \frac{C\widetilde{\sigma}^2}{\lambda_0^2}\left(s_1s_2s_3\log p + \sum_{k=1}^3(s_k^2r_k^2 + r_{k+1}^2r_{k+2}^2)\right),$$ 
	we have the following signal-noise-ratio assumption for denoising problem: $\widetilde{\A} = \bcH + \bcR$,
	\begin{equation*}
	\min_k \sigma_{r_k}(\mathcal{M}_k(\bcH)) \geq \frac{C\widetilde{\sigma}_\zeta}{\sqrt{n_1}}\left((s_1s_2s_3\log p)^{1/2} + \sum_{k=1}^3(s_kr_k+r_{k+1}r_{k+1})\right).
	\end{equation*}
	By \cite[Theorem 4]{zhang2017optimal-statsvd} (with mild modifications to the proof to accommodate the fact that $\bcR_{[I_1, I_2, I_3]}$ here is projection of i.i.d. Gaussian but not exactly i.i.d. Gaussian), the STAT-SVD with the tuning parameter $\widehat{\sigma}={\rm Med}(|\rmvec(\widetilde{\bcA})|/0.6744)$ (where 0.6744 is the $75\%$ quantile of standard Gaussian) yields
	\begin{equation}\label{ineq:sparse-regression-2}
	\begin{split}
	\left\|\sin\Theta(\widetilde{\U}_k, \U_k)\right\|_F \leq & \frac{C\widetilde{\sigma}_\zeta\sqrt{(s_kr_k + s_k\log(p_k))/n_1}}{\sigma_{r_k}(\mathcal{M}_k(\bcH))}\\
	\overset{\eqref{ineq:sparse-regression-1}\eqref{ineq:tilde-sigma-eta-sigma}}{\leq} &  \frac{C\widetilde{\sigma}\sqrt{(s_kr_k+s_k\log(p_k))/n_1}}{\lambda_k}, \quad k\in J_s,
	\end{split}
	\end{equation}
	\begin{equation}
	\begin{split}
	\left\|\sin\Theta(\widetilde{\U}_k, \U_k)\right\|_F \leq & \frac{C\widetilde{\sigma}_\zeta\sqrt{p_kr_k/n_1}}{\sigma_{r_k}(\mathcal{M}_k(\bcH))} \overset{\eqref{ineq:sparse-regression-1}\eqref{ineq:tilde-sigma-eta-sigma}}{\leq}  \frac{C\widetilde{\sigma}\sqrt{p_kr_k/n_1}}{\lambda_k}, \quad k\notin J_s,
	\end{split}
	\end{equation}
	\begin{equation}\label{ineq:sparse-regression-3}
	\begin{split}
	\text{and} \quad \left\|\llbracket\widetilde{\bcA}; P_{\widetilde{\U}_1}, P_{\widetilde{\U}_2}, P_{\widetilde{\U}_3}\rrbracket - \bcH\right\|_{\tHS}^2 \leq & \frac{C\widetilde{\sigma}_\eta^2}{n_1}\Big(r_1r_2r_3 + \sum_{k\in J_s}s_k(r_k+\log p) + \sum_{k\notin J_s} p_kr_k\Big) \\
	\overset{\eqref{eq:m_s}}{\leq} & \frac{C\widetilde{\sigma}^2m_s}{n_1}
	\end{split}
	\end{equation}
	with probability at least $1 - p^{-C}$, where $\widetilde{\U}_1, \widetilde{\U}_2, \widetilde{\U}_3$ are the outcomes of STAT-SVD procedure. Since the leading right singular vectors of $\mathcal{M}_k\left(\llbracket\widetilde{\bcA}; P_{\widetilde{\U}_1}, P_{\widetilde{\U}_2}, P_{\widetilde{\U}_3}\rrbracket\right)$ and $\mathcal{M}_k(\bcA)$ are $\widetilde{\W}_k$ and $\W_k$, respectively, we have
	\begin{equation*}
	\begin{split}
	& \left\|\sin\Theta(\widetilde{\W}_k, \W_k)\right\|_F = \left\|\widetilde{\W}_{k\perp}^\top \W_k\right\|_F \leq \frac{\|\widetilde{\W}_{k\perp}^\top \W_k\W_k^\top \mathcal{M}_k(\bcH)^\top \|_F}{\sigma_{r_k}\left(\W_k^\top\mathcal{M}_k(\bcH)^\top\right)}\\
	= & \frac{\|\widetilde{\W}_{k\perp}^\top \mathcal{M}_k(\bcH)^\top \|_F}{\sigma_{r_k}\left(\mathcal{M}_k(\bcH)\right)} \overset{\text{Lemma \ref{lm:SVD-projection}}}{\leq} \frac{\left\|\mathcal{M}_k\left(\llbracket\widetilde{\A}; P_{\widetilde{\U}_1}, P_{\widetilde{\U}_2}, P_{\widetilde{\U}_3}\rrbracket\right)-\mathcal{M}_k(\A)\right\|_F}{\sigma_{r_k}(\mathcal{M}_k(\bcH))}\\
	\overset{\eqref{ineq:sparse-regression-1}\eqref{ineq:sparse-regression-3}}{\leq} & C\widetilde{\sigma}\frac{\sqrt{m_s/n_1}}{\lambda_k},\quad k=1,2,3. 
	\end{split}
	\end{equation*}
	\begin{equation*}
	\begin{split}
	\left\|\A_k\widetilde{\W}_{k\perp}\right\|_F = &  \left\|\A_k\W_k\W_k^\top\widetilde{\W}_{k\perp}\right\|_F \leq \left\|\W_k^\top \widetilde{\W}_{k\perp}\right\|_F\cdot \|\A_k\|\\ 
	= & \left\|\sin\Theta(\widetilde{\W}_k, \W_k)\right\|_F\cdot \|\A_k\| \leq C\kappa\widetilde{\sigma}\sqrt{m_s/n_1}.
	\end{split}
	\end{equation*}

	Since $\widetilde{\U}_k$ and $\U_k$ are the leading left singular values of $\mathcal{M}_k\left(\llbracket \widetilde{\bcA}; P_{\widetilde{\U}_1}, P_{\widetilde\U_2}, P_{\widetilde{\U}_3}\rrbracket\right)$ and $\A_k$, respectively,
	\begin{equation*}
	\begin{split}
	\left\|\widetilde{\U}_{k\perp}^\top \A_k\right\|_F = & \left\|\widetilde{\U}_{k\perp}^\top \U_k\U_k^\top\A_k\right\|_F \leq \left\|\widetilde{\U}_{k\perp}^\top \U_k\right\|_F\cdot \left\|\U_k^\top \A_k\right\| =\left\|\sin\Theta(\widetilde{\U}_k, \U_k)\right\|_F\cdot \|\A_k\|\\
	\leq & \left\{\begin{array}{ll} \frac{C\widetilde{\sigma}\sqrt{(s_kr_k+s_k\log(p_k))/n_1}}{\lambda_k} \cdot \|\A_k\| \leq C\kappa\widetilde{\sigma}\sqrt{(s_kr_k+s_k\log(p_k))/n_1}, & k\in J_s;\\
	\frac{C\widetilde{\sigma}\sqrt{p_kr_k/n_1}}{\lambda_k} \cdot \|\A_k\| \leq C\kappa\widetilde{\sigma}\sqrt{p_kr_k/n_1}, & k\notin J_s.
	\end{array}\right.
	\end{split}
	\end{equation*}	
	In summary, in the previous two steps, we have shown
	\begin{equation}\label{ineq:STAT-SVD-conclusion}
	\begin{split}
	& \left\|\sin\Theta(\widetilde{\U}_k, \U_k)\right\|_F \leq \left\{\begin{array}{ll}
	\frac{C\widetilde{\sigma}\sqrt{(s_kr_k+s_k\log(p_k))/n_1}}{\lambda_k}, & k\in J_s;\\
	\frac{C\widetilde{\sigma}\sqrt{p_kr_k/n_1}}{\lambda_k}, & k\notin J_s,
	\end{array}\right. \\ 
	& \left\|\widetilde{\U}_{k\perp}^\top \A_k\right\|_F \leq C\kappa\widetilde{\sigma}\sqrt{m_s/n_1},\\
	& \left\|\sin\Theta(\widetilde{\W}_k, \W_k)\right\|_F \leq \frac{C\widetilde{\sigma}\sqrt{m_s/n_1}}{\lambda_k},\\
	&  \left\|\A_k\widetilde{\W}_{k\perp}\right\|_F\leq C\kappa\widetilde{\sigma}\sqrt{m_s/n_1},\quad \text{for}\quad k=1,2,3
	\end{split}
	\end{equation}
	with probability at least $1 - p^{-C}$.
	
	\item[Step 3] Next, we move on to analyze the second batch of samples $\{\bcX_j^{(2)}, \varepsilon_j^{(2)}\}_{j=1}^{n_2}$. We first introduce the following notations,
	\begin{equation*}
	\widehat{\sigma}_{\bcB}^2 = \sigma^2 + \left\|P_{(\widetilde{\U}_3\otimes \widetilde{\U}_2\otimes \widetilde{\U}_1)_\perp}\rmvec(\bcA)\right\|_2^2,\quad  \widehat{\sigma}^2_{\E_k} = \sigma^2 + \left\|P_{\left(\mathcal{R}_k(\widetilde{\W}_k\otimes \I_{p_k})\right)_\perp} \rmvec(\bcA)\right\|_2^2.
	\end{equation*}
	In this step, we give an upper bound for $\widehat{\sigma}_{\bcB}^2$ and $\widehat{\sigma}_{\E_k}^2$ given \eqref{ineq:STAT-SVD-conclusion} holds. Note that
	\begin{equation*}
	\begin{split}
	& \left\|P_{(\widetilde{\U}_3\otimes \widetilde{\U}_2\otimes \widetilde{\U}_1)_\perp}\rmvec(\bcA)\right\|_2 \\
	= & \left\|\rmvec(\bcA) - P_{(\widetilde{\U}_3\otimes \widetilde{\U}_2\otimes \widetilde{\U}_1)}\rmvec(\bcA)\right\|_2 =  \left\|\bcA - \llbracket \bcA; P_{\widetilde{\U}_1}, P_{\widetilde{\U}_2}, P_{\widetilde{\U}_3}\rrbracket\right\|_{\tHS} \\
	= & \left\|\llbracket\bcA; P_{\widetilde{\U}_1}+P_{\widetilde{\U}_{1\perp}}, P_{\widetilde{\U}_2}+P_{\widetilde{\U}_{2\perp}}, P_{\widetilde{\U}_3}+P_{\widetilde{\U}_{3\perp}}\rrbracket- \llbracket \bcA; P_{\widetilde{\U}_1}, P_{\widetilde{\U}_2}, P_{\widetilde{\U}_3}\rrbracket\right\|_{\tHS}\\
	\leq & \left\|\bcA; P_{\widetilde{\U}_{1\perp}}, P_{\widetilde{\U}_2}, P_{\widetilde{\U}_3}\right\|_{\tHS} + \left\|\bcA; \I_{p_1}, P_{\widetilde{\U}_{2\perp}}, P_{\widetilde{\U}_3}\right\|_{\tHS} + \left\|\bcA; \I_{p_1}, \I_{p_2}, P_{\widetilde{\U}_{3\perp}}\right\|_{\tHS}\\
	\leq & \left\|\widetilde{\U}_{1\perp}^\top \A_1\right\|_F + \left\|\widetilde{\U}_{2\perp}^\top \A_2\right\|_F + \left\|\widetilde{\U}_{3\perp}^\top \A_3\right\|_F\\
	\overset{\eqref{ineq:STAT-SVD-conclusion}}{\leq} & C\kappa\widetilde{\sigma}\sqrt{m_s/n_1},
	\end{split}
	\end{equation*}
	\begin{equation*}
	\begin{split}
	& \left\|P_{\left(\mathcal{R}_k(\widetilde{\W}_k\otimes \I_{p_k})\right)_\perp}\rmvec(\bcA)\right\|_2 = \left\|\rmvec(\bcA) - P_{\mathcal{R}_k(\widetilde{\W}_k\otimes \I_{p_k})}\rmvec(\bcA)\right\|_2 \\
	= & \left\|\A_k P_{\widetilde{\W}_{k\perp}}\right\|_F = \left\|\A_k\widetilde{\W}_{k\perp}\right\|_F \leq C\kappa\widetilde{\sigma}\sqrt{m_s/n_1}.
	\end{split}
	\end{equation*}
	Therefore,
	\begin{equation}\label{ineq:sigma_B-sigma_E}
	\widehat{\sigma}_{\bcB}^2 \leq \sigma^2 + \frac{Cm_s\kappa^2\widetilde{\sigma}^2}{n_1}, \quad \widehat{\sigma}_{\E_k}^2 \leq \sigma^2 + \frac{Cm_s\kappa^2\widetilde{\sigma}^2}{n_1},\quad k=1,2,3.
	\end{equation}
	\item[Step 4] In this step, we analyze the estimation error for $\widehat{\bcB}$ and $\widehat{\E}_k$ under the assumption that \eqref{ineq:STAT-SVD-conclusion} hold (which further means \eqref{ineq:sigma_B-sigma_E} holds). Recall the partial linear models on importance sketching covariates (see \eqref{eq:tilde-B-tilde-E} - \eqref{eq:epsilon_B-epsilon_E}; also see the proof of Theorem \ref{th:upper_bound_sparse_general}),
	\begin{equation*}
	\begin{split}
	y^{(2)} = \widetilde{\X}_\bcB \rmvec(\widetilde{\bcB}) + \widetilde{\varepsilon}_{\bcB}, 
	\end{split}
	\end{equation*}
	\begin{equation*}
	y^{(2)} = \widetilde{\X}_{\E_k} \rmvec(\widetilde{\E}_k) + \widetilde{\bvarepsilon}_{\E_k}, \quad k=1,2,3,
	\end{equation*}
	where the covariates, parameters, and noises of these two regressions are
	\begin{equation*}
	\widetilde{\X}_\bcB \in \mathbb{R}^{n_2\times (r_1r_2r_3)}, \quad (\widetilde{\X}_\bcB)_{i\cdot} = \rmvec\left(\bcX_i^{(2)} \times_1 \widetilde{\U}_1 \times_2 \widetilde{\U}_2 \times_3 \widetilde{\U}_3\right);
	\end{equation*}
	\begin{equation*}
	\begin{split}
	\widetilde{\X}_{\E_k}\in \mathbb{R}^{n_2\times (p_kr_k)}, \quad (\widetilde{\X}_{\E_k})_{i\cdot} = & \rmvec\left( \X_{ik}^{(2)}\left(\widetilde{\U}_{k+2}\otimes \widetilde{\U}_{k+1}\right)\widetilde{\V}_k\right) \\
	= & \rmvec\left(\X_{ik}^{(2)}\widetilde{\W}_k\right), \quad k=1,2,3;
	\end{split}
	\end{equation*}
	\begin{equation*}
	\begin{split}
	&\widetilde{\bvarepsilon}_\bcB \in \mathbb{R}^n, \quad (\widetilde{\bvarepsilon}_\bcB)_j = \left\langle \rmvec\left(\bcX^{(2)}_j\right); P_{(\widetilde{\U}_3\otimes \widetilde{\U}_2\otimes \widetilde{\U}_1)_{\perp}}\rmvec(\bcA)\right\rangle + \varepsilon_j^{(2)}, \\
	& \widetilde{\bvarepsilon}_{\E_k} \in \mathbb{R}^n, \quad 
	(\widetilde{\bvarepsilon}_{\E_k})_j = \left\langle \rmvec\left(\bcX^{(2)}_{j}\right), P_{\left(\mathcal{R}_k(\widetilde{\W}_k\otimes\I_{p_k})\right)_\perp}\rmvec(\bcA)\right\rangle + \varepsilon_j^{(2)}, \quad k=1,2,3;
	\end{split}
	\end{equation*}
	\begin{equation*}
	\begin{split}
	& \rmvec(\widetilde{\bcB}) = \rmvec(\llbracket \bcA; \widetilde{\U}_1^\top, \widetilde{\U}_2^\top, \widetilde{\U}_3^\top\rrbracket) = (\widetilde{\U}_3\otimes \widetilde{\U}_2\otimes \widetilde{\U}_1) \rmvec(\bcA) \in \mathbb{R}^{r_1r_2r_3};\\
	\text{and} \quad & \widetilde{\E}_k = \mathcal{M}_k\left(\bcA\times_{k+1}\widetilde{\U}_{k+1}^\top\times_{k+2}\widetilde{\U}_{k+2}^\top\right)\widetilde{\V}_{k} = \A_k \W_k\in \mathbb{R}^{p_k\times r_k},\quad k=1,2,3.
	\end{split}
	\end{equation*}
	These quantities satisfy the following properties.
	\begin{itemize}[leftmargin=*]
		\item Based on the proof of Theorem \ref{th:upper_bound_sparse_general}, $\widetilde{\E}_k, k\in J_s$ are group-wise sparse,
		\begin{equation*}
		\begin{split}
		\left\|\rmvec(\widetilde{\E}_k)\right\|_{0, 2} = \sum_{i=1}^{p_k} 1_{\left\{(\rmvec(\widetilde{\E}_k))_{G_i^k} \neq 0\right\}} \leq s_k,
		\end{split}
		\end{equation*}
		where $G_i^k = \{i+p_k, \ldots, i+p_k(r_k-1)\}, i=1,\ldots, p_k, k\in J_s$. 
		\item Conditioning on fixed values of $\widetilde{\U}_k \widetilde{\V}_k, \widetilde{\W}_k$, the noise distribution satisfies
		\begin{equation*}
		\begin{split}
		& \widetilde{\bvarepsilon}_\bcB\Big| \widetilde{\U}_k, \widetilde{\V}_k, \widetilde{\W}_k \overset{iid}{\sim} N\left(0, \sigma^2 + \left\|P_{(\widetilde{\U}_3\otimes \widetilde{\U}_2\otimes \U_1)_\perp}[\bcA]\right\|_{\tHS}\right) \sim N(0, \widehat{\sigma}^2_{\bcB}); \\
		& \widetilde{\bvarepsilon}_{\E_k}\Big| \widetilde{\U}_k, \widetilde{\V}_k, \widetilde{\W}_k \overset{iid}{\sim} N\left(0, \sigma^2 + \left\|P_{\left(\mathcal{R}_k(\widetilde{\W}_k\otimes \I_{p_k})\right)_\perp}[\bcA]\right\|_{\tHS}\right) \sim N(0, \widehat{\sigma}^2_{\E_k}).
		\end{split}
		\end{equation*}
		\item Note that $\widetilde{\X}_\bcB$ is an $n_2$-by-$(r_1r_2r_3)$ matrix with i.i.d. Gaussian entries. Similarly to the argument in Step 5 in the proof of Theorem \ref{th:upper_bound_regression}, 
		\begin{equation*}
		\begin{split}
		& \left\|\left(\widetilde{\X}^\top_{\bcB}\widetilde{\X}_{\bcB}\right)^{-1}\widetilde{\X}_\bcB^\top \widetilde{\bvarepsilon}_\bcB\right\|_2^2 \\
		\leq &  \frac{\widehat{\sigma}^2_\bcB\left(n_2+2\sqrt{n_2C\log(p)}+2C\log(p)\right)\left(r_1r_2r_3 + 2\sqrt{Cr_1r_2r_3\log(p)}+2C\log(p)\right)}{\left(\sqrt{n_2}-\sqrt{r_1r_2r_3} - C\log(p)\right)^4}\\
		\leq & \frac{\widehat{\sigma}_\bcB^2}{n_2}\frac{\left(1+2\sqrt{\frac{C\log p}{n_2}}+2\frac{\log p}{n_2}\right)Cm_s}{\left(1 - \sqrt{\frac{r_1r_2r_3}{n_2}} - C\sqrt{\frac{\log(p)}{n_2}}\right)^4} \leq \frac{C\widehat{\sigma}_\bcB^2m_s}{n_2}.
		\end{split}
		\end{equation*}
		with probability at least $1 - p^{-C}$. Here, the second last inequality is due to $\sqrt{r_1r_2r_3\log(p)} \leq \frac{1}{2}\left(r_1r_2r_3 + \log(p)\right) \leq m_s$ and the last inequality is due to $n_2 \geq Cm_s$. By the proof of Theorem \ref{th:upper_bound_sparse_general},
		\begin{equation}\label{ineq:X_BX_B^-1epsilon}
		\begin{split}
		& \left\|\widehat{\bcB} - \widetilde{\bcB}\right\|_{\tHS}^2 \overset{\eqref{eq:hat_B-B-identity}}{=} \left\|\left(\widetilde{\X}^\top_{\bcB}\widetilde{\X}_{\bcB}\right)^{-1}\widetilde{\X}_\bcB^\top \widetilde{\bvarepsilon}_\bcB\right\|_2^2 \leq \frac{Cm_s\widehat{\sigma}^2_{\bcB}}{n_2}.
		\end{split}
		\end{equation}
		with probability at least $1 - p^{-C}$. Similarly, we can show for $k\notin J_s$, the least square estimator $\widehat{\E}_k$ satisfies
		\begin{equation}\label{ineq:X_EkX_Ek^-1epsilon}
		\left\|\widehat{\E}_k - \E_k\right\|_F^2 \overset{\eqref{eq:hat_Ek-Ek-identity}}{=} \left\|\left(\widetilde{\X}_{\E_k}^\top\widetilde{\X}_{\E_k}\right)^{-1}\widetilde{\X}_{\E_k}^\top\widetilde{\bvarepsilon}_{\E_k}\right\|_2^2 \leq \frac{Cm_s\widehat{\sigma}_{\E_k}^2}{n_2}.
		\end{equation}
		\item By Lemma \ref{lm:GE->GRIP} and $n_2\geq Cm_s$ for large constant $C>0$, $\widetilde{\X}_{\D_k}$ satisfies group restricted isometry property with $\delta = 1/4$ with probability at least $1 - \exp(-cn)$. 
		
		Next, since $\widetilde\varepsilon_{\E_k}\overset{iid}{\sim}N_{n_2}\left(0, \widehat{\sigma}_{\E_k}^2\right)$ and $(\widetilde{\X}_{\E_k}^i)^\top\widetilde{\bvarepsilon}_{\E_k} \bigg| \|\widetilde{\bvarepsilon}_{\E_k}\|_2^2 \sim N_{r_k}\left(0, \|\widetilde{\varepsilon}_{\E_j}\|_2^2\right)$, we know
		$$\|\widetilde{\bvarepsilon}_{\E_k}\|_2^2  \sim \widehat{\sigma}_{\E_k}^2\chi^2_{n_2}\quad \text{and}\quad \|(\widetilde{\X}_{\E_k}^i)^\top\widetilde{\bvarepsilon}_{\E_k}\|_2^2 \bigg| \|\widetilde{\bvarepsilon}_{\E_k}\|_2^2 \sim \|\widetilde{\bvarepsilon}_{\E_k}\|_2^2\cdot\chi_{r_k}^2$$
		By the tail bound of $\chi^2$ distribution,
		\begin{equation*}
		\begin{split}
		\left\|(\widetilde{\X}^i_{\E_k})^\top \widetilde{\bvarepsilon}_{\E_k}\right\|_2^2 \leq & \widehat{\sigma}_{\E_k}^2\left(n_2+2\sqrt{n_2C\log(p)}+2C\log(p)\right)\left(r_k + 2\sqrt{r_kC\log(p)} + 2C\log(p)\right)\\
		\leq & Cn_2\widehat{\sigma}_{\E_k}^2(r_k+\log(p))
		\end{split}
		\end{equation*}
		with probability at least $1 - p^{-C}$. Since $\log(p_k)\asymp \log(p)$, we have
		\begin{equation}\label{ineq:maxtildeX_Evarepsilon}
		\max_{1\leq i\leq p_k} \left\|(\widetilde{\X}^i_{\E_k})^\top \widetilde{\bvarepsilon}_{\E_k}\right\|_2^2 \leq Cn_2\widehat{\sigma}^2_{\E_k}(r_k+\log(p_k))
		\end{equation}
		with probability at least $1 - p^{-C}$. 
		\item Similarly as the Step 5 in the proof of Theorem \ref{th:upper_bound_regression}, one can show
		\begin{equation*}
		\left\|\widehat{\E}_k (\widetilde{\U}_k^\top \widehat{\E}_k)^{-1}\right\| \leq 1 + \frac{C_1\kappa\widetilde{\sigma}}{\lambda_k}\sqrt{\frac{m_s}{n_1}} + \frac{C_2\kappa\widetilde{\sigma}}{\lambda_k}\sqrt{\frac{m_s}{n_2}}\leq 1+c,\quad k=1,2,3
		\end{equation*}
		for constant $0<c<1/2$.
	\end{itemize}
	By previous arguments, we have shown the conditions of Theorem \ref{th:upper_bound_sparse_general} hold with probability at least $1 - p^{-C}$ under the scenario of Theorem \ref{th:upper_bound_sparse_tensor_regression}. Finally, Theorem \ref{th:upper_bound_sparse_general} implies
	\begin{equation*}
	\begin{split}
	& \left\|\widehat{\bcA} - \bcA\right\|_{\tHS}^2 \\
	\leq & \left(1+\frac{C_1\kappa\widetilde{\sigma}}{\lambda_0}\sqrt{\frac{m_s}{n_1\wedge n_2}}\right)\Bigg(\left\|(\widetilde{\X}_\bcB^\top\widetilde{\X}_\bcB)^{-1}\widetilde{\X}_\bcB^\top \widetilde{\bvarepsilon}_\bcB\right\|_2^2 + C_2\sum_{k\in J_s} s_k \max_{1\leq i\leq p_k} \left\|(\widetilde{\X}_{\E_k}^i)^\top \widetilde{\bvarepsilon}_{\E_k}/n_2\right\|_2^2\\
	& \quad\quad + \sum_{k\notin J_s} \left\|(\widetilde{\X}_{\E_k}^\top\widetilde{\X}_{\E_k})^{-1}\widetilde{\X}_{\E_k}^\top \widetilde{\bvarepsilon}_{\E_k}\right\|_2^2\Bigg)\\
	\overset{(a)}{\leq} & C\left(\frac{m_s(\widehat{\sigma}_{\bcB}^2 + \widehat{\sigma}_{\E_k}^2)}{n_2} + C\sum_{k=1}^3 \frac{s_k(r_k+\log(p_k))\widehat{\sigma}_{\E_k}^2}{n_2}\right) \\
	\overset{(b)}{\leq} & \frac{C_1m_s}{n_2}\left(\sigma^2 + \frac{C_2m_s\kappa^2\widetilde{\sigma}^2}{n_1}\right)
	\end{split}
	\end{equation*}
	with probability at least $1 - p^{-C}$. Here, (a) is due to \eqref{ineq:X_BX_B^-1epsilon}, \eqref{ineq:X_EkX_Ek^-1epsilon}, and \eqref{ineq:maxtildeX_Evarepsilon}; (b) is due to \eqref{ineq:sigma_B-sigma_E}. \quad $\square$
\end{enumerate}

\subsection{Proof of Theorem \ref{th:lower_bound_sparse_tensor_regression}}\label{sec:proof-lower_bound_sparse_tensor_regression}

This theorem gives a lower bound on the estimation error of sparse low-rank tensor regression. In order to prove the desired lower bound, we only need to prove the forthcoming \eqref{ineq:lower-sparse-to-show-1} and \eqref{ineq:lower-sparse-to-show-2}, respectively. To prove each inequality, we first construct a series of tensor parameters $\bcA^{(j)}$ that satisfy: (1) there are sufficient distances between $\bcA^{(j)}$ and $\bcA^{(l)}$ for any $j\neq l$; (2) the Kullback-Leiber divergence between the resulting observations, $\{y_i^{(j)}, \bcX_i^{(j)}\}_{i=1}^n$ and $\{y_i^{(l)}, \bcX_i^{(l)}\}_{i=1}^n$, are close. Finally, the lower bound is proved by an application of the generalized Fano's Lemma.

In order to prove this theorem, we only need to show 
\begin{equation*}
\inf_{\widehat{\bcA}}\sup_{\bcA\in \mathcal{A}_{\bp, \bs, \br}}\mathbb{E}\left\|\widehat{\bcA} - \bcA\right\|_{\tHS}^2 \geq \max\left\{\frac{cr_1r_2r_3\sigma^2}{n}, \max_{l=1,2,3}\frac{c\sigma^2\left(s_lr_l+s_l\log(ep_l/s_l)\right)}{n}\right\}.
\end{equation*}
\begin{enumerate}[leftmargin=*]
	\item If
	\begin{equation*}
	r_1r_2r_3 = \max\left\{r_1r_2r_3, \max_{k=1,2,3} \left(s_kr_k+s_k\log(ep_k/s_k)\right)\right\}, 
	\end{equation*}
	we only need to prove
	\begin{equation}\label{ineq:lower-sparse-to-show-1}
	\begin{split}
	\inf_{\widehat{\bcA}}\sup_{\bcA\in \mathcal{A}_{\bp, \bs, \br}}\mathbb{E}\left\|\widehat{\bcA} - \bcA\right\|_{\tHS}^2 \geq \frac{cr_1r_2r_3\sigma^2}{n},
	\end{split}
	\end{equation}
	for $r_1r_2r_3\geq 9$ in order to finish the proof of this theorem. Construct $\bcS_0$ as an $r_1$-by-$r_2$-by-$r_3$ tensor with i.i.d. Gaussian entries. Since $r_k\geq r_{k+1}r_{k+2}$ for $k=1,2,3$, $\bcS_0$ has Tucker rank-$(r_1, r_2, r_3)$ with probability one. Let $\U_1, \U_2, \U_3$ be arbitrary fixed orthogonal matrices that satisfy
	$$\U_k\in \mathbb{O}_{p_k, r_k}, \quad \|\U_k\|_{0, 2} = \sum_{i=1}^{p_k}1_{\{(\U_k)_{[i,:]}\neq 0\}}\leq s_k, \quad k=1,2,3.$$
	By Varshamov-Gilbert bound \cite[Lemma 4.7]{massart2007concentration}, we can find $\bcB^{(1)},\ldots, \bcB^{(N)}\subseteq \{-1, 1\}^{r_1\times r_2\times r_3}$ such that 
	$$\forall j\neq l, \quad \|\bcB^{(j)}-\bcB^{(l)}\|_{\tHS}^2 = 2\sum_{i_1, i_2}|\bcB^{(j)}_{[i_1,i_2]} - \bcB^{(l)}_{[i_1,i_2]}| \geq 2r_1r_2r_3 \quad \text{and}\quad N\geq \exp(r_1r_2r_3/8).$$
	On the other hand,
	\begin{equation}\label{ineq:thm6-1}
	\|\bcB^{(j)} - \bcB^{(l)}\|_{\tHS}^2 \leq 2\|\bcB^{(j)}\|_{\tHS}^2 + 2\|\bcB^{(l)}\|_{\tHS}^2 \leq 4r_1r_2r_3.
	\end{equation} 
	Since $r_1r_2r_3\geq 9$, $N\geq 3$. Then we construct
	\begin{equation*}
	\bcA^{(j)} = \llbracket \bcS_0 + \tau \bcB_j; \U_1, \U_2, \U_3\rrbracket, \quad j=1,\ldots, N,
	\end{equation*}
	where $\tau>0$ is a constant to be determined a little while later. By such the configuration, $\bcA^{(1)},\ldots, \bcA^{(N)}\subseteq \mathcal{A}_{\bp, \bs, \br}$. Now, the Kullback–Leibler divergence between the samples generated from $\bcA^{(j)}$ and the samples generated from $\bcA^{(l)}$ satisfy
	\begin{equation}
	\begin{split}
	& D_{KL}\left(\{\bcX_i, y_i^{(j)}\}_{i=1}^n\Big|\Big| \{\bcX_i, y_i^{(l)}\}_{i=1}^n\right) \overset{\text{Lemma \ref{lm:regression-KL}}}{=} \frac{n}{2\sigma^{2}}\left\|\bcA^{(j)} - \bcA^{(l)}\right\|_{\tHS}^{2} \\
	\leq & \frac{n}{2\sigma^{2}}\left\|\tau\bcB^{(j)} - \tau\bcB^{(l)}\right\|_{\tHS}^{2} \overset{\eqref{ineq:thm6-1}}{\leq} \frac{n}{2\sigma^{2}}(4\tau^2 r_1r_2r_3)
	\end{split}
	\end{equation}
	and
	\begin{equation*}
	\forall j\neq l,\quad  \left\|\bcA^{(j)} - \bcA^{(l)}\right\|_{\tHS}^2 = \left\|\tau\bcB^{(j)} - \tau\bcB^{(l)}\right\|_{\tHS}^2 \geq 2\tau^2 r_1r_2r_3.
	\end{equation*}
	By generalized Fano's lemma,
	\begin{equation*}
	\begin{split}
	\inf_{\widehat{\bcA}}\sup_{\bcA\in  \mathcal{A}_{\bp, \bs, \br}}\left\|\widehat{\bcA} - \bcA\right\|_{\tHS}^2 \geq &	\inf_{\widehat{\bcA}}\sup_{\bcA\in \left\{\bcA^{(1)},\ldots, \bcA^{(N)}\right\}}\left\|\widehat{\bcA} - \bcA\right\|_{\tHS}^2\\
	\geq &  \tau^2r_1r_2r_3\left(1 - \frac{2\tau^2r_1r_2r_3n/\sigma^2+\log(2)}{\log(N)}\right).
	\end{split}
	\end{equation*}
	By setting $\tau^2 = \sigma^2\log(N/2.5)/(2r_1r_2r_3n)$, we have
	\begin{equation*}
	\inf_{\widehat{\bcA}}\sup_{\bcA\in  \mathcal{A}_{\bp, \bs, \br}}\left\|\widehat{\bcA} - \bcA\right\|_{\tHS}^2 \geq c\tau^2r_1r_2r_3 = \frac{c\sigma^2r_1r_2r_3}{n},
	\end{equation*}
	which has shown \eqref{ineq:lower-sparse-to-show-1} if $r_1r_2r_3\geq 9$.
	\item If
	\begin{equation*}
	s_kr_k+s_k\log(ep_k/s_k) = \max\left\{r_1r_2r_3, \max_{l=1,2,3}\left(s_l r_l +s_k \log(ep_l/s_l)\right)\right\}, 
	\end{equation*}
	we only need to prove
	\begin{equation}\label{ineq:lower-sparse-to-show-2}
	\begin{split}
	\inf_{\widehat{\bcA}}\sup_{\bcA\in \mathcal{A}_{\bp, \br}}\mathbb{E}\left\|\widehat{\bcA} - \bcA\right\|_{\tHS}^2 \geq  \frac{c\sigma^2\left(s_kr_k+s_k\log(ep_k/s_k)\right)}{n},
	\end{split}
	\end{equation}
	provided that $s_kr_k+s_k\log(ep_k/s_k)\geq C$ for large constant $C>0$. Without loss of generality we assume $k=1$. 
	
	To this end, we randomly generate an orthogonal matrix $\S\in \mathbb{O}_{r_2r_3, r_1}$ and construct $\bcS\in \mathbb{R}^{r_1\times r_2\times r_3}$ such that $\mathcal{M}_1(\bcS) = \S^\top$. We also construct $\U_2$ and $\U_3$ as fixed orthogonal matrices that satisfies $\|\U_2\|_{0, 2} \leq s_2$ and $\|\U_3\|_{0, 2}\leq s_3$. By Lemma \ref{lm:sparse-Varshamov-Gilbert}, there exists $\{\U_1^{(k)}\}_{k=1}^N \subseteq \{1, 0, -1\}^{p_1\times r_1}$ such that
	\begin{equation}\label{ineq:U_1-condition}
	\begin{split}
	& \|\U_1^{(j)}\|_{0, 2} = \sum_{i=1}^{p_1} 1_{\left\{(\U_1^{(j)})_{[i,:]}\neq 0\right\}} \leq s_1, \quad j=1,\ldots, N,\\
	&  \left\|\U_1^{(j)} - \U_1^{(l)}\right\|_{1,1} = \sum_{i,j} \left|(\U_1^{(j)})_{ij} - (\U_1^{(l)})_{ij}\right| > s_1r_1/2,\quad 1\leq j \neq l \leq N,
	\end{split}
	\end{equation}
	and $N\geq \exp\left(c(s_1r_1+s_1\log(ep_1/s_1))\right)$. We further let
	\begin{equation*}
	\bcA^{(j)} = \llbracket \tau\bcS; \U_1^{(j)}, \U_2, \U_3\rrbracket, \quad j=1,2,\ldots, N,
	\end{equation*}
	where $\tau$ is a fixed and to-be-determined value. By such the construction, for any $1\leq j \neq l \leq N$,
	\begin{equation*}
	\begin{split}
	\left\|\bcA^{(j)} - \bcA^{(l)}\right\|_{\tHS}^2 = & \tau^2\left\|\U_1^{(j)}\mathcal{M}_1(\bcS)\U_3^\top\otimes \U_2^\top - \U_1^{(l)}\mathcal{M}_1(\bcS)\U_3^\top\otimes \U_2^\top \right\|_F^2 \\
	= & \tau^2\left\|\U_1^{(j)}\S^\top \U_3^\top\otimes \U_2^\top - \U_1^{(l)}\S^\top\U_3^\top\otimes \U_2^\top \right\|_F^2 = \tau^2\left\|\U_1^{(j)} - \U_1^{(l)}\right\|_F^2\\
	& \quad \text{(since all entries of $\U_1^{(j)}, \U_1^{(l)}\in \{-1, 0, 1\}$)}\\
	\geq & \tau^2\left\|\U_1^{(j)} - \U_1^{(l)}\right\|_{1,1} > \tau^2 s_1r_1/2,
	\end{split}
	\end{equation*}
	\begin{equation}
	\begin{split}
	\text{and}\quad & D_{KL}\left(\{\bcX_i, y_i^{(j)}\}_{i=1}^n\Big|\Big| \{\bcX_i, y_i^{(l)}\}_{i=1}^n\right) = \frac{n}{2\sigma^{2}}\left\|\bcA^{(j)} - \bcA^{(l)}\right\|_{\tHS}^{2}\\
	= & \frac{n}{2\sigma^2}\tau^2\left\|\U_1^{(j)} - \U_1^{(l)}\right\|_{F}^2 \leq \frac{n\tau^2}{2\sigma^2} 2\left(\|\U_1^{(j)}\|_2^2 + \|\U_1^{(l)}\|_2^2\right) \leq \frac{n\tau^2}{2\sigma^2}\cdot 4s_1r_1. 
	\end{split}
	\end{equation}
	By setting $\tau^2 = \sigma^2\log(N/2.5)/(2ns_1r_1)$, we have
	\begin{equation*}
	\begin{split}
	& \inf_{\widehat{\bcA}}\sup_{\bcA\in  \mathcal{A}_{\bp, \br}}\left\|\widehat{\bcA} - \bcA\right\|_{\tHS}^2 \geq \frac{\tau^2s_1r_1}{4}\left(1 - \frac{\frac{2n\tau^2s_1r_1}{\sigma^2} - \log(2)}{\log(N)}\right) \\
	\geq & \frac{2\sigma^2\log(N/2.5)}{4ns_1r_1}\cdot \frac{s_1r_1}{4}\cdot c \geq \frac{c\sigma^2\left(s_1r_1 + s_1\log(ep_1/s_1)\right)}{n},
	\end{split}
	\end{equation*}
	which has shown \eqref{ineq:lower-sparse-to-show-2}.
\end{enumerate}
In summary of the previous two parts, we have finished the proof of this theorem. \quad $\square$

\section{Technical Lemmas}

\begin{Lemma}[Kronecker Product, Vectorization, and Matricization]\label{lm:Kronecker-vectorization-matricization}
	Suppose $\A\in \mathbb{R}^{p_1\times p_2}$, $\bcA\in\mathbb{R}^{p_1\times p_2\times \ldots \times p_d}$, $\B_k\in \mathbb{R}^{p_k\times r_k}$, $\B'_k \in \mathbb{R}^{r_k\times d_k}$, $k=1,\ldots, d$. Then,
	\begin{equation}\label{eq:Kronecker-0}
	(\B_1\otimes \cdots \otimes\B_d)\cdot (\B_1'\otimes \cdots \otimes\B_d') = (\B_1\B_1')\otimes \cdots \otimes(\B_d\B_d'),
	\end{equation}
	\begin{equation}\label{eq:kronecker-1}
	\rmvec\left(\B_1^\top \A \B_2 \right) = (\B_2^\top\otimes \B_1^\top)\rmvec(\A),
	\end{equation}
	\begin{equation}\label{eq:kronecker-2}
	\rmvec\left(\llbracket\bcA; \B_1^\top, \ldots, \B_d^\top\rrbracket\right) = (\B_d^\top\otimes \cdots \otimes \B_1^\top)\rmvec(\bcA),
	\end{equation}
	\begin{equation}\label{eq:kronecker-3}
	\mathcal{M}_k\left(\llbracket \bcA; \B_1^\top, \ldots, \B_d^\top \rrbracket\right) = \B_k^\top \mathcal{M}_k(\bcA) \left(\B_d\otimes \cdots \otimes \B_{k+1}\otimes \B_{k-1}\otimes \cdots \otimes \B_1\right).
	\end{equation}
	Finally, for any $\V_k\in \mathbb{R}^{r_{-k}\times r_k}$, 
	\begin{equation}\label{eq:kronecker-4}
	\begin{split}
	& \rmvec\left(\B_k^\top \mathcal{M}_k\left(\left\llbracket\bcA; \B_1^\top, \ldots, \B_{k-1}^\top, \B_{k+1}^\top,\ldots, \B_d^\top \right\rrbracket\right)\V_k\right)\\
	= & \V_k^\top\left(\B_d^\top \otimes \cdots \otimes \B_{k+1}^\top \otimes \B_{k-1}^\top \otimes \cdots \otimes\B_1^\top \right)\otimes(\B_k^\top) \cdot \rmvec(\mathcal{M}_k(\bcA))\\
	\end{split}
	\end{equation}
\end{Lemma}
{\noindent\bf Proof of Lemma \ref{lm:Kronecker-vectorization-matricization}.} See \cite{kolda2009tensor,kolda2006multilinear} for the proof of \eqref{eq:Kronecker-0}, \eqref{eq:kronecker-2} and \eqref{eq:kronecker-3}. We shall also note that \eqref{eq:kronecker-1} is the order-2 case of \eqref{eq:kronecker-2}. Finally,
\begin{equation*}
\begin{split}
& \rmvec\left(\B_k^\top \mathcal{M}_k\left(\left\llbracket\bcA; \B_1^\top, \ldots, \B_{k-1}^\top, \B_{k+1}^\top,\ldots, \B_d^\top \right\rrbracket\right)\V_k\right)\\
\overset{\eqref{eq:kronecker-1}}{=} & (\V_k^\top\otimes \B_k^\top )\rmvec\left(\mathcal{M}_k\left(\left\llbracket\bcA; \B_1^\top, \ldots, \B_{k-1}^\top, \I_{p_k}, \B_{k+1}^\top,\ldots, \B_d^\top \right\rrbracket\right)\right)\\
\overset{\eqref{eq:kronecker-3}}{=} & (\V_k^\top\otimes \B_k^\top) \rmvec\left(\mathcal{M}_k(\bcA)(\B_d\otimes \cdots \otimes  \B_{k+1} \otimes \B_{k-1}\otimes \cdots \otimes \B_1)\right)\\
\overset{\eqref{eq:kronecker-1}}{=} & (\V_k^\top\otimes \B_k^\top) \left(\B_d^\top\otimes \cdots \otimes \B_{k+1}^\top \otimes \B_{k-1}^\top\otimes \cdots \otimes \B_1^\top \otimes \I\right) \rmvec(\mathcal{M}_k(\bcA))\\
= & \V_k^\top\left(\B_d^\top \otimes \cdots \otimes \B_{k+1}^\top \otimes \B_{k-1}^\top \otimes \cdots \otimes\B_1^\top \right)\otimes(\B_k^\top) \cdot \rmvec(\mathcal{M}_k(\bcA))\\
\end{split}
\end{equation*}
\quad $\square$

\begin{Lemma}\label{lm:concatenation-singular-value}
	Suppose $\A\in \mathbb{R}^{p\times r}$ and $\U\in \mathbb{O}_{p, m}$. Then,
	\begin{equation*}
	\sigma_r^2(\A) \geq \sigma_r^2(\U^\top\A) + \sigma_r^2(\U_{\perp}^\top\A),\quad 
	\left\|\A\right\|^2 \leq \left\|\U^\top \A\right\|^2 + \left\|\U_{\perp}^\top \A\right\|^2.
	\end{equation*}
\end{Lemma}
{\noindent\bf Proof of Lemma \ref{lm:concatenation-singular-value}.} 
Let $\v$ be the right singular vector associated with the $r$-th singular value of $\A$. Then $\|\A \v\|_2 = \sigma_r(\A)\|\v\|_2 = \sigma_r(\A)$ and
\begin{equation*}
\begin{split}
\sigma_r^2(\A) = & \|\A \v\|_2^2 = \|P_{\U}\A \v\|_2^2 + \|P_{\U_{\perp}}\A \v\|_2^2 =\|\U^\top \A \v\|_2^2 + \|\U_{\perp}^\top \A \v\|_2^2 \\
\geq & \sigma_r^2(\U^\top \A )\|\v\|_2^2 + \sigma_r^2(\U_{\perp}^\top \A )\|\v\|_2^2 = \sigma_r^2(\U^\top \A ) + \sigma_r^2(\U_{\perp}^\top \A ).
\end{split}
\end{equation*}
On the other hand,
\begin{equation*}
\begin{split}
\|\A \|^2 = & \max_{\v: \|\v\|_2\leq 1} \|\A \v\|_2^2 = \max_{\v: \|\v\|_2\leq 1} \left(\|P_\U \A \v\|_2^2 + \|P_{\U_{\perp}}\A \v\|_2^2\right) \\
\leq & \max_{\v: \|\v\|_2\leq 1} \|P_\U \A \v \|_2^2 + \max_{\v: \|\v\|_2\leq 1}\|P_{\U_\perp}\A \v\|_2^2 = \|\U^\top \A \|^2 + \|\U_{\perp}^\top \A \|^2.
\end{split}
\end{equation*}
\quad $\square$

The following lemma establish a deterministic upper bound for $\|\widehat\F\widehat\G^{-1}\widehat\H - \F\G^{-1}\H\|$ in terms of $\|\widehat\F-\F\|_F, \|\widehat\G-\G\|_F, \|\widehat\H-\H\|_F$ and its more general high-order form. This result serves as a key technical lemma for the theoretical analysis of the oracle inequalities.
\begin{Lemma}\label{lm:FGH}
	Suppose $\F, \widehat{\F}\in \mathbb{R}^{p_1\times r}, \G, \widehat{\G}\in \mathbb{R}^{r\times r}, \H, \widehat{\H}\in \mathbb{R}^{r\times p_2}$. If $\G$ and $\widehat{\G}$ are invertible, $\|\F\G^{-1}\|\leq \lambda_1$, $\|\G^{-1}\H\|\leq \lambda_2$, and $\|\widehat{\G}^{-1}\widehat{\H}\| \leq \lambda_2$, we have
	\begin{equation}\label{ineq:FGH1}
	\left\|\widehat{\F}\widehat{\G}^{-1}\widehat{\H} - \F\G^{-1}\H \right\|_F \leq \lambda_2\|\widehat{\F}-\F\|_F + \lambda_1\|\widehat{\H}-\H\|_F + \lambda_1\lambda_2\|\widehat{\G}-\G\|_F.
	\end{equation}
	More generally for any $d\geq 1$, suppose $\widehat{\bcF}, \bcF\in \mathbb{R}^{r_1\times \cdots \times  r_d}$ are order-$d$ tensors, $\G_k, \widehat{\G}_k \in \mathbb{R}^{r_k\times r_k}$ $\H_k, \widehat{\H}_k \in \mathbb{R}^{p_k\times r_k}$. If $\|\H_k\G_k^{-1}\|\leq\lambda_k$, $\|\widehat{\H}_k\widehat{\G}_k^{-1}\|\leq\lambda_k$, and $\|\G_k^{-1}\mathcal{M}_k(\bcF)\|\leq \pi_k$, we have
	\begin{equation}\label{ineq:FGH2}
	\begin{split}
	& \left\|\llbracket \widehat{\bcF}; (\widehat{\H}_1\widehat{\G}_1^{-1}), \ldots, (\widehat{\H}_d\widehat{\G}_d^{-1}) \rrbracket - \llbracket \bcF; (\H_1\G_1^{-1}), \ldots, (\H_d\G_d^{-1}) \rrbracket\right\|_{\tHS}\\
	\leq & \lambda_1\cdots \lambda_d \|\widehat{\bcF}-\bcF\|_{\tHS} + \sum_{k=1}^d \pi_k\lambda_1\cdots \lambda_d\|\widehat{\G}-\G\|_F + \sum_{k=1}^d \pi_k \lambda_1\cdots \lambda_d/\lambda_k \|\widehat{\H}_k - \H_k\|_F.
	\end{split}
	\end{equation}
\end{Lemma}
{\bf\noindent Proof of Lemma \ref{lm:FGH}.} First, it is easy to check the following identity for any non-singular matrices $\G$ and $\widehat{\G}$,
\begin{equation*}
\widehat{\G}^{-1} = \G^{-1} - \G^{-1} (\widehat{\G} - \G) \widehat{\G}^{-1}.
\end{equation*}
Thus,
\begin{equation*}
\begin{split}
& \left\|\widehat{\F}\widehat{\G}^{-1}\widehat{\H} - \F\G^{-1}\H\right\|_F \\
\leq & \left\|(\widehat{\F} - \F)\widehat{\G}^{-1}\widehat{\H}\right\|_F + \left\|\F \left(\G^{-1} - \G^{-1}(\widehat{\G}-\G)\widehat{\G}^{-1}\right)\widehat{\H} - \F\G^{-1}\H\right\|_F\\
\leq & \left\|\widehat{\F} - \F\right\|_F \cdot \left\|\widehat{\G}^{-1}\widehat{\H}\right\| + \left\|\F\G^{-1}\widehat{\H} - \F\G^{-1}\H\right\|_F + \left\|\F\G^{-1}(\widehat{\G}-\G)\widehat{\G}^{-1}\widehat{\H}\right\|_F\\
\leq & \left\|\widehat{\F} - \F\right\|_F \left\|\widehat{\G}^{-1}\widehat{\H}\right\| + \left\|\F\G^{-1}\right\|\left\|\widehat{\H} - \H\right\|_F + \left\|\F\G^{-1}\right\|\left\|\widehat{\G}-\G\right\|_F\left\|\widehat{\G}^{-1}\widehat{\H}\right\|\\
\leq & \lambda_2\|\widehat{\F}-\F\|_F + \lambda_1\|\widehat{\H}-\H\|_F + \lambda_1\lambda_2\|\widehat{\G}-\G\|_F.
\end{split}
\end{equation*}
Then we consider the proof of \eqref{ineq:FGH2}. Define
\begin{equation*}
\begin{split}
& \widehat{\widetilde{\F}}_d = \mathcal{M}_d(\widehat{\bcF})\left(\widehat{\H}_{d-1}\widehat{\G}_{d-1}^{-1}\otimes \cdots \otimes \widehat{\H}_1\widehat{\G}_1^{-1}\right)^\top,\\
& \widetilde{\F}_d = \mathcal{M}_d(\bcF)\left(\H_{d-1}\G_{d-1}^{-1}\otimes \cdots \otimes \H_1\G_1^{-1}\right)^\top.
\end{split}
\end{equation*}
We shall note that
\begin{equation*}
\begin{split}
& \left\|\G_d^{-1}\widetilde{\F}_d\right\| = \left\|\G_d^{-1}\mathcal{M}_d(\bcF)\left(\H_{d-1}\G_{d-1}^{-1}\otimes \cdots \otimes \H_1\G_1^{-1}\right)\right\|\\
\leq & \left\|\G_d^{-1}\mathcal{M}_d(\bcF)\right\|\cdot \|\H_{d-1}\G_{d-1}^{-1}\| \cdots \|\H_1\G_1^{-1}\| \leq \pi_d\lambda_1\cdots \lambda_{d-1},
\end{split}
\end{equation*}
\begin{equation*}
\begin{split}
\left\|\H_d\G_d^{-1}\right\| \leq \lambda_d, \quad \|\widehat{\H}_d\widehat{\G}_d^{-1}\| \leq \lambda_d.
\end{split}
\end{equation*}
By the first part of this lemma and tensor algebra,
\begin{equation}\label{ineq:FGH-1}
\begin{split}
& \left\|\llbracket\widehat{\bcF}; \widehat{\H}_1\widehat{\G}_1^{-1},\ldots,\widehat{\H}_{d}\widehat{\G}_{d}^{-1}\rrbracket - \llbracket\bcF; \H_1\G_1^{-1},\ldots, \H_{d} \G_{d}^{-1}\rrbracket\right\|_{\tHS}\\
= & \left\|\mathcal{M}_d\left(\llbracket\widehat{\bcF}; \widehat{\H}_1\widehat{\G}_1^{-1},\ldots,\widehat{\H}_{d}\widehat{\G}_{d}^{-1}\rrbracket\right) - \mathcal{M}_d\left(\llbracket\bcF; \H_1\G_1^{-1},\ldots, \H_{d} \G_{d}^{-1}\rrbracket\right)\right\|_F\\
\overset{\text{Lemma \ref{lm:Kronecker-vectorization-matricization}}}{=} & \left\|\widehat{\H}_d\widehat{\G}_d^{-1}\widehat{\widetilde{\F}}_d - \H_d\G_d^{-1}\widetilde{\F}_d\right\|_F \\
\leq & \lambda_d \|\widehat{\widetilde{\F}}_d - \widetilde{\F}_d\|_F + \lambda_1\cdots\lambda_d \pi_d \|\widehat{\G}_d - \G_d\|_F + \lambda_1\cdots \lambda_{d-1}\pi_d \|\widehat{\H}_d - \H_d\|_F.
\end{split}
\end{equation}
Next, we analyze $\|\widehat{\widetilde{\F}}_d - \widetilde{\F}_d\|_F$. Define
\begin{equation*}
\begin{split}
& \widehat{\widetilde{\F}}_{d-1} = \mathcal{M}_{d-1}(\widehat{\bcF})\left(\I_{r_d}\otimes\widehat{\H}_{d-2}\widehat{\G}_{d-2}^{-1}\otimes \cdots \otimes \widehat{\H}_1\widehat{\G}_1^{-1}\right)^\top,\\
& \widetilde{\F}_{d-1} = \mathcal{M}_{d-1}(\bcF)\left(\I_{r_d} \otimes\H_{d-2}\G_{d-2}^{-1}\otimes \cdots \otimes \H_1\G_1^{-1}\right)^\top.
\end{split}
\end{equation*}
Then by tensor algebra (Lemma \ref{lm:Kronecker-vectorization-matricization}),
\begin{equation*}
\begin{split}
& \|\widehat{\widetilde{\F}}_d - \widetilde{\F}_d\|_{F} = \left\|\llbracket\widehat{\bcF}; \widehat{\H}_1\widehat{\G}_1^{-1}, \ldots, \widehat{\H}_{d-1}\widehat{\G}_{d-1}^{-1}, \I_{r_d}\rrbracket - \llbracket\bcF; \H_1\G_1^{-1}, \ldots, \H_{d-1}\G_{d-1}^{-1}, \I_{r_d}\rrbracket\right\|_{\tHS}\\
= & \left\|\mathcal{M}_{d-1}\left(\llbracket\widehat{\bcF}; \widehat{\H}_1\widehat{\G}_1^{-1}, \ldots, \widehat{\H}_{d-1}\widehat{\G}_{d-1}^{-1}, \I_{r_d}\rrbracket\right) - \mathcal{M}_{d-1}\left(\llbracket\bcF; \H_1\G_1^{-1}, \ldots, \H_{d-1}\G_{d-1}^{-1}, \I_{r_d}\rrbracket\right)\right\|_F\\
= & \left\|\widehat{\H}_{d-1}\widehat{\G}_{d-1}^{-1}\widehat{\widetilde{\F}}_{d-1} - \H_{d-1}\G_{d-1}^{-1}\widetilde{\F}_{d-1}\right\|_F.
\end{split}
\end{equation*}
Similarly as the previous argument, one can show by the first part of this lemma that
\begin{equation*}
\begin{split}
& \|\widehat{\widetilde{\F}}_d - \widetilde{\F}_d\|_{F} =  \left\|\widehat{\H}_{d-1}\widehat{\G}_{d-1}^{-1}\widehat{\widetilde{\F}}_{d-1} - \H_{d-1}\G_{d-1}^{-1}\widetilde{\F}_{d-1}\right\|_F\\
\leq & \lambda_{d-1}\|\widehat{\widetilde{\F}}_{d-1}-\widetilde{\F}_{d-1}\|_F + \lambda_1\cdots \lambda_{d-1}\pi_{d-1}\|\widehat{\G}_{d-1}-\G_{d-1}\|_F + \lambda_1\cdots \lambda_{d-2}\pi_{d-1}\|\widehat{\H}_{d-1}-\H_{d-1}\|_F.\\
\end{split}
\end{equation*}
Therefore, by \eqref{ineq:FGH-1} and the previous inequality,
\begin{equation*}
\begin{split}
& \left\|\llbracket\widehat{\bcF}; \widehat{\H}_1\widehat{\G}_1^{-1},\ldots,\widehat{\H}_{d}\widehat{\G}_{d}^{-1}\rrbracket - \llbracket\bcF; \H_1\G_1^{-1},\ldots, \H_{d} \G_{d}^{-1}\rrbracket\right\|_{\tHS}\\
\leq & \lambda_{d-1}\lambda_d\left\|\widehat{\widetilde{\F}}_{d-1}-\widetilde{\F}_{d-1}\right\|_F + \sum_{k=d-1, d}\lambda_1\cdots \lambda_d\pi_{k}\|\widehat{\G}_k-\G_k\|_F + \sum_{k = d-1, d} \frac{\lambda_1\cdots \lambda_d\pi_{k}}{\lambda_k}\left\|\widehat{\H}_{k}-\H_k\right\|_F.
\end{split}
\end{equation*}
We further introduce $\widehat{\widetilde{\F}}_{d-2}, \widetilde{\F}_{d-2}, \ldots, \widehat{\widetilde{\F}}_1, \widetilde{\F}_1$, repeat the previous argument for $d$ time, and can finally obtain
\begin{equation*}
\begin{split}
& \left\|\llbracket\widehat{\bcF}; \widehat{\H}_1\widehat{\G}_1^{-1},\ldots,\widehat{\H}_{d}\widehat{\G}_{d}^{-1}\rrbracket - \llbracket\bcF; \H_1\G_1^{-1},\ldots, \H_{d} \G_{d}^{-1}\rrbracket\right\|_{\tHS} \\
\leq & \lambda_1\cdots \lambda_d \|\widehat{\bcF} - \bcF\|_{\tHS} + \sum_{k=1}^d \lambda\cdots \lambda_d\pi_k  \|\widehat{\G}_k-\G_k\|_F + \sum_{k=1}^d \frac{\lambda_1\cdots \lambda_d \pi_k}{\lambda_k} \|\widehat{\H}_k - \H_k\|_F,
\end{split}
\end{equation*}
which has finished the proof of this lemma. \quad $\square$

The following lemma characterizes the concentration of Gaussian ensemble measurements, which will be extensively used in the proof of Theorem \ref{th:upper_bound_regression}. 
\begin{Lemma}[Gaussian Ensemble Concentration Inequality for Matrices]\label{lm:concentration-Gaussian-ensemble}
	Suppose $\A  \in \mathbb{R}^{a\times b}$ is a fixed matrix, $\X_1,\ldots, \X_n\in \mathbb{R}^{a\times b}$ are random matrices with i.i.d. standard Gaussian entries, and $\varepsilon_1,\ldots, \varepsilon_n\overset{iid}{\sim}N(0,\sigma^2)$. Let $\E = \frac{1}{n}\sum_{i=1}^n \left(\langle \A, \X_i\rangle +\varepsilon_i\right) \X_i$. Then there exists a uniform constant $C>0$ such that,
	\begin{equation}\label{ineq:concentration-target}
	\bbP\left(\left\|\E - \A\right\| \geq C\sqrt{(a+b)(\|\A\|_F^2+\sigma^2)}\left(\sqrt{\frac{\log(a+b) + t}{n}} + \frac{\log(a+b) + t}{n}\right) \right) \leq \exp(-t)
	\end{equation}
\end{Lemma}

{\bf\noindent Proof of Lemma \ref{lm:concentration-Gaussian-ensemble}.} Denote $\Z_i = \left(\langle \A, \X_i \rangle + \varepsilon_i\right)\X_i$. It is easy to check that $ \mathbb{E} \Z_i = \A$. Then,
\begin{equation*}
\begin{split}
\mathbb{E} (\Z_i-\A)(\Z_i-\A)^\top = & \mathbb{E} \Z_i\Z_i^\top -\A(\mathbb{E} \Z_i)^\top - (\mathbb{E}\Z_i) \A^\top + \A\A^\top = \mathbb{E}\Z_i\Z_i^\top- \A\A^\top \\
= & \mathbb{E} \langle \A, \X_i\rangle^2 \X_i\X_i^\top + \sigma^2 \mathbb{E}\X_i\X_i^\top -\A\A^\top\\
= & \mathbb{E} \langle \A, \X_i\rangle^2 \X_i\X_i^\top + \sigma^2\cdot b\I_{a} - \A\A^\top
\end{split}
\end{equation*}
Note that for any entry $(\X_i)_{[j, k]}$, $\mathbb{E}(\X_i)_{[j, k]}=0, \mathbb{E}(\X_i)_{[j, k]}^2=1, \mathbb{E}(\X_i)_{[j, k]}^3=0, \mathbb{E}(\X_i)_{[j, k]}^4=3$. When $j\neq k$,
\begin{equation*}
\begin{split}
& \left(\mathbb{E} \langle \A, \X_i\rangle^2 \X_i \X_i^\top\right)_{jk} = \mathbb{E} \langle \A, \X_i\rangle^2 \sum_{l=1}^{b} (\X_i)_{[j, l]}(\X_i)_{[k, l]}\\
= & \mathbb{E} \sum_{l=1}^b \left(2\A_{[j,l]} \A_{[k,l]} (\X_i)_{[i, l]} (\X_i)_{[k,l]}\right) (\X_i)_{[i,l]} (\X_i)_{[k,l]}\\
= & 2 \sum_{l=1}^b \A_{[j,l]}\A_{[k,l]} = 2(\A\A^\top)_{[j,k]};
\end{split}
\end{equation*}
when $j = k$, 
\begin{equation*}
\begin{split}
& \left(\mathbb{E} \langle \A, \X_i\rangle^2 \X_i \X_i^\top\right)_{[j,j]} = \mathbb{E} \langle \A, \X_i\rangle^2 \sum_{l=1}^b (\X_i)_{[j,l]}^2 \\
= & \mathbb{E} \sum_{j'=1}^{a}\sum_{l'=1}^{b} \left(\A_{[j',l']}^2 (\X_i)_{[j',l']}^2\right) \cdot \sum_{l=1}^b (\X_i)_{[j,l]}^2 = \sum_{j'=1}^{a}\sum_{l'=1}^b (\A_{[j',l']}^2)\cdot b + 2\sum_{l=1}^b \A_{[j,l]}^2\\
= & b\|\A\|_F^2 + 2(\A\A^\top)_{[j,j]}.
\end{split}
\end{equation*}
Therefore, $\mathbb{E}\langle \A, \X_i\rangle^2\X_i\X_i^\top = 2\A\A^\top +b\|\A\|_F^2 \I_a$, and
\begin{equation}\label{eq:lm-expectation1}
\left\|\mathbb{E}(\Z_i - \A)(\Z_i - \A)^\top\right\| = \left\|2\A\A^\top +b\|\A\|_F^2\I_a + b\sigma^2\I_a - \A\A^\top\right\| = \|\A\|^2+b\|\A\|_F^2 + b\sigma^2.
\end{equation}
Similarly, we can also show
\begin{equation}\label{eq:lm-expectation2}
\left\|\mathbb{E}(\Z_i - \A)^\top(\Z_i - \A)\right\| = \left\|2\A^\top \A +a\|\A\|_F^2\I_b + \sigma^2\I_a - \A^\top \A\right\| = \|\A\|^2+a\|\A\|_F^2 + a\sigma^2.
\end{equation}
Next, we consider the spectral norm of $\Z_i$ and aim to show that
\begin{equation}
\big\|\left\|\Z_i - \A\right\|\big\|_{\psi_1} = \inf_{u\geq 0}\left\{u:\mathbb{E}\exp\left(\frac{\|\Z_i-\A\|}{u}\right) \leq 2\right\} \leq C \left(\sqrt{a}+\sqrt{b}\right)\sqrt{\|\A\|_F^2+\sigma^2}
\end{equation} 
for uniform constant $C>0$. Note that $\langle \A, \X_i \rangle +\varepsilon_i \sim N\left(0, \|\A\|_F^2 + \sigma^2\right)$, $\X_i$ is a random matrix, by Gaussian tail bound inequality and random matrix theory (Corollary 5.35 in \cite{vershynin2010introduction}),  \begin{equation}\label{ineq:lm-tail-probability}
\begin{split}
& \bbP\left(\left|\langle \A, \X_i\rangle+\varepsilon_i\right| \geq t\sqrt{\|\A\|_F^2 + \sigma^2}\right) \leq 2\exp(-t^2/3),\\ 
& \bbP\left(\|\X_i\| \geq \sqrt{a}+\sqrt{b} + t\right) \leq \exp(-t^2/2).
\end{split}
\end{equation}
We set $u = C_0\left(\sqrt{a}+\sqrt{b}\right)\sqrt{\|\A\|_F^2+\sigma^2}$ for large uniform constant $C_0\geq 80$. Thus, for any $x \geq 1$,
\begin{equation*}
\begin{split}
& \bbP\left(\|\Z_i - \A\| \geq xu\right) \leq \bbP\left(\| (\langle \A, \X_i\rangle + \varepsilon_i) \X_i\| \geq xu-\|\A\|\right) \\
\leq & \bbP\left(\| (\langle \A, \X_i\rangle + \varepsilon_i ) \X_i\| \geq \frac{xC_0(\sqrt{a}+\sqrt{b})}{2}\sqrt{\|\A\|_F^2 + \sigma^2}\right)\\
\leq & \bbP\left( | \langle \A, \X_i\rangle  + \varepsilon_i | \geq \sqrt{\frac{xC_0}{2}\cdot \left(\|\A\|_F^2 + \sigma^2\right)}\right) + \bbP\left(\|\X_i\| \geq \sqrt{\frac{xC_0}{2}}\cdot (\sqrt{a}+\sqrt{b})\right)\\
\overset{\text{\eqref{ineq:lm-tail-probability}}}{\leq} & 3\exp(- C_0x/6).
\end{split}
\end{equation*}
For any real valued function smooth $g$ and non-negative random variable $Y$ with density $f_Y$, the following identity holds,
$$\mathbb{E} g(Y) = \int_0^\infty g'(y)P(Y\geq y) dy.$$
Thus,
\begin{equation*}
\begin{split}
& \mathbb{E}\exp\left(\frac{\|\Z_i - \A \|}{u}\right) = \int_0^\infty \exp\left(x\right) \bbP\left(\frac{\|\Z_i-\A \|}{u} \geq x\right) dx\\
\leq & \int_0^1 \exp(u) du + \int_{1}^\infty \exp(x) \cdot 3\exp(-C_0x/6)dx \\
\leq & \exp(1)-1 + \frac{3}{C_0/6-1} \leq 2,
\end{split}
\end{equation*}
which implies $\big\|\|\Z_i - \A \|\big\|_{\psi_1} \leq C_0\left(\sqrt{a}+\sqrt{b}\right)\sqrt{\|\A \|_F^2 + \sigma^2}$ for some uniform constant $C_0>0$. 

Finally we apply the Bernstein-type matrix concentration inequality (c.f., Proposition 2 in \cite{koltchinskii2011nuclear} and Theorem 4 in \cite{koltchinskii2013remark}), 
\begin{equation}\label{ineq:Z_i-A}
\begin{split}
\left\|\frac{1}{n}\sum_{i=1}^n \Z_i - \A \right\| & \leq C\max\Bigg\{  \sigma_Z\sqrt{\frac{t+\log(a+b)}{n}},\\
&  (\sqrt{a}+\sqrt{b})\sqrt{\|\A \|_F^2+\sigma^2} \log\left(\frac{C(\sqrt{a}+\sqrt{b})\sqrt{\|\A \|_F^2 + \sigma^2}}{\sigma_Z}\right)\cdot \frac{t+\log(a+b)}{n}\Bigg\}
\end{split}
\end{equation}
with probability at least $1 - \exp(-t)$. Here,
\begin{equation*}
\begin{split}
\sigma_Z := & \max\left\{\left\|\frac{1}{n}\sum_{i=1}^n\mathbb{E}(\Z_i-\A )(\Z_i-\A )^\top\right\|^{1/2}, \left\|\frac{1}{n}\sum_{i=1}^n\mathbb{E}(\Z_i-\A )^\top(\Z_i-\A )\right\|^{1/2}\right\}\\
= & \sqrt{\|\A \|^2 + (a\vee b) \left(\|\A \|_F^2 + \sigma^2\right)}.
\end{split}
\end{equation*}
Noting that $\sqrt{(a\vee b)(\|\A \|_F^2+\sigma^2)} \leq \sigma_Z \leq \sqrt{(a\vee b+1)(\|\A \|_F^2+\sigma^2)}$, \eqref{ineq:Z_i-A} implies \eqref{ineq:concentration-target}. \quad $\square$

\begin{Lemma}[Gaussian Ensemble Concentration Inequality for Vector]\label{lm:concentration-Gaussina-vector}
	Suppose $x_1,\ldots, x_n\overset{iid}{\sim}N(0, \I_m)$ are i.i.d. $m$-dimensional random vectors, $\varepsilon_1,\ldots, \varepsilon_n\overset{iid}{\sim}N(0, \sigma^2)$, and $a\in \mathbb{R}^m$ is a fixed vector. Then
	\begin{equation*}
	\bbP\left(\left\|\frac{1}{n}\sum_{i=1}^n\left(\langle \x_i, \a\rangle + \varepsilon_i\right) \x_i - \a\right\|_2 \leq \frac{C\sqrt{\|\a\|_2^2+\sigma^2}\left(\sqrt{n}+\sqrt{t}\right)\left(\sqrt{m}+\sqrt{t}\right)}{n} \right) \geq 1 - 5\exp(-t).
	\end{equation*}
\end{Lemma}
{\bf\noindent Proof of Lemma \ref{lm:concentration-Gaussina-vector}.} Denote 
$$\x_i = (x_{i1},\ldots, x_{im})^\top, \quad i=1,\ldots, n.$$ 
Since the distribution of Gaussian random vectors are invariant after orthogonal transformation, without loss of generality we assume $\a = (\theta, 0, \ldots, 0)$. Then
\begin{equation*}
\frac{1}{n}\left(\sum_{i=1}^n\langle \x_i, \a\rangle + \varepsilon_i\right) \x_i - \a = \begin{pmatrix}
\frac{1}{n}\sum_{i=1}^n (x_{i1}^2-1)\theta\\
\frac{1}{n}\sum_{i=1}^n x_{i1}\theta x_{i2}\\
\vdots\\
\frac{1}{n}\sum_{i=1}^n x_{i1}\theta x_{im}
\end{pmatrix} + \frac{1}{n}\sum_{i=1}^n \varepsilon_i \x_i := h + \frac{1}{n}\sum_{i=1}^n \varepsilon_i \x_i;
\end{equation*}
Note that $\sum_{i=1}^n x_{i1}^2 \sim \chi^2_n$, by tail bounds of $\chi^2$ (c.f., \cite[Lemma 1]{laurent2000adaptive}),
$$\bbP\left(n - 2\sqrt{nt} \leq \sum_{i=1}^n x_{i1}^2\right) \geq 1 - \exp(-t),\quad \bbP\left(\sum_{i=1}^n x_{i1}^2 \leq n + 2\sqrt{nt} + 2t \right) \geq 1 - \exp(-t).$$
Conditioning on the fixed value of $\xi := \sum_{i=1}^n x_{i1}^2$, we have 
$$\frac{1}{n}\sum_{i=1}^n x_{i1}\theta x_{ik}\Big|\xi \sim N\left(0, \frac{\theta^2\xi}{n^2}\right), \quad k=2,\ldots, n,$$ 
\begin{equation*}
\|h\|_2^2 \Big| \xi  \sim \left(\frac{\xi}{n} - 1\right)^2\theta^2 + \frac{\theta^2\xi}{n^2} \chi^2_{m-1}.
\end{equation*}
Thus,
\begin{equation*}
\begin{split}
& \bbP\left(\left\|h\right\|_2^2 \geq 4\theta^2\left(\sqrt{\frac{t}{n}} + \frac{t}{n}\right)^2 + \frac{\theta^2\left(n+2\sqrt{nt}+2t\right)\left(m-1+2\sqrt{(m-1)t}+2t\right)}{n^2} \right) \\
\leq & \bbP\left(\xi \geq n+2\sqrt{nt}+2t \right) + \bbP\left(\xi \leq n-2\sqrt{nt}\right) + \bbP\left(\frac{\theta^2\xi}{n^2}\chi_{m-1}^2 \geq \frac{\theta^2\xi(m-1+2\sqrt{(m-1)t}+2t)}{n^2}\right)\\
\leq & 3\exp(-t).
\end{split}
\end{equation*}
Conditioning on fixed values of $\|\bvarepsilon\|_2^2 = \sum_i \varepsilon_i^2$, 
$$\left\|\frac{1}{n}\sum_{i=1}^n \varepsilon_i \x_i \right\|_2^2\Bigg| \|\bvarepsilon\|_2^2 \sim \frac{\sigma^2\|\bvarepsilon\|_2^2}{n^2}\chi^2_m.$$ 
Additionally, $\bbP\left(\|\bvarepsilon\|_2^2 \geq \sigma^2(n+2\sqrt{nt}+2t)\right) \leq \exp(-t)$, which means 
\begin{equation*}
\begin{split}
& \bbP\left(\left\|\frac{1}{n}\sum_{i=1}^n \varepsilon_i \x_i\right\|_2^2 \geq \frac{\sigma^2\left(n+2\sqrt{nt}+2t\right)\left(m+2\sqrt{mt}+2t\right)}{n^2} \right) \\
\leq & \bbP\left(\|\bvarepsilon\|_2^2 \geq \sigma^2(n+2\sqrt{nt}+2t)\right) + \bbP\left(\left\|\frac{1}{n}\sum_{i=1}^n \varepsilon_i \x_i\right\|_2^2\Bigg| \|\bvarepsilon\|_2^2 \geq \frac{\sigma^2\|\bvarepsilon\|_2^2}{n^2}\left(m+2\sqrt{mt}+2t\right)\right)\\
\leq & 2\exp(-t).
\end{split}
\end{equation*}
Combining the previous two inequalities, we finally obtain
\begin{equation*}
\begin{split}
& \bbP\left(\left\|\frac{1}{n}\sum_{i=1}^n\left(\langle \x_i, \a\rangle + \varepsilon_i\right) \x_i - \a\right\|_2 \leq \frac{C\sqrt{\theta^2+\sigma^2}\left(\sqrt{n}+\sqrt{t}\right)\left(\sqrt{m}+\sqrt{t}\right)}{n}\right) \\
\geq & 1 - 5\exp(-t).
\end{split}
\end{equation*}
for constant $C>0$. \quad$\square$

\begin{Lemma}\label{lm:concentration-independent}
	Suppose $\X_1,\ldots, \X_n\in \mathbb{R}^{a\times b}$ ($a\leq b$) are i.i.d. standard Gaussian matrices, $\xi_1,\ldots, \xi_n \overset{iid}{\sim} N(0, \tau^2)$, and $\E = \frac{1}{n}\sum_{i=1}^n \xi_i \X_i$. Then the largest and smallest singular values of $\E$ satisfies the following tail probability,
	\begin{equation*}
	\bbP\left(\sigma_{\max}^2(\E) \geq \tau^2\frac{n + 2\sqrt{nx}+2x}{n^2}\left(\sqrt{a}+\sqrt{b}+\sqrt{2x}\right)^2 \right) \leq 2\exp(-x),
	\end{equation*}
	\begin{equation*}
	\bbP\left(\sigma_{\min}^2(\E) \leq \tau^2\frac{n - 2\sqrt{nx}}{n^2}\left(\sqrt{b}-\sqrt{a}-\sqrt{2x}\right)^2 \right) \leq 2\exp(-x).
	\end{equation*}
\end{Lemma}
{\bf\noindent Proof of Lemma \ref{lm:concentration-independent}.} In the given setting, $\|\xi\|_2^2 = \sum_{i=1}^n \xi_i^2 \sim \tau^2\chi^2_n$, and 
$$\E = \frac{1}{n}\sum_{i=1}^n \xi_i \X_i \Big| \|\xi\|_2 \overset{iid}{\sim} N\left(0, \frac{\|\xi\|_2^2}{n^2}\right).$$
By Corollary 5.35 in \cite{vershynin2010introduction},
\begin{equation}\label{ineq:random-matrix-tail-bound}
\begin{split}
\bbP\left(\sigma_{\max}^2(\E) \geq \frac{\|\xi\|_2^2}{n^2}\left(\sqrt{a}+\sqrt{b}+\sqrt{2x} \right)^2\Big| \|\xi\|_2\right) \leq \exp(-x), \\
\bbP\left(\sigma_{\min}^2(\E) \leq \frac{\|\xi\|_2^2}{n^2}\left(\sqrt{b}-\sqrt{a}-\sqrt{2x} \right)^2\Big| \|\xi\|_2\right) \leq \exp(-x).
\end{split}
\end{equation}
By the tail bound of $\chi^2$ distribution (Lemma 1 in \cite{laurent2000adaptive}), 
\begin{equation}\label{ineq:chi-square-tail-bound}
\bbP\left(\|\xi\|_2^2 \geq \tau^2\left(n + 2\sqrt{nx}+2x\right)\right) \leq e^{-x},\quad \bbP\left(\|\xi_2\|_2^2 \leq \tau^2\left(n - 2\sqrt{nx}\right)\right) \leq e^{-x}.
\end{equation}
By \eqref{ineq:random-matrix-tail-bound} and \eqref{ineq:chi-square-tail-bound}, we have
\begin{equation*}
\begin{split}
& \bbP\left(\sigma_{\max}^2(\E) \geq \tau^2\frac{n + 2\sqrt{nx}+2x}{n^2}\left(\sqrt{a}+\sqrt{b}+\sqrt{2x}\right)^2\right)\\
\leq & \bbP\left(\sigma_{\max}^2(\E) \geq \frac{\|\xi\|_2^2}{n^2}\left(\sqrt{a}+\sqrt{b}+\sqrt{2x}\right)^2 \text{~or~} \|\xi\|_2^2 \geq \tau^2\left(n + 2\sqrt{nx}+2x\right)\right)\\
\leq & \exp(-x) + \exp(-x) = 2\exp(-x);
\end{split}
\end{equation*}
\begin{equation*}
\begin{split}
& \bbP\left(\sigma_{\min}^2(\E) \leq \tau^2\frac{n - 2\sqrt{nx}}{n^2}\left(\sqrt{b}-\sqrt{a}-\sqrt{2x}\right)^2\right)\\
\leq & \bbP\left(\sigma_{\min}^2(\E) \leq \frac{\|\xi\|_2^2}{n^2}\left(\sqrt{b}-\sqrt{a}-\sqrt{2x}\right)^2 \text{~or~} \|\xi\|_2^2 \leq \tau^2\left(n - 2\sqrt{nx}\right)\right)\\
\leq & \exp(-x) + \exp(-x) = 2\exp(-x).
\end{split}
\end{equation*}
\quad $\square$

The next lemma provides an upper bound for the projection error after perturbation, which is useful in the singular subspace perturbation analysis in the proofs of the main results. 
\begin{Lemma}[Projection error after perturbation]\label{lm:SVD-projection}
	Suppose $\A , \Z$ are two matrices of the same dimension and $\widehat{\U} = \SVD_r(\A +\Z)$. Then,
	\begin{equation*}
	\left\|P_{\widehat{\U}_\perp}\A \right\|\leq \sigma_{r+1}(\A )+2\|\Z\|,\quad \left\|P_{\widehat{\U}_\perp}\A \right\|_F\leq \sqrt{\sum_{k\geq r+1}\sigma_k^2(\A )} + 2\|\Z\|_F.
	\end{equation*}
	In particular when $\rank(\A )\leq r$, 
	$$\left\|P_{\widehat{\U}_\perp}\A \right\|\leq 2\|\Z\|, \quad \left\|P_{\widehat{\U}_\perp}\A \right\|_F\leq 2\min\left\{\|\Z\|_F, \sqrt{r}\|\Z\|\right\}.$$
\end{Lemma}

{\noindent \bf Proof of Lemma \ref{lm:SVD-projection}.} Suppose $\A = \sum_k \sigma_k(\A)\u_k \v_k^\top$ is the singular value decomposition. Then,
\begin{equation*}
\begin{split}
\|P_{\widehat{\U}_\perp}\A \| \leq & \left\|P_{\widehat{\U}_\perp}(\A +\Z)\right\| + \|\Z\| = \sigma_{r+1}(\A + \Z)+ \|\Z\|\\
= & \min_{\rank(\M)\leq r}\|\A +\Z-\M\| + \|\Z\|\\
\leq & \left\|\A  + \Z - \sum_{k=1}^{r}\sigma_k(\A)\u_k \v_k^\top\right\| + \|\Z\| = \left\|\Z + \sum_{k\geq r+1} \sigma_k(\A)\u_k \v_k^\top\right\| + \|\Z\| \\
\leq & \sigma_{r+1}(\A) + 2\|\Z\|.
\end{split}
\end{equation*}
\begin{equation*}
\begin{split}
\|P_{\widehat{\U}_\perp}\A \|_F \leq & \left\|P_{\widehat{\U}_\perp}(\A +\Z)\right\|_F + \|P_{\widehat{\U}_\perp}\Z\|_F = \sqrt{\sum_{k\geq r+1}\sigma_{k}^2(\A +\Z)}+ \|\Z\|_F\\
= & \min_{\rank(\M)\leq r}\|\A +\Z-\M\|_F + \|\Z\|_F\\
\leq & \left\|\A  + \Z - \sum_{k=1}^r\sigma_k(\A)\u_k \v_k^\top\right\|_F + \|\Z\|_F \leq \sqrt{\sum_{k\geq r+1}\sigma^2_k(\A )} + 2\|\Z\|_F.
\end{split}
\end{equation*}
Finally, when $\rank(\A )\leq r$, $\rank(P_{\widehat{\U}_{\perp}}\A ) \leq \rank(\A ) \leq r$, then 
$$\|P_{\widehat{\U}_{\perp}}\A \|_F \leq \min\left\{\sqrt{\sum_{k\geq r+1}\sigma_k^2(\A)} + 2\|\Z\|_F, \sqrt{r}\left\|P_{\widehat{\U}_{\perp}}\A \right\|\right\} \leq \min\left\{2\|\Z\|_F, 2\sqrt{r}\|\Z\|\right\}.$$
\quad $\square$

The Lemma \ref{lm:projection-difference} below provides a inequality for tensors after tensor-matrix product projections.
\begin{Lemma}\label{lm:projection-difference}
	Suppose $\bcA \in \mathbb{R}^{p_1\times \cdots \times p_d}$ is an order-$d$ tensor and $\widetilde{\U}_k \in \mathbb{O}_{p_k, r_k}$, $k=1,\ldots, d$, are orthogonal matrices. Let $\|\cdot\|_\bullet$ be a tensor norm that satisfies sub-multiplicative inequality, i.e., $\|\bcA\times_k \B\|_\bullet \leq \|\bcA\|_\bullet \cdot \|\B\|$ for any tensor $\bcA$ and matrix $\B$ (in particular, the tensor Hilbert-Schmitt norm satisfies this condition), we have
	\begin{equation*}
	\left\|\llbracket\bcA; P_{\widetilde{\U}_1}, \ldots, P_{\widetilde{\U}_d}\rrbracket - \bcA\right\|_{\bullet} \leq \sum_{k=1}^d \left\|\bcA\times_k P_{\widetilde{\U}_{k\perp}}\right\|_{\bullet}.
	\end{equation*}
	Specifically, 
	$$\left\|\llbracket\bcA; P_{\widetilde{\U}_1}, \ldots, P_{\widetilde{\U}_d}\rrbracket - \bcA\right\|_{\tHS} = \left\|P_{(\widetilde{\U}_d\otimes \cdots \otimes \widetilde{\U}_1)_\perp}\rmvec(\bcA)\right\|_2 \leq \sum_{k=1}^d \left\|\widetilde{\U}_{k\perp}^\top\mathcal{M}_k(\bcA)\right\|_F.$$
\end{Lemma}

{\bf\noindent Proof of Lemma \ref{lm:projection-difference}.} Note that
\begin{equation*}
\begin{split}
\bcA = & \left\llbracket \bcA; \left(P_{\widetilde{\U}_1} + P_{\widetilde{\U}_{1\perp}}\right), \ldots, \left(P_{\widetilde{\U}_d} + P_{\widetilde{\U}_{d\perp}}\right)\right\rrbracket \\
= & \left\llbracket \bcA; P_{\widetilde{\U}_1}, \ldots, P_{\widetilde{\U}_d}\right\rrbracket + \left\llbracket \bcA; P_{\widetilde{\U}_{1\perp}}, \ldots, P_{\widetilde{\U}_d}\right\rrbracket + \left\llbracket \bcA; \I_{p_1}, P_{\widetilde{\U}_{2\perp}}, \ldots, P_{\widetilde{\U}_d}\right\rrbracket\\
& + \cdots + \left\llbracket \bcA; \I_{p_1}, \I_{p_2}, \ldots, P_{\widetilde{\U}_{d\perp}}\right\rrbracket.
\end{split}
\end{equation*}
Additionally, $\|P_{\widetilde{\U}_k}\|\leq 1, \|P_{\widetilde{\U}_{k\perp}}\|\leq 1$. Thus,
\begin{equation*}
\begin{split}
\left\|\llbracket\bcA; P_{\widetilde{\U}_1}, \ldots, P_{\widetilde{\U}_d}\rrbracket - \bcA\right\|_\bullet \leq & \left\|\left\llbracket \bcA; P_{\widetilde{\U}_{1\perp}}, \ldots, P_{\widetilde{\U}_d}\right\rrbracket\right\|_\bullet + \left\|\left\llbracket \bcA; \I_{p_1}, P_{\widetilde{\U}_{2\perp}}, \ldots, P_{\widetilde{\U}_d}\right\rrbracket\right\|_\bullet\\
& + \cdots + \left\|\left\llbracket \bcA; \I_{p_1}, \I_{p_2}, \ldots, P_{\widetilde{\U}_{d\perp}}\right\rrbracket\right\|_\bullet\\
\leq & \sum_{k=1}^d \left\|\bcA\times_k P_{\widetilde{\U}_{k\perp}}\right\|_\bullet.
\end{split}
\end{equation*}
Specifically for the Hilbert-Schmitt norm,
\begin{equation*}
\begin{split}
& \left\|P_{(\widetilde{\U}_d\otimes \cdots \otimes \widetilde{\U}_1)_\perp}\rmvec(\bcA)\right\|_2 = \left\|P_{(\widetilde{\U}_d\otimes \cdots \otimes \widetilde{\U}_1)}\rmvec(\bcA)-\rmvec(\bcA)\right\|_2 \\
\leq & \left\|\llbracket\bcA; P_{\widetilde{\U}_1}, \ldots, P_{\widetilde{\U}_d}\rrbracket - \bcA\right\|_{\tHS} \leq \sum_{k=1}^d \left\|\bcA \times_k P_{\widetilde{\U}_{k\perp}}\right\|_{\tHS} = \sum_{k=1}^d \left\|\mathcal{M}_k\left(\bcA \times_k P_{\widetilde{\U}_{k\perp}}\right)\right\|_F\\ 
= & \left\|\widetilde{\U}_{k\perp}^\top\mathcal{M}_k(\bcA)\right\|_{F}.
\end{split}
\end{equation*}
Therefore, we have finished the proof of lemma \ref{lm:projection-difference}. \quad $\square$

The next Lemma \ref{lm:projection_remainder} introduces a useful inequality for the tensor projected orthogonal to a Cross structure (i.e., $\widetilde{\U}$ in the statement below).
\begin{Lemma}\label{lm:projection_remainder}
	Suppose $\bcA = \llbracket\bcS; \U_1, \U_2, \U_3 \rrbracket$ is a rank-$(r_1, r_2 ,r_3)$ tensor. $\U_k\in \mathbb{O}_{p_k, r_k}$ and $\W_k\in \mathbb{O}_{p_{k+1}p_{k+2}, p_k}$ are the left and right singular subspaces of $\mathcal{M}_k(\bcA) := \A_k$, respectively. Suppose $\widetilde{\U}_k\in \mathbb{O}_{p_k, r_k}$ and
	$$\widetilde{\W}_1 = (\widetilde{\U}_3\otimes \widetilde{\U}_2)\widetilde{\V}_1\in \mathbb{O}_{p_1, r_1}, \quad \widetilde{\W}_2 = (\widetilde{\U}_3\otimes \widetilde{\U}_1)\widetilde{\V}_2\in \mathbb{O}_{p_2, r_2}, \quad \widetilde{\W}_3 = (\widetilde{\U}_2\otimes \widetilde{\U}_1)\widetilde{\V}_3\in \mathbb{O}_{p_3, r_3}$$ 
	are sample estimates of $\U$ and $\W_k$, respectively. Assume $\widetilde{\U}_k$ and $\widetilde{\W}_k$ satisfy
	\begin{equation*}
	\|\sin\Theta(\widetilde{\U}_k, \U_k)\| \leq \theta_k, \quad \|\widetilde{\U}_{k\perp}^\top \A _k\|_F \leq \eta_k, \quad \|\A _k\widetilde{\W}_{k\perp}\|_F \leq \xi_k, \quad  k=1,2,3.
	\end{equation*}
	Let
	$$\widetilde{\U} = \begin{bmatrix}
	\widetilde{\U}_3\otimes \widetilde{\U}_2\otimes \widetilde{\U}_1, ~~ \mathcal{R}_1(\widetilde{\W}_1\otimes \widetilde{\U}_{1 \perp}), ~~ \mathcal{R}_2(\widetilde{\W}_2\otimes \widetilde{\U}_{2\perp}), ~~ \mathcal{R}_3(\widetilde{\W}_3\otimes \widetilde{\U}_{3\perp})
	\end{bmatrix},$$
	where $\mathcal{R}_k(\cdot)$ is the row-permutation operator that matches the row indices of $\widetilde{\W}_k\otimes \widetilde{\U}_{k\perp}$ to $\rmvec(\bcA)$ and the actual definitions of $\mathcal{R}_k$ are provided in Section \ref{sec:row-permutation-operator} in the supplementary materials. Recall $\widetilde{\U}_\perp$ is the orthogonal complement of $\U$. Then,
	\begin{equation*}
	\begin{split}
	\|P_{\widetilde{\U}_\perp}\rmvec(\bcA)\|_2^2 \leq & \sum_{k=1,2,3} \left(\theta_k^2\xi_k^2 + \min\{\theta_{k+1}^2\eta_{k+2}^2, \theta_{k+2}^2\eta_{k+1}^2\}\right)\\
	& + \min\{\eta_1^2\theta_2^2\theta_3^2, \theta_1^2\eta_2^2\theta_3^2, \theta_1^2\theta_2^2\eta_3^2\}.
	\end{split}
	\end{equation*}
\end{Lemma}

{\bf\noindent Proof of Lemma \ref{lm:projection_remainder}.} 
Since 
$$\widetilde{\U} = \begin{bmatrix}
\widetilde{\U}_3\otimes \widetilde{\U}_2\otimes \widetilde{\U}_1, ~~ \mathcal{R}_1(\widetilde{\W}_1\otimes \widetilde{\U}_{1\perp}), ~~ \mathcal{R}_2(\widetilde{\W}_2\otimes \widetilde{\U}_{2\perp}), ~~ \mathcal{R}_3( \widetilde{\W}_3 \otimes\widetilde{\U}_{3\perp})
\end{bmatrix}\in \mathbb{O}_{p_1p_2p_3, m},$$
where $m = r_1r_2r_3 + (p_1-r_1)r_1 + (p_2-r_2)r_2 + (p_3-r_3)r_3$. Denote
\begin{equation}\label{eq:tilde-U-1k}
\begin{split}
\widetilde{\U}_{11} = 
\mathcal{R}_1\left(\left((\widetilde{\U}_{3}\otimes \widetilde{\U}_{2})\widetilde{\V}_{1\perp}\right) \otimes \widetilde{\U}_{1\perp}\right) \in \mathbb{O}_{p_1p_2p_3, (p_1-r_1)(r_{2}r_{3}-r_1)},\\
\widetilde{\U}_{12} = \mathcal{R}_2\left(\left((\widetilde{\U}_{3}\otimes \widetilde{\U}_{1})\widetilde{\V}_{2\perp}\right) \otimes \widetilde{\U}_{2\perp}\right) \in \mathbb{O}_{p_1p_2p_3, (p_2-r_2)(r_{1}r_{3}-r_2)},\\
\widetilde{\U}_{13} =\mathcal{R}_3\left(\left((\widetilde{\U}_{2}\otimes \widetilde{\U}_{1})\widetilde{\V}_{3\perp}\right) \otimes \widetilde{\U}_{3\perp}\right) \in \mathbb{O}_{p_1p_2p_3, (p_3-r_3)(r_{2}r_{1}-r_3)},\\
\end{split}
\end{equation}
\begin{equation}\label{eq:tilde-U-2k}
\begin{split}
\widetilde{\U}_{21} = \widetilde{\U}_{3\perp} \otimes \widetilde{\U}_{2\perp} \otimes \widetilde{\U}_{1} \in \mathbb{O}_{p_1p_2p_3, r_1(p_2-r_2)(p_3-r_3)};\\
\widetilde{\U}_{22} = \widetilde{\U}_{3\perp} \otimes \widetilde{\U}_{2} \otimes \widetilde{\U}_{1 \perp} \in \mathbb{O}_{p_1p_2p_3, r_2(p_1-r_1)(p_3-r_3)};\\
\widetilde{\U}_{23} = \widetilde{\U}_{3} \otimes \widetilde{\U}_{2\perp} \otimes \widetilde{\U}_{1 \perp} \in \mathbb{O}_{p_1p_2p_3, r_3(p_1-r_1)(p_2-r_2)};\\
\end{split}
\end{equation}
\begin{equation}\label{eq:tilde-U-3}
\widetilde{\U}_{3\ast} =
\widetilde{\U}_{3\perp} \otimes \widetilde{\U}_{2\perp} \otimes \widetilde{\U}_{1\perp} \in \mathbb{O}_{p_1p_2p_3, (p_1-r_1)(p_2-r_2)(p_3-r_3)}.
\end{equation}
Then it is not hard to verify that $[\widetilde{\U}_{11}, \widetilde{\U}_{12}, \widetilde{\U}_{13}, \widetilde{\U}_{21}, \widetilde{\U}_{22}, \widetilde{\U}_{23}, \widetilde{\U}_{3\ast}]$ forms an orthogonal complement of $\widetilde{\U}$. Thus, we have the following decomposition,
\begin{equation*}
\begin{split}
\|P_{\widetilde{\U}_{\perp}} \rmvec(\bcA)\|_2^2 = \sum_{k=1,2,3} \|P_{\widetilde{\U}_{1k}} \rmvec(\bcA)\|_2^2 + \sum_{k=1,2,3} \|P_{\widetilde{\U}_{2k}} \rmvec(\bcA)\|_2^2 + \|P_{\widetilde{\U}_{3\ast}} \rmvec(\bcA)\|_2^2.
\end{split}
\end{equation*}
We analyze each term separately as follows.
\begin{itemize}[leftmargin=*]
	\item Note that
	$$\begin{bmatrix}
	(\widetilde{\U}_{3}\otimes \widetilde{\U}_{2})\widetilde{\V}_1, ~~ (\widetilde{\U}_{3}\otimes \widetilde{\U}_{2})\widetilde{\V}_{1\perp}, ~~ (\widetilde{\U}_{3}\otimes \widetilde{\U}_{2})_\perp
	\end{bmatrix}$$
	is a square orthogonal matrix, we know 
	$$\begin{bmatrix}
	(\widetilde{\U}_{3}\otimes \widetilde{\U}_{2})\widetilde{\V}_{1\perp}, ~~ (\widetilde{\U}_{3}\otimes \widetilde{\U}_{2})_\perp
	\end{bmatrix}$$
	is an orthogonal complement to $\widetilde{\W}_1$. Given the left and right singular subspaces of $\A _1$ are $\U_1$ and $\W_1$, we have
	\begin{equation*}
	\begin{split}
	& \|P_{\widetilde{\U}_{11}}\rmvec(\bcA)\|_F^2 \overset{\eqref{eq:tilde-U-1k}}{=} \left\|\widetilde{\U}_{1\perp}^\top \A_1\left((\widetilde{\U}_3\otimes \widetilde{\U}_2)\widetilde{\V}_{1\perp}\right)\right\|_F^2\\
	\leq & \left\|\widetilde{\U}_{1\perp}^\top \A_1\widetilde{\W}_{1\perp}\right\|_F^2 = \left\|\widetilde{\U}_{1\perp}^\top \U_1 \U_1^\top  \A_1\widetilde{\W}_{1\perp}\right\|_F^2\leq \|\widetilde{\U}_{1\perp}^\top \U_1\|^2\cdot \|\A_1\widetilde{\W}_{1\perp}\|_F^2 \\
	\leq & \|\sin\Theta(\widetilde{\U}_1, \U_1)\|^2 \cdot \|\A_1\widetilde{\W}_{1\perp}\|_F^2 \leq \theta_1^2\xi_1^2.
	\end{split}
	\end{equation*}
	Similar inequalities also hold for $\|P_{\widetilde{\U}_{12}}\rmvec(\bcA)\|_F^2$ and $\|P_{\widetilde{\U}_{13}}\rmvec(\bcA)\|_F^2$.
	\item 
	\begin{equation*}
	\begin{split}
	& \|P_{\widetilde{\U}_{21}}\rmvec(\bcA)\|_2^2 = \|\widetilde{\U}_{21}^\top \rmvec(\bcA)\|_2^2 = \|\bcA \times_1 \widetilde{\U}_1^\top \times_2 \widetilde{\U}_{2\perp}^\top \times_3 \widetilde{\U}_{3\perp}^\top \|_{\tHS}^2\\
	= & \|\widetilde{\U}_{2\perp}^\top \A_2 (\widetilde{\U}_{3\perp}\otimes \widetilde{\U}_1)\|_F^2 = \|\widetilde{\U}_{2\perp}^\top \U_2\U_2^\top \A_2 (\widetilde{\U}_{3\perp}\otimes \widetilde{\U}_1)\|_F^2 \\
	\leq & \|\widetilde{\U}_{2\perp}^\top \U_2\|^2 \cdot\|\U_2^\top \A_2 (\widetilde{\U}_{3\perp}\otimes \widetilde{\U}_1)\|_F^2\\
	\leq & \|\sin\Theta(\widetilde{\U}_2, \U_2)\|^2\cdot \|\A_2(\widetilde{\U}_{3\perp}\otimes \widetilde{\U}_1)\|_F^2 = \theta_2^2 \cdot \|\bcA \times_1\widetilde{\U}_1^\top \times_3 \widetilde{\U}_{3\perp}^\top\|_{\tHS}^2\\
	= & \theta_2^2 \cdot \|\widetilde{\U}_{3\perp}^\top \A_3\|_F^2\leq \theta_2^2 \eta_3^2.
	\end{split}
	\end{equation*}
	By symmetry, $\|P_{\widetilde{\U}_{21}}\rmvec(\bcA)\|_2^2\leq \theta_3^2\eta_2^2$. Similar inequalities also hold for $\|P_{\widetilde{\U}_{22}}\rmvec(\bcA)\|_2^2$ and $\|P_{\widetilde{\U}_{23}}\rmvec(\bcA)\|_2^2$. Therefore,
	\begin{equation}
	\|P_{\widetilde{\U}_{2k}}\rmvec(\bcA)\|_2^2 \leq \min\{\theta_{k+1}^2\eta_{k+2}^2, \theta_{k+2}^2\eta_{k+1}^2\}, \quad \text{for} ~~ k=1,2,3.
	\end{equation}
	\item Similarly as the previous part,
	\begin{equation*}
	\begin{split}
	\|P_{\widetilde{\U}_{3\ast}}\rmvec(\bcA)\|_2 \leq & \left\|\bcA \times_1 \widetilde{\U}_{1\perp}^\top \times_2 \widetilde{\U}_{2\perp}^\top \times_3 \widetilde{\U}_{3\perp}^\top\right\|_{\tHS}\\
	= & \left\|\widetilde{\U}_{1\perp}^\top \bcA_1 \left(\widetilde{\U}_{3\perp} \otimes \widetilde{\U}_{2\perp}\right)\right\|_F\\
	= & \left\|\widetilde{\U}_{1\perp}^\top \A_1 (\U_3\otimes \U_2) (\U_3\otimes \U_2)^\top (\widetilde{\U}_{3\perp}\otimes \widetilde{\U}_{2\perp})\right\|_F\\
	\leq & \left\|\widetilde{\U}_{1\perp}^\top \A_1 (\U_3\otimes \U_2)\right\|_F \cdot \left\|(\U_3^\top \widetilde{\U}_{3\perp})\otimes (\U_2^\top \widetilde{\U}_{2\perp})\right\| \\
	\leq & \left\|\widetilde{\U}_{1\perp}^\top \A_1 \right\|_F\cdot \left\|(\U_3^\top \widetilde{\U}_{3\perp})\right\|\cdot \left\| (\U_2^\top \widetilde{\U}_{2\perp})\right\|\leq  \eta_1\theta_2\theta_3.
	\end{split}
	\end{equation*}
	Similar upper bounds of $\theta_1\eta_2\theta_3$ and $\theta_1\theta_2\eta_3$ also hold. Thus,
	\begin{equation*}
	\|P_{\widetilde{\U}_{3\ast}}\rmvec(\bcA)\|_2^2 \leq \min\{\eta_1^2\theta_2^2\theta_3^2, \theta_1^2\eta_2^2\theta_3^2, \theta_1^2\theta_2^2\eta_3^2\}.
	\end{equation*}
\end{itemize}
In summary,
\begin{equation*}
\begin{split}
\|P_{\U_\perp}\rmvec(\bcA)\|_2^2 \leq & \sum_{k=1,2,3} \left(\theta_k^2\xi_k^2 + \min\{\theta_{k+1}^2\eta_{k+2}^2, \theta_{k+2}^2\eta_{k+1}^2\}\right)\\
& + \min\{\eta_1^2\theta_2^2\theta_3^2, \theta_1^2\eta_2^2\theta_3^2, \theta_1^2\theta_2^2\eta_3^2\}.
\end{split}
\end{equation*}
\quad $\square$

The following lemma discusses the Bayes risk of regular linear regression. Though it is a standard result in statistical decision theory (c.f., Exercise 5.8, p. 403 in \cite{lehmann2006theory}), we present the proof here for completeness of statement.
\begin{Lemma}\label{lm:linear-regression-lower-bound}
	Consider the linear regression model $\y = \X\bbeta+\varepsilon$. Here, $\varepsilon\overset{iid}{\sim} N(0, \sigma^2)$; the parameter $\bbeta$ is generated from a prior distribution: $\bbeta\overset{iid}{\sim} N(0, \tau^2)$. We aim to estimate $\bbeta$ based on $(\y, \X)$ with the minimal $\ell_2$ risk. Then, the Bayes estimator for $\bbeta$ and the corresponding Bayes risk are
	$$\widehat{\bbeta} = \left(\frac{\sigma^2 \I}{\tau^2} + \X^\top\X\right)^{-1}\X^\top \quad\text{and}\quad \mathbb{E}\left((\widehat{\bbeta} - \bbeta)^2| \X\right) = \tr\left(\left(\frac{\I}{\tau^2} + \frac{\X^\top\X}{\sigma^2}\right)^{-1}\right).$$
\end{Lemma}

{\bf\noindent Proof of Lemma \ref{lm:linear-regression-lower-bound}.} When $\bbeta\overset{iid}{\sim}N(0, \tau^2)$ and $\varepsilon\overset{iid}{\sim}N(0, \sigma^2)$,
\begin{equation}
\begin{split}
p(\bbeta \Big| \X, \y) ~~ \varpropto ~~ & p(\y|\X, \bbeta) \cdot p(\bbeta) \\
\varpropto ~~ & \exp\left(-\|\y-\X\bbeta\|_2^2/(2\sigma^2)\right)\cdot\exp(-\bbeta^\top\bbeta/(2\tau^2))\\
\varpropto ~~ & \exp\left(- \frac{\bbeta^\top\bbeta}{2\tau^2} - \frac{\bbeta^\top \X^\top\X \bbeta}{2\sigma^2}  + \frac{\y^\top \X\bbeta}{\sigma^2}\right)\\
\varpropto ~~ & \exp\left(- \frac{1}{2}\left\|\left(\frac{\I}{\tau^2} + \frac{\X^\top\X}{\sigma^2}\right)^{-1/2}\frac{\X^\top \y}{\sigma^2} - \left(\frac{\I}{\tau^2} + \frac{\X^\top\X}{\sigma^2}\right)^{1/2}\bbeta\right\|_2^2 \right)
\end{split}
\end{equation}
Thus, the posterior distribution of $\bbeta$ is
$$\bbeta\Big| \X, \y ~~ \sim ~~ N\left(\left(\frac{\sigma^2 \I}{\tau^2} + \X^\top\X\right)^{-1}\X^\top \y, \left(\frac{\I}{\tau^2} + \frac{\X^\top\X}{\sigma^2}\right)^{-1}\right).$$
Then, the Bayes estimator, i.e., the posterior mean, and the corresponding Bayes risk are
$$\widehat\bbeta = \mathbb{E}(\bbeta|\X, \y) =  \left(\frac{\sigma^2 \I}{\tau^2} + \X\X^\top\right)^{-1}\X^\top \y, \quad  \mathbb{E}((\widehat{\bbeta} - \bbeta)^2| \X, \y) = \tr\left(\left(\frac{\I}{\tau^2} + \frac{\X\X^\top}{\sigma^2}\right)^{-1}\right),$$
respectively. Thus, we have finished the proof of this lemma.\quad $\square$

The following lemma provides a deterministic bound for the group Lasso estimator under group restricted isometry property.
\begin{Lemma}\label{lm:group-oracle-RIP}
	Suppose $\X\in \mathbb{R}^{n\times pr}$, $\{G_1,\ldots, G_p\}$ is a partition of $\{1,\ldots, pr\}$ and $|G_1|=\cdots =|G_p|$. Assume $\X$ satisfies group restricted isometry condition, such that
	$$ (1-\delta)n\|\bbeta\|_2^2 \leq \|\X\bbeta\|_2^2 \leq (1+\delta)n\|\bbeta\|_2^2,\quad \forall \bbeta ~\text{such that}~ \sum_{i=1}^m 1_{\{\bbeta_{G_i}\neq 0\}}\leq 2s.$$
	Suppose $\y = \X\bbeta + \varepsilon$ and $\sum_{i=1}^p 1_{\{\bbeta_{G_i}\neq 0\}}\leq s.$ Consider the following group Lasso estimator
	\begin{equation}\label{eq:hat-beta-group-lasso}
	\widehat{\bbeta} = \argmin_{\bgamma\in\mathbb{R}^{pr}}\left\{\frac{1}{2}\|\y - \X\bgamma\|_2^2 + \eta\sum_{i=1}^p \| \bgamma_{G_i}\|_2\right\}.
	\end{equation}
	For $\eta \geq 3\max_{1\leq j\leq p}\|(\X_{[:, G_j]})^\top \varepsilon\|_2$ and $\delta<2/7$, the optimal solution of \eqref{eq:hat-beta-group-lasso} yields
	\begin{equation}
	\|\widehat{\bbeta} - \bbeta\|_2 \leq \frac{4\eta\sqrt{s/3}}{n(1-7\delta/2)}.
	\end{equation}
\end{Lemma}
{\noindent\bf Proof of Lemma \ref{lm:group-oracle-RIP}.} For convenience, define the $(2, \infty)$- and $(2, 1)$-norms of any vector $v\in \mathbb{R}^{pr}$ as
$$\|\v\|_{2, \infty} = \max_{j=1,\ldots, p} \|\v_{G_j}\|_2\quad \text{and}\quad \|\v\|_{2, 1} = \sum_{j=1}^p \|\v_{G_j}\|_2.$$ 
Then, $\|\cdot\|_{2, \infty}$ and $\|\cdot \|_{2, 1}$ satisfies 
$\|\v\|_{2, \infty} \cdot \|w\|_{2, 1} \geq \langle \v, w\rangle$. We also define $J = \{j: \bbeta_{G_j} \neq 0\}$ as the group support of $\bbeta$, then $|J|\leq s$ based on the assumption. Suppose $h = \widehat{\bbeta} - \bbeta\in \mathbb{R}^{pr}$. By definition,
\begin{equation*}
\begin{split}
\frac{1}{2}\|\y - \X\widehat{\bbeta}\|_2^2 + \eta \|\widehat{\bbeta}\|_{2,1} \leq  \frac{1}{2}\|\y-\X\bbeta\|_2^2 + \eta\|\bbeta\|_{2,1}.
\end{split}
\end{equation*}
Noting that
\begin{equation*}
\begin{split}
& \frac{1}{2}\left(\|\y - \X\widehat{\bbeta}\|_2^2 - \|\y-\X\bbeta\|_2^2\right) = \frac{1}{2}\left(\|\varepsilon-\X h\|_2^2 - \|\bvarepsilon\|_2^2\right)\\
= & -\frac{1}{2} (2\varepsilon - \X h)^\top (\X h) \geq -\varepsilon^\top \X h \geq - \|\X^\top \varepsilon\|_{2, \infty} \cdot \|h\|_{2, 1} \\
= & - \|\X^\top \varepsilon\|_{2, \infty}(\|h_{J}\|_{2, 1} +\|h_{J^c}\|_{2, 1}),
\end{split}
\end{equation*}
\begin{equation*}
\begin{split}
\eta\left(\|\bbeta\|_{2,1} - \|\widehat{\bbeta}\|_{2,1}\right) = \eta\left(\|\bbeta_J\|_{2, 1} - \|\widehat{\bbeta}_J\|_{2,1} - \|\widehat{\bbeta}_{J^c}\|_{2,1}\right) \leq \eta\left(\|h_{J}\|_{2,1} - \|h_{J^c}\|_{2,1}\right),
\end{split}
\end{equation*}
we have
\begin{equation*}
\begin{split}
& -\|\X^\top \varepsilon\|_{2, \infty}(\|h_{J}\|_{2,1} + \|h_{J^c}\|_{2,1}) \leq \eta(\|h_{J}\|_{2,1} - \|h_{J^c}\|_{2,1}),\\
\Rightarrow \quad & \|h_{J^c}\|_{2,1} \leq \frac{\eta + \|\X^\top\varepsilon\|_{2,\infty}}{\eta - \|\X^\top \varepsilon\|_{2,\infty}}\|h_{J}\|_{2,1}.
\end{split}
\end{equation*}
Given $\eta\geq 3\|\X^\top \varepsilon \|_{2, \infty}$, we have
\begin{equation}\label{ineq:h_J^c<=h_J}
\|h_{J^c}\|_{2,1} \leq 2\|h_{J}\|_{2,1}.
\end{equation}

Now we can sort all groups of $h$ by their $\ell_2$ norm and suppose $\|h_{G_{i_1}}\|_2\geq \cdots \geq \|h_{G_{i_p}}\|_2$, where $\{i_1,\ldots, i_p\}$ as a permutation of $\{1,\ldots, p\}$. Let
$$h_{\max(s)} \in \mathbb{R}^{pr}, \quad (h_{\max(s)})_j = \left\{\begin{array}{ll}
h_j, & j\in G_{i_1}\cup \cdots \cup G_{i_s};\\
0, & \text{otherwise},
\end{array}\right.$$
Then $h_{\max(s)}$ is the vector $h$ with all but the $s$ largest groups in $\ell_2$ norm set to zero. We also denote $h_{-\max(s)} = h - h_{\max(s)}$. Then \eqref{ineq:h_J^c<=h_J} implies
\begin{equation}\label{ineq:h_max(s)<=h_-max(s)}
\|h_{-\max(s)}\|_{2,1} \leq \|h_{J^c}\|_{2,1} \leq 2\|h_{J}\|_{2,1} \leq 2\|h_{\max(s)}\|_{2,1}.
\end{equation}

Let $\v\in \mathbb{R}^p$ with $\v_i = \|h_{G_i}\|_2, 1\leq i \leq p$ be the $\ell_2$ norms of each group of $h$. We can similarly define $\v_{\max(s)}$ as the vector $\v$ with all but the $s$ largest entries set to zero, and $\v_{-\max(s)} = \v - \v_{\max(s)}$. Then, $(\v_{\max(s)})_i = \|(h_{\max(s)})_{G_i}\|_2$ and $(\v_{-\max(s)})_i = \|(h_{-\max(s)})_{G_i}\|_2$. Let 
$$\alpha = \max\{\|h_{-\max(s)}\|_{2, \infty}, \|h_{-\max(s)}\|_{2, 1}/s\} = \max\{\|\v_{-\max(s)}\|_{\infty}, \|\v_{-\max(s)}\|_{1}/s\}.$$ 
By the polytope representation lemma (Lemma 1 in \cite{cai2014sparse}) with $\alpha$, one can find a finite series of vectors $\v^{(1)}, \cdots, \v^{(N)} \in \mathbb{R}^{p}$ and weights $\pi_1, \ldots, \pi_N$ such that
$$\supp(\v^{(j)}) \subseteq \supp(\v_{-\max(s)}), \quad \|\v^{(j)}\|_0 \leq s, \quad \|\v^{(j)}\|_\infty \leq \alpha, \quad \|\v^{(j)}\|_1 = \|\v_{-\max(s)}\|_1,$$
$$\v_{-\max(s)} = \sum_{j=1}^N \pi_j \v^{(j)}, \quad 0\leq \pi_j \leq 1,\quad \text{and}\quad \sum_{j=1}^N \pi_j = 1.$$
Now we construct
\begin{equation}
h^{(j)}\in \mathbb{R}^{pr}, \quad \text{where}\quad (h^{(j)})_{G_i} = \frac{(h_{-\max(s)})_{G_i}}{\|(h_{-\max(s)})_{G_i}\|_2}\cdot \v^{(j)}_i, \quad i=1,\ldots, p; j=1,\ldots, N.
\end{equation}
Then $\{h^{(j)}\}_{j=1}^N$ satisfy
\begin{equation}\label{eq:h^(j)}
\begin{split}
& \supp(h^{(j)}) \subseteq \supp(h_{-\max(s)}), \quad \sum_{i=1}^p 1_{\{(h^{(j)})_{G_i}\neq 0\}} \leq s, \quad \|h^{(j)}\|_{2,\infty} \leq \alpha, \\
\quad & \|h^{(j)}\|_{2, 1} = \|h_{-\max(s)}\|_{2,1},\quad h_{-\max(s)} = \sum_{j=1}^N \pi_j h^{(j)}, \quad 0\leq \pi_j \leq 1,\quad \sum_{j=1}^N \pi_j = 1.
\end{split}
\end{equation}
Therefore, $h_{\max(s)}$ and $h^{(j)}$ have distinct supports, $\sum_{i=1}^m 1_{(h_{\max(s)} + h^{(j)})_{G_i}\neq 0}\leq 2s$, $\|h_{\max(s)} + h^{(j)}\|_2^2 = \|h_{\max(s)}\|_2^2 + \|h^{(j)}\|_2^2$, and \begin{equation*}
    \begin{split}
    \|h^{(j)}\|_2^2 \leq & \|h^{(j)}\|_{2,1}\cdot \|h^{(j)}\|_{2,\infty} \overset{\eqref{eq:h^(j)}}{\leq} \|h_{-\max(s)}\|_{2,1}\cdot \alpha  \\
    \overset{\eqref{ineq:h_max(s)<=h_-max(s)}}{\leq} & 2\|h_{\max(s)}\|_{2,1}\cdot \max\left\{\|h_{-\max(s)}\|_{2,\infty}, \|h_{-\max(s)}\|_{2,1}/s\right\}\\
    \leq & 2\|h_{\max(s)}\|_{2,1}\cdot \max\left\{\min_{j: \|h_{G_j}\|_2\neq 0}\|h_{G_j}\|_2, 2\|h_{\max(s)}\|_{2,1}/s\right\}\\
    \leq & 4\|h_{\max(s)}\|_{2,1}^2/s \leq 4\|h_{\max(s)}\|_2^2.
    \end{split}
\end{equation*}
Thus,
\begin{equation*}
\begin{split}
& \left|\langle \X h_{\max(s)}, \X h_{-\max(s)}\rangle\right| \leq \sum_{j=1}^N \pi_j \left|\langle \X h_{\max(s)}, \X h^{(j)}\rangle\right|\\
= & \sum_{j=1}^N \frac{\pi_j}{4} \left|\|\X h_{\max(s)}+ \X h^{(j)}\|_2^2 - \|\X h_{\max(s)}- \X h^{(j)}\|_2^2\right|\\
\leq & \sum_{j=1}^N \frac{\pi_j}{4}\left(n(1+\delta)(\|h_{\max(s)}\|_2^2 + \|h^{(j)}\|_2^2) - n(1-\delta)(\|h_{\max(s)}\|_2^2 + \|h^{(j)}\|_2^2)\right)\\
\leq & \frac{\delta n}{2} \left(\|h_{\max(s)}\|_{2}^2 + 4\|h_{\max(s)}\|_{2}^2\right) = \frac{5\delta n}{2}\|h_{\max(s)}\|_2^2,
\end{split}
\end{equation*}
which means
\begin{equation}\label{ineq:inner-product}
\begin{split}
& \langle \X h_{\max(s)}, \X h\rangle =\|\X h_{\max(s)}\|_2^2 +  \langle \X h_{\max(s)}, \X h_{-\max(s)}\rangle \\
\geq & n(1-\delta)\|h_{\max(s)}\|_2^2 - \frac{5\delta n}{2}\|h_{\max(s)}\|_2^2 = n(1-7\delta/2)\|h_{\max(s)}\|_2^2.
\end{split}
\end{equation}

Next, by the KKT condition of $\widehat{\bbeta}$ being the optimizer of \eqref{eq:hat-beta-group-lasso},
\begin{equation*}
\|\X^\top(y - \X\widehat{\bbeta})\|_{2, \infty} \leq \eta.
\end{equation*}
In addition, $\|\X^\top (y - \X\bbeta)\|_{2, \infty} = \|\X^\top \varepsilon\|_{2, \infty} \leq \eta/3$, which means
\begin{equation}
\begin{split}
\langle \X h_{\max(s)}, \X h\rangle = & h_{\max(s)}^\top \X^\top \X h \leq \|h_{\max(s)}\|_{2,1}\cdot \|\X^\top \X h\|_{2,\infty} \\
\leq & \|h_{\max(s)}\|_{2,1}\cdot \left(\|\X^\top (y - \X\widehat{\bbeta})\|_{2,\infty} + \|\X^\top (y - \X\bbeta)\|_{2,\infty}\right) \\
\leq & 4\eta/3\cdot \|h_{\max(s)}\|_{2,1} \leq 4\eta/3 \cdot \sqrt{s}\|h_{\max(s)}\|_2.
\end{split}
\end{equation}
Combining the above inequality with \eqref{ineq:inner-product}, one has
\begin{equation*}
\frac{4\eta}{3}\sqrt{s}\|h_{\max(s)}\|_{2} \geq n(1-7\delta/2)\|h_{\max(s)}\|_2^2,
\end{equation*}
namely 
$$\|h_{\max(s)}\|_2 \leq \frac{\frac{4}{3}\eta\sqrt{s}}{n(1-7\delta/2)}.$$
Finally, 
\begin{equation*}
\begin{split}
\|h_{-\max(s)}\|_2^2 \leq & \|h_{-\max(s)}\|_{2,1} \cdot \|h_{-\max(s)}\|_{2, \infty} \\
\leq & 2\|h_{\max(s)}\|_{2, 1}\cdot \min_{j: (h_{\max(s)})_{G_j}\neq 0} \|(h_{\max(s)})_{G_j}\|_2 \\
\leq & 2\|h_{\max(s)}\|_2^2.
\end{split}
\end{equation*}
Therefore,
$$\|h\|_2 = \sqrt{\|h_{-\max(s)}\|_2^2 + \|h_{\max(s)}\|_2^2} \leq \sqrt{3}\|h_{\max(s)}\|_2 \leq \frac{4\eta\sqrt{s/3}}{n(1-7\delta/2)},$$
which has finished the proof of this lemma.\quad $\square$

The next Lemma \ref{lm:GE->GRIP} shows that the Gaussian Ensemble satisfies group restricted isometry property with high probability.
\begin{Lemma}\label{lm:GE->GRIP}
	Suppose $\X\in \mathbb{R}^{n\times (pr)}$, $G_1, \ldots, G_p$ is a partition of $\{1,\ldots pr\}$ and $|G_1|=\cdots |G_p| = r$. If $\X\overset{iid}{\sim}N(0, 1)$ and $n \geq C(sr/\delta + s\log(ep/s))$ for large constant $C>0$, $\X$ satisfies the following group restricted isometry (GRIP) 
	\begin{equation}\label{ineq:lm-GRIP}
	n(1 - \delta)\|\bbeta\|_2^2\leq \left\|\X \bbeta\right\|_2^2 \leq n(1 + \delta)\|\bbeta\|_2^2, \quad \forall \bbeta \text{ such that } \sum_{i=1}^p 1_{\{\bbeta_{\G_i} \neq 0\}} \leq s 
	\end{equation}
	with probability at least $1 - \exp(-cn)$.
\end{Lemma}

{\bf\noindent Proof of Lemma \ref{lm:GE->GRIP}.} First, the statement \eqref{ineq:lm-GRIP} is equivalently to 
\begin{equation}\label{ineq:GRIP-equivalent}
\begin{split}
& \forall \text{ distinct } i_1,\ldots, i_s \subseteq \{1,\ldots, p\}, \\ 
& n(1-\delta) \leq \sigma_{\min}^2(\X_{[:, G_{i_1}\cup \cdots \cup G_{i_s}]}) \leq \sigma_{\max}^2(\X_{[:, G_{i_1}\cup \cdots \cup G_{i_s}]}) \leq n(1+\delta).
\end{split}
\end{equation}
Since $\X_{[:, G_{i_1}\cup \cdots \cup G_{i_s}]}$ is an $n$-by-$sr$ matrix with i.i.d. Gaussian entries, by random matrix theory (c.f., \cite[Corollary 5.35]{vershynin2010introduction}), 
\begin{equation*}
\begin{split}
& \bbP\left(\sqrt{n}-\sqrt{sr}-x \leq \sigma_{\min}(\X_{[:, G_{i_1}\cup \cdots \cup G_{i_s}]}) \leq \sigma_{\max}(\X_{[:, G_{i_1}\cup \cdots \cup G_{i_s}]}) \leq \sqrt{n}+\sqrt{sr}+x\right) \\
\geq & 1 - 2\exp(-x^2/2),
\end{split}
\end{equation*}
which means
\begin{equation*}
\begin{split}
& \bbP\left(\eqref{ineq:GRIP-equivalent} \text{ does not hold}\right)\\
\leq & \sum_{\substack{\text{distinct } i_1,\ldots, i_s\\
		\subseteq \{1,\ldots, p\}}}\bbP\left(\left\{n(1-\delta) \leq \sigma_{\min}^2(\X_{[:, G_{i_1}\cup \cdots \cup G_{i_s}]}) \leq \sigma_{\max}^2(\X_{[:, G_{i_1}\cup \cdots \cup G_{i_s}]}) \leq n(1+\delta)\right\}^c\right) \\
\leq &  2\binom{p}{s}\exp\left(-\left(\sqrt{n}-\sqrt{n(1-\delta)}-\sqrt{sr}\right)_+^2 \wedge\left(\sqrt{n(1+\delta)}-\sqrt{n}-\sqrt{sr}\right)_+^2 \right),
\end{split}
\end{equation*}

Provided that $n\geq C(sr/\delta+s\log(ep/s))$ for large constant $C>0$, we have
\begin{equation*}
\left(\sqrt{n}-\sqrt{n(1-\delta)}-\sqrt{sr}\right)_+^2 \wedge \left(\sqrt{n(1+\delta)} - \sqrt{n}-\sqrt{sr}\right)_+^2 \geq (1-c)n,
\end{equation*}
\begin{equation*}
\begin{split}
(1-c)n\geq (1-c)Cs\log(ep/s) \geq (1-c)C\log\left(\binom{p}{s}\right).
\end{split}
\end{equation*}
Therefore, we have 
$$\bbP\left(\eqref{ineq:GRIP-equivalent} \text{ does not hold}\right) \leq \exp\left(\log\left(2\binom{p}{s}\right)-(1-c)n\right) \leq \exp(-cn)$$
and have finished the proof of this lemma.\quad $\square$

The next lemma gives the Kullback–Leibler divergence between two regression models with random designs, which will be used in the lower bound argument in this paper.
\begin{Lemma}\label{lm:regression-KL}
	Consider two linear regression models $\y^{(1)} = \X \bbeta^{(1)} + \varepsilon$ and $y^{(2)} = \X \bbeta^{(2)} + \varepsilon$. Here, $\y^{(1)}, y^{(2)}\in \mathbb{R}^{n}$ and $\X\in \mathbb{R}^{n\times p}$, $\bbeta^{(1)}, \bbeta^{(2)}\in \mathbb{R}^p$, and $\varepsilon\in \mathbb{R}^n$. Assume $\X\overset{iid}{\sim}N(0, 1)$, $\varepsilon\overset{iid}{\sim}N(0, \sigma^2)$, and $\bbeta^{(1)}, \bbeta^{(2)}$ are fixed. Then,
	\begin{equation}
	D_{KL}\left(\{\X, \y^{(1)}\} \Big|\Big|\{\X, y^{(2)}\} \right) = \frac{n}{2\sigma^2}\left\|\bbeta^{(1)} - \bbeta^{(2)}\right\|_2^2.
	\end{equation}
\end{Lemma}
{\noindent\bf Proof of Lemma \ref{lm:regression-KL}.} Denote the $j$-th row vector of $\X$ as $x_j$, i.e., $\X = [x_1^\top \cdots x_n^\top]^\top$. Then, $(x_1^\top, y_1^{(1)\top}), \ldots, (x_n^\top, y_n^{(1)\top})$ are i.i.d. distributed vectors, $y_j^{(1)} = x_j^\top \bbeta^{(1)}+\varepsilon_j$, and
$$\left(x_j^\top, y_j^{(1)}\right) \sim N\left(0, \Sigma_1\right), \quad \Sigma_1 = \begin{bmatrix}
\I_p & \bbeta^{(1)}\\
\bbeta^{(1)\top} & \|\bbeta^{(1)}\|_2^2 + \sigma^2
\end{bmatrix}.$$
Similarly, 
$$\left(x_j^\top, y_j^{(2)}\right) \sim N(0, \Sigma_2), \quad \Sigma_2 = \begin{bmatrix}
\I_p & \bbeta^{(2)}\\
\bbeta^{(2)\top} & \|\bbeta^{(2)}\|_2^2 + \sigma^2
\end{bmatrix}.$$ 
Additionally,
\begin{equation*}
\det(\Sigma_i) = \det\left(\begin{bmatrix}
\I_p & 0\\
-\bbeta^{(i)\top} & 1
\end{bmatrix}\cdot \begin{bmatrix}
\I_p & \bbeta^{(i)}\\
\bbeta^{(i)\top} &  \|\bbeta^{(i)}\|_2^2 + \sigma^2
\end{bmatrix}\right) = \det\left(\begin{bmatrix}
\I_p & \bbeta^{(i)} \\
0 & \sigma^2
\end{bmatrix}\right) = \sigma^2,\quad i=1,2,
\end{equation*}
\begin{equation*}
\Sigma_i^{-1} = \begin{bmatrix}
\I_p + \bbeta^{(i)}\bbeta^{(i)\top} \sigma^{-2} & -\bbeta^{(i)}\sigma^{-2}\\
-\bbeta^{(i)\top} \sigma^{-2} & \sigma^{-2}
\end{bmatrix}, \quad i=1,2.
\end{equation*}
By the formula for multivariate normal distribution KL-divergence,
\begin{equation*}
\begin{split}
& D_{KL}\left(\left\{x_j^\top, y_j^{(1)}\right\}\Big|\Big| \left\{x_j^\top, y_j^{(2)}\right\}\right) \\
= & \frac{1}{2}\left(\tr\left(\Sigma_2^{-1}\Sigma_1\right) - (p+1) + \log\left(\frac{\det(\Sigma_1)}{\det(\Sigma_2)}\right) \right)\\
= & \frac{\sigma^{-2}}{2}\left(\tr\left(\bbeta^{(2)}\bbeta^{(2)\top} - \bbeta^{(2)}\bbeta^{(1)^\top} - \bbeta^{(1)\top}\bbeta^{(2)}\right)+\|\bbeta^{(1)}\|_2^2\right)\\
= & \frac{1}{2\sigma^2}\left\|\bbeta^{(1)} - \bbeta^{(2)}\right\|_2^2.
\end{split}
\end{equation*}
Therefore,
\begin{equation*}
\begin{split}
D_{KL}\left(\left\{x_j^\top, y_j^{(1)}\right\}_{j=1}^n\Big|\Big| \left\{x_j^\top, y_j^{(2)}\right\}_{j=1}^n\right) = & nD_{KL}\left(\left\{x_j^\top, y_j^{(1)}\right\}\Big|\Big| \left\{x_j^\top, y_j^{(2)}\right\}\right) \\
= & \frac{n}{2\sigma^2}\left\|\bbeta^{(1)} - \bbeta^{(2)}\right\|_2^2.
\end{split}
\end{equation*}
\quad $\square$

The next lemma can be seen as a sparse version of Varshamov-Gilbert bound \cite[Lemma 4.7]{massart2007concentration}. This result is crucial in the proof of the lower bound argument in sparse tensor regression (Theorem \ref{th:lower_bound_sparse_tensor_regression}).
\begin{Lemma}\label{lm:sparse-Varshamov-Gilbert}
	There exists a series of matrices $\A^{(1)},\ldots, \A^{(N)} \in \{1, 0, -1\}^{p\times r}$, such that
	\begin{equation}\label{eq:targeting-condition}
	\|\A^{(k)}\|_{0, 2} := \sum_{i=1}^p1_{\left\{\A^{(k)}_{[i,:]}\neq 0\right\}} \leq s,\quad \|\A^{(k)} - \A^{(l)}\|_{1, 1} = \sum_{i=1}^p\sum_{j=1}^r \left|\A^{(k)}_{[i,j]} - \A^{(l)}_{[i,j]}\right| > sr/2
	\end{equation}
	for all $k, l$, and $N\geq \exp\left(c(sr+s\log(ep/s))\right)$ for some uniform constant $c>0$.
\end{Lemma}

{\bf\noindent Proof of Lemma \ref{lm:sparse-Varshamov-Gilbert}} First, if $p/s \leq C$ for some constant $C>0$, the lemma directly follows from the Varshamov-Gilbert bound by restricting on the top $s\times r$ submatrices of $\A_1,\ldots, \A_N$. Thus, without loss of generality, we assume $p\geq 10s$ throughout the rest of the proof. 

Next for $k=1,\ldots, N$, we randomly draw $s$ elements from $\{1,\ldots, p\}$ without replacement, form $\Omega^{(k)}$ as a random subset of $\{1,\ldots, p\}$, and generate
\begin{equation*}
\A^{(k)}\in \mathbb{R}^{p\times r}, \quad \left(\A^{(k)}\right)_{ij} \left\{\begin{array}{ll}
\sim \text{Rademacher}, & i \in \Omega^{(k)};\\
=0, & i \notin \Omega^{(k)},
\end{array}\right.
\end{equation*}
for $k=1,2,\ldots, N$. Here, $A\sim$ Rademacher if $A$ is equally distributed on -1 and 1. By such the construction,
\begin{equation*}
\|\A^{(k)}\|_{0, 2} = \sum_{i=1}^p1_{\left\{\A^{(k)}_{[i,:]}\neq 0\right\}} \leq s.
\end{equation*}
For any $k\neq l$,
\begin{equation}\label{eq:A^{(k)}-A^{(l)}}
\begin{split}
\left\|\A^{(k)} - \A^{(l)}\right\|_{1,1} \sim &  r|\Omega^{(k)}\backslash \Omega^{(l)}| + r|\Omega^{(l)}\backslash \Omega^{(k)}| + 2\cdot \text{Bin}\left(r\left|\Omega^{(k)}\cap \Omega^{(l)}\right|, 1/2\right)\\
= & 2sr - 2r|\Omega^{(l)}\cap \Omega^{(k)}| - 2\cdot\text{Bin}\left(r|\Omega^{(l)}\cap \Omega^{(k)}|, 1/2\right)\\
\sim & 2sr - 2\cdot\text{Bin}\left(r|\Omega^{(l)}\cap \Omega^{(k)}|, 1/2\right).
\end{split}
\end{equation}
Here, we used the fact that $|\Omega^{(k)}\backslash \Omega^{(l)}| = |\Omega^{(k)}| - |\Omega^{(k)}\cap \Omega^{(l)}| = s - |\Omega^{(k)}\cap \Omega^{(l)}|$. Moreover, $|\Omega^{(l)}\cap \Omega^{(k)}|$ satisfies the following hyper-geometric distribution: 
\begin{equation*}
\bbP\left(\left|\Omega^{(l)}\cap \Omega^{(k)}\right| = t\right) = \frac{\binom{s}{t}{\binom{p-s}{s-t}}}{\binom{p}{s}}, \quad t = 0,\ldots, s.
\end{equation*}
Let $Z_{kl} = \left|\Omega^{(l)}\cap \Omega^{(k)}\right|$. Then for any $s/2 \leq t\leq s$,
\begin{equation}\label{ineq:P(Z=t)}
\begin{split}
\bbP\left(Z = t\right) = & \frac{\frac{s\cdots (s-t+1)}{t!}\cdot \frac{(p-s)\cdots (p-2s+t+1)}{(s-t)!}}{\frac{p\cdots (p-s+1)}{s!}} \leq \binom{s}{t} \cdot \left(\frac{s}{p-s+1}\right)^t\\
\leq & 2^s\left(\frac{s}{p-s+1}\right)^t \leq \left(\frac{4s}{p-s+1}\right)^t.
\end{split}
\end{equation}
Next, by Bernstein's inequality,
\begin{equation*}
\begin{split}
& \bbP\left(\left\|\A^{(k)} - \A^{(l)}\right\|_{1,1} \leq sr/2\Big| Z\right) \overset{\eqref{eq:A^{(k)}-A^{(l)}}}{=} \bbP\left(\text{Bin}\left(rZ, 1/2\right) \geq  3sr/4\Big|Z\right)\\
= & \bbP\left(2\text{Bin}(rZ, 1/2) - rZ \geq \frac{3sr}{2} - rZ\right) \\
\leq & \left\{\begin{array}{ll}
2\exp\left(-\frac{(3sr/2-Zr)^2}{rZ + (3sr/2 - Zr)/3}\right), & s/2 \leq Z\leq s;\\
0, & Z< s/2.
\end{array}\right.
\end{split}
\end{equation*}
Thus,
\begin{equation*}
\begin{split}
\bbP\left(\left\|\A^{(k)} - \A^{(l)}\right\|_{1,1} \leq sr/2 \right) \leq & \sum_{s/2\leq t \leq s} \bbP\left(\left\|\A^{(k)} - \A^{(l)}\right\|_{1,1} \leq sr/2\Big| Z=t\right)\cdot \bbP\left(Z=t\right)\\
\leq & \sum_{s/2\leq t \leq s} 2\exp\left(-\frac{(3sr/2-tr)^2}{rt + (3sr/2 - tr)/3}\right) \left(\frac{4s}{p-s+1}\right)^t\\
\leq & \sum_{s/2\leq t \leq s} 2\exp\left(-\frac{(3sr/2-sr)^2}{sr + (3sr/2 - sr)/3}\right) \left(\frac{4s}{p-s+1}\right)^t\\
\leq & \sum_{t\geq s/2} 2\exp\left(-sr/14\right) \cdot \left(4s/(p-r+1)\right)^t\\
\leq & 2\exp(-sr/14) 2\cdot \left(4s/(p-s+1)\right)^{s/2}\\
\leq & 4\exp\left(-c(sr + s\log(ep/s))\right)
\end{split}
\end{equation*}
for some uniform constant $c>0$. Finally,
\begin{equation*}
\begin{split}
& \bbP\left(\forall 1\leq k\neq l\leq N, \left\|\A^{(k)} - \A^{(l)}\right\|_{1,1} > sr/2\right)\\ 
\geq &  1 - \binom{N}{2}\bbP\left(\left\|\A^{(k)} - \A^{(l)}\right\|_{1,1} \leq sr/2 \right) \geq 1 - \frac{N^2}{2} \cdot 4\exp\left(-c(sr + s\log(ep/s))\right)
\end{split}
\end{equation*}
We can see if $N \leq \exp(c(sr+s\log(ep/s)))$ for some uniform constant $c>0$, the previous event happens with a positive probability, which means there exists fixed $\A^{(1)},\ldots, \A^{(N)}$ satisfying the targeting condition \eqref{eq:targeting-condition} for some $N \geq \exp(c(sr+s\log(p/s)))$.\quad $\square$


\begin{thebibliography}{100}
	
	\bibitem{allen2012regularized}
	Genevera~I Allen.
	\newblock Regularized tensor factorizations and higher-order principal
	components analysis.
	\newblock {\em arXiv preprint arXiv:1202.2476}, 2012.
	
	\bibitem{anandkumar2014tensor}
	Animashree Anandkumar, Rong Ge, Daniel Hsu, Sham~M Kakade, and Matus Telgarsky.
	\newblock Tensor decompositions for learning latent variable models.
	\newblock {\em The Journal of Machine Learning Research}, 15(1):2773--2832,
	2014.
	
	\bibitem{avron2016sharper}
	Haim Avron, Kenneth~L Clarkson, and David~P Woodruff.
	\newblock Sharper bounds for regression and low-rank approximation with
	regularization.
	\newblock {\em arXiv preprint arXiv:1611.03225}, 6, 2016.
	
	\bibitem{avron2014subspace}
	Haim Avron, Huy Nguyen, and David Woodruff.
	\newblock Subspace embeddings for the polynomial kernel.
	\newblock In {\em Advances in Neural Information Processing Systems}, pages
	2258--2266, 2014.
	
	\bibitem{balasubramanian2018tensor}
	Krishnakumar Balasubramanian, Jianqing Fan, and Zhuoran Yang.
	\newblock Tensor methods for additive index models under discordance and
	heterogeneity.
	\newblock {\em arXiv preprint arXiv:1807.06693}, 2018.
	
	\bibitem{baldin2018optimal}
	Nicolai Baldin and Quentin Berthet.
	\newblock Optimal link prediction with matrix logistic regression.
	\newblock {\em arXiv preprint arXiv:1803.07054}, 2018.
	
	\bibitem{ballani2013projection}
	Jonas Ballani and Lars Grasedyck.
	\newblock A projection method to solve linear systems in tensor format.
	\newblock {\em Numerical linear algebra with applications}, 20(1):27--43, 2013.
	
	\bibitem{ban2019ptas}
	Frank Ban, Vijay Bhattiprolu, Karl Bringmann, Pavel Kolev, Euiwoong Lee, and
	David~P Woodruff.
	\newblock A ptas for $\ell_p$-low rank approximation.
	\newblock In {\em Proceedings of the Thirtieth Annual ACM-SIAM Symposium on
		Discrete Algorithms}, pages 747--766. SIAM, 2019.
	
	\bibitem{bebendorf2011adaptive}
	Mario Bebendorf.
	\newblock Adaptive cross approximation of multivariate functions.
	\newblock {\em Constructive approximation}, 34(2):149--179, 2011.
	
	\bibitem{beylkin2005algorithms}
	Gregory Beylkin and Martin~J Mohlenkamp.
	\newblock Algorithms for numerical analysis in high dimensions.
	\newblock {\em SIAM Journal on Scientific Computing}, 26(6):2133--2159, 2005.
	
	\bibitem{bi2018multilayer}
	Xuan Bi, Annie Qu, and Xiaotong Shen.
	\newblock Multilayer tensor factorization with applications to recommender
	systems.
	\newblock {\em The Annals of Statistics}, 46(6B):3308--3333, 2018.
	
	\bibitem{bousse2017linear}
	M~Bouss{\'e}, I~Domanov, and L~De~Lathauwer.
	\newblock Linear systems with a multilinear singular value decomposition
	constrained solution.
	\newblock {\em ESAT-STADIUS, KU Leuven, Belgium, Tech. Rep}, 2017.
	
	\bibitem{bousse2018linear}
	Martijn Bouss{\'e}, Nico Vervliet, Ignat Domanov, Otto Debals, and Lieven
	De~Lathauwer.
	\newblock Linear systems with a canonical polyadic decomposition constrained
	solution: Algorithms and applications.
	\newblock {\em Numerical Linear Algebra with Applications}, 25(6):e2190, 2018.
	
	\bibitem{boutsidis2017optimal}
	Christos Boutsidis and David~P Woodruff.
	\newblock Optimal cur matrix decompositions.
	\newblock {\em SIAM Journal on Computing}, 46(2):543--589, 2017.
	
	\bibitem{cai2010singular}
	Jian-Feng Cai, Emmanuel~J Cand{\`e}s, and Zuowei Shen.
	\newblock A singular value thresholding algorithm for matrix completion.
	\newblock {\em SIAM Journal on Optimization}, 20(4):1956--1982, 2010.
	
	\bibitem{cai2016optimal}
	T~Tony Cai, Xiaodong Li, and Zongming Ma.
	\newblock Optimal rates of convergence for noisy sparse phase retrieval via
	thresholded wirtinger flow.
	\newblock {\em The Annals of Statistics}, 44(5):2221--2251, 2016.
	
	\bibitem{cai2014sparse}
	T~Tony Cai and Anru Zhang.
	\newblock Sparse representation of a polytope and recovery of sparse signals
	and low-rank matrices.
	\newblock {\em IEEE transactions on information theory}, 60(1):122--132, 2014.
	
	\bibitem{cai2015rop}
	T~Tony Cai and Anru Zhang.
	\newblock {ROP}: Matrix recovery via rank-one projections.
	\newblock {\em The Annals of Statistics}, 43(1):102--138, 2015.
	
	\bibitem{cai2018rate}
	T~Tony Cai and Anru Zhang.
	\newblock Rate-optimal perturbation bounds for singular subspaces with
	applications to high-dimensional statistics.
	\newblock {\em The Annals of Statistics}, 46(1):60--89, 2018.
	
	\bibitem{cai2016structured}
	Tianxi Cai, T.~Tony Cai, and Anru Zhang.
	\newblock Structured matrix completion with applications to genomic data
	integration.
	\newblock {\em Journal of the American Statistical Association},
	111(514):621--633, 2016.
	
	\bibitem{caiafa2010generalizing}
	Cesar~F Caiafa and Andrzej Cichocki.
	\newblock Generalizing the column--row matrix decomposition to multi-way
	arrays.
	\newblock {\em Linear Algebra and its Applications}, 433(3):557--573, 2010.
	
	\bibitem{camoriano2016nytro}
	Raffaello Camoriano, Tom{\'a}s Angles, Alessandro Rudi, and Lorenzo Rosasco.
	\newblock Nytro: When subsampling meets early stopping.
	\newblock In {\em Artificial Intelligence and Statistics}, pages 1403--1411,
	2016.
	
	\bibitem{candes2015phase}
	Emmanuel~J Candes, Xiaodong Li, and Mahdi Soltanolkotabi.
	\newblock Phase retrieval via wirtinger flow: Theory and algorithms.
	\newblock {\em IEEE Transactions on Information Theory}, 61(4):1985--2007,
	2015.
	
	\bibitem{candes2010matrix}
	Emmanuel~J Candes and Yaniv Plan.
	\newblock Matrix completion with noise.
	\newblock {\em Proceedings of the IEEE}, 98(6):925--936, 2010.
	
	\bibitem{candes2011tight}
	Emmanuel~J Candes and Yaniv Plan.
	\newblock Tight oracle inequalities for low-rank matrix recovery from a minimal
	number of noisy random measurements.
	\newblock {\em IEEE Transactions on Information Theory}, 57(4):2342--2359,
	2011.
	
	\bibitem{candes2005decoding}
	Emmanuel~J Candes and Terence Tao.
	\newblock Decoding by linear programming.
	\newblock {\em IEEE transactions on information theory}, 51(12):4203--4215,
	2005.
	
	\bibitem{candes2010power}
	Emmanuel~J Cand{\`e}s and Terence Tao.
	\newblock The power of convex relaxation: Near-optimal matrix completion.
	\newblock {\em IEEE Transactions on Information Theory}, 56(5):2053--2080,
	2010.
	
	\bibitem{charikar2002finding}
	Moses Charikar, Kevin Chen, and Martin Farach-Colton.
	\newblock Finding frequent items in data streams.
	\newblock In {\em International Colloquium on Automata, Languages, and
		Programming}, pages 693--703. Springer, 2002.
	
	\bibitem{chen2016non}
	Han Chen, Garvesh Raskutti, and Ming Yuan.
	\newblock Non-convex projected gradient descent for generalized low-rank tensor
	regression.
	\newblock {\em arXiv preprint arXiv:1611.10349}, 2016.
	
	\bibitem{chen2015exact}
	Yuxin Chen, Yuejie Chi, and Andrea~J Goldsmith.
	\newblock Exact and stable covariance estimation from quadratic sampling via
	convex programming.
	\newblock {\em Information Theory, IEEE Transactions on}, 61(7):4034--4059,
	2015.
	
	\bibitem{chierichetti2017algorithms}
	Flavio Chierichetti, Sreenivas Gollapudi, Ravi Kumar, Silvio Lattanzi, Rina
	Panigrahy, and David~P Woodruff.
	\newblock Algorithms for $\ell_p$ low-rank approximation.
	\newblock In {\em Proceedings of the 34th International Conference on Machine
		Learning-Volume 70}, pages 806--814. JMLR. org, 2017.
	
	\bibitem{cichocki2015tensor}
	Andrzej Cichocki, Danilo Mandic, Lieven De~Lathauwer, Guoxu Zhou, Qibin Zhao,
	Cesar Caiafa, and Huy~Anh Phan.
	\newblock Tensor decompositions for signal processing applications: From
	two-way to multiway component analysis.
	\newblock {\em IEEE Signal Processing Magazine}, 32(2):145--163, 2015.
	
	\bibitem{clarkson2015input}
	Kenneth~L Clarkson and David~P Woodruff.
	\newblock Input sparsity and hardness for robust subspace approximation.
	\newblock In {\em 2015 IEEE 56th Annual Symposium on Foundations of Computer
		Science}, pages 310--329. IEEE, 2015.
	
	\bibitem{clarkson2017low}
	Kenneth~L Clarkson and David~P Woodruff.
	\newblock Low-rank approximation and regression in input sparsity time.
	\newblock {\em Journal of the ACM (JACM)}, 63(6):54, 2017.
	
	\bibitem{dasarathy2015sketching}
	Gautam Dasarathy, Parikshit Shah, Badri~Narayan Bhaskar, and Robert~D Nowak.
	\newblock Sketching sparse matrices, covariances, and graphs via tensor
	products.
	\newblock {\em IEEE Transactions on Information Theory}, 61(3):1373--1388,
	2015.
	
	\bibitem{de2000best}
	Lieven De~Lathauwer, Bart De~Moor, and Joos Vandewalle.
	\newblock On the best rank-1 and rank-(r 1, r 2,..., rn) approximation of
	higher-order tensors.
	\newblock {\em SIAM Journal on Matrix Analysis and Applications},
	21(4):1324--1342, 2000.
	
	\bibitem{diao2018sketching}
	Huaian Diao, Zhao Song, Wen Sun, and David Woodruff.
	\newblock Sketching for kronecker product regression and p-splines.
	\newblock In {\em International Conference on Artificial Intelligence and
		Statistics}, pages 1299--1308, 2018.
	
	\bibitem{dobriban2018new}
	Edgar Dobriban and Sifan Liu.
	\newblock A new theory for sketching in linear regression.
	\newblock {\em arXiv preprint arXiv:1810.06089}, 2018.
	
	\bibitem{drineas2012fast}
	Petros Drineas, Malik Magdon-Ismail, Michael~W Mahoney, and David~P Woodruff.
	\newblock Fast approximation of matrix coherence and statistical leverage.
	\newblock {\em Journal of Machine Learning Research}, 13(Dec):3475--3506, 2012.
	
	\bibitem{drineas2010effective}
	Petros Drineas and Michael~W Mahoney.
	\newblock Effective resistances, statistical leverage, and applications to
	linear equation solving.
	\newblock {\em arXiv preprint arXiv:1005.3097}, 2010.
	
	\bibitem{elden2009newton}
	Lars Eld{\'e}n and Berkant Savas.
	\newblock A newton--grassmann method for computing the best multilinear
	rank-(r\_1, r\_2, r\_3) approximation of a tensor.
	\newblock {\em SIAM Journal on Matrix Analysis and applications},
	31(2):248--271, 2009.
	
	\bibitem{espig2012variational}
	Mike Espig, Wolfgang Hackbusch, Thorsten Rohwedder, and Reinhold Schneider.
	\newblock Variational calculus with sums of elementary tensors of fixed rank.
	\newblock {\em Numerische Mathematik}, 122(3):469--488, 2012.
	
	\bibitem{fan2017generalized}
	Jianqing Fan, Wenyan Gong, and Ziwei Zhu.
	\newblock Generalized high-dimensional trace regression via nuclear norm
	regularization.
	\newblock {\em arXiv preprint arXiv:1710.08083}, 2017.
	
	\bibitem{fan2008sure}
	Jianqing Fan and Jinchi Lv.
	\newblock Sure independence screening for ultrahigh dimensional feature space.
	\newblock {\em Journal of the Royal Statistical Society: Series B (Statistical
		Methodology)}, 70(5):849--911, 2008.
	
	\bibitem{fan2016shrinkage}
	Jianqing Fan, Weichen Wang, and Ziwei Zhu.
	\newblock A shrinkage principle for heavy-tailed data: High-dimensional robust
	low-rank matrix recovery.
	\newblock {\em arXiv preprint arXiv:1603.08315}, 2016.
	
	\bibitem{friedman2010note}
	Jerome Friedman, Trevor Hastie, and Robert Tibshirani.
	\newblock A note on the group lasso and a sparse group lasso.
	\newblock {\em arXiv preprint arXiv:1001.0736}, 2010.
	
	\bibitem{ge2015escaping}
	Rong Ge, Furong Huang, Chi Jin, and Yang Yuan.
	\newblock Escaping from saddle points—online stochastic gradient for tensor
	decomposition.
	\newblock In {\em Conference on Learning Theory}, pages 797--842, 2015.
	
	\bibitem{georgieva2019greedy}
	Irina Georgieva and Clemens Hofreither.
	\newblock Greedy low-rank approximation in tucker format of solutions of tensor
	linear systems.
	\newblock {\em Journal of Computational and Applied Mathematics}, 358:206--220,
	2019.
	
	\bibitem{goreinov2012wedderburn}
	SA~Goreinov, Ivan~V Oseledets, and Dmitry~V Savostyanov.
	\newblock Wedderburn rank reduction and krylov subspace method for tensor
	approximation. part 1: Tucker case.
	\newblock {\em SIAM Journal on Scientific Computing}, 34(1):A1--A27, 2012.
	
	\bibitem{grasedyck2010hierarchical}
	Lars Grasedyck.
	\newblock Hierarchical singular value decomposition of tensors.
	\newblock {\em SIAM Journal on Matrix Analysis and Applications},
	31(4):2029--2054, 2010.
	
	\bibitem{grasedyck2013literature}
	Lars Grasedyck, Daniel Kressner, and Christine Tobler.
	\newblock A literature survey of low-rank tensor approximation techniques.
	\newblock {\em GAMM-Mitteilungen}, 36(1):53--78, 2013.
	
	\bibitem{guhaniyogi2015bayesian}
	Rajarshi Guhaniyogi, Shaan Qamar, and David~B Dunson.
	\newblock Bayesian tensor regression.
	\newblock {\em arXiv preprint arXiv:1509.06490}, 2015.
	
	\bibitem{guo2012tensor}
	Weiwei Guo, Irene Kotsia, and Ioannis Patras.
	\newblock Tensor learning for regression.
	\newblock {\em IEEE Transactions on Image Processing}, 21(2):816--827, 2012.
	
	\bibitem{hackbusch2009new}
	Wolfgang Hackbusch and Stefan K{\"u}hn.
	\newblock A new scheme for the tensor representation.
	\newblock {\em Journal of Fourier analysis and applications}, 15(5):706--722,
	2009.
	
	\bibitem{hao2018sparse}
	Botao Hao, Anru Zhang, and Guang Cheng.
	\newblock Sparse and low-rank tensor estimation via cubic sketchings.
	\newblock {\em arXiv preprint arXiv:1801.09326}, 2018.
	
	\bibitem{haupt2017near}
	Jarvis Haupt, Xingguo Li, and David~P Woodruff.
	\newblock Near optimal sketching of low-rank tensor regression.
	\newblock {\em arXiv preprint arXiv:1709.07093}, 2017.
	
	\bibitem{he2014graphical}
	Shiyuan He, Jianxin Yin, Hongzhe Li, and Xing Wang.
	\newblock Graphical model selection and estimation for high dimensional tensor
	data.
	\newblock {\em Journal of Multivariate Analysis}, 128:165--185, 2014.
	
	\bibitem{hoff2015multilinear}
	Peter~D Hoff.
	\newblock Multilinear tensor regression for longitudinal relational data.
	\newblock {\em The Annals of Applied Statistics}, 9(3):1169, 2015.
	
	\bibitem{hofreither2018black}
	Clemens Hofreither.
	\newblock A black-box low-rank approximation algorithm for fast matrix assembly
	in isogeometric analysis.
	\newblock {\em Computer Methods in Applied Mechanics and Engineering},
	333:311--330, 2018.
	
	\bibitem{hughes2005isogeometric}
	Thomas~JR Hughes, John~A Cottrell, and Yuri Bazilevs.
	\newblock Isogeometric analysis: Cad, finite elements, nurbs, exact geometry
	and mesh refinement.
	\newblock {\em Computer methods in applied mechanics and engineering},
	194(39-41):4135--4195, 2005.
	
	\bibitem{ishteva2011best}
	Mariya Ishteva, P-A Absil, Sabine Van~Huffel, and Lieven De~Lathauwer.
	\newblock Best low multilinear rank approximation of higher-order tensors,
	based on the riemannian trust-region scheme.
	\newblock {\em SIAM Journal on Matrix Analysis and Applications},
	32(1):115--135, 2011.
	
	\bibitem{ishteva2009differential}
	Mariya Ishteva, Lieven De~Lathauwer, P-A Absil, and Sabine Van~Huffel.
	\newblock Differential-geometric newton method for the best rank-(r 1, r 2, r
	3) approximation of tensors.
	\newblock {\em Numerical Algorithms}, 51(2):179--194, 2009.
	
	\bibitem{janzamin2014score}
	Majid Janzamin, Hanie Sedghi, and Anima Anandkumar.
	\newblock Score function features for discriminative learning: Matrix and
	tensor framework.
	\newblock {\em arXiv preprint arXiv:1412.2863}, 2014.
	
	\bibitem{kane2014sparser}
	Daniel~M Kane and Jelani Nelson.
	\newblock Sparser johnson-lindenstrauss transforms.
	\newblock {\em Journal of the ACM (JACM)}, 61(1):4, 2014.
	
	\bibitem{kolda2009tensor}
	Tamara~G Kolda and Brett~W Bader.
	\newblock Tensor decompositions and applications.
	\newblock {\em SIAM review}, 51(3):455--500, 2009.
	
	\bibitem{kolda2006multilinear}
	Tamara~Gibson Kolda.
	\newblock {\em Multilinear operators for higher-order decompositions},
	volume~2.
	\newblock United States. Department of Energy, 2006.
	
	\bibitem{koltchinskii2013remark}
	Vladimir Koltchinskii.
	\newblock A remark on low rank matrix recovery and noncommutative bernstein
	type inequalities.
	\newblock In {\em From Probability to Statistics and Back: High-Dimensional
		Models and Processes--A Festschrift in Honor of Jon A. Wellner}, pages
	213--226. Institute of Mathematical Statistics, 2013.
	
	\bibitem{koltchinskii2011nuclear}
	Vladimir Koltchinskii, Karim Lounici, and Alexandre~B Tsybakov.
	\newblock Nuclear-norm penalization and optimal rates for noisy low-rank matrix
	completion.
	\newblock {\em The Annals of Statistics}, 39(5):2302--2329, 2011.
	
	\bibitem{kressner2016preconditioned}
	Daniel Kressner, Michael Steinlechner, and Bart Vandereycken.
	\newblock Preconditioned low-rank riemannian optimization for linear systems
	with tensor product structure.
	\newblock {\em SIAM Journal on Scientific Computing}, 38(4):A2018--A2044, 2016.
	
	\bibitem{kressner2010krylov}
	Daniel Kressner and Christine Tobler.
	\newblock Krylov subspace methods for linear systems with tensor product
	structure.
	\newblock {\em SIAM journal on matrix analysis and applications},
	31(4):1688--1714, 2010.
	
	\bibitem{kroonenberg2008applied}
	Pieter~M Kroonenberg.
	\newblock {\em Applied multiway data analysis}, volume 702.
	\newblock John Wiley \& Sons, 2008.
	
	\bibitem{laurent2000adaptive}
	Beatrice Laurent and Pascal Massart.
	\newblock Adaptive estimation of a quadratic functional by model selection.
	\newblock {\em Annals of Statistics}, pages 1302--1338, 2000.
	
	\bibitem{lee2010practical}
	Jason~D Lee, Ben Recht, Nathan Srebro, Joel Tropp, and Ruslan~R Salakhutdinov.
	\newblock Practical large-scale optimization for max-norm regularization.
	\newblock In {\em Advances in Neural Information Processing Systems}, pages
	1297--1305, 2010.
	
	\bibitem{lehmann2006theory}
	Erich~L Lehmann and George Casella.
	\newblock {\em Theory of point estimation}.
	\newblock Springer Science \& Business Media, 2006.
	
	\bibitem{li2017parsimonious}
	Lexin Li and Xin Zhang.
	\newblock Parsimonious tensor response regression.
	\newblock {\em Journal of the American Statistical Association}, pages 1--16,
	2017.
	
	\bibitem{li2010tensor}
	Nan Li and Baoxin Li.
	\newblock Tensor completion for on-board compression of hyperspectral images.
	\newblock In {\em 2010 IEEE International Conference on Image Processing},
	pages 517--520. IEEE, 2010.
	
	\bibitem{li2013tucker}
	Xiaoshan Li, Da~Xu, Hua Zhou, and Lexin Li.
	\newblock Tucker tensor regression and neuroimaging analysis.
	\newblock {\em Statistics in Biosciences}, pages 1--26, 2018.
	
	\bibitem{liu2013tensor}
	Ji~Liu, Przemyslaw Musialski, Peter Wonka, and Jieping Ye.
	\newblock Tensor completion for estimating missing values in visual data.
	\newblock {\em IEEE Transactions on Pattern Analysis and Machine Intelligence},
	35(1):208--220, 2013.
	
	\bibitem{lounici2011oracle}
	Karim Lounici, Massimiliano Pontil, Sara Van De~Geer, and Alexandre~B Tsybakov.
	\newblock Oracle inequalities and optimal inference under group sparsity.
	\newblock {\em The Annals of Statistics}, 39(4):2164--2204, 2011.
	
	\bibitem{lynch1964tensor}
	RE~Lynch, JOHN~R Rice, and DONALD~H Thomas.
	\newblock Tensor product analysis of partial difference equations.
	\newblock {\em Bulletin of the American Mathematical Society}, 70(3):378--384,
	1964.
	
	\bibitem{lyu2019tensor}
	Xiang Lyu, Will~Wei Sun, Zhaoran Wang, Han Liu, Jian Yang, and Guang Cheng.
	\newblock Tensor graphical model: Non-convex optimization and statistical
	inference.
	\newblock {\em IEEE transactions on pattern analysis and machine intelligence},
	2019.
	
	\bibitem{mahoney2011randomized}
	Michael~W Mahoney.
	\newblock Randomized algorithms for matrices and data.
	\newblock {\em Foundations and Trends{\textregistered} in Machine Learning},
	3(2):123--224, 2011.
	
	\bibitem{mahoney2008tensor}
	Michael~W Mahoney, Mauro Maggioni, and Petros Drineas.
	\newblock Tensor-cur decompositions for tensor-based data.
	\newblock {\em SIAM Journal on Matrix Analysis and Applications},
	30(3):957--987, 2008.
	
	\bibitem{manceur2013maximum}
	Ameur~M Manceur and Pierre Dutilleul.
	\newblock Maximum likelihood estimation for the tensor normal distribution:
	Algorithm, minimum sample size, and empirical bias and dispersion.
	\newblock {\em Journal of Computational and Applied Mathematics}, 239:37--49,
	2013.
	
	\bibitem{markopoulos2014optimal}
	Panos~P Markopoulos, George~N Karystinos, and Dimitris~A Pados.
	\newblock Optimal algorithms for $ l\_ $\{$1$\}$ $-subspace signal processing.
	\newblock {\em IEEE Transactions on Signal Processing}, 62(19):5046--5058,
	2014.
	
	\bibitem{markopoulos2017efficient}
	Panos~P Markopoulos, Sandipan Kundu, Shubham Chamadia, and Dimitris~A Pados.
	\newblock Efficient l1-norm principal-component analysis via bit flipping.
	\newblock {\em IEEE Transactions on Signal Processing}, 65(16):4252--4264,
	2017.
	
	\bibitem{massart2007concentration}
	Pascal Massart.
	\newblock {\em Concentration inequalities and model selection}.
	\newblock Springer, 2007.
	
	\bibitem{meng2013cyclic}
	Deyu Meng, Zongben Xu, Lei Zhang, and Ji~Zhao.
	\newblock A cyclic weighted median method for l1 low-rank matrix factorization
	with missing entries.
	\newblock In {\em Twenty-Seventh AAAI Conference on Artificial Intelligence},
	2013.
	
	\bibitem{meng2013low}
	Xiangrui Meng and Michael~W Mahoney.
	\newblock Low-distortion subspace embeddings in input-sparsity time and
	applications to robust linear regression.
	\newblock In {\em Proceedings of the forty-fifth annual ACM symposium on Theory
		of computing}, pages 91--100. ACM, 2013.
	
	\bibitem{montanari2016spectral}
	Andrea Montanari and Nike Sun.
	\newblock Spectral algorithms for tensor completion.
	\newblock {\em arXiv preprint arXiv:1612.07866}, 2016.
	
	\bibitem{mu2014square}
	Cun Mu, Bo~Huang, John Wright, and Donald Goldfarb.
	\newblock Square deal: Lower bounds and improved relaxations for tensor
	recovery.
	\newblock In {\em ICML}, pages 73--81, 2014.
	
	\bibitem{nelson2013osnap}
	Jelani Nelson and Huy~L Nguy{\^e}n.
	\newblock Osnap: Faster numerical linear algebra algorithms via sparser
	subspace embeddings.
	\newblock In {\em 2013 IEEE 54th Annual Symposium on Foundations of Computer
		Science}, pages 117--126. IEEE, 2013.
	
	\bibitem{oseledets2011tensor}
	Ivan~V Oseledets.
	\newblock Tensor-train decomposition.
	\newblock {\em SIAM Journal on Scientific Computing}, 33(5):2295--2317, 2011.
	
	\bibitem{oseledets2008tucker}
	Ivan~V Oseledets, DV~Savostianov, and Eugene~E Tyrtyshnikov.
	\newblock Tucker dimensionality reduction of three-dimensional arrays in linear
	time.
	\newblock {\em SIAM Journal on Matrix Analysis and Applications},
	30(3):939--956, 2008.
	
	\bibitem{oseledets2010cross}
	Ivan~V Oseledets, Dmitry~V Savostyanov, and Eugene~E Tyrtyshnikov.
	\newblock Cross approximation in tensor electron density computations.
	\newblock {\em Numerical Linear Algebra with Applications}, 17(6):935--952,
	2010.
	
	\bibitem{oseledets2009breaking}
	Ivan~V Oseledets and Eugene~E Tyrtyshnikov.
	\newblock Breaking the curse of dimensionality, or how to use svd in many
	dimensions.
	\newblock {\em SIAM Journal on Scientific Computing}, 31(5):3744--3759, 2009.
	
	\bibitem{pagh2013compressed}
	Rasmus Pagh.
	\newblock Compressed matrix multiplication.
	\newblock {\em ACM Transactions on Computation Theory (TOCT)}, 5(3):9, 2013.
	
	\bibitem{pan2018covariate}
	Yuqing Pan, Qing Mai, and Xin Zhang.
	\newblock Covariate-adjusted tensor classification in high dimensions.
	\newblock {\em Journal of the American Statistical Association}, pages 1--15,
	2018.
	
	\bibitem{pham2013fast}
	Ninh Pham and Rasmus Pagh.
	\newblock Fast and scalable polynomial kernels via explicit feature maps.
	\newblock In {\em Proceedings of the 19th ACM SIGKDD international conference
		on Knowledge discovery and data mining}, pages 239--247. ACM, 2013.
	
	\bibitem{pilanci2015randomized}
	Mert Pilanci and Martin~J Wainwright.
	\newblock Randomized sketches of convex programs with sharp guarantees.
	\newblock {\em IEEE Transactions on Information Theory}, 61(9):5096--5115,
	2015.
	
	\bibitem{pilanci2016iterative}
	Mert Pilanci and Martin~J Wainwright.
	\newblock Iterative hessian sketch: Fast and accurate solution approximation
	for constrained least-squares.
	\newblock {\em The Journal of Machine Learning Research}, 17(1):1842--1879,
	2016.
	
	\bibitem{raskutti2014statistical}
	Garvesh Raskutti and Michael Mahoney.
	\newblock A statistical perspective on randomized sketching for ordinary
	least-squares.
	\newblock {\em arXiv preprint arXiv:1406.5986}, 2014.
	
	\bibitem{raskutti2015convex}
	Garvesh Raskutti, Ming Yuan, and Han Chen.
	\newblock Convex regularization for high-dimensional multi-response tensor
	regression.
	\newblock {\em arXiv preprint arXiv:1512.01215}, 2015.
	
	\bibitem{recht2010guaranteed}
	Benjamin Recht, Maryam Fazel, and Pablo~A Parrilo.
	\newblock Guaranteed minimum-rank solutions of linear matrix equations via
	nuclear norm minimization.
	\newblock {\em SIAM review}, 52(3):471--501, 2010.
	
	\bibitem{savas2013krylov}
	Berkant Savas and Lars Eld{\'e}n.
	\newblock Krylov-type methods for tensor computations i.
	\newblock {\em Linear Algebra and its Applications}, 438(2):891--918, 2013.
	
	\bibitem{savas2010quasi}
	Berkant Savas and Lek-Heng Lim.
	\newblock Quasi-newton methods on grassmannians and multilinear approximations
	of tensors.
	\newblock {\em SIAM Journal on Scientific Computing}, 32(6):3352--3393, 2010.
	
	\bibitem{sidiropoulos2017tensor}
	Nicholas~D Sidiropoulos, Lieven De~Lathauwer, Xiao Fu, Kejun Huang, Evangelos~E
	Papalexakis, and Christos Faloutsos.
	\newblock Tensor decomposition for signal processing and machine learning.
	\newblock {\em IEEE Transactions on Signal Processing}, 65(13):3551--3582,
	2017.
	
	\bibitem{sidiropoulos2012multi}
	Nicholas~D Sidiropoulos and Anastasios Kyrillidis.
	\newblock Multi-way compressed sensing for sparse low-rank tensors.
	\newblock {\em IEEE Signal Processing Letters}, 19(11):757--760, 2012.
	
	\bibitem{sidiropoulos2014parallel}
	Nicholas~D Sidiropoulos, Evangelos~E Papalexakis, and Christos Faloutsos.
	\newblock Parallel randomly compressed cubes: A scalable distributed
	architecture for big tensor decomposition.
	\newblock {\em IEEE Signal Processing Magazine}, 31(5):57--70, 2014.
	
	\bibitem{song2017low}
	Zhao Song, David~P Woodruff, and Peilin Zhong.
	\newblock Low rank approximation with entrywise l 1-norm error.
	\newblock In {\em Proceedings of the 49th Annual ACM SIGACT Symposium on Theory
		of Computing}, pages 688--701. ACM, 2017.
	
	\bibitem{song2019relative}
	Zhao Song, David~P Woodruff, and Peilin Zhong.
	\newblock Relative error tensor low rank approximation.
	\newblock In {\em Proceedings of the Thirtieth Annual ACM-SIAM Symposium on
		Discrete Algorithms}, pages 2772--2789. Society for Industrial and Applied
	Mathematics, 2019.
	
	\bibitem{sun2016sparse}
	Will~Wei Sun and Lexin Li.
	\newblock Sparse low-rank tensor response regression.
	\newblock {\em arXiv preprint arXiv:1609.04523}, 2016.
	
	\bibitem{sun2017store}
	Will~Wei Sun and Lexin Li.
	\newblock Store: sparse tensor response regression and neuroimaging analysis.
	\newblock {\em The Journal of Machine Learning Research}, 18(1):4908--4944,
	2017.
	
	\bibitem{sun2019low}
	Yiming Sun, Yang Guo, Charlene Luo, Joel Tropp, and Madeleine Udell.
	\newblock Low-rank tucker approximation of a tensor from streaming data.
	\newblock {\em arXiv preprint arXiv:1904.10951}, 2019.
	
	\bibitem{toh2010accelerated}
	Kim-Chuan Toh and Sangwoon Yun.
	\newblock An accelerated proximal gradient algorithm for nuclear norm
	regularized linear least squares problems.
	\newblock {\em Pacific Journal of Optimization}, 6(615-640):15, 2010.
	
	\bibitem{tomioka2013convex}
	Ryota Tomioka and Taiji Suzuki.
	\newblock Convex tensor decomposition via structured schatten norm
	regularization.
	\newblock In {\em Advances in neural information processing systems}, pages
	1331--1339, 2013.
	
	\bibitem{tomioka2011statistical}
	Ryota Tomioka, Taiji Suzuki, Kohei Hayashi, and Hisashi Kashima.
	\newblock Statistical performance of convex tensor decomposition.
	\newblock In {\em Advances in Neural Information Processing Systems}, pages
	972--980, 2011.
	
	\bibitem{tropp2017practical}
	Joel~A Tropp, Alp Yurtsever, Madeleine Udell, and Volkan Cevher.
	\newblock Practical sketching algorithms for low-rank matrix approximation.
	\newblock {\em SIAM Journal on Matrix Analysis and Applications},
	38(4):1454--1485, 2017.
	
	\bibitem{tu2016low}
	Stephen Tu, Ross Boczar, Max Simchowitz, Mahdi Soltanolkotabi, and Ben Recht.
	\newblock Low-rank solutions of linear matrix equations via procrustes flow.
	\newblock In {\em International Conference on Machine Learning}, pages
	964--973, 2016.
	
	\bibitem{tucker1966some}
	Ledyard~R Tucker.
	\newblock Some mathematical notes on three-mode factor analysis.
	\newblock {\em Psychometrika}, 31(3):279--311, 1966.
	
	\bibitem{udell2019big}
	Madeleine Udell and Alex Townsend.
	\newblock Why are big data matrices approximately low rank?
	\newblock {\em SIAM Journal on Mathematics of Data Science}, 1(1):144--160,
	2019.
	
	\bibitem{vershynin2010introduction}
	Roman Vershynin.
	\newblock Introduction to the non-asymptotic analysis of random matrices.
	\newblock {\em arXiv preprint arXiv:1011.3027}, 2010.
	
	\bibitem{vervliet2015randomized}
	Nico Vervliet and Lieven De~Lathauwer.
	\newblock A randomized block sampling approach to canonical polyadic
	decomposition of large-scale tensors.
	\newblock {\em IEEE Journal of Selected Topics in Signal Processing},
	10(2):284--295, 2015.
	
	\bibitem{wang2017sketching}
	Jialei Wang, Jason~D Lee, Mehrdad Mahdavi, Mladen Kolar, Nathan Srebro, et~al.
	\newblock Sketching meets random projection in the dual: A provable recovery
	algorithm for big and high-dimensional data.
	\newblock {\em Electronic Journal of Statistics}, 11(2):4896--4944, 2017.
	
	\bibitem{wang2015fast}
	Yining Wang, Hsiao-Yu Tung, Alexander~J Smola, and Anima Anandkumar.
	\newblock Fast and guaranteed tensor decomposition via sketching.
	\newblock In {\em Advances in Neural Information Processing Systems}, pages
	991--999, 2015.
	
	\bibitem{woodruff2014sketching}
	David~P Woodruff.
	\newblock Sketching as a tool for numerical linear algebra.
	\newblock {\em Foundations and Trends{\textregistered} in Theoretical Computer
		Science}, 10(1--2):1--157, 2014.
	
	\bibitem{xia2017polynomial}
	Dong Xia and Ming Yuan.
	\newblock On polynomial time methods for exact low rank tensor completion.
	\newblock {\em arXiv preprint arXiv:1702.06980}, 2017.
	
	\bibitem{xia2017statistically}
	Dong Xia, Ming Yuan, and Cun-Hui Zhang.
	\newblock Statistically optimal and computationally efficient low rank tensor
	completion from noisy entries.
	\newblock {\em arXiv preprint arXiv:1711.04934}, 2017.
	
	\bibitem{xue2011sure}
	Lingzhou Xue and Hui Zou.
	\newblock Sure independence screening and compressed random sensing.
	\newblock {\em Biometrika}, pages 371--380, 2011.
	
	\bibitem{yang2014sparse}
	Dan Yang, Zongming Ma, and Andreas Buja.
	\newblock A sparse singular value decomposition method for high-dimensional
	data.
	\newblock {\em Journal of Computational and Graphical Statistics},
	23(4):923--942, 2014.
	
	\bibitem{yang2015fast}
	Yi~Yang and Hui Zou.
	\newblock A fast unified algorithm for solving group-lasso penalize learning
	problems.
	\newblock {\em Statistics and Computing}, 25(6):1129--1141, 2015.
	
	\bibitem{yu2018recovery}
	Ming Yu, Zhaoran Wang, Varun Gupta, and Mladen Kolar.
	\newblock Recovery of simultaneous low rank and two-way sparse coefficient
	matrices, a nonconvex approach.
	\newblock {\em arXiv preprint arXiv:1802.06967}, 2018.
	
	\bibitem{yuan2006model}
	Ming Yuan and Yi~Lin.
	\newblock Model selection and estimation in regression with grouped variables.
	\newblock {\em Journal of the Royal Statistical Society: Series B (Statistical
		Methodology)}, 68(1):49--67, 2006.
	
	\bibitem{yuan2014tensor}
	Ming Yuan and Cun-Hui Zhang.
	\newblock On tensor completion via nuclear norm minimization.
	\newblock {\em Foundations of Computational Mathematics}, pages 1--38, 2014.
	
	\bibitem{zhang2019cross}
	Anru Zhang.
	\newblock Cross: Efficient low-rank tensor completion.
	\newblock {\em The Annals of Statistics}, 47(2):936--964, 2019.
	
	\bibitem{zhang2017optimal-statsvd}
	Anru Zhang and Rungang Han.
	\newblock Optimal sparse singular value decomposition for high-dimensional
	high-order data.
	\newblock {\em Journal of the American Statistical Association}, page to
	appear, 2018.
	
	\bibitem{ISLET-supplement}
	Anru Zhang, Yuetian Luo, Garvesh Raskutti, and Ming Yuan.
	\newblock Supplement to ``{ISLET}: Fast and optimal low-rank tensor regression
	via importance sketching", 2018.
	
	\bibitem{zhang2019HOOI}
	Anru Zhang, Yuetian Luo, Garvesh Raskutti, and Ming Yuan.
	\newblock A sharp blockwise tensor perturbation bound for higher-order
	orthogonal iteration.
	\newblock {\em preprint}, 2019.
	
	\bibitem{zhang2017tensor}
	Anru Zhang and Dong Xia.
	\newblock Tensor {SVD}: Statistical and computational limits.
	\newblock {\em IEEE Transactions on Information Theory}, 64(11):7311--7338,
	2018.
	
	\bibitem{zhang2014random}
	Lijun Zhang, Mehrdad Mahdavi, Rong Jin, Tianbao Yang, and Shenghuo Zhu.
	\newblock Random projections for classification: A recovery approach.
	\newblock {\em IEEE Transactions on Information Theory}, 60(11):7300--7316,
	2014.
	
	\bibitem{zheng2012practical}
	Yinqiang Zheng, Guangcan Liu, Shigeki Sugimoto, Shuicheng Yan, and Masatoshi
	Okutomi.
	\newblock Practical low-rank matrix approximation under robust l 1-norm.
	\newblock In {\em 2012 IEEE Conference on Computer Vision and Pattern
		Recognition}, pages 1410--1417. IEEE, 2012.
	
	\bibitem{zhou2017matlab}
	Hua Zhou.
	\newblock Matlab tensorreg toolbox version 1.0, 2017.
	\newblock Available online at https://hua-zhou.github.io/TensorReg/.
	
	\bibitem{zhou2013tensor}
	Hua Zhou, Lexin Li, and Hongtu Zhu.
	\newblock Tensor regression with applications in neuroimaging data analysis.
	\newblock {\em Journal of the American Statistical Association},
	108(502):540--552, 2013.
	
	\bibitem{zhou2014gemini}
	Shuheng Zhou.
	\newblock Gemini: Graph estimation with matrix variate normal instances.
	\newblock {\em The Annals of Statistics}, 42(2):532--562, 2014.
	
\end{thebibliography}
\end{document}